\documentclass[11pt]{article}%
\input{macros}
\usepackage[normalem]{ulem}

\begin{document}

\title{The Interpolation Phase Transition in Neural Networks:\\ 
Memorization and Generalization under Lazy Training}
\date{\today}

\author{Andrea Montanari\thanks{Department of Electrical Engineering and
  Department of Statistics, Stanford University}\;\; and \;\; Yiqiao Zhong\footnotemark[1]}

\date{\today}
\maketitle

\begin{abstract}
  Modern neural networks are often operated in a strongly overparametrized regime: they comprise so many
  parameters that they can interpolate the training set, even if actual labels are replaced by purely random ones.
  Despite this, they achieve good prediction error on unseen data: interpolating the training set does not
  lead to a large generalization error. Further, overparametrization appears to be beneficial in that it simplifies the optimization
  landscape.  Here we study these phenomena in the context of two-layers neural networks in the
  neural tangent (NT) regime. We consider a simple data model, with isotropic covariates vectors in $d$ dimensions,
  and $N$ hidden neurons. We assume that both the sample size $n$ and the dimension $d$ are large, 
  and they are polynomially related. 
Our first main result is a characterization of the eigenstructure 
  of the empirical NT kernel in the overparametrized regime $Nd\gg n$.
  This characterization implies as a corollary that the minimum eigenvalue 
  of the empirical NT kernel is bounded away from zero as soon as $Nd\gg n$,
  and therefore the network can exactly interpolate
  arbitrary labels in the same regime.

Our second main result is a characterization of the generalization error of NT ridge regression 
including, as a special case, min-$\ell_2$ norm interpolation.
We prove that, as soon as $Nd\gg n$, the test error is well approximated by the 
one of kernel ridge regression with respect to the infinite-width kernel.
The latter is in turn well approximated by the error of polynomial ridge regression,
whereby the regularization parameter is increased by a `self-induced' term related to 
the high-degree components of the activation function. The polynomial  degree depends on the
sample size and the dimension (in particular on $\log n/\log d$).
\end{abstract}

\tableofcontents
\newpage

\section{Introduction}
\label{sec:intro}

Tractability and generalization are two key problems in statistical learning.
Classically, tractability is achieved by crafting suitable convex objectives,
and generalization by regularizing (or restricting) the
function class of interest to guarantee uniform convergence.
In modern neural networks, a different mechanism
appears to be often at work \cite{neyshabur2015search,zhang2016understanding,belkin2019reconciling}.
Empirical risk minimization becomes tractable despite non-convexity
because the model is overparametrized. In fact, so overparametrized that a model
interpolating perfectly the training set is found in the neighborhood of most initializations.
Despite this, the resulting model generalizes well to unseen data: the inductive bias produced
by gradient-based algorithms is sufficient to select models that generalize well.

Elements of this picture have been rigorously established in special regimes. In particular, it is known
that for neural networks with a sufficiently large number of neurons, gradient descent converges
quickly to a model with vanishing training error \cite{du2018gradient,allen2019convergence,oymak2020towards}.
In a parallel research direction, the generalization properties
of several examples of interpolating models have been studied in detail 
\cite{bartlett2020benign,liang2018just,hastie2019surprises,belkin2019two,mei2019generalization,montanari2019generalization}.
The present paper continues along the last direction, by studying the generalization properties 
of linear interpolating models that arise from the analysis of neural networks.

In this context, many fundamental questions remain challenging. (We refer to Section \ref{sec:related}
for pointers to recent progress on these questions.)
\begin{enumerate}
\item[{\sf Q1.}] When is a neural
network sufficiently complex to interpolate $n$ data points? Counting degrees of freedom would
suggest that this happens as soon as the number of parameters in the network is larger than $n$.
Does this lower bound predict the correct threshold? What are the architectures that achieve this lower bound?
\item[{\sf Q2.}]
  Assume that the answer to the previous question is positive, namely a network with the order of $n$
  parameters can interpolate $n$ data points. Can such a network be found efficiently, 
  using standard gradient descent (GD) or stochastic gradient descent (SGD)? 
\item[{\sf Q3.}] Can we characterize the generalization error above this interpolation threshold?
  Does it decrease with the number of parameters? What is the nature of the implicit regularization
  and of the resulting model $\hat f(\xx)$?
\end{enumerate}
Here we address these questions by studying a class of linear models known as `neural tangent' models. 
Our focus will be on characterizing test error and generalization, i.e.,
on the last question  {\sf Q3}, but our results are also relevant to {\sf Q1} and {\sf Q2}.

We assume to be given data
$\{(\xx_i,y_i)\}_{i\le n}$ with i.i.d.\ $d$-dimensional covariates vectors 
$\xx_i\sim\Unif(\S^{d-1}(\sqrt{d}))$
(we denote by $\S^{d-1}(r)$ the sphere of radius $r$ in $d$ dimensions).
In addressing questions {\sf Q1} and {\sf Q2}, we assume the labels $y_i$ to be arbitrary.
For {\sf Q3} we will assume $y_i = f_*(\xx_i)+\veps_i$, 
where $\veps_1, \ldots, \veps_n$ are independent noise variables with $\E(\veps_i) = 0$ and 
$\E(\veps^2_i) = \sigma_{\veps}^2$. Crucially, we allow for a general target function 
$f_*$ under the minimal condition $\E\{f_*(\xx_i)^2\}<\infty$. We assume both the dimension and
sample size to be large and polynomially related, namely $c_0 d\le n\le d^{1/c_0}$
for some small constant $c_0$.

A substantial literature \cite{du2018gradient,allen2019convergence,zou2020gradient,oymak2020towards,liu2020linearity} shows that,
 \emph{under certain training schemes}, multi-layer
neural networks can be well approximated by linear models  with a nonlinear (randomized)
featurization map that depends on the network architecture and its initialization.
In this paper, we will focus on the simplest such models.
Given a set of weights $\WW = (\ww_1,\dots, \ww_N)$, we define the following 
function class with $Nd$ parameters $\aa := (\aa_1,\dots, \aa_N)$:
\begin{align}
  \cF_{\NT}^N(\WW) :=\Big\{f(\xx;\aa) :=\frac{1}{\sqrt{Nd}}\sum_{i=1}^N
  \<\aa_i,\xx\> \sigma'(\<\ww_i,\xx\>)\, \;\;\; \aa_i\in\R^d\Big\}\, .\label{eq:FNT}
\end{align}
In other words, to a vector of covariates $\xx\in\R^d$, the NT model associates a (random) 
features vector
\begin{equation}
\bPhi(\xx) = 
\frac{1}{\sqrt{Nd}} [ \sigma'(\<\xx, \ww_1\>) \xx^\sT, \ldots, \sigma'(\< \xx, \ww_N \rangle) \xx^\sT ]
\in \R^{Nd}.\label{eq:FeaturizationMap}
\end{equation}
The model then computes a linear function $f(\xx;\aa)= \<\aa,\bPhi(\xx)\>$.
We fit
the parameters $\aa$ via minimum $\ell_2$-norm interpolation:
\begin{equation}\label{eq:InterpolationFirst}
\begin{aligned}
\mbox{minimize} &\;\;\;\; \|\aa\|_2\, ,\\
\mbox{subj. to} &\;\;\;\; f(\xx_i;\aa) = y_i\;\;\; \forall i\le n\, .
\end{aligned}
\end{equation}
We will also study a generalization of min $\ell_2$-norm interpolation
which is given by least-squares regression with ridge penalty.  Interpolation is
recovered in the limit in which the ridge parameter tends to $0$.

As mentioned, the  construction of  model \eqref{eq:FNT} and the choice of 
minimum $\ell_2$-norm interpolation \eqref{eq:InterpolationFirst} are motivated 
by the analysis of two-layers neural networks. In a nutshell, for a suitable scaling of 
the initialization, two-layer networks trained via gradient descent are well approximated
by the min-norm interpolator described above. 

While our focus is not on establishing or expanding the connection
between neural networks and neural tangent models,  Section \ref{sec:ConnectionOptimization}
discusses the relation between model \eqref{eq:FNT} and two-layer networks, mainly
based on \cite{chizat2019lazy,bartlett_montanari_rakhlin_2021}.
We also emphasize that the model \eqref{eq:FNT} is the neural tangent model
corresponding to a two-layer network in which only first-layer weights are trained. 
If second-layer weights were trained as well, this model would have to be slightly 
modified (see Section \ref{sec:ConnectionOptimization}). However, our proofs and results
would remain essentially unchanged at the cost of substantial notational burden. 
The reason is intuitively clear: in large dimensions, the number of second layer weights
$N$ is negligible compared to the number of first-layer weights $Nd$. 

We next summarize our results. We assume the weights $(\ww_k)_{k\le N}$ to be i.i.d.\ with 
$\ww_k\sim\Unif(\S^{d-1}(1))$. In large dimensions, this choice is closely related to one of the most
standard initialization of gradient descent, that is $\ww_k\sim\normal(\bzero,\id_d/d)$.
In the summary below $C$ denotes constants that can change from line to line,
and various statements hold `with high probability' (i.e., with probability converging 
to one as $N,d,n\to\infty$).
\begin{description}
\item[Interpolation threshold.] Considering ---as mentioned above--- feature vectors
  $(\xx_i)_{i\le n}\sim\Unif(\S^{d-1}(1))$, and arbitrary labels $y_i\in\R$ we prove that, 
  if $Nd/(\log Nd)^C\ge n$ then with high probability there exists $f\in\cF^{N}_{\NT}(\WW)$
  that interpolates the data. Namely $f(\xx_i;\aa) = y_i$ for all $i\le n$.

  Finding such an interpolator amounts to solving the $n$ linear equations $f(\xx_i;\aa)=y_i$, $i\le n$
  in the $Nd$ unknowns $\aa_1,\dots, \aa_N$, which parametrize $\cF^{N}_{\NT}(\WW)$,
  cf. Eq.~\eqref{eq:FNT}. Hence the function $f$ can be found efficiently, e.g. via
  gradient descent with respect to the square loss.
\item[Minimum eigenvalue of the empirical kernel.] In order to prove the previous upper bound on the
  interpolation threshold, we show that the linear system $f(\xx_i;\aa)=y_i$, $i\le n$ has full row rank
  provided $Nd/(\log Nd)^C\ge n$. In fact, our proof provides quantitative control on the 
  eigenstructure of the associated empirical kernel matrix. 
  
  Denoting by $\bPhi\in\reals^{n\times (Nd)}$ the matrix whose $i$-th row
  is the $i$-th feature vector $\bPhi(\xx_i)$, the empirical kernel is given 
  by $\bK_N:= \bPhi\bPhi^{\top}$. We then prove that, for $Nd/(\log Nd)^C\ge n$,
  $\bK_N$ tightly concentrates around its expectation with respect to the weights
  $\bK:= \E_{\WW}[\bPhi\bPhi^{\top}]$, which can be interpreted as the kernel associated to
  an infinite-width network. We then prove that the latter is well approximated
  by $\bK^{p}+\gamma_{>\ell}\id_n$, where $\bK^{p}$ is a polynomial kernel of constant degree $\ell$,
  and $\gamma_{>\ell}$ is bounded away from zero, and depends on the high-degree components
  of the activation function.
  This result implies a tight lower bound on the minimum eigenvalue of 
  the kernel $\bK_N$ as well as estimates of all the other eigenvalues.
  The term $\gamma_{>\ell}\id_n$ acts as a `self-induced' ridge regularization.
  
  We note that $\lambda_{\min}(\bK_N)$ has a direct algorithmic relevance, as discussed in
  \cite{du2018gradient,chizat2019lazy,oymak2020towards}.
\item[Generalization error.] Most of our work is devoted to the analysis 
 of the generalization error of the min-norm NT interpolator
\eqref{eq:InterpolationFirst}.
 As mentioned above, we consider general labels of the form $y_i=f_*(\xx_i)+\eps_i$, 
 for $f_*$ a target function with finite second moment. We prove that,
 as soon as $Nd/(\log Nd)^C\ge n$, \emph{the risk of NT interpolation is close 
 to the one of polynomial ridge regression with kernel $\bK^p$ and ridge parameter 
 $\gamma_{>\ell}$}.
 
 The degree $\ell$ of the effective polynomial kernel is defined by the condition 
 $d^{\ell}(\log d)^C\le n\ll d^{\ell+1}/(\log d)^C$ for $\ell\ge 2$ or
 $d/C\le n\ll d^{2}/(\log d)^C$ for $\ell=1$
 (in particular our result does not cover the case $n\asymp d^{\ell}$ for an integer $\ell\ge 2$,
 but they cover $n\asymp d^{\alpha}$ for any non-integer $\alpha$).

Remarkably, even if the original NT model interpolates the data, the 
equivalent polynomial regression model does not interpolate and has a positive 
`self-induced' regularization 
parameter $\gamma_{>\ell}$. 

From a technical viewpoint, the characterization of the generalization behavior
is significantly more challenging than the analysis of the eigenstructure of the kernel $\bK_N$
at the previous point. Indeed understanding generalization requires to study the
eigenvectors of $\bK_N$, and how they change when adding new sample points. 
\end{description}

These results provide a clear picture of how neural networks achieve good generalizations
in the neural tangent regime, despite interpolating the data. 
First, the model is nonlinear in the input covariates, and sufficiently overparametrized. Thanks to this flexibility
it can interpolate the data points.
Second, gradient descent selects the min-norm interpolator.
This results in a model that is very close to $\hat{f}_{\ell}(\xx)$ a polynomial of degree $\ell$
at `most' points $\xx$ (more formally, the model is close to a polynomial in the 
$L^2$ sense). Third, because of the large dimension $d$, the empirical kernel 
matrix contains a portion proportional to the identity, which acts as 
a \emph{self-induced regularization} term.

The rest of the paper is organized as follows. 
In the next section, we illustrate our results through some simple numerical simulations. 
We formally present our main results in Section~\ref{sec:Main}.
We state our characterization of the NT kernel, and of the generalization error of NT regression.
 In Section~\ref{sec:related}, we briefly overview related work on interpolation and 
 generalization of neural networks.
 We discuss the connection between GD-trained neural networks and neural tangent models 
 in Section \ref{sec:ConnectionOptimization}.
 Section~\ref{sec:pre} provides some technical background on orthogonal polynomials,
 which is useful for the proofs.
  The proofs of the main theorems are outlined in Sections~\ref{sec:ProofInvertibility}
  and \ref{sec:ProofGenMain},
  with most technical work  deferred to the appendices.

\newpage

\section{Two numerical illustrations}\label{sec:numerical}

\subsection{Comparing NT models and Kernel Ridge Regression}

\begin{figure}[H]
\centering
\begin{subfigure}[t]{0.7\textwidth}
\centering
\includegraphics[width=\linewidth]{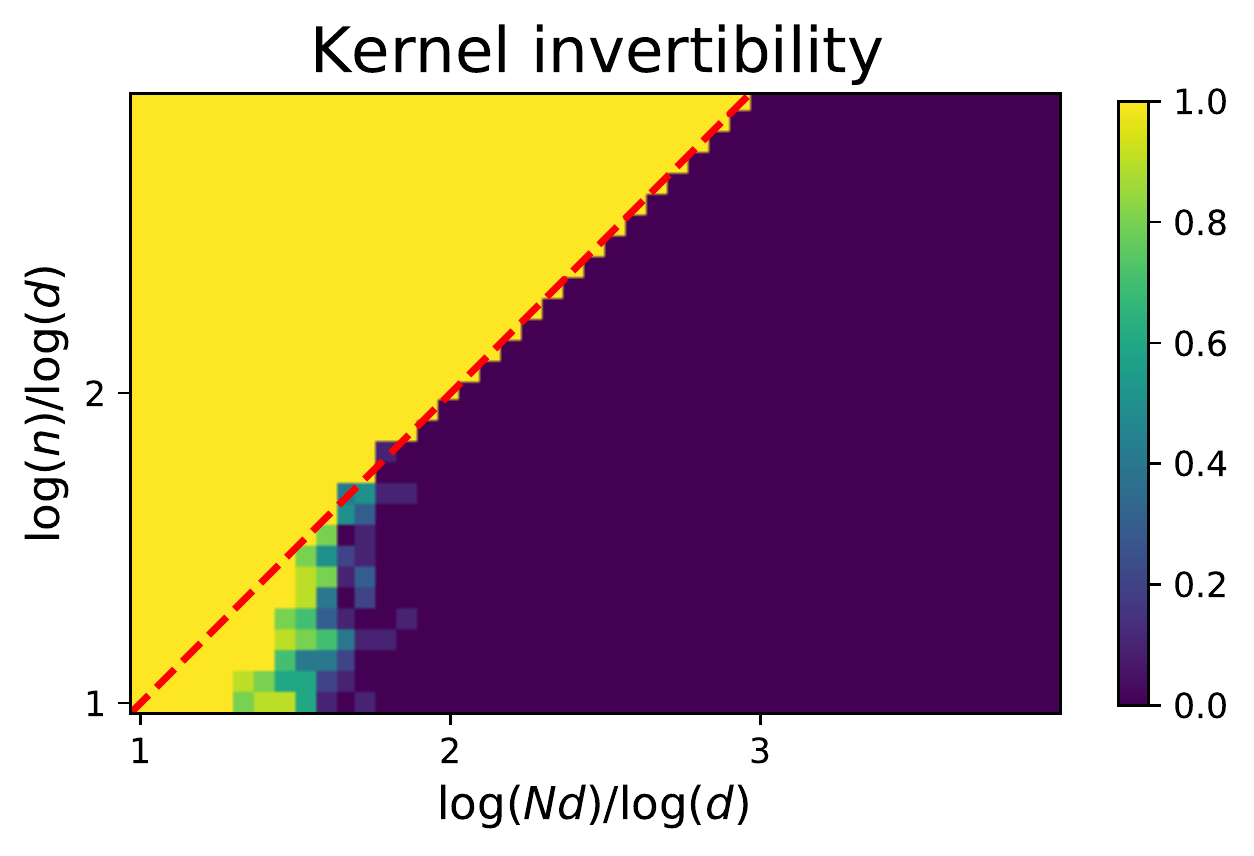}
\end{subfigure}
\begin{subfigure}[t]{0.7\textwidth}
\centering
\includegraphics[width=\linewidth]{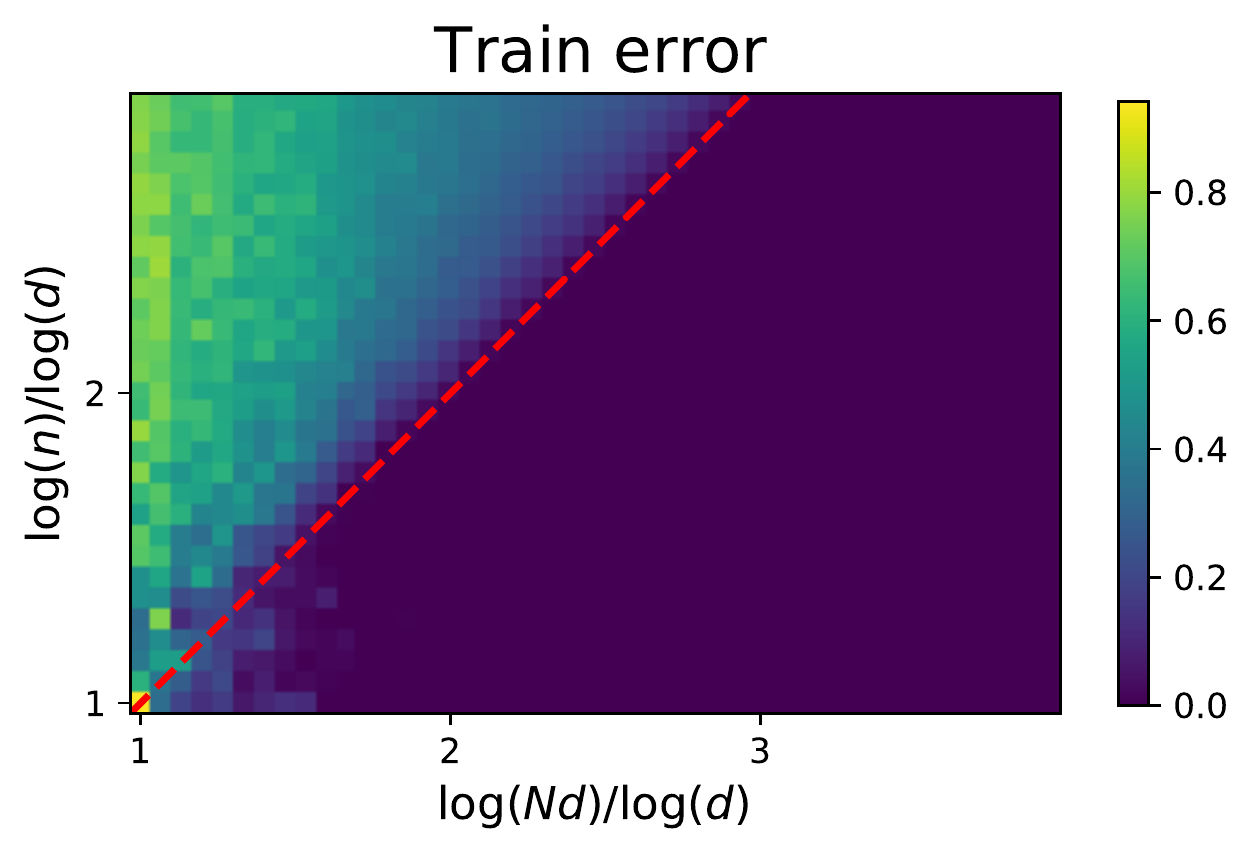}
\end{subfigure}
\vspace{-3mm}
\caption{In both heatmaps, we fix the dimension $d = 20$ and use min-$\ell_2$ norm NT regression to fit data generated according to 
\eqref{eq:Target}. Results are averaged over $n_{\text{rep}}=10$ repetitions. 
\textbf{Top:} For varying network parameters $Nd$ and sample size $n$, we check 
if $\bK_N$ is singular (we report the empirical probability).  
\textbf{Bottom:} We calculate the train error of min-$\ell_2$ norm NT regression.}\label{fig:phasetran1} 
\end{figure}

In the first experiment, we generated data according to the model
$y_i = f_*(\xx_i)+\veps_i$, with $\xx_i\sim\Unif(\S^{d-1}(\sqrt{d}))$,
$\veps_i\sim\normal(0,\sigma^2_{\veps})$, $\sigma_{\veps}=0.5$, and 
\begin{align}
f_*(\xx) = \sqrt{\frac{4}{10}} \,h_1(\<\bbeta_*,\xx\>)+
\sqrt{\frac{4}{10}} \, h_2(\<\bbeta_*,\xx\>)+\sqrt{\frac{2}{10}} \,h_4(\<\bbeta_*,\xx\>)\, .\label{eq:Target}
\end{align}

Here $\bbeta_*$ is a fixed unit norm vector (randomly generated and then fixed throughout), and $(h_k)_{k\ge 0}$ are orthonormal
Hermite polynomials, e.g. $h_1(x) =x$, $h_1(x) =(x^2-1)/\sqrt{2}$ 
(see Section \ref{sec:pre} for general definitions).
We fix $d=20$, and compute the min $\ell_2$-norm NT interpolator
for  ReLU activations $\sigma(t)=\max(t,0)$, and random weights 
$\ww_k\sim\Unif(\S^{d-1}(\sqrt{d}))$. We average our results
over $n_{\text{rep}}=10$ repetitions.

\begin{figure}[H]
\centering
\begin{subfigure}{0.7\textwidth}
\centering
\includegraphics[width=\linewidth]{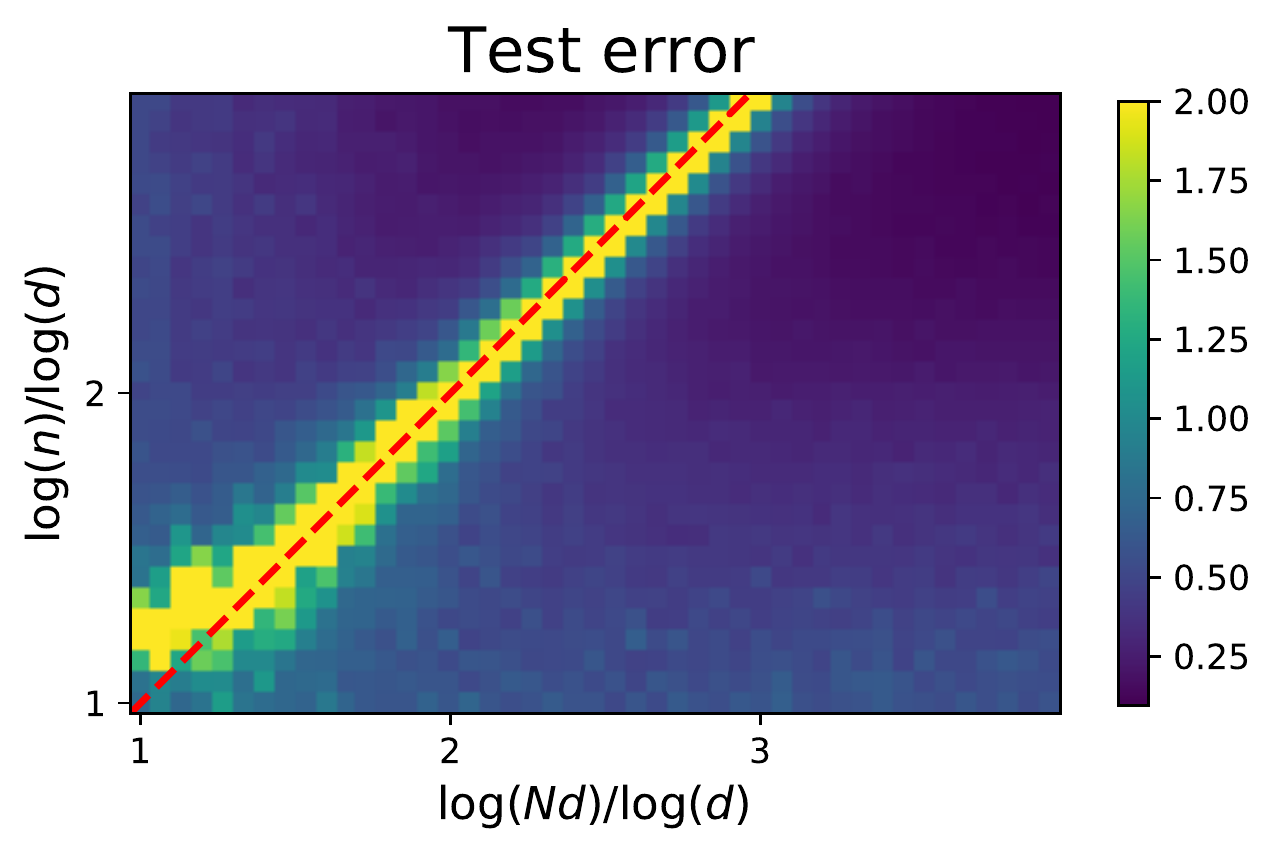}
\end{subfigure}
\begin{subfigure}{0.7\textwidth}
\centering
\includegraphics[width=0.9\linewidth]{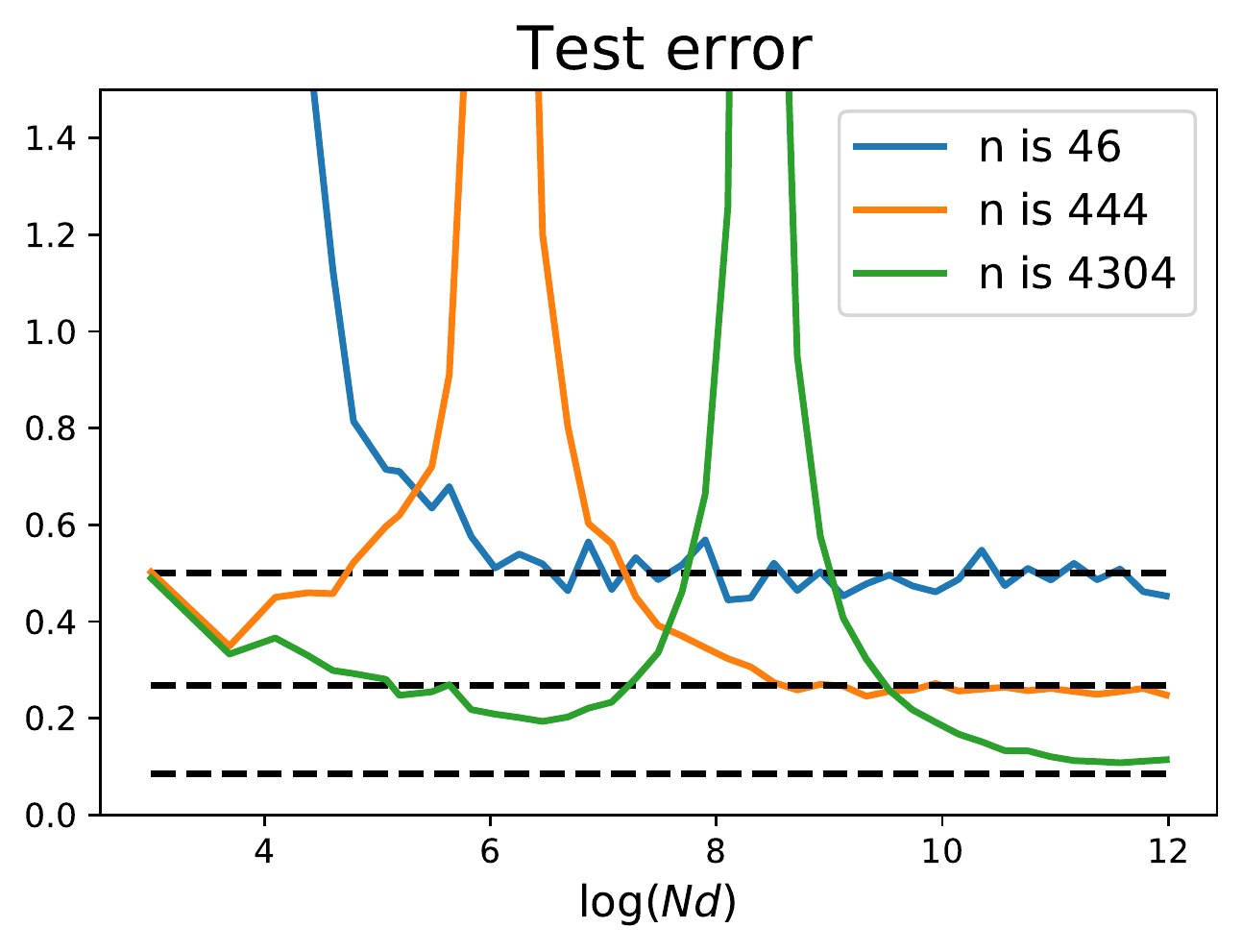}
\end{subfigure}
\vspace{-3mm}
\caption{We fix the dimension $d = 20$ and use min-$\ell_2$ norm NT regression as before. 
We calculated the test error using an independent test set of size $4000$. Results are 
averaged over $n_{\text{rep}}=10$ repetitions. \textbf{Top:} Test errors are plotted for 
varying $Nd$ and $n$; we cap the errors at $2$ so that the blowup near the dashed line 
is easy to visualize. \textbf{Bottom:} We examine the test errors at three particular sample 
sizes $n=46, 444$ and $4304$. The dashed horizontal lines represent test errors of KRR with 
the infinite-width kernel $\bK$ at three corresponding sample sizes.} \label{fig:phasetran2}
\end{figure}

From Figure~\ref{fig:phasetran1}, we observe that $(1)$~the minimum eigenvalue of the kernel becomes strictly positive 
very sharply as soon as $Nd/n\gtrsim 1$, $(2)$~as a consequence, the train error vanishes sharply as  
$Nd/n$ crosses $1$. Both phenomena are captured by
our theorems in the next section, although we require the 
condition $Nd/(\log Nd)^C\ge n$ which is suboptimal by a 
 polylogarithmic factor.

From Figure~\ref{fig:phasetran2}, we make the following observations and remarks.
\begin{enumerate}
 \item The number of samples $n$ and the number of parameters $Nd$ play a strikingly symmetric role. The
 test error is large when $Nd\approx n$ (the interpolation threshold)
 and decreases rapidly when either $Nd$ or $n$ increases (i.e.\ moving either along horizontal
 or vertical lines). In the context of random features model, a form of this symmetry property was established 
 rigorously in \cite{mei2021generalization}. For the present work, we only focus on the overparametrized regime $Nd\gg n$.
  
\item The test error rapidly decays to a limit value as $Nd$ grows at fixed $n$. 
We interpret the limit value as the infinite-width limit, and indeed matches 
the risk of kernel ridge(--less) regression with respect to the infinite-width kernel $\bK$
(dashed lines); see Theorem \ref{thm:gen}.

 \item Considering the most favorable case, namely $Nd\gg n$, the test error appears
 to remain bounded away from zero even when  $n\approx d^2$. This appears to be surprising
 given the simplicity of the target function \eqref{eq:Target}. 
This phenomenon can be explained by Theorem~\ref{thm:gen}, in which we show that for $d^\ell \ll n\ll d^{\ell+1}$ with $\ell$ an integer,  
 NT regression is roughly equivalent to regression with respect to 
 degree-$\ell$ polynomials. In particular, for $n\ll d^3$ it will not capture 
 components  of degree larger than 2 in the target $f_*$ of Eq.~\eqref{eq:Target}. 
\end{enumerate}

\begin{figure}[t]
\centering
\includegraphics[scale=0.8]{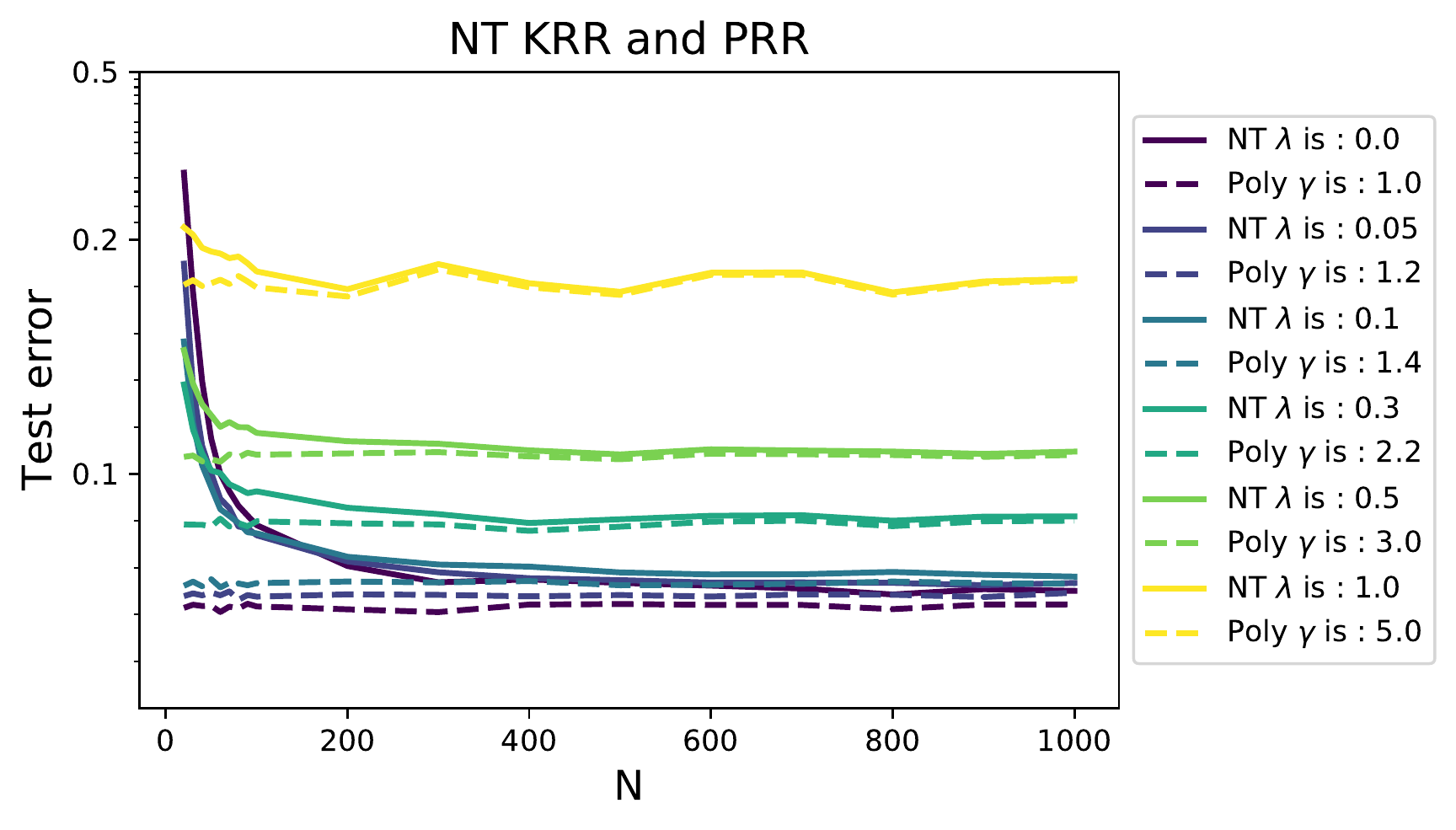}
\caption{Test/generalization errors of NT kernel ridge regression (NT) and polynomial ridge regression (Poly). We fix $n = 4000$ and $d=500$. For each regularization parameter $\lambda$ (which corresponds to one color), we plot two curves (solid: NT, dashed: Poly) that represent $R_{\NT}(\lambda)$ and $R_{\slin}(\gamma_{\seff}(\lambda, \sigma))$ respectively. }\label{fig:sim}
\end{figure}

\begin{figure}[t]
\centering
\includegraphics[scale=0.8]{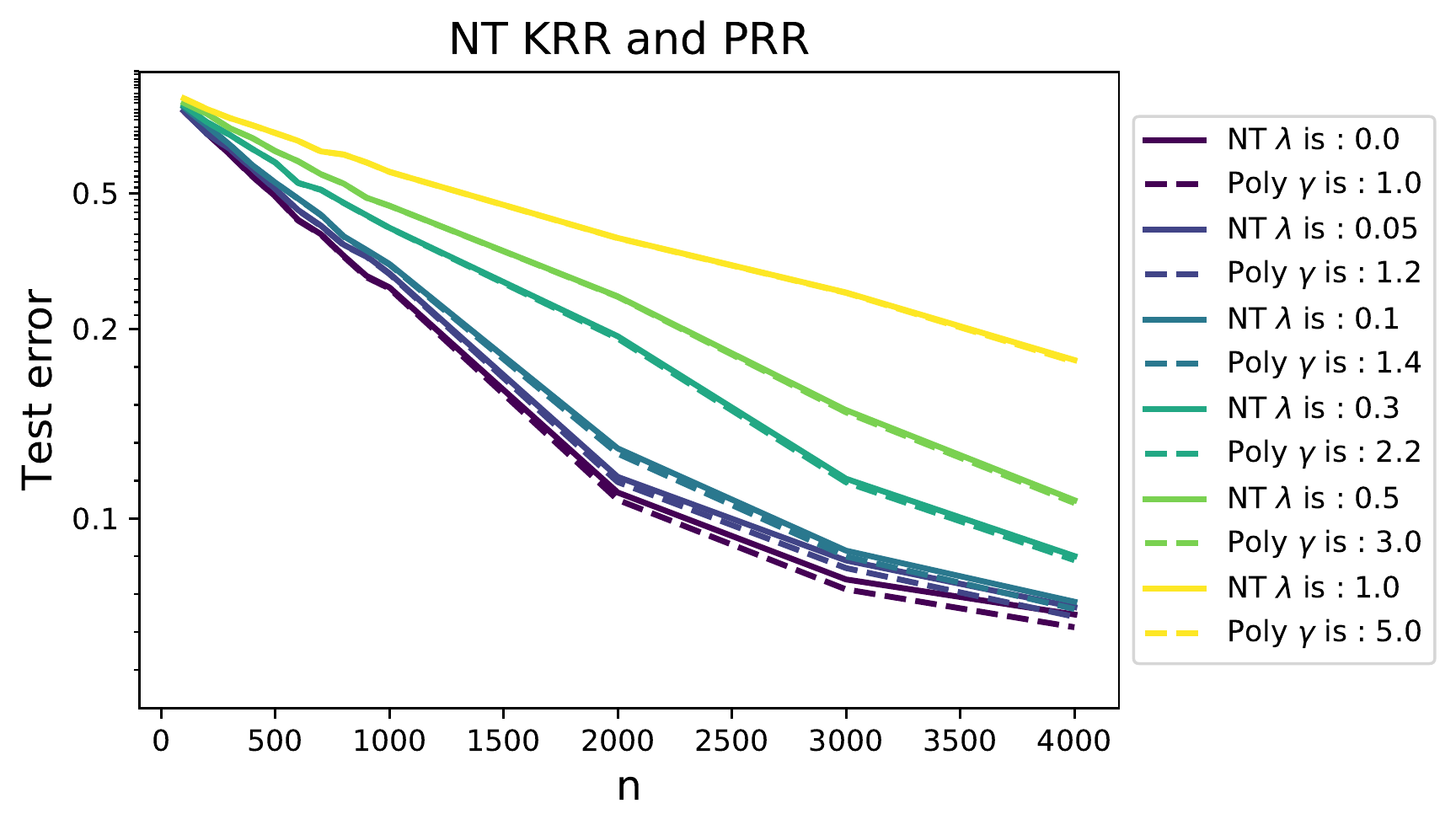}
\caption{Test/generalization errors of NT and Poly.  As before, for each regularization parameter $\lambda$ (which corresponds to one color), we plot two curves. 
}
\label{fig:sim2}
\end{figure}

\subsection{Comparing NT regression and polynomial regression}

In the second experiment, we generated data from a linear model
 $y_i = \langle \xx_i, \bbeta_* \rangle + \veps_i$. As before, $\xx_i\sim\Unif(\S^{d-1}(\sqrt{d}))$,
$\veps_i\sim\normal(0,\sigma^2_{\veps})$, $\sigma_{\veps}=0.5$, and $\bbeta_*$ 
is a fixed unit norm vector (randomly generated and then fixed throughout).
In Figure~\ref{fig:sim}, We fix $n = 4000$ and $d=500$, and vary the number of neurons 
\seqsplit{$N \in \{10,20,\ldots,90,100, \ldots, 1000\}$}.  In Figure~\ref{fig:sim2}, we 
fix $N = 800$ and $d=500$, and vary the sample size 
\seqsplit{$n \in \{100,200,\ldots,900,1000,1500,\ldots,4000\}$}. The results are averaged over
 $n_{\text{rep}}=10$ repetitions. 
 
 Notice that in all of these experiments $n\ll d^2$.  The theory developed below
 (see in particular Theorem \ref{thm:gen} and Corollary~\ref{cor:linear})
 implies that the risk of NT ridge regression should be well approximated by the risk
 polynomial ridge regression, although with an inflated ridge parameter. 
 For $n\ll d^2$, the polynomial degree is $\ell=1$. If $\lambda\ge 0$ denotes the 
 regularization parameter in NT ridge regression, the equivalent regularization
 in polynomial regression is predicted to be
 \begin{equation}
\gamma = \frac{\lambda + \var(\sigma')}{\big\{ \E \big[\sigma'(G)] \big\}^2},
\label{eq:MatchGammaLambda}
\end{equation}
where $G \sim \cN(0,1)$.

In Figures~\ref{fig:sim} and~\ref{fig:sim2} we fit NT ridge regression (NT) and polynomial
 ridge regression of degree $\ell=1$ (Poly).
 We use different pairs of regularization parameters $\lambda$ (for NT) and $\gamma$ (for Poly),
 satisfying  Eq.~\eqref{eq:MatchGammaLambda}. We observe a close match  between the risk 
 of NT regression
 and polynomial, in agreement with the theory established below.

We also trained a two-layer neural network to fit this linear model. Under specific 
initialization of weights, the test error is well aligned with the one from NT KRR and 
polynomial ridge regression. Details can be found in the appendix.

\section{Main results}
\label{sec:Main}

\subsection{Notations}\label{sec:notation}

For a positive integer, we denote by $[n]$ the set $\{1,2,\ldots,n\}$.  
The $\ell_2$ norm of a vector $\uu\in\reals^m$  is denoted by $\| \uu \|$. We denote by 
$\S^{d-1}(r) = \{\uu\in \R^d: \| \uu \|=r\}$ the sphere of radius $r$ in $d$ dimensions
 (sometimes we simply write $\S^{d-1} := \S^{d-1}(1)$). 

Let $\bA \in \R^{n \times m}$ be a matrix. 
We use $\sigma_j(\bA)$ to denote the $j$-th largest singular value of $\bA$, and we also denote 
$\sigma_{\max}(\bA) = \sigma_{1}(\bA)$ and $\sigma_{\min}(\bA) = \sigma_{\min\{m,n\}}(\bA)$. 
If $\bA$ is a symmetric matrix, we use $\lambda_j(\bA)$ to denote its $j$-th largest eigenvalue. We denote by $\| \bA \|_{\op} = \max_{\uu \in \S^{m-1}} \| \bA \uu \|$ 
the operator norm,  by $\| \bA \|_{\max} = \max_{i \in [n], j \in [m]} |A_{ij}|$ the maximum norm, 
 by $\| \bA \|_{F} = \big( \sum_{i,j} A_{ij}^2 \big)^{1/2}$ the Frobenius norm, and 
  by $\| \bA \|_* = \sum_{j=1}^{\min\{n,m\}} \sigma_j(\bA)$ the nuclear norm 
  (where $\sigma_j$ is the $j$-th singular value). If $\bA \in \R^{n \times n}$ is a square matrix, 
  the trace of $\bA$ is denoted by $\Tr(\bA) = \sum_{i \in [n]} A_{ii}$. 
  Positive semi-definite is abbreviated as p.s.d.

We will use $O_d(\cdot)$ and $o_d(\cdot)$ for the standard big-$O$ and small-$o$ notation, 
where $d$ is the asymptotic variable. 
We write $a_d = \Omega_d(b_d)$ for scalars $a_d,b_d$ if there exists $d_0, C>0$ 
such that $a_d \ge C b_d$ for $d > d_0$.
 For random variables $\xi_1(d)$ and $\xi_2(d)$, $\xi_1(d) = O_{d,\P}(\xi_2(d))$ if for any
  $\veps$, there exists $C_{\veps}>0$ and $d_{\veps}>0$ such that
\begin{equation*}
\P\big( | \xi_1(d) / \xi_2(d) | > C_\veps \big) \le \veps, \qquad \text{for all}~d > d_\veps.
\end{equation*}
Similarly, 
 $\xi_1(d) = o_{d,\P}(\xi_2(d))$ if $\xi_1(d) / \xi_2(d)$ converges to $0$ in probability. 
Occasionally, we use the notation $\tilde O_{d,\P}(\cdot)$ and $\tilde o_{d,\P}(\cdot)$: 
we write $\xi_1(d) = \tilde O_{d,\P}(\xi_2(d))$ if there exists a constant $C>0$ such that 
$\xi_1(d) = O_{d,\P}((\log d)^C \xi_2(d))$, and similarly we write 
$\xi_1(d) = \tilde o_{d,\P}(\xi_2(d))$ if there exists a constant $C>0$ such that 
$\xi_1(d) = o_{d,\P}((\log d)^C \xi_2(d))$. We may drop the subscript $d$ if there is no confusion 
in a context.

For a nonnegative sequence $(a_d)_{d \ge 1}$, we write $a_d = \poly(d)$ if there 
exists a constant $C>0$ such that $a_d \le Cd^C$.

Throughout, we will use $C,C_0,C_1,C_2,C_3$ to refer to constants that do not depend on $d$. 
In particular, for notational convenience, the value of $C$  may change from line to line.

Most of our statement apply to settings in which $N,n,d$ all grow to $\infty$,
while satisfying certain conditions. Without loss of generality, one can thing that 
such sequences are indexed by $d$, with $N,n$ functions of $d$.

Recall that we say that $A_{N,n,d}$ happens 
\textit{with  high probability} (abbreviated as w.h.p.) if its probability tends
 to one as $N,n,d\to \infty$
(in whatever way is specified in the text).
In certain proofs, we will say that an event $A_{N,n,d}$ happens 
\textit{with very high probability} if for every $\beta > 0$, we have 
$\lim_{d \to \infty} d^{\beta} \P(A_{N,n,d}^c) = 0$ (again, we think of $N,n\to\infty$
when $d\to\infty$ in whatever way prescribed in the text).

\subsection{Definitions and assumptions}

Throughout, we assume $\xx_i \sim_\iid \Unif(\S^{d-1}(\sqrt{d}))$ and 
$\ww_k \sim_\iid \Unif(\S^{d-1})$. 
To a vector of covariates $\xx\in\R^d$, the NT model associates a (random) 
features vector $\bPhi(\xx)$ as per Eq.~\eqref{eq:FeaturizationMap}.
We denote by $\bPhi \in \R^{n \times (Nd)}$ the matrix whose $i$-th row contains the feature vector of the $i$-th sample,
and by $\bK_N := \bPhi\bPhi^{\top}$ the corresponding empirical kernel:
\begin{equation*}
  \bPhi = \left[\begin{matrix} \bPhi(\xx_1)^{\top}\\
     \bPhi(\xx_2)^{\top}\\ \ldots\\
      \bPhi(\xx_n)^{\top} \end{matrix}\right] \in \R^{n \times (Nd)},\;\;\;\;\;\; 
      \bK_N := \bPhi\bPhi^{\top}\in\R^{n\times n}\, .
\end{equation*}
The entries of the kernel matrix take the form
\begin{equation}
[\bK_N]_{ij} = \frac{1}{Nd}\sum_{k=1}^N \sigma'(\< \xx_i, \ww_k \> ) \sigma'(\< \xx_j, \ww_k \> ) \< \xx_i, \xx_j \>\, .\label{eq:KernelDef}
\end{equation}
The infinite-width kernel matrix is given by
\begin{equation}
\label{eq:InfiniteWidthKernelDef}
[\bK]_{ij} = \E_\ww \big[ \sigma'(\langle \xx_i, \ww \rangle) \sigma'(\langle \xx_j, \ww \rangle) \big] \frac{\langle \xx_i, \xx_j \rangle}{d}.
\end{equation}
In terms of the featurization map $\bPhi$, an NT function $f\in\cF_{\NT}^N(\WW)$ reads
\begin{align}
f(\xx;\aa) = \<\aa,\bPhi(\xx)\>\, ,
\end{align}
where $\aa = (\aa_1,\dots \aa_N)^\top\in\R^{Nd}$.

\begin{ass}\label{ass:Asymp}
Given an arbitrary integer $\ell \ge 1$, and a small constant $c_0>0$, we assume that the following hold for a 
sufficiently large constant $C_0$ (depending on $c_0$, $\ell$, and on the activation $\sigma$)
\begin{align}
&  c_0d \le n \le \frac{d^2}{(\log d)^{C_0} }, & \text{if}~\ell = 1, \label{eq:AssLge1}\\
& d^{\ell} (\log d)^{C_0} \le n \le \frac{d^{\ell + 1}}{(\log d)^{C_0} }, & \text{if}~\ell > 1.
\label{eq:AssLge2}
\end{align}
\end{ass}

Throughout, we assume that the activation function $\sigma:\R\to\R$ satisfies the following condition. Note that commonly used activation functions, such as ReLU, sigmoid, tanh, leaky ReLU, satisfy this condition. Low-degree polynomials are excluded from this assumption. We denote by $\mu_k(\sigma')$ the $k$-th coefficient in the Hermite expansion of $\sigma'$; see its definition in Section~\ref{sec:pre}, especially Eqn.~\ref{decomp:Hermite}.

\begin{ass}[polynomial growth]\label{ass:Sigma}
We assume that $\sigma$ is weakly differentiable with weak derivative $\sigma'$ satisfying 
$|\sigma'(x)|\le B(1+|x|)^B$ for some finite constant $B>0$, and that
 $\sum_{k \ge \ell} [\mu_k(\sigma')]^2 > 0$.
\end{ass}

Note that this condition is extremely mild: it is satisfied by any activation function
of practical use. The existence of the weak derivative $\sigma'$
is needed for the NT model to make sense at all. 
The assumption of polynomial growth for  $\sigma'$ 
ensures that we can use harmonic analysis  on the sphere to analyze its behavior.
Finally the condition $\sum_{k \ge \ell} [\mu_k(\sigma')]^2 > 0$ is satisfied 
if $\sigma$ is not a polynomial of degree $\ell$.

\subsection{Structure of the kernel matrix}

Given points $\xx_1,\dots \xx_n$, $\bPhi$ has full row rank, i.e.\ $\rank(\bPhi) = n$ if and 
only if for any choice of the labels or responses $y_1,\dots,y_n\in\R$,
there exists  a function $f\in\cF_{\NT}^N(\WW)$  interpolating those data, i.e. $y_i=f(\xx_i)$ for all $i\le n$.  This of course requires $Nd\ge n$.
Our first result---which is a direct corollary of Theorem~\ref{thm:MinEigenvalue} that we present below---shows that this lower bound is roughly correct.
\begin{cor}\label{thm:1}
Assume $(\xx_i)_{i\le n}\sim_{\iid} \Unif(\S^{d-1}(\sqrt{d}))$, and  
$(\ww_k)_{k \le N}\sim_{\iid}\Unif(\S^{d-1}(1))$, and let $B,c_0>0$, $\ell\in\naturals$ be fixed.
 Then, there exist constants $C_0, C>0$ such that the following holds. 

If Assumptions~\ref{ass:Asymp} and~\ref{ass:Sigma} hold with constant $\ell,c_0,C_0, B$,
and
if $Nd/(\log (Nd))^C\ge n$, then $\rank(\bPhi) = n$ with high probability. 
In particular an NT interpolator exists with high probability for any choice of the responses $(y_i)_{i\le n}$.
\end{cor}
\begin{remark}
Note that any NT model \eqref{eq:FNT} can be approximated arbitrarily well by a two-layer
neural network with $2N$ neurons. This can be seen by taking $\eps\to 0$ in
$f_{\veps}(\xx) = \sum_{k=1}^N\frac{1}{\veps}\big\{\sigma(\<\ww_k,\xx\>+\eps\<\aa_k,\xx\>)
  - \sigma(\<\ww_k,\xx\>-\veps\<\aa_k,\xx\>) \big\}$.  
As a consequence, in the above
  setting,
  a $2N$-neurons neural network can interpolate $n$ data points 
  with arbitrarily small approximation error with high probability
  provided $Nd/(\log (Nd))^C\ge n$.

To the best of our knowledge, this is the first result of this type for regression. 
The concurrent paper \cite{daniely2020memorizing} proves a similar memorization 
result for classification but exploits in a crucial way the fact that in classification 
it is sufficient to ensure $y_if(\xx_i)>0$. Regression is considered in 
\cite{bubeck2020network} but the number of required neurons depends on the interpolation accuracy.
\end{remark}

While we stated the above as an independent result because of its interest, it is in fact an immediate corollary of
a quantitative lower bound on the minimum eigenvalue of the kernel $\bK_N\in\R^{n\times n}$, stated below. Given $g:\reals\to\reals$, we let $\mu_k(g)=\E[g(G)h_k(G)]$ denote its $k$-th 
Hermite coefficient (here $G\sim\normal(0,1)$, $\E[h_k(G)^2]=1$);
 see also Section~\ref{sec:pre}, Eq.~\ref{decomp:Hermite}.

\begin{thm}\label{thm:MinEigenvalue}
Assume $(\xx_i)_{i\le n}\sim_{\iid} \Unif(\S^{d-1}(\sqrt{d}))$, and  
$(\ww_k)_{k \le N}\sim_{\iid}\Unif(\S^{d-1}(1))$, and let $B,c_0>0$, $\ell\in\naturals$ be fixed.
 Then, there exist constants $C_0, C>0$ such that the following holds. 

If Assumptions~\ref{ass:Asymp} and~\ref{ass:Sigma} hold with constant $\ell,c_0,C_0, B$,
and further $n\ge d+1$ and $Nd/(\log (Nd))^C\ge n$, then, defining 
$v(\sigma) := \sum_{k\ge\ell} [\mu_k(\sigma')]^2$,
we have 
   \begin{align}
     \lambda_{\min}(\bK_N) = v(\sigma)+o_{d,\P}(1)\, .
   \end{align}
If the assumption $n\ge d+1$ is replaced by $n\ge c_0d$ for a strictly positive constant $c_0$,
then   $\lambda_{\min}(\bK_N) \ge  v(\sigma)-o_{d,\P}(1)$.
\end{thm}
\begin{remark}\label{rmk:var}
In the case $\ell=1$, the eigenvalue lower bound $v(\sigma)$ is simply 
$v(\sigma) = \Var(\sigma'(G)) = \E[\sigma'(G)^2]-\{\E[\sigma'(G)]\}^2$ where $G\sim\normal(0,1)$.
\end{remark}
The only earlier result comparable to Theorem \ref{thm:MinEigenvalue} 
was obtained in \cite{soltanolkotabi2018theoretical} which, in a similar setting,
proved that, with high probability, $\lambda_{\min}(\bK_N) \ge c_*$ for a strictly positive
constant $c_*$, provided $N\ge 2d$.

The proof of Theorem \ref{thm:MinEigenvalue}  is presented in Section~\ref{sec:ProofInvertibility}. 
The key is to show that the NT kernel concentrates to the infinite-width kernel for large enough $N$.

For any constant $\gamma>0$, we define an event $\cA_\gamma$ as follows.
\begin{equation}
\cA_\gamma = \big\{ \bK \succeq \gamma \bI_n, ~ \| \XX \|_\op \le 2(\sqrt{n}+\sqrt{d}) \big\}.
\end{equation}
This event only involves $(\xx_i)_{i\le n}$ (formally speaking, it lies in the sigma-algebra generated by 
$(\xx_i)_{i\le n}$). We will show that, if $\gamma>0$ is a constant no larger than $v(\sigma)/2$, then the
 event $\cA_\gamma$ happens with
 very high probability under Assumption \ref{ass:Sigma}, provided $n\le d^{\ell+1}/(\log d)^{C_0}$. 
 Below we will use $\P_{\ww}$ to mean the probability over the randomness of 
 $\ww_1,\ldots,\ww_N$
 (equivalently, conditional on $(\xx_i)_{i \le n}$).

\begin{thm}[Kernel concentration]\label{thm:invert2}
Assume $(\xx_i)_{i\le n}\sim_{\iid} \Unif(\S^{d-1}(\sqrt{d}))$, and  
$(\ww_k)_{k \le N}\sim_{\iid}\Unif(\S^{d-1}(1))$. Let $\gamma = v(\sigma)/2$. Then, there exist constants 
$C', C_0>0$ such that the following holds. Under Assumption~\ref{ass:Sigma}, the event $A_\gamma$ holds with very high probability, and on the event $A_\gamma$, for any constant $\beta > 0$, 
\begin{equation*}
d^\beta \cdot \P_\ww \left( \big\| \bK^{-1/2} \bK_N \bK^{-1/2} - \bI_n \big\|_\op > 
\sqrt{\frac{(n+d)( \log (nNd))^{C'}}{Nd}} + \frac{(n+d)( \log (nNd))^{C'}}{Nd} \right) = o_d(1).
\end{equation*}
As a consequence, if $n\le d^{\ell+1}/(\log d)^{C_0}$ (i.e.\ the upper bound in 
Assumption~\ref{ass:Asymp} holds), with the same $\ell$ as in  Assumption~\ref{ass:Sigma}, 
then with very high probability,
\begin{equation}\label{ineq:invert2}
\big\| \bK^{-1/2} \bK_N \bK^{-1/2} - \bI_n \big\|_\op  \le  \sqrt{\frac{(n+d)( \log (nNd))^{C'}}{Nd}} + 
\frac{(n+d)( \log (nNd))^{C'}}{Nd}.
\end{equation}

\end{thm}

When the right-hand side of \eqref{ineq:invert2} is $o_{d,\P}(1)$, it is clear that the eigenvalues of $\bK_N$ are bounded from below, since
\begin{equation}\label{ineq:KNtwosides}
\big\| \bK^{-1/2} \bK_N \bK^{-1/2} - \bI_n \big\|_\op \le \eta' \quad \Longrightarrow \quad   (1-\eta') \bK  \preceq  \bK_N \preceq (1+\eta') \bK.
\end{equation} 
If we further have $n \le  Nd/(\log(Nd))^{C}$, $n\ge c_0d$ for $c_0>0$ and a sufficiently large constant $C$,
 then we can take $\eta'= (n( \log (Nd))^{C'}/Nd)^{1/2}$ here. 
 
\begin{remark}\label{rmk:nonpoly}
 If $\sigma'$ is not a polynomial, then $\cA_{\gamma}$ holds with high probability
as soon as $n\le d^{C''}$ for some constant $C''$. Hence, in this case, the stronger assumption 
$n\le d^{\ell+1}/(\log d)^{C_0}$ is not needed for the second part of this theorem.
\end{remark}

\subsection{Test error}
\label{sec:Gen}

In order to study the generalization properties of the NT model, we consider a general regression model for
the data distribution.  Data $(\xx_1, y_1), \ldots, (\xx_n, y_n)$ are i.i.d.\
with $\xx_i\sim\Unif(\S^{d-1}(\sqrt{d}))$ and $y_i = f_*(\xx_i)+\veps_i$ where 
$f_*\in L^2 := L^2(\S^{d-1}(\sqrt{d}), \tau_{d-1})$. Here, $L^2(\S^{d-1}(\sqrt{d}), \tau_{d-1})$ is the space of functions on $\S^{d-1}(\sqrt{d})$ which are square integrable
   with respect to the uniform measure. In matrix notation, 
we let $\XX\in\R^{n\times d}$ denote the matrix
whose $i$-th row is $\xx_i$, $\yy:= (y_i)_{i\le n}$, and $\ff_* = (f_*(\xx_i))_{i \le n}$. We then have
\begin{equation}\label{model:linear}
\yy =  \ff_*  + \bveps, \qquad \where~\var(\veps_i) = \sigma_\veps^2. 
\end{equation}
The noise variables  $\veps_1,\ldots,\veps_n \sim_{\iid} \P_\veps$ are assumed to have zero mean 
and finite second moment, i.e.,
$\sigma_\veps^2= \E(\veps_1^2)$.

We fit the coefficients $\aa= (\aa_1,\dots,\aa_N)\in\R^{Nd}$ of the NT model using  ridge regression. Namely
\begin{align}
  \hba(\lambda) := \arg\min_{\aa\in\R^{Nd}}\Big\{ \sum_{i=1}^n(y_i-f(\xx_i;\aa)\big)^2 + \lambda \|\aa\|^2\Big\}\, .
  \label{eq:NT-Regression}
\end{align}
where $f(\,\cdot\,;\aa)$ is defined as per Eq.~\eqref{eq:FNT}.
Explicitly, we have
\begin{equation}
\hba(\lambda) = \bPhi^\top \big( \bPhi \bPhi^\top + \lambda \bI_n)^{-1} \yy.
\label{eq:NT-Regression-Explicit}
\end{equation}
We evaluate this approach on a new input $\xx_0$ that has the same distribution as the training input.
Our analysis covers the case $\lambda=0$, in which case we obtain the minimum $\ell_2$-norm interpolator.

The test  error is defined as
\begin{align}
  R_{\NT}(f_*;\lambda) = \E_{\xx_0}\big[ (f_*(\xx_0)-\<\hba(\lambda),\bPhi(\xx_0)\>)^2\big]\, .
\end{align}
We occasionally call this `generalization error', with a slight abuse of terminology
(sometimes this term is referred to the difference between test error 
and the train error
$n^{-1}\sum_{i\le n} (y_i-\<\hba,\bPhi(\xx_i)\>)^2$.)

Our main results on the generalization behavior of NT KRR establish its equivalence with  simpler
methods. Namely, we perform kernel ridge regression with the infinite-width kernel $\bK$. Let $\gamma\ge 0$ be any ridge regularization parameter. The prediction function fitted by the data alongside its associated risk is given by
\begin{align*}
& \hat f_{\KRR}^\gamma(\xx) = \bK(\cdot, \xx)^\top ( \gamma \bI_n + \bK)^{-1} \yy, \\
&R_{\KRR}(f_*;\gamma) := \E_{\xx_0}\big[ (f_*(\xx_0)-\hat f_{\KRR}^\gamma(\xx_0))^2\big].
\end{align*}
We also define the polynomial ridge regression as follows. The infinite-width kernel 
$K$ can be decomposed uniquely in orthogonal polynomials as 
$K(\xx, \xx') = \sum_{k=0}^\infty \gamma_k Q_k^{(d)}(\langle \xx, \xx' \rangle)$,
where $Q_k^{(d)}(z)$ is the Gegenbauer polynomial of degree $k$ 
(see Lemma~\ref{lem:K_Harmonic}). We consider truncating the kernel function up to the degree-$\ell$ polynomials:
\begin{equation}\label{def:Kp}
K^p(\xx, \xx') = \sum_{k=0}^\ell \gamma_k Q_k^{(d)}(\langle \xx, \xx' \rangle).
\end{equation}
The superscript refers to the name ``polynomial''. We also define
\begin{equation}
\label{eq:MatrixKp}
\bK^p = \big(K^p(\xx_i, \xx_j)\big)_{i,j\le n}, \qquad \bK^{p}(\cdot, \xx) = \big( K^p(\xx_i, \xx) \big)_{i\le n}.
\end{equation}
For example, in the case $\ell=1$, we have equivalence $\bK^p = \gamma_0 \bone_n \bone_n^\top + \frac{\gamma_1}{d} \XX \XX^\top$ (the kernel of linear regression with an intercept). The prediction function fitted by the data and its associated risk are
\begin{align*}
&\hat f_{\PRR}^\gamma(\xx) = \bK^p(\cdot, \xx)^\top (\gamma \bI_n + \bK^p)^{-1} \yy, \\
&R_{\PRR}(f_*;\gamma) := \E_{\xx_0}\big[ (f_*(\xx_0)-\hat f_{\PRR}^\gamma(\xx_0))^2\big].
\end{align*}
Kernel ridge regression and polynomial ridge regression are well understood. Our next result establishes a relation
between these risks: the neural tangent model behaves as the polynomial model, albeit with a different value of the regularization parameter.

\begin{thm}\label{thm:gen}
Assume $(\xx_i)_{i\le n}\sim_{\iid} \Unif(\S^{d-1}(\sqrt{d}))$, and  
$(\ww_k)_{k \le N}\sim_{\iid}\Unif(\S^{d-1}(1))$. Recall $v(\sigma) := \sum_{k\ge\ell} [\mu_k(\sigma')]^2$
and let $B,c_0>0$, $\ell\in\naturals$  be fixed. 
 Then, there exist constants $C_0, C,C'>0$ such that the following holds.  

If Assumptions~\ref{ass:Asymp} and~\ref{ass:Sigma} hold with constants $B,c_0,\ell$, 
and $Nd/(\log (Nd))^C\ge n$, then for any $\lambda\ge 0$,
\begin{align*}
R_{\NT}(f_*;\lambda) &= R_{\KRR}(f_*;\lambda) + O_{d,\P}\Big(\tau^2\sqrt{\frac{n(\log (Nd))^{C'}}{Nd}}\Big) \\
&=R_{\PRR}(f_*;\lambda + v(\sigma)) + O_{d,\P}\Big(\tau^2\sqrt{\frac{n(\log (Nd))^{C'}}{Nd}} \, + \tau^2 \sqrt{\frac{n (\log n)^C}{d^{\ell+1}}} \Big),
\end{align*}
where $\tau^2 := \|f_*\|^2_{L^2}+\sigma^2_{\veps}$.
\end{thm}

\paragraph{Example: $n\ll d^2$.} 
 Suppose  Assumptions~\ref{ass:Asymp} 
 and~\ref{ass:Sigma} hold with $\ell=1$, i.e.\  $c_0 d \le n \le d^2/(\log d)^{C_0}$.
 In this case, Theorem \ref{thm:gen} implies that NT regression can only fit the 
 linear component in the target function $f_*(\xx)$. In order to simplify our treatment,
 we assume that the target is linear: $f_*(\xx) = \<\bbeta_*,\xx\>$.
 
 Consider  ridge regression with respect to the linear features, with regularization 
 $\gamma \ge 0$:
\begin{align}
  \hbbeta(\gamma) &:=  \arg\min_{\bbeta\in\R^{d}}\Big\{\frac{1}{d} \sum_{i=1}^n(y_i-\<\bbeta,\xx_i\>\big)^2 + \gamma\|\bbeta\|_2^2
                    \Big\}\, ,\\
  R_{\slin}(f_*;\gamma) &:= \E_{\xx_0}\big[ (\<\bbeta_*,\xx_0\>-\<\hbbeta(\gamma),\xx_0\>)^2\big]\, .
\end{align}
Note that $R_{\slin}(f_*;\gamma)$ is essentially the same as $R_{\PRR}(f_*;\lambda + v(\sigma))$,
with the minor difference that we are not fitting an intercept.
 The scaling factor $d^{-1}$ is specially chosen for comparison with NT KRR. 
 Theorem~\ref{thm:gen} implies the following correspondence.
\begin{cor}[NT KRR for linear model]\label{cor:linear}
Assume $(\xx_i)_{i\le n}\sim_{\iid} \Unif(\S^{d-1}(\sqrt{d}))$, 
$(\ww_k)_{k \le N}\sim_{\iid}\Unif(\S^{d-1}(1))$, and 
$\E(\veps_i^4)\le C\sigma_{\veps}^4$. Denote $\tau^2 = \| \bbeta_* \|^2 + \sigma_\veps^2$.
 Then, there exist constants $C_0, C>0$ such that the following holds. 
 Under Assumption~\ref{ass:Asymp} and~\ref{ass:Sigma} with $\ell=1$ and $Nd/(\log(Nd))^C \ge n$, 
\begin{align}
    R_{\NT}(f_*;\lambda) &= R_{\slin}(f_*;\gamma_{\seff}(\lambda,\sigma))+O_{d,\P}\Big(\tau^2\sqrt{\frac{n (\log d)^C}{Nd}} + \frac{\tau^2}{\sqrt{d}} + \tau^2\sqrt{\frac{n(\log n)^C}{d^2}}\Big)\, , \qquad \where \label{eq:NT-lin-expr} \\
    \gamma_{\seff}(\lambda,\sigma)&:=\frac{\lambda+v(\sigma)}{\{ \E[\sigma'(G)] \}^2}.\,  \label{def:gammaeff}
    %
  \end{align}
Further, 
\begin{align}
R_{\slin}(f_*;\gamma) &= \|\bbeta_*\|^2_2\cuB_{\slin}(\gamma) + \sigma_\veps^2\cuV_{\slin}(\gamma) + 
O_{d,\P}( \tau^2/\sqrt{d} ), \label{eq:Lin-NT}\\
\cuB_{\slin}(\gamma) & := \frac{\gamma^2}{d}  \Tr\Big( \big(\gamma\id_d+\XX^{\sT}\XX/d\big)^{-2} \Big), \label{eq:Lin-NT1} \\
    \cuV_{\slin}(\gamma) &:=\frac{1}{d^2} \Tr\Big(\XX^{\sT}\XX\big(\gamma\id_d+\XX^{\sT}\XX/d\big)^{-2}\Big)\, . \label{eq:Lin-NT2}
  \end{align}
\end{cor}

\begin{remark}\label{rmk:explicit}
The error term $O_{d,\P}(\tau^2 d^{-1/2})$ in \eqref{eq:NT-lin-expr} is due to the effect of the intercept. 
It is easy to derive the asymptotic formulas  
if $n/d \to\kappa\in (0,\infty)$ where $\kappa$ is a constant \cite{hastie2019surprises}:
  \begin{align}
    \cuB_{\slin}(\kappa,\gamma) & = \frac{1}{2}\left\{1-\kappa+\sqrt{(\kappa-1+\gamma)^2+4\gamma}-
                                  \frac{\gamma(1+\kappa+\gamma)}{\sqrt{(\kappa-1+\gamma)^2+4\gamma}}\right\}+o_{d,\P}(1)\, , \label{eq:Lin-B}\\
    \cuV_{\slin}(\kappa,\gamma) &= \frac{1}{2}\left\{-1+\frac{\kappa+\gamma+1}{\sqrt{(\kappa-1+\gamma)^2+4\gamma}}\right\}+o_{d,\P}(1)\, . \label{eq:Lin-V}
  \end{align}
Here we emphasized the dependence on $\kappa$. Also, as $\kappa \to\infty$,  $\cuB_{\slin}(\kappa,\gamma) = \gamma^2\kappa^{-2} +O(\kappa^{-3})$ and
  $\cuV_{\slin}(\kappa,\gamma) = \kappa^{-1} +O(\kappa^{-2})$.
\end{remark}

In particular, the ridgeless NT model at $\lambda=0$ corresponds to linear regression with regularization
  $\gamma = v(\sigma)/\{ \E[\sigma'(G)] \}^2$. Note that the denominator in \eqref{def:gammaeff} is due to the different scaling for the linear term (cf. Eq.~\ref{def:Kp}). 
Since $\cuB_{\slin}(\kappa,\gamma) = O(d^2/n^2)$ and $\cuV_{\slin}(\kappa,\gamma) = O(d/n)$ from the formulas \eqref{eq:Lin-B} and \eqref{eq:Lin-V}, we find that the error term of order $\sqrt{n/Nd}$  in Eq.~\eqref{eq:NT-lin-expr}
  is negligible for $1\le n/d\ll N^{1/3}$. We leave
  to future work the problem of obtaining optimal bounds on the error term.

\subsection{An upper bound on the memorization capacity}\label{sec:lowerbnd}

In the regression setting, naively counting degrees of freedom would imply that the memorization
capacity of a two-layer network with $N$ neurons is at most given by the number of parameters, 
namely $N(d+1)$.
More careful consideration reveals that, as long as the data $\{(\xx_i,y_i)\}_{i\le n}$
are in generic positions $N=O(1)$ neurons are sufficient to achieve 
memorization within any fixed accuracy, using a `sawlike' activation function. 

These constructions are of course fragile and a non-trivial upper bound on 
the memorization capacity can be obtained if we constrain the weights.
In this section, we prove such an upper bound. We state it for binary classification, but 
it is immediate to see that it implies an upper bound on regression which is of the same order.

To simplify the argument, we assume $(\xx_i)_{i\le n} \sim_{\iid} \cN(\bzero, \bI_d)$ 
(which is very close to the setting $(\xx_i)_{i\le n} \sim_{\iid} \Unif(\S^{d-1}(\sqrt{d})$) 
and $y_1,\ldots,y_n \sim_{\iid} \Unif(\{+1,-1\})$. Consider a subset of two-layers neural networks,
by constraining the magnitude of the parameters:
\begin{equation*}
\cF_{\NN}^{N,L} := \Big\{ f(\xx; \bb, \WW) = \sum_{k=1}^N b_{k} \sigma\big( \langle \ww_\ell, \xx \rangle \big), \quad N^{-1/2}\| \bb \| \le L, \|  \ww_k \| \le L,~ \forall\, k \in [N] \Big\}.
\end{equation*}
To get binary outputs from $f(\xx)\in \R$, we take the sign. In particular, the label predict
on sample $\xx_i$ is $\sign(f(\xx_i))$. 

The value $y_i f(\xx_i)$ is referred to the \textit{margin} for input $\xx_i$. In regression we require
$y_i f(\xx_i)  = y_i^2=1$ (the latter equality holds for $\{+1,-1\}$ labels), while in classification we ask for
$yf(\xx_i)>0$.
The next result states that $Nd$ must be roughly of order $n$ in order for a neural network to fit
a nontrivial fraction of data with $\delta$ margin.  
\begin{prop}\label{thm:lowerbound2}
  Let Assumptions~\ref{ass:Asymp} and~\ref{ass:Sigma} hold, and further assume $(\xx_i)_{i\le n}\sim_{\iid}\normal(\bzero,\id_d)$. Fix constants $\eta_1,\eta_2 > 0$, and let $L = L_d \ge 1$, $\delta = \delta_d>0$ be general functions of
  $d$. Then there exists a constant $C$ such that the following holds.

Assume that, with probability larger than $\eta_1$, there exists a function $f \in \cF_{\NN}^{N,L}$ such that 
\begin{equation}\label{ineq:0.51}
\frac{1}{n} \sum_{i=1}^n \bone\{ y_i f(\xx_i) > \delta \} \ge \frac{1+\eta_2}{2}.
\end{equation}
(In words, $f$ achieves margin $\delta$ in at least a fraction $(1+\eta_2)/2$ of the samples.)

Then we must have $n \le C^{-1} Nd\log (L d/\delta)$. 
\end{prop}

If  $L$ and $1/\delta$ are upper bounded by a polynomial of $d$, then this result implies an upper bound
on the network capacity of order $Nd\log d$. This  matches the interpolation and invertibility thresholds
in Corollary~\ref{thm:1}, and Theorem~\ref{thm:MinEigenvalue} up to a logarithmic factor.
The proof follows a discretization--approximation argument, which can be found in the appendix. 

\begin{remark}
  Notice that the memorization capacity upper bound $C^{-1} Nd\log (L d/\delta)$ tends to infinity when the margin vanishes $\delta\to 0$.
  This is not an artifact of the proof. As mentioned above, if we allow for an arbitrarily small margin,
  it is possible to construct a `sawlike' activation function
  $\sigma$ such that the corresponding network correctly classifies $n$ points with binary labels $y_i\in \{+1,-1\}$
  despite $n\gg Nd$. Note that our result is different from \cite{yun2019small, bartlett2019nearly} in which the activation function is piecewise linear/polynomial.
\end{remark}

\begin{remark}
While we stated Proposition \ref{thm:lowerbound2} for classification, it has obvious consequences
for regression interpolation. Indeed, considering  $y_i\in\{+1,-1\}$,
the interpolation constraint $f(\xx_i) = y_i$ implies $y_if(\xx_i) \ge \delta_0=1$.
\end{remark}

\section{Related work}\label{sec:related}

As discussed in the introduction, we addressed three questions:
{\sf Q1:} What is the maximum number of training samples
$n$ that a network of given width $N$ can interpolate (or memorize)?
{\sf Q2:} Can such an interpolator be found efficiently?
{\sf Q3:} What are the generalization properties of the interpolator?
While questions {\sf Q1} and {\sf Q2} have some history (which we briefly review next),
much less is known about {\sf Q3}, which is our main goal in this paper.

In the context of binary classification, question {\sf Q1} was first studied by Tom Cover \cite{cover1965geometrical},
who considered the case of a simple perceptron network ($N=1$)
when the feature vectors $(\xx_i)_{i\le n}$ are in a generic position, and the labels $y_i\in\{+1,-1\}$
are independent and uniformly random. He proved that this model can memorize $n$ training samples
with high probability if $n \le 2d(1-\eps)$, and cannot memorize them with high probability if
$n \ge 2d(1+\eps)$. Following Cover, this maximum number of samples is sometimes referred to as the
network capacity but, for greater clarity, we also use the expression \emph{network memorization capacity.}

The case of two-layers network was studied by Baum
\cite{baum1988capabilities} who proved that, again for any set of points in
general positions, the memorization capacity is at least $Nd$. Upper bound of the same order were proved,
among others, in \cite{sakurai1992,kowalczyk1994counting}. Generalizations to multilayer 
networks were proven recently in \cite{yun2019small,vershynin2020memory}.
Can these networks be found efficiently? In the context of classification,
the recent work of \cite{daniely2019neural} provides an efficient
algorithm that can memorize all but a fraction $\eps$ of the training samples in polynomial time, 
provided the $Nd\ge Cn/\eps^2$. For the case of Gaussian feature vectors, \cite{daniely2020memorizing}
proves that exact memorization can be achieved efficiently provided $Nd \ge Cn(\log d)^4$.

Here we are interested in achieving memorization in regression, which is more 
challenging than for classification. Indeed, in binary classification a function
$f:\R^d\to\R$ memorizes the data if $y_if(\xx_i)>0$ for all $i\le n$. On the other hand, in
our setting, memorization amounts to $f(\xx_i)=y_i$ for all $i\le n$. The techniques developed for binary
classification exploit in a crucial way the flexibility provided by the inequality constraint, which we cannot do
here.

In concurrent work, \cite{bubeck2020network} studied the interpolation properties
of two-layers networks. Generalizing the construction of Baum
\cite{baum1988capabilities}, they show that,  for $N\ge 4\lceil n/d\rceil$ there exists a two-layers ReLU
network interpolating $n$ points in generic positions. This is however unlikely to be the network produced by
gradient-based training. They also construct a model that interpolates the data with error
$\eps$, provided $Nd\ge n\log(1/\eps)$. 
 In contrast, we obtain exact interpolation provided $Nd\ge n(\log Nd)^C$
From a more fundamental point of view, our work does not only construct a network that memorizes
the data, but also characterizes the eigenstructure of the kernel matrix.
While our paper is under review, more papers on memorization are posted, including \cite{vardi2021optimal, park2021provable} that study the effect of depth.

As discussed in the introduction, we focus here on the lazy or neural tangent regime in which weights
change only slightly with respect to a random initialization \cite{jacot2018neural}.
This regime attracted considerable attention over the last two years, although the focus
has been so far on its implications for optimization, rather than on its statistical 
properties.

It was first shown in \cite{du2018gradient} that, for sufficiency overparametrized networks, and under suitable 
initializations, gradient-based training indeed converges to an interpolator that is well approximated by
an NT model. The proof of  \cite{du2018gradient} required (in the present context) $N\ge
Cn^6/\lambda_{\min}(\bK)^4$, where $\bK$ is the infinite width kernel of Eq.~\eqref{eq:InfiniteWidthKernelDef}.
This bound was improved over the last two years 
\cite{du2018gradient,allen2019convergence,lee2019wide,zou2020gradient,oymak2020towards,liu2020linearity,weinan2020comparative}.
 In particular, \cite{oymak2020towards} prove that, for $Nd\ge Cn^2$, gradient descent converges
 to an interpolator. The authors also point out the gap between
 this result and the natural lower bound $Nd\gtrsim n$.

 A key step in the analysis of gradient descent in the NT regime is to prove that the 
 tangent feature map at the initialization (the matrix
 $\bPhi\in\R^{n\times Nd}$) or, equivalently, the associated kernel (i.e. the matrix
 $\bK_N=\bPhi\bPhi^\top$) is nonsingular. Our Theorem~\ref{thm:MinEigenvalue} establishes that this is the
 case for $Nd\ge n(\log d)^C$. As discussed in Section \ref{sec:ConnectionOptimization},
 this implies convergence of gradient descent under the near-optimal condition
 $Nd\ge n(\log d)^C$, under a suitable scaling of the weights 
\cite{chizat2019lazy}. More generally, our characterization of the eigenstructure of $\bK_N$
 is a foundational step towards a sharper analysis of the gradient descent dynamics.
 
 As mentioned several times, our main contribution is a characterization 
 of the test error of minimum norm regression in the NT model 
 \eqref{eq:FNT}. We are not aware of any comparable results. 
 
 Upper bounds on the generalization error of neural networks based on NT theory 
 were proved, among others, in \cite{arora2019fine,allen2019learning,cao2019generalization,ji2019polylogarithmic,chen2019much,nitanda2019gradient}.
 These works assume a more general data distribution than ours.
 However their objectives are of very different nature from ours.
 We characterize the test error of interpolators, while most of these works do not
 consider interpolators. Our main result is a sharp characterization of the difference 
 between test error of NT regression (with a finite number of neurons $N$)
 and kernel ridge regression (corresponding to $N=\infty$), and between this and polynomial 
 regression. None of the earlier work is sharp enough to provide to control these quantities.
 More in detail:
 \begin{itemize}
 \item The upper bound of \cite{arora2019fine} applies to interpolators but controls
 the generalization error by $(\<\yy,\bK^{-1}\yy\>/n)^{1/2}$ which, in the setting studied in
 the present paper,  is a quantity of order one.
\item The upper bounds of \cite{allen2019learning,cao2019generalization} do not apply to 
interpolators since SGD is run in a one-pass fashion.
Further they require large overparametrization, namely  $N\gtrsim n^7$ in 
\cite{cao2019generalization} and $N$ depending on the generalization error in \cite{allen2019learning}.
Finally, they bound generalization error rather than test error, and do not bound the 
difference between NT regression and kernel ridge regression.
 \item The upper bounds of \cite{ji2019polylogarithmic,chen2019much,nitanda2019gradient} 
 apply to classification and require a large margin condition. As before, these papers do not bound the 
difference between NT regression and kernel ridge regression.
 \end{itemize}

Results similar to ours
 were recently obtained in the context of simpler random features models in 
\cite{hastie2019surprises,ghorbani2019linearized,mei2019generalization,montanari2019generalization,ghorbani2020neural,mei2021generalization}.
 The models studied in these works corresponds to a two-layer network in which the first layer 
 is random and the second is trained. Their analysis is simpler because the featurization map has independent coordinates.

 Finally, the closest line of work in the literature is one that characterizes
 the risk of KRR and KRR interpolators in high dimensions 
 \cite{bach2017breaking,liang2018just,ghorbani2019linearized,bietti2019inductive,liang2020multiple,mei2021generalization}.
 It is easy to see that the infinite width limit ($N\to\infty$ at $n,d$ fixed), 
 NT regression converges to kernel  ridge(--less) 
 regression (KRR) with the rotationally invariant kernel $\bK$. 
 Our work address the natural open question in these studies: how large $N$ should
 be for NT regression to approximate KRR with respect to the limit kernel?
 By considering scalings in which $N,n,d$ are all large and polynomially related,
 we showed that NT performs KRR already when $Nd/(\log Nd)^C\ge n$.

\section{Connections with gradient descent training of neural networks}
\label{sec:ConnectionOptimization}

In this section we discuss in greater detail the relation between the
NT model \eqref{eq:FNT} studied in the rest of the paper, and fully nonlinear two-layer
neural networks. For clarity, we will adopt the following notation for the 
NT model: 
\begin{align}
f_{\NT}(\xx;\aa,\WW) :=\frac{1}{\sqrt{N}}\sum_{k=1}^N
  \<\aa_k,\xx\> \sigma'(\<\ww_k,\xx\>)\, .\label{eq:NT-Redef}
\end{align}
We changed the normalization purely for aesthetic reasons: since the model is linear, 
this has no impact on the min-norm interpolant.
We will compare it to the following neural network
\begin{align}
f_{\NN}(\xx;\tilde{\WW}) := \frac{\alpha}{\sqrt{N}}\sum_{k=1}^{2N} b_k\sigma(\<\tilde{\ww}_k,\xx\>),~~
  b_1=\dots=b_N=+1\,,\; b_{N+1}=\dots=b_{2N}=-1\, .\label{eq:Two-Layers-Net}
\end{align}
Note that the network has $2N$ hidden units and the second layer weights 
$b_k$ are fixed. We train the first-layer weights using gradient flow with respect to 
the empirical risk
\begin{align}
\frac{\de\phantom{t}}{\de t}\tilde \WW^t = -\nabla_{\WW}\hR_n(\tilde \WW^t)\, ,\;\;\;\;
\where ~~ \hR_n(\tilde\WW):= \frac{1}{n}\sum_{i=1}^n \big(y_i-f_{\NN}(\xx;\tilde{\WW})\big)^2\, .
\end{align}
We use the following initialization:
\begin{align}
\tilde \ww^{0}_k=\ww^{0}_{N+k} = \tilde \ww_k\,\;\;\; \forall k\le N\, .\label{eq:SymmInit}
\end{align}
Under this initialization, the network evaluates to $0$ 
at $t=0$. Namely,  $f_{\NN}(\xx;\tilde \WW^0)=0$ for all $\xx\in\reals^d$. 
Let us also emphasize that the weights $\ww_k$ in the NT model \eqref{eq:NT-Redef}
are chosen to match the ones in the initialization \eqref{eq:SymmInit}:
in particular, we can take $(\ww_k)_{k\le N}\sim_{\iid}\Unif(\S^{d-1}(1))$ 
to recover the neural tangent model treated in the rest of the paper.
This symmetrization is a standard technique \cite{chizat2019lazy}: it simplifies the analysis
because it implies $f_{\NN}(\xx;\tilde \WW^0)=0$. We refer Remark~\ref{rmk:opt} $(ii)$ for explanations.

The next theorem is a slight modification of \cite[Theorem 5.4]{bartlett_montanari_rakhlin_2021}
which in turn is a refined version of the analysis of \cite{chizat2019lazy,oymak2020towards}.
\begin{thm}\label{thm:Two-Layers-Linear}
Consider the two-layer neural network of \eqref{eq:Two-Layers-Net} trained with gradient flow
from initialization \eqref{eq:SymmInit}, and the associated NT model of 
Eq.~\eqref{eq:NT-Redef}.
  Assume $(\ww_k)_{k \le N}\sim_{\iid}\Unif(\S^{d-1}(1))$, the activation function
  to have bounded second derivative $\sup_{t\in\reals}|\sigma''(t)|\le C$, and  its
   Hermite coefficients to satisfy 
  $\mu_{k}(\sigma)\neq 0$ for all $k \le \ell_0$ for some constant $\ell_0$.
  Further assume $\{(\xx_i,y_i)\}_{i\le n}$ are i.i.d.\ with $\xx_i\sim\Unif(\S^{d-1}(\sqrt{d}))$
  and $y_i$ to be $C$-subgaussian.
  
  Then there exist constants $C_i$,
  depending uniquely on $\sigma, \ell_0$, such that
 if  $d\le n\le d^{\ell_0}/(\log d)^{C_0}$, as well as
  \begin{align}
    n\le \frac{Nd}{(\log Nd)^{C_0}}\;\;\; \mbox{ and }\;\;\; \alpha\ge C_0 \sqrt{\frac{n^2}{Nd}}\,
    ,\label{eq:ConditionLinearization2}
  \end{align}
  then,
  with probability at least $1-2\exp\{-n/C_0\}$,
  the following holds.
  \begin{enumerate}
  \item Gradient flow converges exponentially fast to a global
  minimizer. Specifically, letting
    $\lambda_* = \lambda_{\min}(\bK_N)\alpha^2d/(4n)$, we have, for all $t\ge 0$,
    \begin{align}
      \hR_n(\tilde \WW^t)\le \hR_n(\tilde \WW^0) \, e^{-\lambda_* t}\, .\label{eq:TrainingConvergence}
    \end{align}
    In particular the rate is lower bounded as $\lambda_* \ge C_1\alpha^2d/n$
    by Theorem \ref{thm:MinEigenvalue}.
  \item The model learned by gradient flow and min-$\ell_2$ norm interpolant 
   are similar on test data.
    Namely, writing $f_{\NN}(\tilde\WW) := f_{\NN}(\,\cdot\,;\tilde\WW)$
    and $f_{\NT}(\WW) := f_{\NT}(\,\cdot\,;\aa,\WW)$, we have 
    \begin{align}
      \limsup_{t\to\infty}\|f_{\NN}(\tilde \WW^t) -f_{\NT}(\hat{\aa},\WW)\|_{L^2(\P)}\le C_1\left\{\frac{1}{\alpha}\sqrt{\frac{n^2}{Nd}} 
      +\frac{1}{\alpha^2}
      \sqrt{\frac{n^5}{Nd^4}}\right\}\, ,\label{eq:ModelDistanceLin}
      \end{align}
      for $\hat{\aa}$ the coefficients of the min-$\ell_2$ norm interpolant
      of Eq.~\eqref{eq:InterpolationFirst}.
  \end{enumerate}
\end{thm}
Notice that any generic activation function satisfies the condition
$\mu_{k}(\sigma)\neq 0$ for all $k \le \ell_0$. For instance a sigmoid or smoothed ReLU 
with a generic offset satisfy the assumptions of this theorem.

 The only difference with respect to \cite[Theorem 5.4]{bartlett_montanari_rakhlin_2021}
 is that we assume $\xx_i\sim\Unif(\S^{d-1}(\sqrt{d}))$ while 
 \cite{bartlett_montanari_rakhlin_2021}  assumes $\xx_i\sim\normal(\bzero,\id_d)$.
 However, it is immediate to adapt the proof of \cite{bartlett_montanari_rakhlin_2021},
 in particular using Theorem \ref{thm:MinEigenvalue} to control the minimum eigenvalue
 of the kernel at initialization. 
 
 Note that Theorem \ref{thm:Two-Layers-Linear} requires the activation function
 $\sigma$ to be smoother than other results in this paper (it requires a bounded second
 derivative). This assumption is used  to bound the Lipschitz constant of
 the Jacobian $\bD_{\WW} f_{\NN}(\xx;\WW)$ uniformly over $\WW$. We expect that
 a more careful analysis could avoid assuming such a strong uniform bound. 
 
 The difference in test errors between the neural network 
 \eqref{eq:Two-Layers-Net} trained via gradient flow and the min-norm NT interpolant 
 studied in this paper can be bounded using Eq.~\eqref{eq:ModelDistanceLin} by triangular inequality.
 Also, while we focus for simplicity on gradient flow, similar results hold for 
 gradient descent with small  enough step size.

 A few remarks are in order. 
 \begin{remark}\label{rmk:opt}
 Even among two-layer fully connected neural networks, model \eqref{eq:Two-Layers-Net}
presents some simplifications:
\begin{itemize}
\item[$(i)$] Second-layer weights are not trained
and fixed to $b_{j}\in\{+1,-1\}$. If second-layer weights are trained, the 
NT model \eqref{eq:NT-Redef} needs to be modified as follows
\begin{align}
f_{\NT}(\xx;\aa,\WW) :=\frac{1}{\sqrt{N}}\sum_{i=1}^N
  \<\aa_i,\xx\> \sigma'(\<\ww_i,\xx\>)+
  \frac{1}{\sqrt{N}}\sum_{i=1}^N\tilde{a_i}
   \sigma(\<\ww_i,\xx\>)\, .\label{eq:GeneralNT}
\end{align}
The new model has $N(d+1)$ parameters $\aa = (\aa_1,\dots,\aa_N;
\tilde{a}_1,\dots,\tilde{a}_N)$. Notice that, for large dimension, the additional number of 
parameters is negligible. Indeed, going through the proofs reveals that our analysis can 
be extended to this case at the price of additional notational burden, but without changing the 
results.

We also point out that the model in which only second layer weights are trained,
i.e.\ we set $\aa_i=0$ in Eq.~\eqref{eq:GeneralNT}, has been studied in detail in 
\cite{ghorbani2019linearized,mei2021generalization}. These papers support the above
parameter-counting heuristics.
\item[$(ii)$] The specific initialization \eqref{eq:SymmInit} is chosen so that 
$f_{\NN}(\xx;\tilde \WW^0) =0$ identically. If this initialization was modified
(for instance by taking independent $(\tilde \ww_k)_{k\le 2N}$ ) two main elements 
would change in the analysis.
First, the target model $f_*(\xx)$ should be replaced by the difference between target
and initialization $f_*(\xx)-f_{\NN}(\xx;\tilde \WW^0)$. Second, and more importantly,
the approximation results in Theorem \ref{thm:Two-Layers-Linear} become weaker:
we refer to \cite{bartlett_montanari_rakhlin_2021} for a comparison.
\end{itemize}
 \end{remark}
 
 \begin{remark}
 Within the setting of Theorem \ref{thm:Two-Layers-Linear} (in particular 
 $y_i$ being $C$-subgaussian), the null risk  $\|f_*\|^2_{L^2}$ is of order one, and therefore
 we should compare the right-hand side of Eq.~\eqref{eq:ModelDistanceLin} with $\|f_*\|_{L^2}=\Theta(1)$.
  We can point at two specific regimes in which this upper bound guarantees
  that  $\|f_{\NN}(\tilde \WW^{\infty}) -f_{\NT}(\hat{\aa},\WW)\|_{L^2}\ll \|f_*\|_{L^2}$:
 \begin{itemize}
 \item[$(i)$] First letting $\alpha$ grow, the this upper bound vanishes. More precisely,
  it becomes much smaller than one provided $\alpha\gg (n^2/Nd)^{1/2}\vee (n^5/Nd^4)^{1/4}$.
  By training the two-layer neural network with a large scaling parameter $\alpha$, we
  obtain a model that is well approximated by the NT model as soon as the overparametrization 
  condition $n\le Nd/(\log Nd)^C$ is satisfied.
  
  This role of scaling in the network parameters, and its generality were first pointed out
  in \cite{chizat2019lazy}. It implies that the theory developed in this always applied to 
  nonlinear neural networks under a certain specific initialization and training scheme.
 \item[$(ii)$] A standard initialization rule suggests to take $\alpha = \Theta(1)$.
 In this case, the right-hand side of Eq.~\eqref{eq:ModelDistanceLin} becomes small
 for wide enough networks, namely $Nd\gg n^2\vee (n^5/d^4)$. This condition is stronger than
 the overparametrization condition under which we carried our analysis of the NT model.
 
 This means that, for $\alpha=\Theta(1)$ and $n(\log n)^C\ll Nd\lesssim  n^2\vee (n^5/d^4)$,
 we can apply our analysis to neural tangent models but not to actual neural networks. 
 On the other hand, the condition $Nd\gg n^2\vee (n^5/d^4)$ in 
 Theorem \ref{thm:Two-Layers-Linear} is likely to be a proof artifact, and we
 expect that it will be improved in the future. In fact we believe that the refined 
 characterization of the NT model in the present paper is a foundational step towards such improvements.
 \end{itemize}
 \end{remark}
 
 \begin{remark}
  Theorem \ref{thm:Two-Layers-Linear} relates the large-time limit of GD trained neural networks
  to the minimum $\ell_2$-norm NT interpolators, corresponding to the case $\lambda=0$
  of Eq.~\eqref{eq:NT-Regression-Explicit}. The review paper \cite{bartlett_montanari_rakhlin_2021}
  proves a similar bound relating the NN and NT models, trained via gradient descent,
  at all times $t$.  
  Our main result, Theorem \ref{thm:gen} characterizes the test error of NT regression
  with general $\lambda\ge 0$. Going through the proof \cite[Theorem 5.4]{bartlett_montanari_rakhlin_2021},
  it can be seen that a non-vanishing $\lambda>0$ corresponds to regularizing
  GD training of the two-layer neural network by the penalty $\|\WW-\WW^0\|_F^2=\sum_{\ell \le N}
  \|\ww_\ell -\ww_{\ell}^0\|_2^2$. 
 \end{remark}
 
\section{Technical background}\label{sec:pre}

This subsection provides a very short introduction to Hermite polynomials and Gegenbauer polynomials. More background information can be found in \cite{atkinson2012spherical, costas2014spherical, ghorbani2019linearized} for example. Let $\rho$ be standard Gaussian measure on $\R$, namely $\rho(dx) = (2\pi)^{-1/2} e^{-x^2/2}\; dx$. The space $L^2(\R, \rho)$ is a Hilbert space with the inner product $\langle f, g \rangle_{L^2(\R, \rho)} = \E_{G\sim \rho}[ f(G) g(G)]$. 

The \textit{Hermite polynomials} form a complete orthogonal basis of $L^2(\R, \rho)$. Throughout this paper, we will use normalized Hermite polynomials $\{h_k\}_{k\ge 0}$:
\begin{equation*}
\E\big[ h_k(G) h_j(G)\big] = \delta_{kj}. 
\end{equation*}
where $\delta_{kj}$ is Kronecker delta, namely $\delta_{kj} = 0$ if $k \neq j$ and $\delta_{kj} = 1$ if $k = j$. For example, the first four Hermite polynomials are given by $h_0(x) = 1, h_1(x) = x, h_2(x) = \frac{1}{\sqrt{2}} (x^2 - 1)$ and $h_3(x) = \frac{1}{\sqrt{6}} (x^3 - 3x)$. For any function $f \in L^2(\R, \rho)$, we have decomposition
\begin{equation*}
f = \sum_{k= 0}^\infty \langle f, h_k\rangle_{L^2(\R, \rho)} \;h_k.
\end{equation*}
In particular, if $\sigma' \in L^2(\R, \rho)$, we denote $\mu_k = \langle \sigma', h_k \rangle_{L^2(\R, \rho)}$ and have
\begin{equation}\label{decomp:Hermite}
\sigma' = \sum_{k=0}^\infty \mu_k h_k.
\end{equation}

Let $\tau_{d-1}$ be the uniform probability measure on $\S^{d-1}(\sqrt{d})$, $\tilde \tau_{d-1}^1$
 be the probability measure of $\sqrt{d}\, \langle \xx, \ee_1\rangle$ where $\xx \sim \tau_{d-1}$, 
 and  $\tau_{d-1}^1$ be the probability measure of $\langle \xx, \ee_1\rangle$. 
 The \textit{Gegenbauer polynomials} $(Q_k^{(d)})_{k = 0}^\infty$ form a basis of 
 $L^2([-d, d], \tilde \tau_{d-1}^1)$ (for simplicity, we may write it as $L^2$ unless confusion arises) where $Q_k^{(d)}$ is a polynomial of degree of $k$ and the they satisfy the normalization
\begin{equation*}
\langle Q_k^{(d)}, Q_j^{(d)} \rangle_{L^2} = \frac{1}{B(d,k)} \delta_{jk}
\end{equation*}
where $B(d,k)$ is a dimension parameter that is monotonically increasing in $k$ and satisfies $B(d,k) = (1+o_d(1))d^k / k!$. This normalization guarantees $Q_k^{(d)}(d) = 1$. The following are some properties of Gegenbauer polynomials we will use. 
\begin{itemize}
\item[$(a)$] {For $\xx, \yy \in \S^{d-1}(\sqrt{d})$,
\begin{equation}\label{eq:innerproductself}
\big\langle Q_{k}^{(d)}(\langle \xx, \cdot \rangle), Q_{j}^{(d)}(\langle \yy, \cdot \rangle) \big\rangle_{L^2(\S^{d-1}(\sqrt{d}), \tau_{d-1})} = \frac{1}{B(d,k)} \delta_{jk} Q_k^{(d)} (\langle \xx, \yy \rangle).
\end{equation}
}
\item[$(b)$] {For $\xx, \yy \in \S^{d-1}(\sqrt{d})$,
\begin{equation}\label{eq:addtheorem}
Q_k^{(d)} (\langle \xx, \yy \rangle) = \frac{1}{B(d,k)} \sum_{i=1}^{B(d,k)} Y_{ki}^{(d)}(\xx) Y_{ki}^{(d)}(\yy),
\end{equation}
where $Y_{k,1}^{(d)}, \ldots, Y_{k,B(d,k)}^{(d)}$ are normalized spherical harmonics of degree $k$; more precisely, each $Y_{k,i}^{(d)}$ is a polynomial of degree $k$ and they satisfy
\begin{equation}\label{eq:harmonicsorthonormal}
\langle Y_{k,i}^{(d)}, Y_{m,j}^{(d)} \rangle_{L^2(\S^{d-1}(\sqrt{d}), \tau_{d-1})} = \delta_{km} \delta_{ij}.
\end{equation}
}
\item [$(c)$] {Recurrence formula. For $t \in [-d,d]$, 
\begin{equation}\label{eq:recurrence}
\frac{t}{d} Q_k^{(d)} (t) = \frac{k}{2k+d-2} Q_{k-1}^{(d)} (t) + \frac{k+d-2}{2k+d-2} Q_{k+1}^{(d)} (t).
\end{equation}
}
\item [$(d)$] {Connection to Hermite polynomials. If $f \in L^2(\R, \rho) \cap L^2([-\sqrt{d},\sqrt{d}], \tau_{d-1}^1)$, then
}
\begin{equation}\label{eq:Hermiteconnect}
\mu_k(f) = \lim_{d\to \infty} \sqrt{B(d,k)} \, \langle f, Q_k^{(d)}(\sqrt{d}\, \cdot) \rangle_{L^2([-\sqrt{d},\sqrt{d}], \tau_{d-1}^1)}.
\end{equation}
\end{itemize}
Note that point $(b)$ and the Cauchy-Schwarz inequality implies 
$\max_{|t| \le d} |Q_k^{(d)}(t)| \le 1$. If the function $\sigma'$ satisfies
 $\sigma' \in L^2([-\sqrt{d},\sqrt{d}], \tau_{d-1}^1)$, then we can decompose $\sigma'$ using Gegenbauer polynomials:
\begin{align}
&\sigma'(x) = \sum_{k=0}^\infty \lambda_{d,k}(\sigma') B(d,k) Q_k^{(d)}(\sqrt{d}\, x), \qquad \where \label{decomp:Gegenbauer}\\
&\lambda_{d,k}(\sigma') := \langle \sigma', Q_k^{(d)}(\sqrt{d}\, \cdot ) \big \rangle_{L^2([-\sqrt{d},\sqrt{d}], \tau_{d-1}^1)}. \label{def:lambda}
\end{align}
We also have $\| \sigma' \|_{L^2([-\sqrt{d}, \sqrt{d}], \tau_{d-1}^1)}^2 = \sum_{k=0}^\infty B(d,k) [\lambda_{d,k}(\sigma')]^2 $. We sometimes drop the subscript $k$ if no confusion arises.

\section{Kernel invertibility and concentration: Proof of  Theorems \ref{thm:MinEigenvalue} and \ref{thm:invert2}}
\label{sec:ProofInvertibility}

Our analysis starts with a study of the infinite-width kernel $\bK$. 
In our setting, we will show that $\bK$ is essentially a regularized low-degree polynomial
 kernel plus a small `noise'. 
 Such a decomposition will play an essential role in our analysis of the generalization error. 

To ease notations, we will henceforth write the target function $f_*$ simply as $f$. 

\subsection{Infinite-width kernel decomposition}

Recall that we denote by $(Q_k^{(d)})_{k \ge 0}$ the Gegenbauer polynomials in $d$ dimensions, and denote by $(Y_{k t})_{t \le B(d,k)}$ the normalized spherical harmonics of degree $k$ in $d$ dimensions. We denote by $\bPsi_{=k} \in \R^{n \times B(d,k)}$ the degree-$k$ spherical harmonics evaluated at $n$ data points, and denote by $\bPsi_{\le k}$ the concatenation of these matrices of degree no larger than $k$; namely, 
\begin{equation*}
\bPsi_{\le \ell}  = \big[ \bPsi_{=0}, \ldots, \bPsi_{= \ell} \big], \qquad \where ~\bPsi_{=k} = \big( Y_{k t}(\xx_i) \big)_{i\le n, t \le B(d,k)}.
\end{equation*}

\begin{lem}[Harmonic decomposition of the infinite-width kernel]\label{lem:K_Harmonic}
The kernel $K$ can be decomposed as
\begin{align}
K(\xx,\xx') =  \sum_{k = 0}^\infty \gamma_k Q_k^{(d)} (\< \xx, \xx' \>)\, ,
\label{eq:KernelHarmonic}
\end{align}
with coefficients:
\begin{align}
\gamma_0 = \big[ \lambda_1(\sigma') \big]^2 , \qquad 
\gamma_k = \frac{k+1}{2k+d} B(d,k+1) \big[ \lambda_{k+1}(\sigma') \big]^2 + 
\frac{k+d-3}{2k+d-4} B(d,k-1)\big[ \lambda_{k-1}(\sigma') \big]^2,
\end{align}
where the coefficients $\lambda_k=\lambda_{d,k}$  are defined in \eqref{def:lambda}. The convergence in \eqref{eq:KernelHarmonic} takes place in any of the following 
interpretations: $(i)$~As functions in 
$L^2(\S^{d-1}(\sqrt{d})\times \S^{d-1}(\sqrt{d}), \tau_{d-1}\otimes \tau_{d-1})$;
$(ii)$~Pointwise, for every $\xx,\xx'\in \S^{d-1}(\sqrt{d})\times \S^{d-1}(\sqrt{d})$;
$(iii)$~In operator norm as integral operators 
$L^2(\S^{d-1}(\sqrt{d}), \tau_{d-1})\to L^2(\S^{d-1}(\sqrt{d}), \tau_{d-1})$.

Finally,  $\gamma_0 = o_d(1)$ and for $k \ge 1$ we have $\gamma_k = \mu_{k-1}^2 + o_d(1)$,
where we recall that $\mu_k$ is the $k$-th coefficient in the 
Hermite expansion of $\sigma'$.
\end{lem}


The use of spherical harmonics and Gegenbauer polynomials to study inner product kernels on the sphere 
(the $N=\infty$ case) is standard since the classical work of Schoenberg \cite{schoenberg1942positive}.
Several recent papers in machine learning use these tools
 \cite{bach2017breaking, bietti2019inductive, ghorbani2019linearized}.
  On the other hand, quantitative properties when both the sample size $n$ and the dimension $d$ 
are large are  much less studied  \cite{ghorbani2019linearized}.  The finite width $N$ case
 is even more challenging since the kernel is no longer of inner-product type.

\begin{proof}
Recall 
the Gegenbauer decomposition of $\sigma'$ in \eqref{decomp:Gegenbauer},
implying that the following holds in $L^2(\S^{d-1})$ (we omit indicating
the measure when this is uniform).
For any fixed $\xx\in \S^{d-1}(\sqrt{d})$
(and writing $\lambda_k(\sigma') := \lambda_{d,k}(\sigma')$):
\begin{align*}
\sigma'(\langle \xx, \,\cdot\, \rangle) = 
\sum_{k=0}^\infty \lambda_k(\sigma')  B(d,k)Q_k^{(d)} (\sqrt{d}\, \langle \xx, \,\cdot\, 
\rangle)\, .
\end{align*}
 We thus obtain, for fixed $\xx$, $\xx'$: 
\begin{align*}
\E_\ww \big[ \sigma'(\langle \xx, \ww \rangle) \sigma'(\langle \xx', \ww \rangle) \big] 
& \stackrel{(i)}{=}  \sum_{k=0}^\infty \big[ \lambda_k(\sigma') \big]^2 [B(d,k)]^2 \E_\ww \big[ Q_k^{(d)} (\sqrt{d}\, \langle \xx, \ww \rangle) Q_k^{(d)} (\sqrt{d}\, \langle \xx', \ww \rangle) \big] \\
&\stackrel{(ii)}{=} \sum_{k=0}^\infty \big[ \lambda_k(\sigma') \big]^2 B(d,k) Q_k^{(d)} ( \langle \xx, \xx' \rangle).
\end{align*}
Here, $(i)$ is due to the the orthogonality of 
$Q_k^{(d)} (\sqrt{d}\, \langle \xx, \,\cdot\, 
\rangle)$ and $Q_{k'}^{(d)} (\sqrt{d}\, \langle \xx', \,\cdot\, 
\rangle)$ for $k\neq k'$ as well as Lemma~\ref{lem:L2conv} (exchanging summation with expectation), 
and $(ii)$ is due to the property \eqref{eq:innerproductself}.

Thus, we obtain
\begin{equation*}
K(\xx,\xx') = \sum_{k=0}^\infty \big[ \lambda_k(\sigma') \big]^2 B(d,k) Q_k^{(d)} ( \langle \xx, \xx' \rangle) 
\frac{\langle \xx, \xx' \rangle}{d}.
\end{equation*}
We use the recurrence formula to simplify the above expression.
\begin{align}
K&(\xx,\xx')  = \sum_{k \ge 1} \frac{k}{2k+d-2}  \big[ \lambda_k(\sigma') \big]^2 B(d,k)Q_{k-1}^{(d)}(\langle \xx, \xx' \rangle) + \sum_{k \ge 0} \frac{k+d-2}{2k+d-2} \big[ \lambda_k(\sigma') \big]^2 B(d,k) Q_{k+1}^{(d)} (\langle \xx, \xx' \rangle) \notag \\
&= \sum_{k \ge 0} \frac{k+1}{2k+d} \big[ \lambda_{k+1}(\sigma') \big]^2 B(d,k+1) Q_k^{(d)}(\langle \xx_i, \xx_j \rangle) + \sum_{k \ge 1} \frac{k+d-3}{2k+d-4} \big[ \lambda_{k-1}(\sigma') \big]^2 B(d,k-1) Q_k^{(k)} ( \langle \xx, \xx' \rangle) \notag \\
&= \sum_{k = 0}^\infty \gamma_k Q_k^{(d)} (\langle \xx, \xx' \rangle)\, . \label{eq:Kdecomp0}
\end{align}

We thus proved that Eq.~\eqref{eq:KernelHarmonic} holds pointwise
for every $\xx,\xx'$. Calling $K^{\ell}$ the sum of the first $\ell$ terms on the right hand side,
we also obtain $\|K-K^{\ell}\|_{L^2(\S^{d-1}(\sqrt{d})\times \S^{d-1}(\sqrt{d}))}\to 0$
by an application of Fatou's lemma. Finally the last norm coincides
with Hilbert-Schmidt norm, when we view these kernels as operators, and hence implies
convergence in operator norm.

The asymptotic formulas for $\gamma_k$ follow by the convergence of Gegenbauer
expansion to Hermite expansion discussed in Section \ref{sec:pre}.
\end{proof}

\begin{lem}[Structure of infinite-width kernel matrix]\label{lem:Kdecomp}
Assume $(\xx_i)_{i\le n}\sim_{iid}\Unif(\S^{d-1}(\sqrt{d}))$, and 
the activation function to satisfy Assumption \ref{ass:Sigma}. Then
we can decompose $\bK$ as follows
\begin{align}
\bK = \gamma_{>\ell} \bI_n + \bPsi_{\le \ell} \bLambda_{\le \ell}^2 \bPsi_{\le \ell}^\top  + 
\bDelta\label{eq:KDecomp}
\end{align}
where $\gamma_{> \ell} := \sum_{k > \ell} \gamma_k= \sum_{k \ge \ell} \mu_k^2 + o_d(1)$ and
\begin{equation*}
\bLambda_{\le \ell}^2 := \diag\big\{ \gamma_0, \underbrace{B(d,1)^{-1}\gamma_1,\ldots, B(d,1)^{-1}\gamma_1}_{B(d,1)}, \ldots, \underbrace{B(d,\ell)^{-1}\gamma_\ell, \ldots, B(d,\ell)^{-1}\gamma_\ell }_{B(d,\ell)} \big\}
\end{equation*}
Here, $\gamma_k \ge 0$ and $\gamma_k = \mu_{k-1}^2 + o_d(1)$ (we denote $\mu_{-1} = 0$).

Further there exists a constant $C$ such that, for $n(\log n)^C\le d^{\ell+1}$,
 the remainder term $\bDelta$ satisfies, with very high probability,
\begin{equation*}
\big\| \bDelta \big\|_\op \le \sqrt{\frac{n (\log n)^C}{d^{\ell+1}}}.
\end{equation*}
Finally, if $n\ge Cd^{\ell}(\log d)^2$, with very high probability,
\begin{equation}\label{eq:psi-conctr}
\big\| n^{-1}\bPsi_{\le \ell}^\top \bPsi_{\le \ell} -  \bI_n \big\|_\op \le C\sqrt{\frac{ d^{\ell}(\log d)^2}{n}}.
\end{equation}
In the case $\ell=1$, the above upper bound can be replaced by $\sqrt{Cd/n}$.
\end{lem}
\begin{proof}[{\bf Proof of Lemma~\ref{lem:Kdecomp}}]
Using Lemma \ref{lem:K_Harmonic}, and using  property \eqref{eq:addtheorem} to express 
\begin{equation*}
Q_k^{(k)} ( \langle \xx_i , \xx_j \rangle) = [B(d,k)]^{-1} \bPsi_{=\ell}(\xx_i) \bPsi_{=\ell}(\xx_j)^\top
\end{equation*}
for $k\le \ell$,
we obtain the decomposition \eqref{eq:KDecomp},
where
\begin{align*}
\bDelta  := \sum_{k=\ell+1}^{\infty}\gamma_k(\QQ_k - \bI_n)\, ,
\end{align*}
and $\QQ_k :=(Q_k^{(d)}(\<\xx_i,\xx_j\>))_{i,j\le n}$. 
Note that
\begin{align*}
\gamma_{>\ell} &= \sum_{k \ge \ell + 2} \frac{k}{2k+d-2} B(d,k) [\lambda_k(\sigma')]^2 + \sum_{k \ge \ell} \frac{k+d-2}{2k+d-2} B(d,k) [\lambda_k(\sigma')]^2 \\
&= \sum_{\ell\le k \le \ell+1} \frac{k+d-2}{2k+d-2} B(d,k) [\lambda_k(\sigma')]^2 + \sum_{k \ge \ell+2} B(d,k) [\lambda_k(\sigma')]^2 \\
&= \sum_{k \ge \ell} \mu_k^2 + o_d(1).
\end{align*}
The last equality involves interchanging the limit $\lim_{d \to \infty}$ with the summation. We explain the validity as follows. For any integer $M>\ell+2$, 
\begin{align*}
\Big| \sum_{k \ge \ell+2} B(d,k) [\lambda_k(\sigma')]^2 - \sum_{k \ge \ell} \mu_k^2\Big| &\le \sum_{k \ge \ell+2}^{M} \big|  B(d,k) [\lambda_k(\sigma')]^2 - \mu_k^2\big| + \sum_{k > M} B(d,k) [\lambda_k(\sigma')]^2 + \sum_{k > M} \mu_k^2.
\end{align*}
By first taking $d \to \infty$ and then $M \to \infty$, we obtain
\begin{align*}
&\limsup_{d\to\infty}\Big| \sum_{k \ge \ell+2} B(d,k) [\lambda_k(\sigma')]^2 - \sum_{k \ge \ell} \mu_k^2\Big| \le 
\lim_{M\to\infty}\limsup_{d\to\infty}\sum_{k > M} B(d,k) [\lambda_k(\sigma')]^2 ,\\
& = \limsup_{d\to\infty} \|\sigma'\|_{L^2(\tau^1_{d-1})}^2-
\lim_{M\to\infty}\liminf_{d\to\infty}\sum_{k \le M} B(d,k) [\lambda_k(\sigma')]^2\\
&=\|\sigma'\|_{L^2(\rho)}^2-\lim_{M\to\infty}\sum_{k \le M} \mu_k(\sigma')^2 = 0\, .
\end{align*}
Here in the last line we recall that $\rho$ is the standard Gaussian measure, and
we used dominated convergence to show that $\|\sigma'\|_{L^2(\tilde{\tau}^1_{d-1})}^2
\to \|\sigma'\|_{L^2(\rho)}^2$.
By Proposition~\ref{prop:Qconctr},  we have with very high probability, 
\begin{equation*}
\sup_{k > \ell} \big\| \QQ_k - \bI_n \big\|_\op \le \sqrt{\frac{ n (\log n)^C}{d^{\ell+1}} }.
\end{equation*} 
Thus, $\bDelta = \sum_{k > \ell} \gamma_\ell (\QQ_k-\bI_n)$ satisfies 
\begin{equation*}
\| \bDelta  \|_\op \le \gamma_{>\ell} \sup_{k > \ell} \big\| \QQ_k - \bI_d \big\|_\op \le \sqrt{\frac{ n (\log n)^C}{d^{\ell+1}} }.
\end{equation*}

Finally, denote $D = \sum_{k \le \ell} B(d,k)$. The matrix $\bPsi_{\le \ell}\in\reals^{n\times D}$ 
has $n$ i.i.d.\ rows, denote them by $\bpsi(\xx_i)$, $i\le n$, where
\begin{align}
\bpsi(\xx) := (Y_{k,t}(\xx)\big)_{t,\le B(d,k), k\le \ell}\, .
\end{align}
By orthonormality of the spherical harmonics \eqref{eq:harmonicsorthonormal}, the covariance of these vectors is
 $\E\{\bpsi(\xx) \bpsi(\xx) ^{\top}\} = \bI_D$. Further, for any $\xx\in\S^{d-1}(\sqrt{d})$,
 we have
 \begin{align}
 \|\bpsi(\xx)\|_2^2 &= \sum_{k=0}^{\ell} \sum_{t = 1}^{B(d,k)} Y_{k,t}(\xx)^2 = \sum_{k=0}^{\ell}  B(d,k) = D\, .
 \end{align}
  By \cite{vershynin2018high}[Theorem 5.6.1 and Exercise 5.6.4], we obtain
  \begin{align}
  \big\| n^{-1}\bPsi_{\le \ell}^\top \bPsi_{\le \ell} -  \bI_n \big\|_\op \le C
  \Big(\sqrt{\frac{D(\log D+u)}{n}}+\frac{D(\log D+u)}{n}\Big)\, ,
  \end{align}
  with probability at least $1-2e^{-u}$. The claim follows by setting $u = (\log d)^2$,
  and noting that $D\le C'd^{\ell}$.
\end{proof}

\subsection{Concentration of neural tangent kernel: Proof of Theorem~\ref{thm:invert2}}

Let $C_0>0$ be a sufficiently large constant (which is determined later). Define the truncated function 
\begin{equation*}
\varphi(x) = \sigma'(x) \bone \{ |x| \le C_0 \log (nNd) \}.
\end{equation*}
For $k \in [n]$, we also define matrices $\bD_k, \bH_k \in \R^{n \times n}$, and truncated kernels $\bK^0$ and $\bK_N^0$ as 
\begin{align}
\label{eq:TruncDef1}\bD_k &= \diag\big\{ \varphi(\langle \xx_1,\ww_k \rangle), \ldots, \varphi(\langle \xx_n,\ww_k \rangle)  \big\}, \qquad \bK^0 = \frac{1}{d} \E_\ww \big[ \bD_k \XX \XX^\top \bD_k \big] \\
\label{eq:TruncDef2}\bH_k &= \frac{1}{d} (\bK^0)^{-1/2} \bD_k \XX \XX^\top \bD_k (\bK^0)^{-1/2}, \qquad \bK_N^0 = \frac{1}{Nd} \sum_{k=1}^N \bD_k \XX \XX^\top \bD_k.
\end{align}
Note that by definition and the assumption on the activation function, namely Assumption~\ref{ass:Sigma}, we have $|\varphi(x)| \le B(1+(C_0 \log (nNd))^B)$. 

Note that $\cA_\gamma$ is a very high probability event as a consequence of Lemma~\ref{lem:Kdecomp}. We shall treat $\xx_1,\ldots, \xx_n$ as deterministic vectors such that the conditions in the event $\cA_\gamma$ holds. 

\noindent\textbf{Step 1: The effect of truncation is small.} First, we realize that 
\begin{align*}
\P_\ww\big( \bK_N^0 \neq \bK_N \big) &\le \P_\ww \big( \max_{i, k} |\langle \xx_i, \ww_k \rangle | > C_0 \log (nNd) \big)  \le Nn \P\big( |\langle \xx_1, \ww_1 \rangle | > C_0 \log (nNd) \big). 
\end{align*}
By the fact that uniform spherical random vector is subgaussian \cite{vershynin2010introduction}[Sect.~5.2.5], we pick $C_0$ sufficiently large such that 
\begin{equation}\label{ineq:sphericalRV}
\P_\ww \big( |\langle \xx_1, \ww_1 \rangle | > C_0 \log (nNd) \big) \le C \exp \Big(- c (\log (nNd))^2 \Big).
\end{equation}
Thus, with very high probability, $\bK_N^0 = \bK_N$. 

Furthermore, for $i,j \in [n]$, we have
\begin{align*}
(\bK - \bK^0)_{ij} &= \E_\ww \big[ \sigma'(\langle \xx_i, \ww \rangle) \sigma'(\langle \xx_j, \ww \rangle) (1-\bone_{A_i}) \bone_{A_j} \big] \frac{\langle \xx_i , \xx_j \rangle}{d} \\
&~ +  \E_\ww \big[ \sigma'(\langle \xx_i, \ww \rangle) \sigma'(\langle \xx_j, \ww \rangle) (1-\bone_{A_j}) \big] \frac{\langle \xx_i , \xx_j \rangle}{d} =: I_1 + I_2
\end{align*}
where $A_i := \{ | \langle \xx_i, \ww \rangle| \le C_0 \log (nNd) \}$. For the first term, we derive
\begin{align*}
|I_1| & \stackrel{(i)}{\le} \Big|  \E_\ww \big[ \sigma'(\langle \xx_i, \ww \rangle) \sigma'(\langle \xx_j, \ww \rangle) (1-\bone_{A_i}) \bone_{A_j} \big] \Big| \\
&\stackrel{(ii)}{\le} \Big\{ \E_\ww \big[ \big( \sigma'(\langle \xx_i, \ww \rangle) \big)^4 \big] \Big\}^{1/4} \cdot \Big\{ \E_\ww \big[ \big( \sigma'(\langle \xx_j, \ww \rangle) \big)^4 \big] \Big\}^{1/4} \cdot \Big\{ \E_\ww \big[ (1 - \bone_{A_i})^2 \big] \Big\}^{1/2} \\
&\stackrel{(iii)}{\le} C(d^B+1) \Big\{  \P_\ww \big( |\langle \xx_i , \ww \rangle| > C_0 \log (nNd) \big) \Big\}^{1/2}
\end{align*}
where \textit{(i)} is due to $|\langle \xx_i, \xx_j \rangle| \le \| \xx_i \| \cdot \|\xx_j \| =  d$, \textit{(ii)} is due to H\"{o}lder's inequality, and \textit{(iii)} follows from the polynomial growth assumption on $\sigma'$. By \eqref{ineq:sphericalRV}, we can choose $C_0$ large enough such that $| I_1| \le C(nNd)^{-2}$. Similarly, we can prove that $|I_2| \le C(nNd)^{-2}$. Thus,
\begin{align}\label{ineq:K0Kdiff}
\big\| \bK^0 - \bK \big\|_\op \le n \max_{ij} \big| (\bK^0 - \bK)_{ij} \big| \le \frac{C}{nNd}.
\end{align}
Since we work on the event $\bK \succeq \gamma \bI_n$, this implies 
\begin{equation}\label{ineq:K0Kdif2}
\big\| \bK^{-1/2} (\bK^0 - \bK) \bK^{-1/2} \big\|_\op \le \frac{C}{\gamma nNd}.
\end{equation}

\noindent\textbf{Step 2: Concentration of truncated kernel.}  Next, we will use matrix concentration inequality \cite{vershynin2018high}[Thm.~5.4.1] for $\sum_{k=1}^N \bH_k$. First, we observe that \eqref{ineq:K0Kdiff} implies that $\bK^0 \succeq (\gamma - o_{d}(1)) \cdot \bI_n$. Together with the bound on $\| \XX \|_\op$ and the deterministic bound on $\| \bD_k \|_\op$, we have a deterministic bound on $\| \bH_k \|_\op$.
\begin{align*}
\| \bH_k \|_\op &\le \frac{1}{d \lambda_{\min}(\bK^0)} \| \bD_k \|_\op^2 \cdot \| \XX \|_\op^2 \\
& \le  \frac{1}{d (\gamma-o_d(1))}  \big[ B(1+(C_0 \log nNd)^B) \big]^2 \cdot 4(\sqrt{n} + \sqrt{d})^2 \\
&\le  \frac{C(n+d)}{d} \big(\log(nNd))^C.
\end{align*}
where $C$ is a sufficiently large constant. 
This also implies $\| \bH_k - \E_\ww [\bH_k] \|_\op \le (n+d)\log(nNd)^C /d$. 
We will make use of the following simple fact: if $\bA_1, \bA_2$ are p.s.d.~that satisfy $\bA_1 \preceq \bA_2 $, then $\QQ \bA_1 \QQ^\top  \preceq \QQ \bA_2 \QQ^\top$. We use this on $\bH_k^2$ and find
\begin{align*}
\bH_k^2 &= \frac{1}{d^2} (\bK^0)^{-1/2} \bD_k \XX \XX^\top \bD_k (\bK^0)^{-1}  \bD_k \XX \XX^\top \bD_k (\bK^0)^{-1/2}  \\
&\preceq \frac{1}{d^2}(\bK^0)^{-1/2} \bD_k \XX \XX^\top\XX \XX^\top \bD_k (\bK^0)^{-1/2} \cdot (\gamma - o_d(1))^{-1} \big[ B(1+(C_0 \log (nNd))^B) \big]^2 \\
&\preceq \frac{1}{d^2} (\bK^0)^{-1/2} \bD_k \XX\XX^\top \bD_k (\bK^0)^{-1/2}   \cdot (\gamma - o_d(1))^{-1} \big[ B(1+(C_0 \log (nNd))^B) \big]^2 \cdot 4(\sqrt{n} + \sqrt{d})^2.
\end{align*}
Taking the expectation $\E_\ww$, we get 
\begin{align*}
\E_\ww \big[ \bH_k^2 \big] \preceq \frac{1}{d} (\gamma - o_d(1))^{-1} \big[ B(1+(C_0 \log (nNd))^B) \big]^2 
\cdot 4(\sqrt{n} + \sqrt{d})^2 \cdot \bI_n 
\preceq \frac{C(n+d)}{d} \big(\log(nNd)\big)^C \cdot \bI_n.
\end{align*}
This implies 
\begin{equation*}
\E_\ww \big( \bH_k - \E_\ww [ \bH_k ] )^2  = \E_\ww [ \bH_k^2 ] - \big(\E_\ww [ \bH_k ]\big)^2 \preceq \E_\ww [ \bH_k^2 ] \preceq O\Big( \frac{n+d}{d} \polylog(nNd) \Big) \cdot \bI_n.
\end{equation*}
Now we apply the matrix Bernstein's inequality \cite{vershynin2018high}[Thm.~5.4.1] to the matrix sum $\sum_{k=1}^N \bH_k - \E_\ww [\bH_k] $. We obtain
\begin{align}\label{ineq:bern}
\P\Big( \Big|\sum_{k=1}^N  \bH_k - \E_\ww [\bH_k]  \Big| \ge t \Big) \le 2n \cdot \exp \left( - \frac{t^2/2}{v + Lt/3} \right),
\end{align}
where 
\begin{align*}
L =  \frac{n+d}{d} \big(\log(nNd) \big)^C, \qquad v = \frac{N(n+d)}{d} \big(\log(nNd) \big)^C\, .
\end{align*}
We choose a sufficient large constant $C'$ and set 
\begin{equation*}
t = \max\Big\{ \sqrt{\frac{N(n+d) (\log (nNd))^{C'}}{d}},  \frac{(n+d) (\log (nNd))^{C'}}{d} \Big\}
\end{equation*}
so that the right-hand side of \eqref{ineq:bern} is no larger than $2n \cdot \exp( - (\log (nNd)^2))$. This proves that with very high probability,
\begin{equation}
\big\| (\bK^0)^{-1/2} \bK_N^0 (\bK^0)^{-1/2} - \bI_n \big\|_\op \le \sqrt{\frac{(n+d)( \log (nNd))^{C'}}{Nd}} +  \frac{(n+d)( \log (nNd))^{C'}}{Nd}.
\end{equation}

\noindent \textbf{Step 3: back to original kernel.} The inequality \eqref{ineq:K0Kdif2} implies that for large enough $n$, $\| (\bK^0)^{1/2} \bK^{-1/2} \|_\op \le C$. Therefore,
\begin{align*}
& ~~~~\big\| \bK^{-1/2} ( \bK_N^0 - \bK) \bK^{-1/2} \big\|_\op   \\
&\le \big\| \bK^{-1/2} ( \bK_N^0 - \bK^0) \bK^{-1/2} \big\|_\op + \big\|\bK^{-1/2} ( \bK^0 - \bK) \bK^{-1/2} \big\|_\op \\
&\le \big\| \bK^{-1/2} (\bK^0)^{1/2} \big\|_\op \cdot \big\| (\bK^0)^{-1/2} ( \bK_N^0 - \bK^0) (\bK^0)^{-1/2} \big\|_\op \cdot \big\|(\bK^0)^{1/2} \bK^{-1/2} \big\|_\op +  \frac{C}{\gamma nNd} \\
&\le \sqrt{\frac{(n+d)( \log (nNd))^{C'}}{Nd}} +  \frac{(n+d)( \log (nNd))^{C'}}{Nd} + \frac{C}{\gamma nNd}.
\end{align*}
Since $1/(\gamma nNd) \le O(1) \cdot C'\sqrt{(n+d)(\log (nNd))^{C'}/(Nd)}$, we can enlarge the constant $C'$ appropriately to obtain that with very high probability, 
\begin{equation*}
\big\| \bK^{-1/2} ( \bK_N^0 - \bK) \bK^{-1/2} \big\|_\op \le  \sqrt{\frac{(n+d)( \log (nNd))^{C'}}{Nd}} + \frac{(n+d)( \log (nNd))^{C'}}{Nd}.
\end{equation*}
This completes the proof of Theorem~\ref{thm:invert2}.

\subsection{Smallest eigenvalue of neural tangent kernel: Proof of Theorem~\ref{thm:MinEigenvalue}}
By Lemma~\ref{lem:Kdecomp}, we have, with high probability,
\begin{align*}
\| \bDelta \|_\op \le \sqrt{\frac{n (\log n)^C}{d^{\ell+1}}} \le \sqrt{\frac{ (\log n)^C}{(\log d)^{C_0}}} \le \sqrt{(\ell+1)^{C} (\log d)^{C - C_0}},
\end{align*}
where the second inequality is because by Assumption~\ref{ass:Asymp}, $d^{\ell+1} \ge n (\log d)^{C_0} \ge n$, so $\log n \le (\ell+1) \log d$. We choose the constant $C_0$ to be larger than $C$. So by Weyl's inequality,
\begin{equation*}
\big| \lambda_{\min}(\bK) - \lambda_{\min}( \gamma_{>\ell} \bI_n + \bPsi_{\le \ell} \bLambda_{\le \ell}^2 \bPsi_{\le \ell}^\top) \big| \le \| \bDelta \|_\op = o_{d,\P}(1).
\end{equation*}
Note that $\bPsi_{\le \ell} \bLambda_{\le \ell}^2 \bPsi_{\le \ell}^\top$ is always p.s.d., and it has rank at most $d+1$ if $\ell=1$ and $O(d^\ell)$ if $\ell > 1$. So 
\begin{equation*}
\lambda_{\min}( \gamma_{>\ell} \bI_n + \bPsi_{\le \ell} \bLambda_{\le \ell}^2 \bPsi_{\le \ell}^\top) \stackrel{(i)}{\ge} \gamma_{>\ell} = v(\sigma) + o_d(1) \qquad \text{and thus} \qquad \lambda_{\min}(\bK) \stackrel{(ii)}{\ge} v(\sigma) - o_d(1)
\end{equation*}
and we can strengthen the inequalities $\ge$ to equalities $=$ in \textit{(i), (ii)} if $n>d+1$. 
By Theorem~\ref{thm:invert2} and Eq.~\eqref{ineq:KNtwosides} in its following remark, 
\begin{equation*}
(1-o_{d,\P}(1)) \cdot \lambda_{\min}(\bK) \le \lambda_{\min}(\bK_N) \le (1+o_{d,\P}(1)) \cdot \lambda_{\min}(\bK).
\end{equation*}
We conclude that $\lambda_{\min}(\bK_N) \ge v(\sigma) - o_{d,\P}(1)$ and that the inequality can be replaced by an equality if $n>d+1$.

\section{Generalization error: Proof outline of Theorem \ref{thm:gen}}
\label{sec:ProofGenMain}

In this section outline the proof our main result Theorem \ref{thm:gen}, which characterize the test error
of NT regression. We describe the main scheme of proof, and how we treat the bias term.
We refer to the appendix where the remaining steps (and most of the technical work) 
are carried out.

 Throughout, we will assume that the setting of that theorem (in
particular, Assumption \ref{ass:Asymp} and~\ref{ass:Sigma}) hold.
We will further assume that Eq.~\eqref{eq:AssLge2} in Assumption \ref{ass:Asymp} holds
for the case $\ell=1$ as well.
In Section \ref{sec:Ell1} we will refine our analysis to eliminate logarithmic factors 
for $\ell=1$.

In  NT regression the coefficients vector is given by Eq.~\eqref{eq:NT-Regression}
and, more explicitly, Eq.~\eqref{eq:NT-Regression-Explicit}. The prediction function
$\hat f_{\NT}(\xx) = \langle \bPhi(\xx), \hat \aa \rangle$ can be written as
\begin{equation*}
\hat f_{\NT}(\xx) = \bK_N(\cdot, \xx)^\top (\lambda \bI_n + \bK_N)^{-1} \yy\, ,
\end{equation*}
where we denote $\bK_N(\cdot, \xx) = (K_N(\xx_1, \xx), \ldots, K_N(\xx_n, \xx))^\top \in \R^n$. Define $\ff = (f(\xx_1),\ldots, f(\xx_n))^\top$ and $\bveps = (\veps_1,\ldots, \veps_n)^\top$. We now decompose the generalization error $R_{\NT}(\lambda)$ into three errors.
\begin{align*}
R_{\NT}(\lambda) &= \E_{\xx} \Big[ \big( f(\xx) - \bK_N(\cdot, \xx)^\top (\lambda \bI_n + \bK_N)^{-1} \yy \big)^2 \Big] \\
& = \E_{\xx} \Big[ \big( f(\xx) - \bK_N(\cdot, \xx)^\top (\lambda \bI_n + \bK_N)^{-1} \ff \big)^2 \Big] \\
& + \bveps^\top (\lambda \bI_n + \bK_N)^{-1} \E_{\xx} \big[ \bK_N(\cdot, \xx) \bK_N(\cdot, \xx)^\top \big] (\lambda \bI_n + \bK_N)^{-1}  \bveps \\
&- 2\bveps^\top (\lambda \bI_n + \bK_N)^{-1} \E_{\xx} \big[ \bK_N(\cdot, \xx) (f(\xx) - \bK_N(\cdot, \xx)^\top (\lambda \bI_n + \bK_N)^{-1} \ff )\big] \\
&=: E_{\bias}^N + E_{\var}^N -2 E_{\cross}^N.
\end{align*}


\noindent In the kernel ridge regression, the prediction function is 
\begin{equation*}
\hat f_{\KRR}(\xx) = \bK(\cdot, \xx)^\top (\lambda \bI_n + \bK)^{-1} \yy.
\end{equation*}
Similarly, we can decompose the associated generalization error $R_{\KRR}(\lambda)$ into three errors.
\begin{align*}
R_{\KRR}(\lambda) &= \E_{\xx} \Big[ \big( f(\xx) - \bK(\cdot, \xx)^\top (\lambda \bI_n + \bK)^{-1} \yy \big)^2 \Big] \\
& = \E_{\xx} \Big[ \big( f(\xx) - \bK(\cdot, \xx)^\top (\lambda \bI_n + \bK)^{-1} \ff \big)^2 \Big] \\
& + \bveps^\top (\lambda \bI_n + \bK)^{-1} \E_{\xx} \big[ \bK(\cdot, \xx) \bK(\cdot, \xx)^\top \big] (\lambda \bI_n + \bK)^{-1}  \bveps \\
&-2 \bveps^\top (\lambda \bI_n + \bK)^{-1} \E_{\xx} \big[ \bK(\cdot, \xx) (f(\xx) - \bK(\cdot, \xx)^\top (\lambda \bI_n + \bK)^{-1} \ff )\big] \\
&=: E_{\bias} + E_{\var} -2 E_{\cross}.
\end{align*}

The first part of Theorem \ref{thm:gen}, establishing that NT regression is well
approximated by kernel ridge regression for overparametrized models, is an immediate
consequence of the next statement.
\begin{thm}[Reduction to kernel ridge regression]\label{thm:reducekrr}
There exists a constant $C'>0$ such that the following holds. If $n \le  (\log(Nd))^{C'} Nd$, then for any $\lambda > 0$, with high probability,
\begin{align}
&\big| E_{\bias}^N - E_{\bias} \big| \le \eta \,\|f\|_{L^2}^2, \qquad \big| E_{\var}^N - E_{\var} \big| \le  \eta\, \sigma_{\veps}^2, 
\label{eq:reducekrr_1}\\
&\big| E_{\cross}^N - E_{\cross} \big| \le \eta\,\|f\|_{L^2}\sigma_{\veps}, \qquad \where~ \eta = \sqrt{\frac{n (C'\log (nNd))^{C'}}{Nd}}\, .
\label{eq:reducekrr_2}
\end{align}
As a consequence, we have $R_{\NT}(\lambda) = R_{\KRR}(\lambda) + O_{d,\P}\big( (\| f \|_{L^2}^2 + \sigma_\veps^2) \sqrt{\frac{n(\log (nNd))^{C'}}{Nd}} \big)$.
\end{thm}

Recall the decomposition of the infinite-width kernel into Gegenbauer polynomials 
introduced in Lemma \ref{lem:K_Harmonic}. In Section \ref{sec:Gen} we defined the
polynomial kernel $K^p(\xx, \xx')$ by truncating $K(\xx,\xx')$ to the degree-$\ell$ 
polynomials. Namely:
\begin{equation}
K^p(\xx, \xx') = \sum_{k=0}^\ell \gamma_k Q_k^{(d)}(\langle \xx, \xx' \rangle).
\end{equation}
We also define the matrix $\bK^p\in\reals^{n\times n}$ and  vector 
$\bK^{p}(\cdot, \xx) \in\reals^n$  as in Eq.~\eqref{eq:MatrixKp}.
In  polynomial ridge regression, the prediction function is 
\begin{equation*}
\hat f_{\PRR}(\xx) = \bK^p(\cdot, \xx)^\top ((\lambda + \gamma_{>\ell}) \bI_n + \bK^p)^{-1} \yy.
\end{equation*}
Its associated generalization error $R_{\PRR}(\lambda)$ into also decomposed into three errors.
\begin{align*}
R_{\PRR}(\lambda) &=  \E_{\xx} \Big[ \big( f(\xx) - \bK^{p}(\cdot, \xx)^\top (\lambda \bI_n + \bK^{p})^{-1} \ff \big)^2 \Big] \\
& + \bveps^\top (\lambda \bI_n + \bK^{p})^{-1} \E_{\xx} \big[ \bK^{p}(\cdot, \xx) \bK^{p}(\cdot, \xx)^\top \big] (\lambda \bI_n + \bK^{p})^{-1}  \bveps \\
&-2 \bveps^\top (\lambda \bI_n + \bK^{p})^{-1} \E_{\xx} \big[ \bK^{p}(\cdot, \xx) (f(\xx) - \bK^{p}(\cdot, \xx)^\top (\lambda \bI_n + \bK^{p})^{-1} \ff )\big] \\
&=: E_{\bias}^{p}(\lambda) + E_{\var}^{p}(\lambda) -2E_{\cross}^{p}(\lambda).
\end{align*}
The second part of Theorem \ref{thm:gen} follows immediately from the following result.

\begin{thm}[Reduction to polynomial ridge regression]\label{thm:reduceprr}
There exists a constant $C'>0$ such that the following holds. For any $\lambda >0$, with high probability,
\begin{align}
&\big| E_{\bias}(\lambda) - E_{\bias}^p(\lambda+\gamma_{>\ell}) \big| \le  \eta' \| f \|_{L^2}^2 , \qquad \big| E_{\var}(\lambda) - E_{\var}^p(\lambda+\gamma_{>\ell}) \big| \le  \eta' \sigma_\veps^2, \\
&\big| E_{\cross}(\lambda) - E_{\cross}^p(\lambda+\gamma_{>\ell}) \big| \le  \eta' \| f \|_{L^2} \sigma_\veps, \qquad \where~ \eta' := \sqrt{\frac{C' (\log n)^{C'}n}{d^{\ell+1}}}
\end{align}
As a consequence, we have $R_{\KRR}(\lambda) = R_{\PRR}(\lambda+\gamma_{>\ell}) + O_{d,\P}\big( (\| f \|_{L^2}^2 + \sigma_\veps^2 ) \sqrt{\frac{C'(\log n)^{C'}n}{d^{\ell+1}}} \big)$.
\end{thm}

\subsection{Proof of Theorem \ref{thm:reducekrr}: Bias term}

In this section we prove the first inequality in Eq.~\eqref{eq:reducekrr_1}.
Since both sides are homogeneous in $f$, we will assume without loss of generality that
$\|f\|_{L^2}=1$.

First let us decompose $E_{\bias}^N$. Define
\begin{align*}
& \bK(\cdot, \xx) = \big(  K(\xx_1,\xx), \ldots, K(\xx_n, \xx) \big)^\top \in \R^n, \qquad \bK^{(2)}  = \E \big[ \bK(\cdot, \xx) \bK(\cdot, \xx)^\top \big] \in \R^{n \times n}, \\
&\bK_N(\cdot, \xx) = \big(  K_N(\xx_1,\xx), \ldots, K_N(\xx_n, \xx) \big)^\top \in \R^n, \qquad \bK^{(2)} _N = \E \big[ \bK_N(\cdot, \xx) \bK_N(\cdot, \xx)^\top \big] \in \R^{n \times n}.
\end{align*}
Then, 
\begin{align*}
E_{\bias}^N & = \E_\xx \big[ \big( f(\xx) - \bK_N(\cdot, \xx)^\top ( \lambda \bI_n + \bK_N)^{-1} \ff \big)^2\big] = \| f \|_{L^2}^2 - 2 I_1^N + I_2^N
\end{align*}
where we define
\begin{align*}
&I_1^N = \ff^\top (\lambda \bI_n + \bK_N)^{-1} \E_\xx \big[ \bK_N(\cdot, \xx) f(\xx) \big], \\
&I_2^N = \ff^\top (\lambda \bI_n + \bK_N)^{-1} \E_\xx \big[ \bK_N(\cdot, \xx) \bK_N(\cdot, \xx)^\top \big] (\lambda \bI_n + \bK_N)^{-1}\ff.
\end{align*}
We also decompose $E_{\bias} = \| f \|_{L^2}^2 - 2 I_1 + I_2$, where
\begin{align*}
&I_1 = \ff^\top (\lambda \bI_n + \bK)^{-1} \E_\xx \big[ \bK(\cdot, \xx) f(\xx) \big], \\
&I_2 = \ff^\top (\lambda \bI_n + \bK)^{-1} \E_\xx \big[ \bK(\cdot, \xx) \bK(\cdot, \xx)^\top \big] (\lambda \bI_n + \bK)^{-1}\ff.
\end{align*}
Our goal is to show that $I_1^N$ and $I_1$ are close, and that $I_2^N$ and $I_2$ are close. %
Let us pause here to explain the challenges and our proof strategy:
\begin{itemize}
\item First, our concentration
 result controls $\| \bK^{-1/2} (\bK_N - \bK) \bK^{-1/2} \|_{\op}$ but not
  $\| \bK^{-1} (\bK_N - \bK) \|_\op$, so it is crucial to balance the matrix differences
  as we do in the  decomposition introduced below. 
  \item Second, 
  the relation between eigenvalues of $\bK_N$ and $\bK$ is not sufficient
  to control the generalization error
   (which is evident in the term $\bK_N^{-1} \bK_N^{(2)}\bK_N^{-1}$).
We will therefore characterize the eigenvector structure as well. 
\item Third our bound of $\| \bK^{-1/2} (\bK_N - \bK) \bK^{-1/2} \|_{\op}$
does not allow us to control  $\bK_N^{(2)} - \bK^{(2)}$ from our previous analysis. 
We develop a new approach that exploits  the independence of $(\ww_k)_{k\le N}$.
\end{itemize}

For the purpose of later use, we need a slightly general setup: let $g \in L^2$ be a function and $\hh \in \R^n$ a random vector. We begin the analysis by defining the following differences.
\begin{align*}
&\delta I_{11}^{g,\hh} = \big[ \hh^\top (\lambda \bI_n + \bK_N)^{-1} \bK_N - \hh^\top (\lambda \bI_n + \bK)^{-1} \bK_N \big] \cdot \bK_N^{-1} \E_\xx \big[ \bK_N(\cdot, \xx) g(\xx) \big], \\
&\delta I_{12}^{g,\hh} = \hh^\top (\lambda \bI_n + \bK)^{-1}  \E_\xx \big[ (\bK_N(\cdot, \xx) - \bK(\cdot, \xx)) g(\xx) \big], \\
&\delta I_{21}^\hh = \big[ \hh^\top (\lambda \bI_n + \bK_N)^{-1} \bK_N - \hh^\top (\lambda \bI_n + \bK)^{-1} \bK_N \big] \cdot  \bK_N^{-1} \bK_N^{(2)}\bK_N^{-1} \cdot \bK_N (\lambda \bI_n + \bK_N)^{-1} \hh, \\
&\delta I_{22}^\hh = \hh^\top (\lambda \bI_n + \bK)^{-1}\bK_N \cdot \bK_N^{-1}   \bK_N^{(2)} \bK_N^{-1} \cdot \big[ \bK_N(\lambda \bI_n + \bK_N)^{-1} \hh  - \bK_N (\lambda \bI_n + \bK)^{-1} \hh  \big], \\
&\delta I_{23}^\hh = \hh^\top (\lambda \bI_n + \bK)^{-1} \big[ \bK_N^{(2)}  - \bK^{(2)} \big] (\lambda \bI_n + \bK)^{-1}  \hh.
\end{align*}
We notice that 
\begin{equation*}
I_1^N - I_1= \delta I_{11}^{f, \ff} + \delta I_{12}^{f,\ff}, \qquad I_2^N - I_2 = \delta I_{21}^\ff + \delta I_{22}^\ff + \delta I_{23}^\ff, 
\end{equation*}
so we only need to bound these delta terms. Below we state a lemma for this general setup. 
Note that $\ff$ satisfies the assumption on the random vector therein, because by the law 
of large numbers $n^{-1} \| \ff \|^2 = (1+o_{n,\P}(1)) \| f \|_{L^2}^2$, so $\| \ff \| \le C_1\sqrt{n}$
 with high probability.
\begin{lem}\label{lem:keybnd}
Suppose that, for $C_1>0$ a constant, we have $\|g\|_{L^2} \le C_1$, and that $\hh \in \R^n$
 is a random vector that satisfies $\| \hh \| \le C_1 \sqrt{n}$ with high probability. 
 Then, there exists a constant $C'>0$ such that the following bounds hold with high probability.
\begin{align}
& \left\| \bK_N^{-1} \E_\xx \big[ \bK_N(\cdot, \xx) g(\xx) \big] \right\| \le  \frac{C'}{\sqrt{n}}, \label{bnd:key1} \\
& \left\| \bK_N^{-1} \bK_N^{(2)} \bK_N^{-1} \right\|_\op  \le   \frac{C'}{n}, \label{bnd:key2} \\
& \left\| \bK (\lambda \bI_n + \bK)^{-1} \hh \right\| \le C' \sqrt{n}, \label{bnd:key3} \\
&  \left\| \bK_N (\lambda \bI_n + \bK_N)^{-1} \hh \right\| \le C'\sqrt{n}, \label{bnd:key4} \\
& \left\| \bK_N (\lambda \bI_n + \bK_N)^{-1} \hh - \bK_N (\lambda \bI_n + \bK)^{-1} \hh \right\| \le C'  \eta \sqrt{n}. \label{bnd:key5}
\end{align}
Here, we denote $\eta = \sqrt{n(\log (nNd))^{C'} / (Nd)}$. As a consequence, we have, with high probability,
\begin{equation}
\max\{ |\delta I_{11}^{g,\ff} |,  |\delta I_{21}^\hh|,  |\delta I_{22}^\hh| \} \le C' \eta.
\end{equation}
\end{lem}

We handle the other two terms $\delta I_{12}^{g,\hh}$ and $\delta I_{23}^\hh$ in a different way. Denote 
\begin{align*}
\vv = (\lambda \bI_n + \bK)^{-1} \hh \in \R^n, \qquad \tilde h(\xx) = \bK(\cdot, \xx)^\top (\lambda \bI_n + \bK)^{-1} \hh.
\end{align*}
The function $\tilde h$ satisfies $\| \tilde h \|_{L^2} \le C\| \hh \|^2 / n$ 
with high probability by the following lemma, which we will prove in Section 
\ref{sec:ProofLemK2}.
\begin{lem}\label{lem:K2}
Define $\bK^{(2)} := \E_\xx [ \bK(\cdot, \xx) \bK(\cdot, \xx)^\top ] \in \R^{n \times n}$ or,
equivalently, 
\begin{align}
K_{ij}^{(2)} = \Big(\E_\xx [ K(\xx, \xx_i) K(\xx, \xx_j)]\Big)_{i,j\le n} \, .
\end{align}
Then, with high probability,
\begin{align}
&\big\| (\lambda \bI_n + \bK)^{-1} \bK^{(2)} (\lambda \bI_n + \bK)^{-1} \big\|_\op \le \frac{C}{n}, \label{eq:K21} \\
&\big\| (\lambda \bI_n  + \bK)^{-1} \E_\xx [ \bK(\cdot, \xx) f(\xx) ] \big\| \le \frac{C}{\sqrt{n}}\|f\|_{L^2}. \label{eq:K22}
\end{align}
\end{lem}

We can rewrite $\delta I_{12}^{g,\hh}$ and $\delta I_{23}^\hh$ as
\begin{align*}
 \delta I_{12}^{g,\hh} &= \vv^\top \E_\xx \big[ (\bK_N(\cdot, \xx) - \bK(\cdot, \xx)) g(\xx) \big] , \\
 \delta I_{23}^\hh &= 2 \vv^\top  \E_\xx \big[ (\bK_N(\cdot, \xx) - \bK(\cdot, \xx)) \tilde h(\xx) \big] + \vv^\top \E_\xx \Big[(\bK_N(\cdot, \xx) - \bK(\cdot, \xx)) (\bK_N(\cdot, \xx) - \bK(\cdot, \xx))^\top \Big]  \vv\\
& =: 2\delta I_{231}^\hh + \delta I_{232}^\hh.
\end{align*}
Note that $\delta I_{232}^\hh$ is always nonnegative. We calculate and bound $\E_\ww [ (\delta I_{12}^{g,\hh})^2 ]$, $\E_\ww [ (\delta I_{231}^\hh)^2 ]$, and $\E_\ww [\delta I_{232}^\hh] $, so that we obtain bounds on $\delta I_{12}^{g,\hh}$ and $\delta I_{23}^\hh$ with high probability. 

\begin{lem}\label{lem:keybnd2}
Suppose that, for $C_1>0$ a constant, we have $\|g\|_{L^2} \le C_1$, and that $\hh \in \R^n$
 is a random vector that satisfies $\| \hh \| \le C_1 \sqrt{n}$ with high probability. 

Let $\zz_1, \zz_2, \zz$ be independent copies of $\xx$. 
Then we have
\begin{align}
& \E_\ww [ (\delta I_{12}^{g,\hh})^2 ] \le \frac{1}{N} \hh^\top (\lambda \bI_n + \bK)^{-1} \bH_1 (\lambda \bI_n + \bK)^{-1} \hh,  \label{ineq:delta12}\\
& \E_\ww [ (\delta I_{231}^\hh)^2 ]  \le \frac{1}{N} \hh^\top (\lambda \bI_n + \bK)^{-1} \bH_2 (\lambda \bI_n + \bK)^{-1} \hh, \label{ineq:delta231} \\
& \E_\ww [ \delta I_{232}^\hh ] \le \frac{1}{N} \hh^\top (\lambda \bI_n + \bK)^{-1} \bH_3 (\lambda \bI_n + \bK)^{-1} \hh, \label{ineq:delta232}
\end{align}
where $\bH_j = (H_j(\xx_i,\xx_j))_{i,j \le n}\in \R^{n \times n}$ $(j=1,2,3)$ are given by
\begin{align*}
H_1(\xx_1, \xx_2) &= \E_{\zz_1, \zz_2, \ww} \Big[ \sigma'(\langle \xx_1, \ww \rangle) \sigma'(\langle \zz_1, \ww \rangle) \sigma'(\langle \xx_2, \ww \rangle)   \sigma'(\langle \zz_2, \ww \rangle)  \frac{\langle \xx_1, \zz_1 \rangle}{d}  \frac{\langle \xx_2, \zz_2 \rangle}{d} g(\zz_1) g(\zz_2) \Big] \\
H_2(\xx_1, \xx_2) &= \E_{\zz_1, \zz_2, \ww} \Big[ \sigma'(\langle \xx_1, \ww \rangle) \sigma'(\langle \zz_1, \ww \rangle) \sigma'(\langle \xx_2, \ww \rangle)   \sigma'(\langle \zz_2, \ww \rangle)  \frac{\langle \xx_1, \zz_1 \rangle}{d}  \frac{\langle \xx_2, \zz_2 \rangle}{d} \tilde h(\zz_1) \tilde h(\zz_2) \Big] \\
H_3(\xx_1, \xx_2) &= \E_{\zz, \ww} \Big[ \sigma'(\langle \xx_1, \ww \rangle) \sigma'(\langle \xx_2, \ww \rangle)  [\sigma'(\langle \zz, \ww \rangle)]^2  \frac{\langle \xx_1, \zz \rangle}{d}  \frac{\langle \xx_2, \zz \rangle}{d}\Big]. 
\end{align*}
Moreover, it holds that with high probability,
\begin{equation*}
\bH_1 \preceq \frac{C}{d} \bK, \qquad \bH_2 \preceq \frac{C}{d} \bK, \qquad \bH_3 \preceq \frac{C}{d} \bK.
\end{equation*}
As a consequence, with high probability, we have
\begin{equation*}
\big| \delta I_{12}^{g,\hh} \big| \le C\sqrt{\frac{ n\log n}{Nd}}, \qquad \big| \delta I_{23}^\hh \big| \le C\sqrt{\frac{ n\log n}{Nd}}.
\end{equation*}
\end{lem}

In Lemma~\ref{lem:keybnd} and~\ref{lem:keybnd2}, we set $\hh = \ff$ and $g = f$. By the conclusions of these two lemmas, we will obtain the bound on the bias term in Theorem~\ref{thm:reducekrr}. We defer the proofs of the two lemmas to the appendix.

\section{Conclusions}\label{sec:discuss}

We  studied the neural tangent model associated to two-layer neural networks, 
focusing on its generalization properties in a regime in which the sample size $n$,
the dimension $d$, and the number of neurons $N$ are all large and polynomially related
($c_0d\le n\le d^{1/c_0}$ for some small constant $c_0>0$),
while in the overparametrized regime $Nd\gg n$.
We assumed an isotropic model for the data $(\xx_i)_{i\le n}$,
and noisy label $(y_i)_{i\le n}$ with a general target $y_i=f_*(\xx_i)+\eps_i$
(with $\eps_i$ independent noise). 

As a fundamental technical step, we obtained a  characterization of the 
eigenstructure of the empirical NT kernel $\bK_N$
in the overparametrized regime.
In particular for non-polynomial activations, we showed that, 
the minimum eigenvalue of $\bK_N$  is bounded away from zero, as soon as  $Nd/(\log Nd)^C\ge n$. 
Further, for 
$d^{\ell}(\log d)^{C_0}\le n\le d^{\ell+1}/(\log d)^{C_0}$, $\ell\in\naturals$, 
$\lambda_{\min}(\bK_N)$  concentrates around a value that is explicitly given in terms of
the Hermite decomposition of the activation. 

An immediate corollary is that, as soon as $Nd\ge n(\log d)^C$,
the neural network can exactly interpolate arbitrary labels.

Our most important result is a  characterization of the test error of minimum
$\ell_2$-norm NT regression. We prove that, for $Nd/(\log Nd)^C\ge n$
the test error is close to the one of the $N=\infty$
limit (i.e.\ of kernel ridgeless regression with respect to the expected kernel).
The latter is in turn well approximated by polynomial regression 
with respect degree-$\ell$ polynomials as long as
$d^{\ell}(\log d)^{C_0}\le n\le d^{\ell+1}/(\log d)^{C_0}$.

Our analysis offers several insight to statistical practice:
\begin{enumerate}
\item NT models provide an effective randomized approximation to kernel methods.
Indeed their statistical properties are analogous to the one of more standard random features 
methods \cite{rahimi2008random}, when we compare them at fixed number of parameters
\cite{mei2021generalization}. However,
the computational complexity at prediction time of NT models is smaller than the one of random feature
models if we keep the same number of parameters.
\item The additional error due to the finite number of hidden units is roughly
bounded by $\sqrt{n/(Nd)}$.
\item In high dimensions, the nonlinearity in the activation function produces an effective 
`self-induced' regularization (diagonal term in the kernel) and as a consequence interpolation
can be optimal.
\end{enumerate}

Finally, our characterization of generalization error applies
to two-layer neural networks (under specific initializations) as long
as the neural tangent theory provides a good approximation for their behavior.
In Section \ref{sec:ConnectionOptimization} we discussed sufficient conditions for this to be the case,
but we do not expect these conditions to be sharp.

\section*{Acknowledgements}
We thank Behrooz Ghorbani and Song Mei for helpful discussion.
This work was supported by NSF through award DMS-2031883 and from the Simons Foundation
through Award 814639 for the Collaboration on the Theoretical Foundations of Deep Learning.
We also acknowledge NSF grants CCF-1714305,
IIS-1741162, and the ONR grant N00014-18-1-2729.

\newpage
\appendix

\section{Additional numerical experiment}

The setup of the experiment is similar to the second experiment in Section~\ref{sec:numerical}: we fit a linear model $y_i = \langle \xx_i, \bbeta_* \rangle + \veps_i$ where covariates satisfy $\xx_i\sim\Unif(\S^{d-1}(\sqrt{d}))$ and the noises satisfy
$\veps_i\sim\normal(0,\sigma^2_{\veps})$, $\sigma_{\veps}=0.5$. We fix $N = 800$ and $d=500$, and vary the sample size $n \in \{100,200,\ldots,900,1000,1500,\ldots,4000\}$.

 Figure \ref{fig:sim3} illustrates the relation between the NT,  2-layer NN, and polynomial ridge regression. We train a two-layer neural network with ReLU activations using the initialization strategy of \cite{chizat2019lazy}. Namely, we draw $(a_k^0)_{k \le N} \sim_{\iid} \cN(0,1)$, $(\ww_k^0)_{k \le N} \sim_{\iid} \cN(\bzero, \bI_d)$ and set $\tilde a_k^0 = - a_k^0$, $\tilde \ww_k^0 = \ww_k^0$. We use Adam optimizer to train the two-layer neural network $f(\xx) =\sum_{k=1}^Na_k \sigma(\<\ww_k,\xx\>) + \sum_{\ell=1}^N\tilde a_k\sigma(\<\tilde \ww_k,\xx\>) $ with parameters $(a_k, \tilde a_k, \ww_\ell, \tilde \ww_\ell)_{k \le N}$ initialized by $(a_k^0, \tilde a_k^0, \ww_k^0, \tilde \ww_k^0)_{k \le N}$. This guarantees that the output is zero at initialization and the parameter scale is much larger than the default, so that we are in the lazy training regime. We compare its generalization error with the one of PRR, and with the theoretical prediction of Remark~\ref{rmk:explicit}. The agreement is excellent.

\begin{figure}[H]
\centering
\includegraphics[scale=0.8]{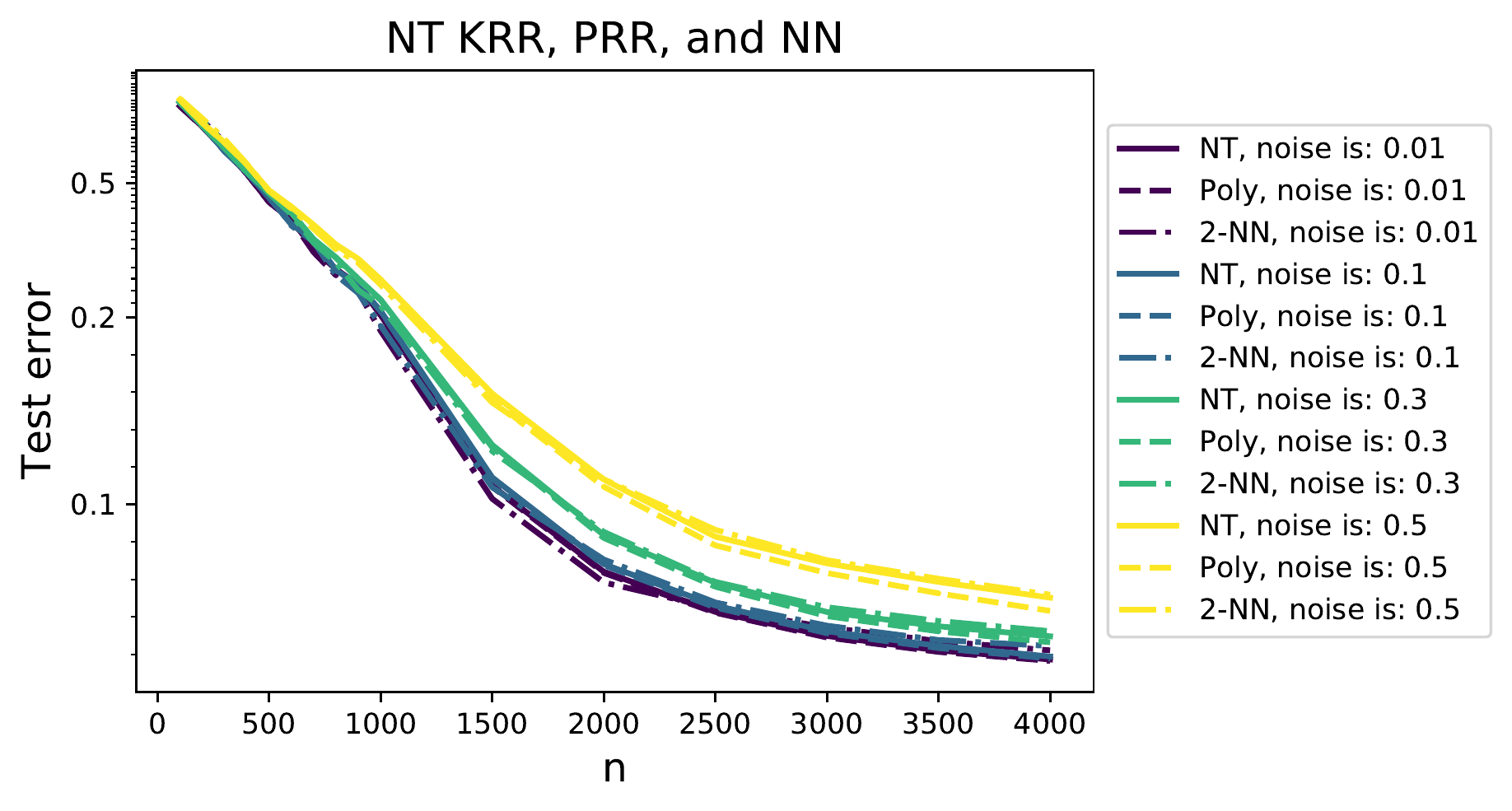}
\caption{Test/generalization errors of NT, Poly, and two-layer neural network (2-NN). We fix $N = 800, d=500$ and vary $n$. For each noise level $\sigma_\veps \in \{0.01, 0.1, 0.3, 0.5\}$ (which corresponds to one color), we plot three curves that represent the test errors for NT, Poly, and 2-NN. All curves of the same color behave similarly.}
\label{fig:sim3}

\end{figure}

\section{Generalization error: Proof of Theorem \ref{thm:gen}}

This section finishes the proof of  our main result,  Theorem \ref{thm:gen}, 
which was initiated in Section \ref{sec:ProofGenMain}. We use definitions and notations 
introduced in that section.

\subsection{Proof of Theorem~\ref{thm:reducekrr}: Variance term and cross term}

By homogeneity we can and will assume, without loss of generality, $\|f\|_{L^2}=
\sigma_{\veps}=1$.

We handle $E_{\var}^N$ and $E_{\cross}^N$ in a way similar to that of $E_{\bias}^N$ (mostly following the same proof with $\hh$ set to $\bveps$ instead of $\ff$). First we observe that 
\begin{align*}
&~~~ \bveps^\top (\lambda \bI_n + \bK_N)^{-1} \bK_N^{(2)} (\lambda \bI_n + \bK_N)^{-1} \bveps - \bveps^\top (\lambda \bI_n + \bK)^{-1} \bK^{(2)} (\lambda \bI_n + \bK)^{-1} \bveps \\
&= \big[ \bveps^\top (\lambda \bI_n + \bK_N)^{-1} \bK_N - \bveps^\top (\lambda \bI_n + \bK)^{-1} \bK_N \big] \cdot  \bK_N^{-1} \bK_N^{(2)}\bK_N^{-1} \cdot \bK_N (\lambda \bI_n + \bK_N)^{-1} \bveps \\
& + \bveps^\top (\lambda \bI_n + \bK)^{-1}\bK_N \cdot \bK_N^{-1}   \bK_N^{(2)} \bK_N^{-1} \cdot \big[ \bK_N(\lambda \bI_n + \bK_N)^{-1} \bveps  - \bK_N (\lambda \bI_n + \bK)^{-1} \bveps  \big] \\
& + \bveps^\top (\lambda \bI_n + \bK)^{-1} \big[ \bK_N^{(2)}  - \bK^{(2)} \big] (\lambda \bI_n + \bK)^{-1} \bveps \\
&=: \delta J_1 + \delta J_2 + \delta  J_3.
\end{align*}
Note that $n^{-1}\| \bveps \|^2 = (1+o_{d,\P}(1)) \sigma_\veps^2$ by the law of large numbers, and
therefore $\| \bveps \| \le C_1 \sqrt{n}$ with high probability.
Applying Lemma~\ref{lem:keybnd} with $\hh$ set to $\bveps$, we obtain the following holds with high probability. 
\begin{align}
& \left\| \bK (\lambda \bI_n + \bK)^{-1} \bveps \right\| \le C\sqrt{n}, \label{bnd:key21}\\
&  \left\| \bK_N (\lambda \bI_n + \bK_N)^{-1} \bveps \right\| \le C\sqrt{n}, \label{bnd:key22}\\
& \left\| \bK_N (\lambda \bI_n + \bK_N)^{-1} \bveps - \bK_N (\lambda \bI_n + \bK)^{-1} \bveps \right\| \le C \eta \sqrt{n}. \label{bnd:key23}
\end{align}
This implies that $|\delta J_1| \le C\eta$ and $|\delta J_2| \le C\eta$ w.h.p. Moreover, denoting
\begin{align*}
\vv = (\lambda \bI_n + \bK)^{-1} \bveps \in \R^n, \qquad \tilde h(\xx) = \bK(\cdot, \xx)^\top (\lambda \bI_n + \bK)^{-1} \bveps,
\end{align*}
we express $\delta J_{3}$ as
\begin{align*}
 \delta J_{3} &= 2 \vv^\top \E_\xx \big[ (\bK_N(\cdot, \xx) - \bK(\cdot, \xx)) \tilde h(\xx) \big] + \vv^\top \E_\xx \Big[(\bK_N(\cdot, \xx) - \bK(\cdot, \xx)) (\bK_N(\cdot, \xx) - \bK(\cdot, \xx))^\top \Big]  \vv\\
& =: 2\delta J_{31} + \delta J_{32}.
\end{align*}
Applying Lemma~\ref{lem:keybnd2} in which we set $\hh = \bveps$, we obtain w.h.p.,
\begin{align*}
& \E_\ww [ (\delta J_{31})^2 ]  \le \frac{1}{N} \bveps^\top (\lambda \bI_n + \bK)^{-1} \bH_2 (\lambda \bI_n + \bK)^{-1} \bveps \le \frac{Cn}{Nd}  \\
& \E_\ww [ \delta J_{32} ] \le \frac{1}{N} \bveps^\top (\lambda \bI_n + \bK)^{-1} \bH_3 (\lambda \bI_n + \bK)^{-1} \bveps \le  \frac{Cn}{Nd}.
\end{align*}
This implies that $| \delta J_{3} | \le C\eta$ w.h.p.~as well. This proves the bound on the variance term in Theorem~\ref{thm:reducekrr}. 

Now for the cross term, we observe that 
\begin{align*}
E_{\cross}^N - E_{\cross} & = \Big[\bveps^\top (\lambda \bI_n + \bK_N)^{-1} \E_{\xx} \big[ \bK_N(\cdot, \xx) f(\xx) \big] - \bveps^\top (\lambda \bI_n + \bK)^{-1} \E_{\xx} \big[ \bK(\cdot, \xx) f(\xx) \big]\Big] \\
&- \Big[\bveps^\top (\lambda \bI_n + \bK_N)^{-1} \bK_N^{(2)} (\lambda \bI_n + \bK_N)^{-1} \ff - \bveps^\top (\lambda \bI_n + \bK)^{-1} \bK^{(2)} (\lambda \bI_n + \bK)^{-1} \ff\Big] \\
&=: \tilde{\delta J}_1 - \tilde{\delta J}_2.
\end{align*}
For the first term $\tilde{\delta J}_1$, we further decompose:
\begin{align*}
\tilde{\delta J}_1  &= \big[ \bveps^\top (\lambda \bI_n + \bK_N)^{-1} \bK_N - \bveps^\top (\lambda \bI_n + \bK)^{-1} \bK_N \big] \cdot \bK_N^{-1} \E_\xx \big[ \bK_N(\cdot, \xx) f(\xx) \big], \\
 &+ \bveps^\top (\lambda \bI_n + \bK)^{-1}  \E_\xx \big[ (\bK_N(\cdot, \xx) - \bK(\cdot, \xx)) f(\xx) \big] =: \tilde{\delta J}_{11} + \tilde{\delta J}_{12}.
\end{align*}
From Eqs.~\eqref{bnd:key1} and \eqref{bnd:key5} in Lemma \ref{lem:keybnd} (setting $g = f$ and $\hh = \bveps$), we have $|\tilde{\delta J}_{11}| \le  C\eta$ w.h.p. By Lemma~\ref{lem:keybnd2} Eq.~\eqref{ineq:delta12} (setting $\hh$ to $\bveps$), w.h.p.,
\begin{equation*}
\E_\ww [ (\tilde{\delta J}_{12})^2 ] \le \frac{1}{N} \bveps^\top (\lambda \bI_n + \bK)^{-1} \bH_1 (\lambda \bI_n + \bK)^{-1} \bveps \le \frac{Cn}{Nd},
\end{equation*}
so with high probability $|\tilde{\delta J}_{12}| \le C\eta$. Thus we proved $|\tilde{\delta J}_{1}| \le C\eta$.

Finally, we further decompose $\tilde{\delta J}_2$ as follows.
\begin{align*}
&\tilde{\delta J}_2 = \big[ \bveps^\top (\lambda \bI_n + \bK_N)^{-1} \bK_N - \bveps^\top (\lambda \bI_n + \bK)^{-1} \bK_N \big] \cdot  \bK_N^{-1} \bK_N^{(2)}\bK_N^{-1} \cdot \bK_N (\lambda \bI_n + \bK_N)^{-1} \ff, \\
& + \bveps^\top (\lambda \bI_n + \bK)^{-1}\bK_N \cdot \bK_N^{-1}   \bK_N^{(2)} \bK_N^{-1} \cdot \big[ \bK_N(\lambda \bI_n + \bK_N)^{-1} \ff - \bK_N (\lambda \bI_n + \bK)^{-1} \ff  \big], \\
& + \bveps^\top (\lambda \bI_n + \bK)^{-1} \big[ \bK_N^{(2)}  - \bK^{(2)} \big] (\lambda \bI_n + \bK)^{-1} \ff \\
&=: \tilde{\delta J}_{21} + \tilde{\delta J}_{22} + \tilde{\delta J}_{23}.
\end{align*}
By Lemma~\ref{lem:keybnd} (in which we set $\hh$ to $\ff$) and Eqs.~\eqref{bnd:key21}--\eqref{bnd:key23}, we have $|\tilde{\delta J}_{21}| \le C\eta$ and $|\tilde{\delta J}_{22}| \le C\eta$ w.h.p.~Moreover, denoting
\begin{align*}
&\vv_1 = (\lambda \bI_n + \bK)^{-1} \ff, \qquad \tilde h_1(\xx) = \bK(\cdot, \xx)^\top (\lambda \bI_n + \bK)^{-1} \ff \\
&\vv_2 =  (\lambda \bI_n + \bK)^{-1} \bveps , \qquad \tilde h_2(\xx) = \bK(\cdot, \xx)^\top (\lambda \bI_n + \bK)^{-1} \bveps,
\end{align*}
we express $\tilde{\delta J}_{23}$ as
\begin{align*}
 \tilde{\delta J}_{23} &=  \vv_2^\top \E_\xx \big[ (\bK_N(\cdot, \xx) - \bK(\cdot, \xx)) \tilde h_1(\xx) \big] + \vv_1^\top \E_\xx \big[ (\bK_N(\cdot, \xx) - \bK(\cdot, \xx)) \tilde h_2(\xx) \big] \\
&+ \vv_2^\top \E_\xx \Big[(\bK_N(\cdot, \xx) - \bK(\cdot, \xx)) (\bK_N(\cdot, \xx) - \bK(\cdot, \xx))^\top \Big]  \vv_1\\
& =:  \tilde{\delta J}_{231} + \tilde{\delta J}_{231}' + \tilde{\delta J}_{232}.
\end{align*}
In Lemma~\ref{lem:keybnd2} Eq.~\eqref{ineq:delta12}, we set $g=\tilde h_1, \hh = \bveps$ to obtain $\E_\ww\big[ (\tilde{\delta J}_{231} )^2\big] \le Cn/(Nd)$; and we set $g=\tilde h_2, \hh = \ff$ to obtain $\E_\ww\big[ (\tilde{\delta J}_{231}' )^2\big] \le Cn/(Nd)$. For $\tilde{\delta J}_{232}$, we use
\begin{align*}
\tilde{\delta J}_{232} &= (\vv_1+\vv_2)^\top  \E_\xx \Big[(\bK_N(\cdot, \xx) - \bK(\cdot, \xx)) (\bK_N(\cdot, \xx) - \bK(\cdot, \xx))^\top \Big](\vv_1+\vv_2) \\
&- \vv_1^\top  \E_\xx \Big[(\bK_N(\cdot, \xx) - \bK(\cdot, \xx)) (\bK_N(\cdot, \xx) - \bK(\cdot, \xx))^\top \Big] \vv_1 \\
&- \vv_2^\top  \E_\xx \Big[(\bK_N(\cdot, \xx) - \bK(\cdot, \xx)) (\bK_N(\cdot, \xx) - \bK(\cdot, \xx))^\top \Big] \vv_2 \\
&\le (\vv_1+\vv_2)^\top  \E_\xx \Big[(\bK_N(\cdot, \xx) - \bK(\cdot, \xx)) (\bK_N(\cdot, \xx) - \bK(\cdot, \xx))^\top \Big](\vv_1+\vv_2).
\end{align*}
Note that $\| \ff + \bveps \| \le C \sqrt{n}$ w.h.p., so applying Lemma~\ref{lem:keybnd2} Eq.~\eqref{ineq:delta232} with $\hh = \ff + \bveps$ leads to $\E_\ww \big[ \tilde{\delta J}_{232} \big] \le Cn/(Nd)$ w.h.p. This proves $|\tilde{\delta J}_{23}| \le Cn/(Nd)$ w.h.p.~and therefore $|\tilde{\delta J}_2| \le C \eta$ w.h.p. Hence, we have proved the bound on the cross term in Theorem~\ref{thm:reducekrr}.

\subsection{Proof of Lemma~\ref{lem:keybnd}}\label{sec:proofkeybnd}

Throughout this subsection, given a sequence of random variables $\xi_d$,
we write $\xi_d = \breve{o}_{d,\P}(1)$ if and only if for any $\veps > 0$ and 
$\beta > 0$, we have $\lim_{d \to 0} d^\beta \P(|\xi_d| > \veps) = 0$. We also assume that 
\begin{align*}
\frac{d^{\ell+1}}{(\log n)^C} \ge n \ge  d^\ell  (\log n)^C, \qquad n \le \frac{Nd}{(\log(Nd))^{C}}
\end{align*}
for a sufficiently large $C>0$ such that we can use the following inequalities  by 
Lemma~\ref{lem:Kdecomp} and Theorem~\ref{thm:invert2}:
\begin{align}\label{eq:quickref}
\big\| \bDelta \big\|_\op = \breve{o}_{d,\P}(1), \qquad \big\| n^{-1} \bPsi_{\le \ell}^\top \bPsi_{\le \ell} - \bI_n \big\|_\op = \breve{o}_{d,\P}(1), \qquad \big\| \bK^{-1/2} (\bK_N - \bK) \bK^{-1/2} \big\|_\op = \breve{o}_{d,\P}(1).
\end{align}
In Section \ref{sec:Ell1} we will refine our analysis to eliminate logarithmic factors 
for $\ell=1$.

We start with the easiest bounds in Lemma~\ref{lem:keybnd}. 

\begin{proof}[{\bf Proof of Lemma~\ref{lem:keybnd}, Eqs.~\eqref{bnd:key3}--\eqref{bnd:key5}}]

The first two bounds follow from the observation
\begin{align*}
& \bK (\lambda \bI_n + \bK)^{-1} \hh = \hh - \lambda (\lambda \bI_n + \bK)^{-1} \hh, \\
&\bK_N (\lambda \bI_n + \bK_N)^{-1} \hh = \hh - \lambda (\lambda \bI_n + \bK_N)^{-1} \hh
\end{align*}
From the invertibility of $\bK$ and $\bK_N$, with high probability,
\begin{align*}
&\left\| \lambda (\lambda \bI_n + \bK)^{-1} \right\|_\op \le \frac{\lambda }{c+\lambda}, \qquad \left\| \lambda (\lambda \bI_n + \bK_N)^{-1} \right\|_\op \le \frac{\lambda }{c+\lambda}.
\end{align*}
We conclude that Eq.~\eqref{bnd:key3} and \eqref{bnd:key4} hold with high probability. 

Next, we will show that w.h.p.,
\begin{align}
& \left\| \bK_N (\lambda \bI_n + \bK_N)^{-1} - \bK (\lambda \bI_n + \bK)^{-1}  \right\|_\op \le C\sqrt{\eta}, \label{ineq:easy1} \\
&\left\| (\bK_N - \bK)(\lambda \bI_n + \bK)^{-1} \hh \right\| \le C \sqrt{\eta n}. \label{ineq:easy2}
\end{align}
Once these are proved, then, w.h.p.,
\begin{align*}
 \left\| \bK_N (\lambda \bI_n + \bK_N)^{-1} \hh - \bK_N (\lambda \bI_n + \bK)^{-1} \hh \right\|  &\le \left\| \bK_N (\lambda \bI_n + \bK_N)^{-1} - \bK (\lambda \bI_n + \bK)^{-1}  \right\|_\op \cdot \| \hh \| \\
+ \left\| (\bK_N - \bK)(\lambda \bI_n + \bK)^{-1} \hh \right\| & \le C \sqrt{\eta n}.
\end{align*}
In order to show \eqref{ineq:easy1}, we observe that
\begin{align*}
&~~~\bK_N (\lambda \bI_n + \bK_N)^{-1} - \bK (\lambda \bI_n + \bK)^{-1} \\
&= -\lambda (\lambda \bI_n + \bK_N)^{-1} + \lambda (\lambda \bI_n + \bK)^{-1} \\
&= \lambda (\lambda \bI_n + \bK_N)^{-1} (\bK_N - \bK) (\lambda \bI_n + \bK)^{-1} \\
&= \lambda (\lambda \bI_n + \bK_N)^{-1}  \bK_N^{1/2} \bK_N^{-1/2} \bK^{1/2} \bK^{-1/2}(\bK_N - \bK) \bK^{-1/2} \bK^{1/2}(\lambda \bI_n + \bK)^{-1} .
\end{align*}
From Theorem~\ref{thm:invert2}, we know that $\| \bK^{-1/2} (\bK_N - \bK) \bK^{-1/2} \|_\op \le \sqrt{\eta}$ w.h.p. It also implies 
\begin{equation*}
\|\bK_N^{-1/2} \bK^{1/2} \|_\op^2 = \| \bK^{1/2} \bK_N^{-1} \bK^{1/2} \big\|_\op = \| (\bK^{-1/2} \bK_N \bK^{-1/2})^{-1} \big\|_\op \le (1-\sqrt{\eta})^{-1}.
\end{equation*}
Also, we notice that
\begin{align*}
&\left\| \lambda^{1/2} \bK^{1/2} (\lambda \bI_n + \bK)^{-1} \right\|_\op \le \max_i\frac{\lambda^{1/2} [\lambda_i(\bK)]^{1/2}}{\lambda + \lambda_i(\bK)} \le \frac{1}{2}, \\
& \left\| \lambda^{1/2} (\lambda \bI_n + \bK_N)^{-1} \bK_N^{1/2} \right\|_\op \le \max_i\frac{\lambda^{1/2}[\lambda_i(\bK_N)]^{1/2}}{\lambda + \lambda_i(\bK_N)} \le \frac{1}{2}, 
\end{align*}
These bounds then lead to the desired bound in \eqref{ineq:easy1}. In order to show \eqref{ineq:easy2}, we recall the notation $\vv = (\lambda \bI_n + \bK)^{-1} \hh$. We also write $\bK_N = N^{-1} \sum_{k \le N} \bK^{(k)}$ where $(\bK^{(k)})_{k \le N}$ are independent conditional on $(\xx_i)_{i\le n}$ and $\E_\ww [\bK^{(k)} ] = \bK$. We calculate
\begin{align}
\E_\ww \big[ \big\| (\bK_N - \bK) \vv \big\|^2 \big] &= \E_\ww \langle \vv, (\bK_N - \bK)^2 \vv \rangle = \frac{1}{N^2} \sum_{k, k' \le N}  \E_\ww \langle \vv, (\bK^{(k)} - \bK)(\bK^{(k')} - \bK) \vv \rangle \\
&= \frac{1}{N^2} \sum_{k \le N} \E_\ww \langle \vv, (\bK^{(k)} - \bK)^2 \vv \rangle \\
&\stackrel{(i)}{\le} \frac{1}{N^2} \sum_{k \le N} \E_\ww \langle \vv, (\bK^{(k)})^2 \vv \rangle
\label{eq:SquareDecoupling}
\end{align}
where \textit{(i)} is due to 
\begin{equation*}
\E_\ww \big[ (\bK^{(k)} - \bK)^2 \big] = \E_\ww \big[ (\bK^{(k)})^2 - \bK^{(k)} \bK - \bK \bK^{(k)} + \bK^2\big] = \E_\ww \big[ (\bK^{(k)})^2 \big] - \bK^2 \preceq \E_\ww \big[ (\bK^{(k)})^2 \big].
\end{equation*}
Recall the notations in the proof of Theorem~\ref{thm:invert2},
cf. Eqs.~\eqref{eq:TruncDef1}, \eqref{eq:TruncDef2}.

 We denote the indicator variable 
$\xi = \bone\{ |\langle \xx_i, \ww_k \rangle| \le C_0 \log (nNd), \forall~i,k \}$. 
If $C_0$ is sufficiently large, we have 
$\E_{\ww}\{\| \bK^{(k)} (1-\xi) \|_\op^2\} \le C/(nNd)^2$.
 Therefore
\begin{align*}
&\bK^{(k)} = \bK^{(k)} \xi + \bK^{(k)} (1-\xi) =d^{-1} \bD_k \XX \XX^\top \bD_k \xi + \bK^{(k)} (1-\xi), \qquad \text{and thus} \\
&\vv^\top (\bK^{(k)})^2 \vv = \vv^\top (\bK^{(k)})^2 \vv \xi + \vv^\top (\bK^{(k)} )^2\vv (1-\xi) \le \vv^\top (d^{-1} \bD_k \XX \XX^\top \bD_k)^2 \vv + \frac{C \| \vv \|^2}{n^2N^2d^2}.
\end{align*}
Now we work on the event 
 $\| \XX \|_\op \le 2(\sqrt{n} + \sqrt{d})$, which happens with high probability.
Continuing from Eq.~\eqref{eq:SquareDecoupling}, we get:
\begin{align*}
\E_\ww \big[ \big\| (\bK_N - \bK) \vv \big\|^2 \big]  & \le \frac{1}{N^2 d^2}  \sum_{k \le N} \E_\ww \Big[ \vv^\top  \bD_k \XX \XX^\top \bD_k^2 \XX \XX^\top \bD_k \vv \Big]  +  \frac{C \| \vv \|^2}{n^2N^2d^2} \\
&\le \frac{n (\log (nNd))^C}{N^2 d^2}  \sum_{k \le N}  \E_\ww \Big[ \vv^\top \bD_k \XX \XX^\top \bD_k \vv \Big] +  \frac{C \| \vv \|^2}{n^2N^2d^2} \\
& = \frac{n (\log (nNd))^C}{N d} \vv^\top \bK \vv+  \frac{C \| \vv \|^2}{n^2N^2d^2}.
\end{align*} 
Finally,  since $\bK \preceq \lambda \bI_n + \bK$, w.h.p.,
\begin{align*}
& \vv^\top \bK \vv \le \hh^\top (\lambda \bI_n + \bK)^{-1} (\lambda \bI_n + \bK) (\lambda \bI_n + \bK)^{-1} \hh \le \frac{1}{c + \lambda} \| \hh \|^2 \le  Cn, \\
& \| \vv \|^2  \le \frac{1}{c+\lambda} \| \hh \|^2 \le Cn.
\end{align*}
We have shown that 
\begin{equation*}
\E_\ww \big[ \big\| (\bK_N - \bK) \vv \big\|^2 \big]  \le \frac{n^2 (\log (nNd))^C}{Nd}.
\end{equation*}
By Markov's inequality, we then obtain $\| (\bK_N - \bK )\vv \| \le \sqrt{\eta n}$ w.h.p.,
as claimed in Eq.~\eqref{ineq:easy2}. 
\end{proof}

Before proving the more difficult inequalities 
\eqref{bnd:key1}, \eqref{bnd:key2}, we establish some useful properties of the 
eigenstructure of $\bK_N$. For convenience, we write 
$\bD_1 = \bD_2 \cdot (1 + \breve{o}_{d,\P}(1))$ if $\bD_1, \bD_2$ are diagonal matrices and 
$\max_i \big| (\bD_1)_{ii} / (\bD_2)_{ii} - 1 \big| = \breve{o}_{d,\P}(1)$; we also write 
$\bD_1 = \bD_2 + \breve{o}_{d,\P}(1)$ if $\bD_1, \bD_2$ are diagonal matrices and 
$\max_i \big| (\bD_1)_{ii} - (\bD_2)_{ii} \big| = \breve{o}_{d,\P}(1)$. Denote 
$D= \sum_{k \le \ell} B(d,k)$.

\begin{lem}[Kernel eigenvalue structure]\label{lem:eigval}
The eigen-decomposition of $\bK - \gamma_{>\ell} \bI_n$ and $\bK_N - \gamma_{>\ell} \bI_n $ 
can be written in the following form:
\begin{align}
&\bK - \gamma_{>\ell} \bI_n = \bU \bD \bU^\top + \bDelta^{(\res)}, \label{eq:Keigdecomp}\\
&\bK_N - \gamma_{>\ell} \bI_n = \bU_N \bD_N \bU_N^\top + \bDelta^{(\res)}_N,
\label{eq:KNeigdecomp}
\end{align}
where $\bD, \bD_N\in\reals^{D\times D}$ are diagonal matrices that contain $D$ eigenvalues of $\bK, \bK_N$ respectively, 
columns of $\bU, \bU_N \in \R^{n \times D}$ are the corresponding eigenvectors 
and $\bDelta^{(\res)}, \bDelta_N^{(\res)}$ correspond to the other eigenvectors 
(in particular $\bDelta^{(\res)}\bU=\bDelta_N^{(\res)}\bU_N=\bzero$).

The eigenvalues in $\bD, \bD_N$ have the following structure.
\begin{align}
&\bD = \diag \Big(\gamma_0 n, \underbrace{\frac{\gamma_1 n}{d}, \ldots, \frac{\gamma_1 n}{d}}_{d}, \ldots, \underbrace{\frac{\gamma_\ell (\ell !)n}{d^\ell}, \ldots, \frac{\gamma_\ell (\ell !)n}{d^\ell}}_{B(d,\ell)}\Big) \cdot (1 + \breve{o}_{d,\P}(1)) + \breve{o}_{d,\P}(1), \label{eq:Keig}\\
&\bD_N = \diag \Big(\gamma_0 n, \underbrace{\frac{\gamma_1 n}{d}, \ldots, \frac{\gamma_1 n}{d}}_{d}, \ldots, \underbrace{\frac{\gamma_\ell (\ell !)n}{d^\ell}, \ldots, \frac{\gamma_\ell (\ell !)n}{d^\ell}}_{B(d,\ell)}\Big) \cdot (1 + \breve{o}_{d,\P}(1)) + \breve{o}_{d,\P}(1). \label{eq:KNeig}
\end{align}
Moreover, the remaining components satisfy $\| \bDelta^{(\res)} \|_\op = \breve{o}_{d,\P}(1)$ and $\| \bDelta^{(\res)}_N \|_\op = \breve{o}_{d,\P}(1)$. \end{lem} 

\begin{proof}[{\bf Proof of Lemma~\ref{lem:eigval}}]
For convenience, for a symmetric matrix $\bA$, we denote by $\blambda(\bA)$ the vector of eigenvalues of $\bA$. We write $\blambda(\bA_1) = \blambda(\bA_2) \cdot (1 + \breve{o}_{d,\P}(1))$ if $\max_i| \lambda_i(\bA_1) / \lambda_i(\bA_2) - 1| = \breve{o}_{d,\P}(1)$. 
From Lemma~\ref{lem:Kdecomp}, we can decompose $\bK$ as
\begin{equation*}
\bK = \bPsi_{\le \ell} \bLambda_{\le \ell}^2 \bPsi_{\le \ell}^\top + \gamma_{>\ell} \bI_n + \bDelta.
\end{equation*}
We claim that 
\begin{equation}\label{claim:lambda}
\blambda(\bPsi_{\le \ell} \bLambda_{\le \ell}^2 \bPsi_{\le \ell}^\top) = 
\Big(\gamma_0 n, \underbrace{\frac{\gamma_1 n}{d}, \ldots, \frac{\gamma_1 n}{d}}_{d}, \ldots, 
\underbrace{\frac{\gamma_\ell (\ell!)n }{d^\ell}, \ldots, \frac{\gamma_\ell (\ell!) n}{d^\ell}}_{B(d,\ell)}, 
\underbrace{0, \ldots, 0}_{n-D} \Big) \cdot (1 + \breve{o}_{d,\P}(1)).
\end{equation}
To prove this claim, first note that \eqref{eq:quickref} implies $n^{-1/2} \| \bPsi_{\le \ell}\|_\op = 1 + \breve{o}_{d,\P}(1)$. Observe that $\bPsi_{\le \ell} \bLambda_{\le \ell}^2 \bPsi_{\le \ell}^\top$ has rank at most $D$. Define $\QQ \in \R^{n \times D}$ such that its columns are the left singular vectors of $\bPsi_{\le \ell}$. We only need to show 
\begin{equation}\label{claim:lambdaQ}
\blambda(\QQ^\top \bPsi_{\le \ell} \bLambda_{\le \ell}^2 \bPsi_{\le \ell}^\top \QQ) = \Big(\gamma_0 n, \underbrace{\frac{\gamma_1 n}{d}, \ldots, \frac{\gamma_1 n}{d}}_{d}, \ldots, \underbrace{\frac{\gamma_\ell (\ell!)n}{d^\ell}, \ldots, \frac{\gamma_\ell (\ell!)n}{d^\ell}}_{B(d,\ell)}\Big) \cdot (1 + \breve{o}_{d,\P}(1)).
\end{equation}
In order to prove the above claim, we use the eigenvalue min-max principle to express the $k$-th eigenvalue of $\QQ^\top \bPsi_{\le \ell} \bLambda_{\le \ell}^2 \bPsi_{\le \ell}^\top \QQ$. We sort the diagonal entries of $\bLambda_{\le \ell}^2$ in the descending order and denote the ranks by $s_1, \ldots, s_D$ (break ties arbitrarily). Define the subspaces $\cV_1 = \spann\{\ee_{s_1}, \ldots, \ee_{s_k} \}$, $\cV_2 = \spann\{ \ee_{s_k}, \ldots, \ee_{s_D} \} \subset \R^D$ and $\cV'_1 = \spann\{\uu \in \R^D:  \bPsi_{\le \ell}^\top \QQ\uu \in \cV_1 \}, \cV'_2 = \spann\{ \uu \in \R^D: \bPsi_{\le \ell}^\top \QQ\uu \in \cV_2 \}$. Note that $\bPsi_{\le \ell}^\top \QQ$ has full rank, 
so $\dim(\cV_1') =k $ and $\dim(\cV_2') = D-k+1$.
\begin{align*}
\lambda_k(\QQ^\top \bPsi_{\le \ell} \bLambda_{\le \ell}^2 \bPsi_{\le \ell}^\top \QQ) &=
 \max_{\cV: \dim(\cV) = k} \min_{\bzero \neq \uu \in \cV} \frac{ \uu^\top\QQ^\top \bPsi_{\le \ell} \bLambda_{\le \ell}^2 \bPsi_{\le \ell}^\top \QQ \uu}{\| \uu \|^2} \\
& \ge (1 - \breve{o}_{d,\P}(1))\min_{\bzero \neq \uu \in \cV_1'} \frac{\uu^\top\QQ^\top \bPsi_{\le \ell} \bLambda_{\le \ell}^2 \bPsi_{\le \ell}^\top \QQ \uu}{\| \bPsi_{\le \ell}^\top \QQ \uu \|^2 / n} \\
&\ge (1 - \breve{o}_{d,\P}(1)) \min_{\bzero \neq \xx \in \cV_1} \frac{\xx^\top \bLambda_{\le \ell}^2 \xx}{\| \xx \|^2 / n} \ge n(1-\breve{o}_{d,\P}(1)) \cdot \lambda_k(\bLambda_{\le \ell}^2).
\end{align*}
Similarly, 
\begin{align*}
\lambda_k(\QQ^\top \bPsi_{\le \ell} \bLambda_{\le \ell}^2 \bPsi_{\le \ell}^\top \QQ) 
&= \min_{\cV: \dim(\cV) = D-k+1} \max_{\uu \in \cV} \frac{\uu^\top\QQ^\top \bPsi_{\le \ell} \bLambda_{\le \ell}^2 \bPsi_{\le \ell}^\top \QQ \uu}{\| \uu \|^2}\\
& \le (1 + \breve{o}_{d,\P}(1))\max_{\uu \in \cV_2'} \frac{\uu^\top \QQ^\top \bPsi_{\le \ell} \bLambda_{\le \ell}^2 \bPsi_{\le \ell}^\top \QQ \uu}{\| \bPsi_{\le \ell}^\top \QQ \uu \|^2 / n} \\
&\le (1 + \breve{o}_{d,\P}(1))\max_{\xx \in \cV_2} \frac{\xx^\top \bLambda_{\le \ell}^2 \xx}{\| \xx \|^2 / n} \le n(1+\breve{o}_{d,\P}(1)) \cdot \lambda_k(\bLambda_{\le \ell}^2).
\end{align*}
Finally, we recall that $B(d,k) = (1+o_d(1)) \cdot d^k / k!$. This then leads to the claim \eqref{claim:lambda}.

In the equality $\bK - \gamma_{>\ell} \bI_n = \bPsi_{\le \ell} \bLambda_{\le \ell}^2 \bPsi_{\le \ell}^\top + \bDelta$, 
we view $\bDelta$ as the perturbation added to $\bPsi_{\le \ell} \bLambda_{\le \ell}^2 \bPsi_{\le \ell}^\top$. By 
Weyl's inequality, 
\begin{equation*}
\big\| \blambda(\bK - \gamma_{>\ell} \bI_n) - \blambda(\bPsi_{\le \ell} \bLambda_{\le \ell}^2 \bPsi_{\le \ell}^\top) \big\|_\infty \le \big\| \bDelta \big\|_\op = \breve{o}_{d,\P}(1),
\end{equation*}
 This proves \eqref{eq:Keig} about the eigenvalues of $\bK - \gamma_{>\ell} \bI_n$.



Finally, from the kernel invertibility result (i.e.\ Theorem~\ref{thm:invert2}), we have 
$(1 + \eta) \bK \succeq \bK_N \succeq (1-\eta) \bK$ for some $\eta = \breve{o}_{d,\P}(1)$. This implies
\begin{equation*}
(1+\eta) \lambda_k(\bK) \ge \lambda_k(\bK_N) \ge (1-\eta)\lambda_k(\bK).
\end{equation*}
So all eigenvalues of $\bK_N$ are up to a factor of $1+\breve{o}_{d,\P}(1)$ compared with those of $\bK$. This implies that the eigenvalue structure of $\bK_N$ is similar to that of $\bK$.
\end{proof}

Lemma \ref{lem:eigval} indicates that the eigenvalues of $\bK$ and $\bK_N$ exhibit a
group structure: roughly speaking, for every $k=0,1,\ldots,\ell$, there are $B(d,\ell)$ eigenvalues centered around $\gamma_{>\ell} + \gamma_k (k!)n/(d^k)$. It is convenient to partition the eigenvectors according to such group structure.

We define
\begin{equation*}
\begin{array}{llll}
\bU &= \big[ \underbrace{\bV_0^{(0)}}_{n\times 1}, \underbrace{\bV_0^{(1)}}_{n\times d}, \ldots, \underbrace{\bV_0^{(\ell)}}_{n\times B(d,\ell)} \big],  &\bU_N &= \big[ \underbrace{\bV^{(0)}}_{n\times 1}, \underbrace{\bV^{(1)}}_{n\times d}, \ldots, \underbrace{\bV^{(\ell)}}_{n\times B(d,\ell)} \big] \\
\bD &= \diag \big(\underbrace{D_0^{(0)}}_{1\times 1}, \underbrace{\bD_0^{(1)}}_{d\times d}, \ldots, \underbrace{\bD_0^{(\ell)}}_{B(d,\ell) \times B(d,\ell)}  \big),  &\bD_N &= \diag \big(\underbrace{D^{(0)}}_{1\times 1}, \underbrace{\bD^{(1)}}_{d\times d}, \ldots, \underbrace{\bD^{(\ell)}}_{B(d,\ell) \times B(d,\ell)}  \big)
\end{array}
\end{equation*}
We further define $\bV_0^{(\ell+1)}, \bV^{(\ell+1)} \in \R^{n \times (n-D)}$ and $\bD_0^{(\ell+1)}, \bD^{(\ell+1)} \in \R^{(n-D) \times (n-D)}$ such that they contain the remaining eigenvectors and eigenvalues of $\bK$, $\bK_N$, respectively. We also denote groups of eigenvectors/eigenvalues of $\bPsi_{\le \ell} \bLambda_{\le \ell}^2 \bPsi_{\le \ell}^\top + \gamma_{>\ell} \bI_n$ by 
\begin{equation}\label{decomp:mainmatrix}
\bU_s = \big[ \underbrace{\bV_s^{(0)}}_{n \times 1},\underbrace{\bV_s^{(1)}}_{n \times d},\ldots, \underbrace{\bV_s^{(\ell)}}_{n \times B(d,\ell)} \big], \qquad \bD_s = \diag \big(\underbrace{D_s^{(0)}}_{1\times 1}, \underbrace{\bD_s^{(1)}}_{d\times d}, \ldots, \underbrace{\bD_s^{(\ell)}}_{B(d,\ell) \times B(d,\ell)}  \big),
\end{equation}
and the remaining eigenvectors/eigenvalues are $\bV_s^{(\ell+1)}$ and $\bD_s^{(\ell+1)}$. The following is a useful result about the kernel eigenvector structure.

\begin{lem}[Kernel eigenvector structure]\label{lem:eigvec}
Let $k \neq k' \in \{ 0,1,\ldots,\ell+1\}$. Denote $\lambda_k = \gamma_{>\ell} + \gamma_k (k!) n/(d^k)$ 
for $k = 0, \ldots, \ell$ and $\lambda_{\ell+1} = \gamma_{>\ell} $.   
\begin{enumerate}
\item[(a)] Suppose that $\min\{\lambda_k/ \lambda_{k'}, \lambda_{k'} / \lambda_k \} \le 1/4$. Then, 
\begin{equation*}
\big\| (\bV_0^{(k')})^\top \bV^{(k)} \big\|_\op = \breve{o}_{d,\P}(1). 
\end{equation*} 
\item[(b)] Recall $\bDelta^{(\res)}$ defined in \eqref{eq:Keigdecomp}. Then, 
\begin{equation*}
\big\| (\bV_s^{(k')})^\top \bV_0^{(k)} \big\|_\op \le \frac{2 \| \bDelta^{(\res)}\|_\op}{\big|\lambda_k - \lambda_{k'}\big| - \breve{o}_{d,\P}(1) }.
\end{equation*}
\end{enumerate}
\end{lem}

\begin{proof}[{\bf Proof of Lemma~\ref{lem:eigvec}}]
First, we prove a useful claim, which is a consequence of classical eigenvector perturbation results \cite{davis1970sin}. We will give a short proof for completeness.
For a diagonal matrix $\bD$, we denote by $\cL(\bD)$ the smallest interval that covers all diagonal entries in $\bD$. If $\cL(\bD_0^{(k')}) \cap \cL(\bD^{(k)}) = \emptyset$, then 
\begin{equation}\label{eq:VV}
\big\| (\bV_0^{(k')})^\top  \bV^{(k)} \big\|_\op \le  \frac{\big\| (\bV_0^{(k')})^\top (\bK_N - \bK) \bV^{(k)} \big\|_\op}{\min_{i,j}\big| (\bD_0^{(k')})_{ii} - (\bD^{(k)})_{jj} \big| }.
\end{equation}
To prove this, we observe that 
\begin{equation*}
(\bK_N - \bK) \bV^{(k)} + \bK \bV^{(k)} = \bK_N \bV^{(k)} = \bV^{(k)} \bD^{(k)}.
\end{equation*}
Left-multiplying both sides by $(\bV_0^{(k')})^\top$ and re-arranging the equality, we obtain
\begin{align*}
-(\bV_0^{(k')})^\top (\bK_N - \bK) \bV^{(k)} &= (\bV_0^{(k')})^\top \bK \bV^{(k)} - (\bV_0^{(k')})^\top \bV^{(k)} \bD^{(k)} \\
&= \bD_0^{(k')}(\bV_0^{(k')})^\top \bV^{(k)} - (\bV_0^{(k')})^\top \bV^{(k)} \bD^{(k)}.
\end{align*}
Denote $\bR =  (\bV_0^{(k')})^\top \bV^{(k)}$ for simplicity. Without loss of generality, we assume $\min_i (\bD_0^{(k')})_{ii} > \max_j (\bD^{(k)})_{jj}$. Let $\uu, \vv$ be the top left/right singular vector of $\bR$. Then, 
\begin{align*}
\big\| \bD_0^{(k')} \bR - \bR \bD^{(k)} \big\|_\op &\ge \uu^\top \bD_0^{(k')} \bR \vv - \uu^\top \bR \bD^{(k)}  \vv = \| \bR \|_\op \cdot \uu^\top \bD_0^{(k')} \uu -  \| \bR \|_\op \vv^\top \bD^{(k)}  \vv \\
&\ge \| \bR \|_\op \cdot \min_i (\bD_0^{(k')})_{ii} - \| \bR \|_\op \cdot\max_j (\bD^{(k)})_{jj}.
\end{align*}
This proves the claim \eqref{eq:VV}. By the kernel eigenvalue structure (Lemma~\ref{lem:eigval}), 
\begin{align}
\min_{i,j} \big| (\bD_0^{(k')})_{ii} - (\bD^{(k)})_{jj} \big| & = \big| \gamma_{>\ell} + (1+\breve{o}_{d,\P}(1)) ( \lambda_k - \gamma_{>\ell}) - \gamma_{>\ell} - (1+\breve{o}_{d,\P}(1)) ( \lambda_{k'} - \gamma_{>\ell}) \big| + \breve{o}_{d,\P}(1) \notag \\
&=\big| (1+\breve{o}_{d,\P}(1)) \lambda_k - (1+\breve{o}_{d,\P}(1)) \lambda_{k'} \big| + \breve{o}_{d,\P}(1). \label{ineq:eigvec0}
\end{align}
By the assumptions on $\lambda_k$ and $\lambda_{k'}$, with very high probability, 
\begin{equation}
\min_{i,j} \big| (\bD_0^{(k')})_{ii} - (\bD^{(k)})_{jj} \big| \ge  \max\{\lambda_k, \lambda_{k'} \big\} / 2.\label{ineq:eigvec1}
\end{equation}
By the kernel invertibility, Theorem \ref{thm:invert2}, we have
\begin{align*}
\big\| (\bV_0^{(k')})^\top (\bK_N - \bK) \bV^{(k)} \big\|_\op &\le \big\| (\bV_0^{(k')})^\top \bK^{1/2} \big\|_\op \cdot \big\| \bK^{-1/2} (\bK_N - \bK) \bK^{-1/2} \big\|_\op \cdot \big\| \bK^{1/2} \bV^{(k)} \big\|_\op \\
&\le \breve{o}_{d,\P}(1) \cdot \big\| (\bV_0^{(k')})^\top \bK^{1/2} \big\|_\op \cdot \big\| \bK^{1/2} \bV^{(k)} \big\|_\op \\
&\le \breve{o}_{d,\P}(1) \cdot \big\| (\bV_0^{(k')})^\top \bK^{1/2} \big\|_\op \cdot \big\| \bK^{1/2} \bK_N^{-1/2} \big\|_\op \cdot \big\| \bK_N^{1/2} \bV^{(k)} \big\|_\op.
\end{align*} 
Since $\bK^{1/2}$ has the same eigenvectors as $\bK$, we have 
\begin{equation}\label{ineq:V0K}
\big\| (\bV_0^{(k')})^\top \bK^{1/2} \big\|_\op = \big\| \big[ \bD_0^{(k')}\big]^{1/2} (\bV_0^{(k')})^\top \big\|_\op 
\le (1 + \breve{o}_{d,\P}(1)) \cdot \sqrt{\lambda_{k'}}+ \breve{o}_{d,\P}(1)\, .
\end{equation}
Similarly, 
\begin{equation}\label{ineq:KV}
\big\| \bK_N^{1/2} \bV^{(k)} \big\|_\op = \big\| \bV^{(k)} (\bD^{(k)})^{1/2} \big\|_\op 
\le (1 + \breve{o}_{d,\P}(1)) \cdot \sqrt{\lambda_k}+ \breve{o}_{d,\P}(1)\,.
\end{equation} 
We also note that 
\begin{align*}
\big\| \bK^{1/2} \bK_N^{-1/2} \big\|_\op^2 = \frac{1}{ \big[ \sigma_{\min} \big( \bK_N^{1/2} \bK^{-1/2}) \big]^2 }  =  \frac{1}{\lambda_{\min}(\bK^{-1/2} \bK_N \bK^{-1/2}) } \le 1 + \breve{o}_{d,\P}(1).
\end{align*}
Therefore, we deduce 
\begin{equation}\label{ineq:eigvec2}
\big\| (\bV_0^{(k')})^\top (\bK_N - \bK) \bV^{(k)} \big\|_\op = \breve{o}_{d,\P}(1) \cdot \sqrt{\lambda_k \lambda_{k'}}.
\end{equation}
Combining \eqref{ineq:eigvec1} and \eqref{ineq:eigvec2}, we have shown that with very high probability,
\begin{equation*}
\big\| (\bV_0^{(k')})^\top  \bV^{(k)} \big\|_\op \le \frac{\breve{o}_{d,\P}(1) \cdot  \sqrt{\lambda_k \lambda_{k'}}}{\max\{ \lambda_k, \lambda_{k'}\}/ 2} \le \breve{o}_{d,\P}(1),
\end{equation*}
which proves $(a)$. Now in order to show $(b)$, we view $\bK$ as a perturbation of 
$\bPsi_{\le \ell} \bLambda_{\le \ell}^2 \bPsi_{\le \ell}^\top + \gamma_{>\ell} \bI_n$, cf. 
Eq.~\eqref{eq:Keigdecomp}. Hence,
\begin{equation}\label{eq:VV2}
\big\| (\bV_s^{(k')})^\top  \bV_0^{(k)} \big\|_\op \le  \frac{\big\| (\bV_s^{(k')})^\top \bDelta^{(\res)} \bV_0^{(k)} \big\|_\op}{\min_{i,j}\big| (\bD_s^{(k')})_{ii} - (\bD_0^{(k)})_{jj} \big| } \le \frac{\big\| \bDelta^{(\res)} \big\|_\op}{\min_{i,j}\big| (\bD_s^{(k')})_{ii} - (\bD_0^{(k)})_{jj} \big|}.
\end{equation}
(The first inequality is proved in the same way as \eqref{eq:VV}.) By Weyl's inequality, $\max_{j} | (\bD_0^{(k)})_{jj} - (\bD_s^{(k)})_{jj} | \le \big\| \bDelta^{(\res)} \big\|_\op$. Therefore, if $\min_{i,j}\big| (\bD_s^{(k')})_{ii} - (\bD_s^{(k)})_{jj} \big| \ge 2\| \bDelta^{(\res)} \|_\op$, then
\begin{align*}
\big\| (\bV_s^{(k')})^\top  \bV_0^{(k)} \big\|_\op &\le \frac{\big\| \bDelta^{(\res)} \big\|_\op}{\min_{i,j}\big| (\bD_s^{(k')})_{ii} - (\bD_s^{(k)})_{jj} \big| - \big\| \bDelta^{(\res)} \big\|_\op} \\
& \le \frac{\big\| \bDelta^{(\res)} \big\|_\op}{\min_{i,j}\big| (\bD_s^{(k')})_{ii} - (\bD_s^{(k)})_{jj} \big| - \min_{i,j}\big| (\bD_s^{(k')})_{ii} - (\bD_s^{(k)})_{jj} \big|/2} \\
&= \frac{2\big\| \bDelta^{(\res)} \big\|_\op}{\min_{i,j}\big| (\bD_s^{(k')})_{ii} - (\bD_s^{(k)})_{jj} \big|}.
\end{align*}
If $\min_{i,j}\big| (\bD_s^{(k')})_{ii} - (\bD_s^{(k)})_{jj} \big| < 2\| \bDelta^{(\res)} \|_\op$, then from a trivial bound we have 
\begin{align*}
\big\| (\bV_s^{(k')})^\top  \bV_0^{(k)} \big\|_\op& \le 1 \le \frac{2\big\| \bDelta^{(\res)} \big\|_\op}{\min_{i,j}\big| (\bD_s^{(k')})_{ii} - (\bD_s^{(k)})_{jj} \big|}.
\end{align*}
as well. In either way, 
\begin{equation*}
\big\| (\bV_s^{(k')})^\top  \bV_0^{(k)} \big\|_\op  \le \frac{2\big\| \bDelta^{(\res)} \big\|_\op}{\min_{i,j}\big| (\bD_s^{(k')})_{ii} - (\bD_s^{(k)})_{jj} \big|} \le \frac{2\big\| \bDelta^{(\res)} \big\|_\op}{\big| \lambda_k - \lambda_{k'} \big| - \breve{o}_{d,\P}(1)}
\end{equation*}
\end{proof}

We state and prove two useful lemmas. Let $\ee_1,\ldots,\ee_n \in \R^n$ be the canonical basis. 
For an integer $m \le n$, define the projection matrix $\bP = [\ee_1,\ldots,\ee_m]
[\ee_1,\ldots,\ee_m]^\top$. 
\begin{lem}\label{lem:rowselect}
Suppose that $c' n \le m \le C' n$, where $c', C' \in (0,1)$ are constants. Then there exists $c>0$ such that with very high probability,
\begin{equation*}
\sigma_{\min}( \bP \bU_s ) \ge c.
\end{equation*}
\end{lem}
\begin{proof}[{\bf Proof of Lemma~\ref{lem:rowselect}}]
Let the singular value decomposition of $\bPsi_{\le \ell}$ be $\bPsi_{\le \ell} = \bU_{\bPsi} \bD_{\bPsi} \bV_{\bPsi}^\top$ where $\bU_{\bPsi} \in \R^{n \times D}, \bD_{\bPsi} \in \R^{D \times D}, \bV_{\bPsi} \in \R^{D \times D}$. By the concentration result \eqref{eq:psi-conctr} we have $\sigma_{\max}(\bD_{\bPsi}) \le (1 + \breve{o}_{d,\P}(1))\sqrt{n}$. By definition, $\bPsi_{\le \ell} \bV_{\bPsi} = \bU_{\bPsi} \bD_{\bPsi}$. Left-multiplying both sides by $\bP$ and rearranging, we have
\begin{equation*}
\bP \bU_{\bPsi} = \bP \bPsi_{\le \ell} \bV_{\bPsi}  \bD_{\bPsi}^{-1}.
\end{equation*}
We use the concentration result \eqref{eq:psi-conctr} to $\bP \bPsi_{\le \ell}$ (where we view $m$ as the dimension) and get $\sigma_{\min} (\bP \bPsi_{\le \ell}) \ge (1 - \breve{o}_{d,\P}(1))\sqrt{m}$. This then leads to
\begin{align}\label{ineq:PU}
\sigma_{\min}(\bP \bU_{\bPsi}) \ge \frac{\sigma_{\min}(\bP \bPsi_{\le \ell} \bV_{\bPsi}) }{\sigma_{\max}(\bD_{\bPsi})} = \frac{\sigma_{\min}(\bP \bPsi_{\le \ell}) } {\sigma_{\max}(\bD_{\bPsi})} \ge (1 - \breve{o}_{d,\P}(1))\sqrt{\frac{m}{n}}.
\end{align}
So with very high probability, $\sigma_{\min}(\bP \bU_{\bPsi}) \ge \sqrt{c'}/2$. Since $\bPsi_{\le \ell}$ and $\bPsi_{\le \ell} \bLambda_{\le \ell}^2 \bPsi_{\le \ell}^\top$ span the same column space, we can find an orthogonal matrix $\bR \in \R^{D \times D}$ such that $\bU_s = \bU_{\bPsi} \bR$. Thus, $\sigma_{\min}(\bP \bU_s) = \sigma_{\min}(\bP \bU_{\bPsi}) \ge \sqrt{c'}/2$ with very high probability.
\end{proof}

\begin{lem}\label{lem:simpleproj}
Suppose that $\bU_a, \bU_b \in \R^{m \times k}$ are two matrices with orthonormal column vectors. Let $\bP \in \R^{m \times m}$ be any orthogonal  projection matrix, and $\bP_a^\perp = \bI_m - \bU_a \bU_a^\top$. Then, we have
\begin{equation*}
\sigma_{\min}(\bP \bU_b) \ge \sigma_{\min}(\bP \bU_a) \sigma_{\min} (\bU_a^\top \bU_b) - \sigma_{\max}(\bP_a^\perp \bU_b)
\end{equation*}
\end{lem}
\begin{proof}[{\bf Proof of Lemma~\ref{lem:simpleproj}}]
By the variational characterization of singular values and the triangle inequality,
\begin{align*}
\sigma_{\min}( \bP  \bU_b ) = \min_{\| \vv \|=1} \big\| \bP  \bU_b \vv \big\| &\ge \min_{\|\vv \|=1} \Big\{ \big\| \bP  \bU_a \bU_a^\top \bU_b \vv \big\| - \big\| \bP  \bP_a^\bot  \bU_b \vv \big\| \Big\} \\
&\ge \sigma_{\min} ( \bP  \bU_a) \min_{\| \vv \|=1}\big\|  \bU_a^\top  \bU_b \vv \big\| - \max_{\| \vv \|=1} \big\|  \bP_a^\bot  \bU_b \vv \big\| \\
&\ge \sigma_{\min} ( \bP  \bU_a) \sigma_{\min}( \bU_a^\top  \bU_b) - \sigma_{\max}( \bP_a^\bot  \bU_b)\, .
\end{align*}
This proves the lemma.
\end{proof}

Suppose $n'$ is a positive integer that satisfies $n \le n' \le (1+C_2) n$ where $C_2>0$ is a constant. Denote $n_0 = n+n'$. Let us sample $n'$ new data $\xx_{n+1},\ldots,\xx_{n_0}$ which are i.i.d.~and have the same distribution as $\xx_1$. We introduce the augmented kernel matrix $\tilde \bK \in \R^{n_0\times n_0}$ as follows. We define
\begin{equation}\label{def:aug}
 [\tilde\bK]_{ij} := \bK_N(\xx_i, \xx_j) = \left( \begin{array}{cc}
\tilde \bK_{11} & \tilde \bK_{12} \\ \tilde \bK_{21} & \tilde \bK_{22} \end{array} \right).
\end{equation}
where $\tilde \bK_{11},  \tilde \bK_{12}, \tilde \bK_{21},  \tilde \bK_{22}$ have size $n\times n$, $n \times n'$, $n' \times n$, $n' \times n'$ 
respectively. Under this definition, clearly $\tilde \bK_{11} = \bK_N$. Note that we can express $\bK_N^{(2)}$ as
(recalling  that $\bK_N(\cdot, \xx)= [\bK_N(\xx_i, \xx)]_{i\le n}\in\reals^n$):
\begin{equation*}
\bK_N^{(2)} = \frac{1}{n'} \sum_{i=1}^{n'} \E_{\xx_{n+1},\ldots,\xx_{n+n'}} \big[ \bK_N(\cdot, \xx_{n+i}) \bK_N(\cdot, \xx_{n+i})^\top \big] 
= \frac{1}{n'}\E_{\xx_{n+1},\ldots,\xx_{n_0}} \big[ \tilde \bK_{12} \tilde \bK_{21}^\top \big]. 
\end{equation*}
So we can reduce our problem using $\tilde \bK$.
\begin{align*}
\big\| \bK_N^{-1} \bK_N^{(2)} \bK_N^{-1} \big\|_\op & \stackrel{(i)}{\le} \frac{1}{n'} \E_{\xx_{n+1},\ldots,\xx_{n_0}}  \big\| (\tilde \bK_{11})^{-1} \tilde \bK_{12} \tilde \bK_{21} (\tilde \bK_{11})^{-1} \big\|_\op \\
& \stackrel{(ii)}{=} \frac{1}{n'} \E_{\xx_{n+1},\ldots,\xx_{n_0}}  \big\| (\tilde \bK_{11})^{-1} (\tilde \bK^2)_{11} (\tilde \bK_{11})^{-1} - \bI_{n} \big\|_\op
\end{align*}
where \textit{(i)} follows from Jensen's inequality, and \textit{(ii)} is due to 
\begin{equation*}
(\tilde \bK^2)_{11} = (\tilde \bK_{11})^2 + \tilde \bK_{12} \tilde \bK_{21}
\end{equation*}
where $(\tilde \bK^2)_{11}$ denotes the top left $n \times n$ matrix block of $\tilde \bK^2$. Similarly, we define the augmented vector 
\begin{equation}\label{def:tildef}
\tilde \ff = (f(\xx_i))_{i \le n_0} = \left( \begin{array}{c} \tilde \ff_1 \\ \tilde \ff_2 \end{array} \right),
\end{equation}
and by Jensen's inequality, we have
\begin{align*}
\big\| \bK_N^{-1} \E_\xx \big[ \bK_N(\cdot, \xx) f(\xx) \big] \big\| &\le \frac{1}{n'} \E_{\xx_{n+1},\ldots,\xx_{n_0}} \big\| \tilde \bK_{11}^{-1} \tilde \bK_{12} \tilde \ff_2 \big\|.
\end{align*}
Hence, proving Eqs.~\eqref{bnd:key1} and~\eqref{bnd:key2} is reduced to studying $(\tilde \bK_{11})^{-1} (\tilde \bK^2)_{11} (\tilde \bK_{11})^{-1}$ and $\tilde \bK_{11}^{-1} \tilde \bK_{12} \tilde \ff_2$, which can be analyzed by making use of the kernel eigenstructure. 

\begin{lem}\label{lem:reduceeig}
There exist constant $C_1,C_2>0$ such that the following holds. With very high probability, 
\begin{align}
&\big\| (\tilde \bK_{11})^{-1} (\tilde \bK^2)_{11} (\tilde \bK_{11})^{-1} - \bI_{n} \big\|_\op \le C_1, \label{ineq:3K1} \\
&\big\| \tilde \bK_{11}^{-1} \tilde \bK_{12} \tilde \ff_2 \big\| \le C_1 \sqrt{n}.\label{ineq:3K2}
\end{align}
Further, with high probability, 
\begin{align}
& \frac{1}{n'} \E_{\xx_{n+1},\ldots,\xx_{2n}}  \big\| (\tilde \bK_{11})^{-1} (\tilde \bK^2)_{11} (\tilde \bK_{11})^{-1} - \bI_{n} \big\|_\op \le \frac{C_2}{n}, \label{ineq:3K3} \\
&\frac{1}{n'} \E_{\xx_{n+1},\ldots,\xx_{2n}} \big\| \tilde \bK_{11}^{-1} \tilde \bK_{12} \tilde \ff_2 \big\| \le \frac{C_2}{\sqrt{n}}. \label{ineq:3K4}
\end{align}
Consequently, we obtain \eqref{bnd:key1} and \eqref{bnd:key2} in Lemma~\ref{lem:keybnd}.
\end{lem}
\begin{proof}[{\bf Proof of Lemma~\ref{lem:reduceeig}}]
\textbf{Step 1.} First we prove \eqref{ineq:3K1} and \eqref{ineq:3K2}. Fix the constant 
$\delta_0 := 1/4$. We apply Lemma~\ref{lem:eigval} to the augmented matrix $\tilde \bK$ 
(where $n$ is replaced by $n_0$) and find that its eigenvalues can be partitioned into $\ell+2$ groups,
with the first $\ell+1$ groups corresponding to $D$ eigenvalues. 
These groups, together with the group of the remaining eigenvalues, are centered around
\begin{equation}\label{groups}
\gamma_0 n_0 + \gamma_{> \ell}, \frac{\gamma_1 n_0}{d} + \gamma_{> \ell}, \ldots, \frac{\gamma_1 (\ell !) n_0}{d^\ell} + \gamma_{> \ell}, \gamma_{> \ell}.
\end{equation} 
These $\ell+2$ values may be unordered and have small gaps (so that Lemma~\ref{lem:eigvec} does not directly apply). 
So we order them in the descending order and denote the ordered values by $\lambda_{(1)}, \ldots, \lambda_{(\ell+2)}$. Let their corresponding group sizes be $s_1,s_2,\ldots,s_{\ell+2} \in \{1,B(d,1),\ldots,B(d,\ell), n_0-D \}$.  Let $\ell_0 \in \{1,\ldots,\ell+1 \}$ be the largest index $k$ such that $\lambda_{(k+1)} / \lambda_{(k)} < \delta_0$. (If such $k$ does not exist, then $\| \tilde \bK \|_\op \le 2\gamma_{>\ell} (1/\delta_0)^{\ell+1}$ with very high probability, and thus \eqref{ineq:3K1} and \eqref{ineq:3K2} follow.) Let the eigen-decomposition of $\tilde \bK$ be
\begin{align*}
\tilde \bK &= \tilde \bU \tilde \bD \tilde \bU^\top  + \tilde \bK_1^{(\res)}
\end{align*}
so that we only keep eigenvalues/eigenvectors in the largest $\ell_0$ groups in $\tilde \bU \tilde \bD \tilde \bU^\top$,
 and the remaining eigenvalues are in $\tilde \bK_1^{(\res)}$ 
 (in particular $\tilde \bK_1^{(\res)}\tilde \bU = \bzero$). Therefore 
\begin{equation*}
\tilde \bD = \diag\big( \lambda_{(1)} \bI_{s_1}, \ldots, \lambda_{(\ell_0)} \bI_{s_{\ell_0}}\big) \cdot (1+\breve{o}_{d,\P}(1)).
\end{equation*}
and $\tilde \bK_1^{(\res)}$ satisfies 
\begin{align}\label{eq:BoundK1Res}
(\gamma_{>\ell}-\breve{o}_{d,\P}(1)) \bI_{n_0} \preceq \tilde \bK_1^{(\res)} \preceq (\gamma_{>\ell}+C + \breve{o}_{d,\P}(1)) \bI_{n_0}
\end{align}
(where $C>0$ is a constant). Taking the square, we have
\begin{align*}
\tilde \bK^2 &= \tilde \bU \tilde \bD^2 \tilde \bU^\top + \tilde \bK_2^{(\res)}, 
\end{align*}
where $\tilde \bK_2^{(\res)}$ satisfies $ (\gamma_{>\ell}^2-\breve{o}_{d,\P}(1)) \bI_{2n} \preceq \tilde \bK_2^{(\res)} \preceq ((\gamma_{>\ell}+C)^2 + \breve{o}_{d,\P}(1)) \bI_{2n}$. 

For convenience, we denote $D_0$ to be the size of $\tilde\bD$ (i.e., number of eigenvalues in the main part of $\tilde \bK$).  
Also denote by $\tilde \bU_1 \in \R^{n \times D}, \tilde \bU_2 \in \R^{n' \times D}$ the submatrix of $\tilde \bU$ formed by its top-$n$/bottom-$n'$ rows respectively. Using this eigen-decomposition, we have
\begin{align}
\big\| (\tilde \bK_{11})^{-1} (\tilde \bK^2)_{11} (\tilde \bK_{11})^{-1} - \bI_{n} \big\|_\op &\le \big\| (\tilde \bK_{11})^{-1} [\tilde \bU \tilde\bD^2 \tilde \bU^\top]_{11} (\tilde \bK_{11})^{-1} \big\|_\op \label{ineq:tildeK2-1} \\
&~~+  \big\| (\tilde \bK_{11})^{-1} \big\|_\op \cdot \big\| \tilde \bK_2^{(\res)} \big\|_\op \cdot \big\| (\tilde \bK_{11})^{-1} \big\|_\op + 1. \label{ineq:tildeK2-2}
\end{align}
Since $\tilde \bK_{11}$ (or equivalently $\bK_N$) has smallest eigenvalue bounded away from zero with very high probability by Theorem~\ref{thm:invert2}, we can see that the two terms in the second line \eqref{ineq:tildeK2-2} are bounded with very high probability. To handle the term on the right-hand side of \eqref{ineq:tildeK2-1}, we use the identity
\begin{align}
\bA^{-1} \bH \bA^{-1} & = \bA_0^{-1} \bH \bA_0^{-1} - \bA^{-1} (\bA - \bA_0) \bA_0^{-1} \bH \bA_0^{-1} -  \bA_0^{-1} \bH  \bA_0^{-1} (\bA - \bA_0) \bA^{-1} \notag \\
& ~~~+ \bA^{-1} (\bA - \bA_0) \bA_0^{-1} \bH  \bA_0^{-1} (\bA - \bA_0) \bA^{-1} \label{eq:magic}
\end{align}
where $\bA_0, \bA$ are nonsingular symmetric matrices and $\bH$ is a symmetric matrix. 
Define $\tilde \bD_\gamma = \tilde \bD - \gamma_{>\ell} \bI_D$, which is 
strictly positive definite with very high probability (since by definition, $\lambda_{(\ell_0)} \ge  \delta_0^{-1} \gamma_{>\ell}$). We set $\bA = \tilde \bK_{11}, \bA_0 = [\tilde \bU \tilde\bD_\lambda 
\tilde \bU^\top]_{11} + \gamma_{>\ell} \bI_{n}$, and $\bH = [\tilde \bU \tilde \bD^2 \tilde 
\bU^\top]_{11}$, and it holds that $\| \bA^{-1} \|_\op \le C$, $\| \bA_0^{-1} \|_\op \le C$ 
and $\| \bA - \bA_0 \|_\op \le C$ with very high probability by the above. 
Therefore, with very high probability,
\begin{align}
&~~~~ \big\| (\tilde \bK_{11})^{-1} [\tilde \bU \tilde \bD^2 \tilde \bU^\top]_{11} (\tilde \bK_{11})^{-1} \big\|_\op \notag \\
&\le C \cdot \Big\| \big([\tilde \bU \tilde\bD_\gamma \tilde \bU^\top]_{11} + \gamma_{>\ell} \bI_{n} \big)^{-1} \big[\tilde \bU \tilde\bD^2 \tilde \bU^\top \big]_{11} \big ( [\tilde \bU \tilde\bD_\gamma \tilde \bU^\top]_{11} + \gamma_{>\ell} \bI_{n} \big)^{-1} \Big\|_\op \notag \\
&= C \cdot \Big\| \big( \tilde \bU_1 \tilde\bD_\gamma \tilde \bU_1^\top  + \gamma_{>\ell} \bI_{n} \big)^{-1} \tilde \bU_1 \tilde\bD^2 \tilde \bU_1^\top \big( \tilde \bU_1 \tilde\bD_\gamma \tilde \bU_1^\top  + \gamma_{>\ell} \bI_{n} \big)^{-1} \Big\|_\op \notag \\
&\stackrel{(i)}{\le} C \cdot \Big\| \big( \tilde \bU_1 \tilde\bD_\gamma \tilde \bU_1^\top  + \gamma_{>\ell} \bI_{n} \big)^{-1} \tilde \bU_1 (\gamma_{>\ell}^2 \bI_D)  \tilde \bU_1^\top \big( \tilde \bU_1 \tilde\bD_\gamma \tilde \bU_1^\top  + \gamma_{>\ell} \bI_{n} \big)^{-1} \Big\|_\op \notag \\
& + C \cdot \Big\| \big( \tilde \bU_1 \tilde\bD_\gamma \tilde \bU_1^\top  + \gamma_{>\ell} \bI_{n} \big)^{-1} \tilde \bU_1 \tilde\bD_\gamma^2 \tilde \bU_1^\top \big( \tilde \bU_1 \tilde\bD_\gamma \tilde \bU_1^\top  + \gamma_{>\ell} \bI_{n} \big)^{-1} \Big\|_\op \notag\\ 
&\stackrel{(ii)}{=} C + C \cdot \Big\| \tilde \bU_1\big( \tilde \bU_1^\top  \tilde \bU_1  + \gamma_{>\ell} (\tilde \bD_\gamma)^{-1} \big)^{-2}\tilde \bU_1^\top \Big\|_\op \notag
\\
&\stackrel{(iii)}{\le} C + C \cdot \big[ \lambda_{\min}(\tilde \bU_1^\top \tilde \bU_1) \big]^{-1} \label{lambdaminUU}
\end{align} 
where \textit{(i)} is because $\tilde \bD^2 \preceq 2\gamma_{>\ell}^2 \bI_D + 2 (\tilde \bD_\gamma)^2$, 
\textit{(ii)} is because by the identity \eqref{eq:algebra} in Lemma~\ref{lem:algebra} we have
\begin{align*}
\big( \tilde \bU_1 \tilde\bD_{\gamma} \tilde \bU_1^\top  + \gamma_{>\ell} \bI_{n} \big)^{-1} \tilde \bU_1 \tilde\bD_{\gamma} 
&= \tilde \bU_1(\tilde\bD_{\gamma})^{1/2}\big(  \tilde\bD_{\gamma}^{1/2} \tilde \bU_1^\top\tilde \bU_1\tilde\bD_{\gamma}^{1/2}  + \gamma_{>\ell} \bI_{D_0} \big)^{-1} \tilde\bD_{\gamma}^{1/2} \\
&=  \tilde \bU_1\big( \tilde \bU_1^\top  \tilde \bU_1  + \gamma_{>\ell} \tilde \bD_{\gamma}^{-1} \big)^{-1};
\end{align*}
and \textit{(iii)} is due to $\| \tilde \bU_1 \|_\op \le \| \tilde \bU \|_\op  =1$. 

Similarly, we have
\begin{align*}
\big\| \tilde \bK_{11}^{-1} \tilde \bK_{12} \tilde \ff_2 \big\| &\le \big\| \tilde \bK_{11}^{-1} \big[ \tilde \bU \tilde\bD \tilde \bU^\top \big]_{12} \tilde \ff_2 \big\| + \big\| \tilde \bK_{11}^{-1} (\tilde \bK_1^{(\res)})_{12} \tilde \ff_2 \big\| \\
&\le  \big\| \tilde \bK_{11}^{-1} \tilde \bU_1 \tilde\bD \tilde \bU_2^\top \tilde \ff_2 \big\| + \big\| \tilde \bK_{11}^{-1}  \big\|_\op \cdot \big\| \tilde \bK_1^{(\res)} \big\|_\op \cdot \big\| \tilde \ff_2 \big\|.
\end{align*}
It is clear that the second term on the last line is bounded by $C \| \tilde \ff_2 \|$ with very high probability,
by Eq.~\eqref{eq:BoundK1Res}. For the first term, with very high probability,
\begin{align}
\big\| \tilde \bK_{11}^{-1} \big[ \tilde \bU \tilde \bD \tilde \bU^\top \big]_{12} \tilde \ff_2 \big\| &\le  C \cdot \Big\| \big( \tilde \bU_1 \tilde \bD_\gamma \tilde \bU_1^\top  + \gamma_{>\ell} \bI_{n} \big)^{-1} \tilde \bU_1 (\gamma_{>\ell} \bI_D) \tilde \bU_2^\top \tilde \ff_2 \Big\|  \\ 
&+ C \cdot \Big\| \big( \tilde \bU_1 \tilde \bD_\gamma \tilde \bU_1^\top  + \gamma_{>\ell} \bI_{n} \big)^{-1} \tilde \bU_1 \tilde \bD_\gamma \tilde \bU_2^\top \tilde \ff_2 \Big\| \notag \\
&\le C\| \tilde \ff_2 \| + C \cdot \Big\| \tilde \bU_1 \big( \tilde \bU_1^\top \tilde \bU_1 + \gamma_{>\ell} (\tilde \bD_\gamma)^{-1} \big)^{-1} \tilde \bU_2^\top \tilde \ff_2 \Big\| \notag\\
&\le C\| \tilde \ff_2 \|+ C \cdot \Big\| \big( \tilde \bU_1^\top \tilde \bU_1 + \gamma_{>\ell} (\tilde \bD_\gamma)^{-1} \big)^{-1} \Big\|_\op \cdot \| \tilde \ff_2 \| \notag\\
&\le C\| \tilde \ff_2 \|+C \cdot \big[ \lambda_{\min}(\tilde \bU_1^\top \tilde \bU_1) \big]^{-1} \cdot \| \tilde \ff_2 \|. \label{lambdaminUU2}
\end{align}
By $\| f \|_{L^2} < C$ and the law of large numbers, $\| \tilde \ff_2 \| \le C\sqrt{n}$ holds with very high probability, so \eqref{lambdaminUU} 
and \eqref{lambdaminUU2} indicate that we only need to prove that 
$\lambda_{\min}(\tilde \bU_1^\top \tilde \bU_1) \ge c_0$.

\textbf{Step 2.} We next prove a the claimed lower bound on $\lambda_{\min}(\tilde \bU_1^\top \tilde \bU_1)$.
We will make use of the augmented infinite-width kernel
\begin{equation*}
[\tilde\bK_0]_{ij} := K(\xx_i, \xx_j), \qquad \where ~ i,j \le n_0.
\end{equation*}
Notice that the eigen-decomposition of  $\tilde\bK_0$ takes the same form as the one 
of $\bK$ as given in Lemma \ref{lem:eigval} (with the change that $n$ should be replaced by $n_0$). 
We denote by $\tilde \bU_0 \in \R^{n_0 \times D_0}$ the groups of eigenvectors that 
correspond to $\bU$ in that lemma and 
have the same eigenvalue structure as $\tilde \bU$.
Namely, in the eigen-decomposition of $\tilde\bK_0$, we keep eigenvalues from the largest 
$\ell_0$ groups in \eqref{groups} and let columns of $\tilde \bU_0$ be the corresponding 
eigenvectors. We can express $\tilde\bU_0$ and $\tilde \bU$ as $\ell_0$ groups of eigenvectors.
\begin{equation}\label{decomp:Utilde0}
\tilde \bU_0 = \big[ \tilde \bV_0^{(1)}, \tilde \bV_0^{(2)} , \ldots,  \tilde \bV_0^{(\ell_0)}  \big], \qquad \tilde \bU = \big[  \tilde \bV^{(1)} ,  \tilde \bV^{(2)} , \ldots,  \tilde \bV^{(\ell_0)}  \big].
\end{equation}
The remaining $\ell+2-\ell_0$ groups of eigenvectors are denoted by $\tilde \bV_0^{(\ell_0+1)}, \ldots, \tilde \bV_0^{(\ell + 2)}$, and $\tilde \bV^{(\ell_0+1)}, \ldots, \tilde \bV^{(\ell + 2)}$.  Note that
\begin{equation*}
\ \tilde \bU_1^\top \tilde \bU_1 = \tilde \bU^\top \bP_1 \tilde \bU , \qquad \where~\bP_1 := \left( \begin{array}{cc} \bI_n & \bzero \\ \bzero & \bzero \end{array} \right).
\end{equation*}
Define $\tilde \bP_0^\bot = \bI_{n_0} - \tilde \bU_0 \tilde \bU_0^\top$. We apply Lemma~\ref{lem:simpleproj} where we set $\bU_a = \tilde \bU_0$, $\bU_b = \tilde \bU$, $\bP = \tilde \bP_1$, and $\bP_a^\perp = \tilde \bP_0^\perp$. This yields
\begin{equation}\label{ineq:UPU}
\Big[ \lambda_{\min}(\tilde \bU^\top \bP_1 \tilde \bU) \Big]^{1/2} = \sigma_{\min}(\bP_1 \tilde \bU) \ge \sigma_{\min} ( \bP_1 \tilde \bU_0) \sigma_{\min}(\tilde \bU_0^\top \tilde \bU) - \sigma_{\max}(\tilde \bP_0^\bot \tilde \bU).
\end{equation} 
Now we invoke Lemma~\ref{lem:eigvec} (a) for the pair $(k,k')$ where $k \le \ell_0$ and $k' > \ell_0$. By the definition of $\ell_0$, the assumption of Lemma~\ref{lem:eigvec} (a) is satisfied, so
$\| (\tilde \bV_0^{(k')})^\top \tilde \bV^{(k)} \|_\op \le \breve{o}_{d,\P}(1)$. Thus,
\begin{align*}
\sigma_{\max}(\tilde \bP_0^\bot \tilde \bU) &= \big\| \sum_{k' > \ell_0} \tilde \bV_0^{(k')}(\tilde \bV_0^{(k')})^\top \tilde \bU \big\|_\op \le  \sum_{k' > \ell_0} \big\| (\tilde \bV_0^{(k')})^\top \tilde \bU \big\|_\op \\
& \le \sum_{k' > \ell_0} \sum_{k \le \ell_0}   \big\| (\tilde \bV_0^{(k')})^\top \tilde \bV^{(k)} \big\|_\op \le (\ell+2)^2 \cdot \breve{o}_{d,\P}(1) = \breve{o}_{d,\P}(1).
\end{align*}
Moreover, we have
\begin{align*}
\sigma_{\min}(\tilde \bU_0^\top \tilde \bU) &= \min_{\| \vv \|=1} \big\| \tilde \bU_0^\top \tilde \bU \vv \big\| \stackrel{(i)}{=} \min_{\| \vv \|=1} \big\| \tilde \bU_0 \tilde \bU_0^\top \tilde \bU \vv  \big\| = \min_{\| \vv \|=1} \big\| \tilde \bU \vv - \tilde \bP_0^\bot \tilde \bU \vv \big\| \\
& \ge \min_{\| \vv \|=1} \big\| \tilde \bU \vv \big\| - \max_{\| \vv \|=1} \big\| \tilde \bP_0^\bot \tilde \bU \vv \big\| \stackrel{(ii)}{=} 1 -  \max_{\| \vv \|=1}  \big\| \tilde \bP_0^\bot \tilde \bU \vv \big\| \\
&\ge 1 - \big\| \tilde \bP_0^\bot \tilde \bU \big\|_\op \ge 1 - \breve{o}_{d,\P}(1),
\end{align*}
where $(i)$ and $(ii)$ uses the fact that $\| \tilde \bU_0 \aa \| = \| \aa \|$ (where $\aa$ is a vector) since $\bU_0$ has orthonormal columns.

Finally, we claim that 
\begin{align}\label{eq:ClaimSigmaMinPU}
\sigma_{\min} ( \bP_1 \tilde \bU_0) \ge c\, ,
\end{align}
 with very high probability 
where $c>0$ is a small constant. We defer the proof of this claim to Step 4. 
Once it is proved, from Eqs.~\eqref{ineq:UPU}, we deduce that $\lambda_{\min}(\tilde \bU_1^\top \tilde \bU_1) \ge c'$. From \eqref{lambdaminUU} and \eqref{lambdaminUU2}, we conclude that, with very high probability,
\begin{align*}
&\big\| (\tilde \bK_{11})^{-1} [\tilde \bU \tilde \bD^2 \tilde \bU^\top]_{11} (\tilde \bK_{11})^{-1} \big\|_\op \le C, \\
&\big\| \tilde \bK_{11}^{-1} \big[ \tilde \bU \tilde \bD \tilde \bU^\top \big]_{12} \tilde \ff_2 \big\|_\op \le C,
\end{align*}
and thus we have proved \eqref{ineq:3K1} and \eqref{ineq:3K2}.

\textbf{Step 3.} We only prove \eqref{ineq:3K3} since \eqref{ineq:3K4} is similar.
 For simplicity, we denote $Z =  \big\| (\tilde \bK_{11})^{-1}$ $(\tilde \bK^2)_{11} (\tilde \bK_{11})^{-1} - \bI_{n} \big\|_\op$. 
 We know by Theorem~\ref{thm:invert2} that there exists a constant $c'>0$ such that 
 $\lambda_{\min} (\tilde \bK_{11}) \ge c'$ with very high probability. 
 Let $C=C_1$ be the same constant as in Eq.~\eqref{ineq:3K1}. For any $\beta > 0$, we also notice that the following event happens with very high probability
\begin{equation} \label{ineq:Pwvhp}
\P_{\xx_{n+1},\ldots, \xx_{n+n''}}(Z > C) > d^{-\beta}.
\end{equation}
(Note that the left-hand side is a random variable depending on $\xx_1,\ldots,\xx_n$ and 
$\ww_1,\ldots,\ww_N$.) Indeed, for any $\beta' > 0$, By Markov's inequality, 
\begin{align*}
d^{\beta'}\, \P \Big( \P_{\xx_{n+1},\ldots, \xx_{2n}}(Z > C)  > d^{-\beta} \Big) &\le d^{\beta + \beta'} \E \Big[ \P_{\xx_{n+1},\ldots, \xx_{2n}}(Z > C)  \Big] \\
&= d^{\beta + \beta'} \P (Z > C) \xrightarrow{d \to \infty} 0.
\end{align*}
Thus, we proved that \eqref{ineq:Pwvhp} happens with high probability. 
Let $\cA$ be the event such that both  
$\lambda_{\min} (\tilde \bK_{11}) \ge c'$ and \eqref{ineq:Pwvhp} happen. By the assumption on 
$\sigma'$, namely Assumption~\ref{ass:Sigma}, on the event $\lambda_{\min} (\tilde \bK_{11}) \ge c'$, a deterministic (and naive) bound on $\| \tilde \bK \|_{\max}$ is $O(d^{2B})$. So on $\cA$, we get 
\begin{equation*}
Z \le 1 + \big\| (\tilde \bK_{11})^{-1} \big\|_\op^2 \cdot \big\|  \tilde \bK \big\|_\op^2 \le C d^{2B}.
\end{equation*}
Thus, on the event $\cA$, 
\begin{align*}
\E_{\xx_{n+1},\ldots, \xx_{n+n'}} \big[ Z \big] &\le\E_{\xx_{n+1},\ldots, \xx_{n+n'}} \big[ Z; Z\le C \big] + \E_{\xx_{n+1},\ldots, \xx_{n+n'}} \big[ Z ; Z>C\big]  \\
&\le C +  C d^{2B} \cdot \P_{\xx_{n+1},\ldots, \xx_{n+n'}} \big( Z > C \big).
\end{align*} 
Therefore,  $\E_{\xx_{n+1},\ldots, \xx_{n+n'}} \big[ Z \big] \le C'$ for some $C' > C$.

\textbf{Step 4.} We now prove the claim \eqref{eq:ClaimSigmaMinPU},
namely $\sigma_{\min}(\bP_1 \tilde \bU_0) \ge c$
which was used in Step 2. We use a similar strategy as in Step 2.  Let $\tilde \bPsi_{\le \ell} \in \R^{n_0 \times D}$ contain the normalized spherical harmonics evaluated on $\xx_1,\ldots,\xx_{n_0}$. Similar to \eqref{decomp:Utilde0}, in the eigen-decomposition of $\tilde \bPsi_{\le \ell}  \bLambda_{\le \ell}^2 \tilde \bPsi_{\le \ell}^\top + \gamma_{>\ell} \bI_{n_0}$, we keep eigenvalues from the largest $\ell_0$ groups in \eqref{groups} and let $\tilde \bD_s$ be those eigenvalues and columns of $\tilde \bU_s$ be the corresponding eigenvectors. Then we partition $\tilde \bU_s, \tilde \bD_s$ into $\ell_0$ groups:
\begin{equation*}
\tilde \bU_s = \big[ \tilde\bV_s^{(1)},\tilde\bV_s^{(2)},\ldots, \tilde\bV_s^{(\ell_0)} \big], \qquad \tilde \bD_s = \diag \big(\tilde D_s^{(1)}, \tilde\bD_s^{(2)}, \ldots, \tilde\bD_s^{(\ell_0)}  \big).
\end{equation*}
The remaining $\ell+2-\ell_0$ groups of eigenvectors are denoted by $\tilde \bV_s^{(\ell_0+1)}, \ldots, \tilde \bV_s^{(\ell + 2)}$. 
Define $\tilde \bP_s^\bot = \bI_{n_0} -  \tilde \bU_s \tilde \bU_s^\top$. We apply Lemma~\ref{lem:simpleproj} again, in which we set $\bU_a = \tilde \bU_s$, $\bU_b = \tilde \bU_0$, $\bP = \bP_1$, and $\bP_a^\perp = \tilde \bP_s^\perp$. This yields
\begin{align*}
\sigma_{\min}(\bP_1 \tilde \bU_0) \ge \sigma_{\min} ( \bP_1 \tilde \bU_s) \sigma_{\min}(\tilde \bU_s^\top \tilde \bU_0) - \sigma_{\max}(\tilde \bP_s^\bot \tilde \bU_0).
\end{align*}
By Lemma~\ref{lem:rowselect}, we have $\sigma_{\min} ( \bP_1 \tilde \bU_s) \ge c >0$ with very high probability for certain constant $c$. Because of the gap $\lambda_{(\ell_0+1)} \le \delta_0 \lambda_{(\ell_0)}$, we can apply Lemma~\ref{lem:eigvec}: for $k' > \ell_0 $ and $k \le \ell_0$, we have
\begin{equation*}
\big\| (\tilde \bV_s^{(k')})^\top \tilde \bV_0^{(k)} \big\|_\op \le
 \frac{2 \| \bDelta^{(\res)}\|_\op}{\big|\lambda_k - \lambda_{k'}\big| - \breve{o}_{d,\P}(1) } 
 \le  \frac{2 \| \bDelta^{(\res)}\|_\op}{(1-\delta_0) \lambda_{(\ell_0)} - \breve{o}_{d,\P}(1) } 
 \le \frac{3(1+\breve{o}_{d,\P}(1)) \| \bDelta^{(\res)}\|_\op}{\gamma_{>\ell}}.
\end{equation*}
Further $\| \bDelta^{(\res)}\|_\op=\breve{o}_{d,\P}(1)$ by Lemma \ref{lem:eigval}.
Thus,
\begin{equation}\label{ineq:PsU}
\sigma_{\max} ( \tilde \bP_s^\bot \tilde \bU_0 ) \le \sum_{k' > \ell_0} \sum_{k \le \ell_0} \big\| (\tilde \bV_s^{(k')})^\top \tilde \bV_0^{(k)} \big\|_\op \le \breve{o}_{d,\P}(1),
\end{equation}
Moreover, $\sigma_{\min}(\tilde \bU_s^\top \tilde \bU_0) \ge 1 - \|  \tilde \bP_s^\bot \tilde \bU_0 \|_\op  \ge 1-\breve{o}_{d,\P}(1)$. Combining the lower bounds on $\sigma_{\min} ( \bP_1 \tilde \bU_s)$, $\sigma_{\min}(\tilde \bU_s^\top \tilde \bU_0)$ and the upper bound \eqref{ineq:PsU}, we obtain 
\begin{equation*}
\sigma_{\min}(\bP_1 \tilde \bU_0) \ge (1-\breve{o}_{d,\P}(1)) \cdot c' - \breve{o}_{d,\P}(1) 
\end{equation*}
with very high probability. This proves the claim.
\end{proof}

\subsection{Proof of Lemma \ref{lem:K2}}
\label{sec:ProofLemK2}

By homogeneity, we will assume, without loss of generality $\|f\|_{L^2}=1$.

Recall that, by Lemma \ref{lem:K_Harmonic},
we have $K(\xx_i, \xx) = \sum_{k=0}^\infty \gamma_k Q_k^{(d)} (\langle \xx_i, \xx \rangle)$ (where we replaced $\xx_j$ by an independent copy $\xx$ for notational convenience). We derive
\begin{align*}
K^{(2)}_{ij} &= \E_\xx\big[ K(\xx_i, \xx) K(\xx, \xx_j) \big] \stackrel{(i)}{=} \lim_{M \to \infty}\sum_{k,m=0}^M \gamma_k \gamma_m \E_\xx \big[ Q_k^{(d)}(\langle \xx_i, \xx \rangle) Q_k^{(d)}(\langle \xx, \xx_j \rangle) \big] \\
& \stackrel{(ii)}{=} \sum_{k=0}^\infty \frac{\gamma_k^2}{B(d,k)}  Q_k^{(d)}(\langle \xx_i, \xx_j \rangle) 
\end{align*}
where in \textit{(i)} we used $\lim_{M\to \infty} \| \sum_{k > M} \gamma_k Q_k^{(d)} (\langle \xx_i, \cdot \rangle) \|_{L^2} = 0$ and convergence in $L^2$-spaces (Lemma~\ref{lem:L2conv}), and in \textit{(ii)} we used the identity \eqref{eq:innerproductself}.
Note that 
\begin{align*}
\gamma_{>\ell}^{(2)} &:= \sum_{k > \ell} \frac{\gamma_k^2}{B(d,\ell)} \le \frac{1}{B(d,\ell+1)} \sum_{ k > \ell} \gamma_k^2 \le \frac{(1+o_d(1))(\ell!) \gamma_{>\ell}^2}{d^{\ell+1}}.
\end{align*}
Using the identity \eqref{eq:addtheorem}, we have $\QQ_k := (\QQ_k^{(d)}(\langle \xx_i, \xx_j \rangle))_{i,j \le n} = [B(d,k)]^{-1} \bPhi_{=\ell}\bPhi_{=\ell}^\top$. 
Thus, we can express $\bK^{(2)}$ as
\begin{equation}\label{eq:K2decomp}
\bK^{(2)} =  \bPsi_{\le \ell} \bLambda_{\le \ell}^4 \bPsi_{\le \ell}^\top + \gamma_{>\ell}^{(2)}  \bI_n + \bDelta^{(2)}
\end{equation}
where $\bDelta^{(2)} := \sum_{k > \ell} [B(d,k)]^{-1} \gamma_k^2 (\QQ_k-\bI_n)$ satisfies,
by Proposition~\ref{prop:Qconctr}, with high probability,
\begin{equation*}
\| \bDelta^{(2)}  \|_\op \le \gamma_{>\ell}^{(2)} \sup_{k > \ell} \big\| \QQ_k - \bI_d \big\|_\op \le \frac{C}{d^{\ell+1}} \sqrt{\frac{ n (\log n)^C}{d^{\ell+1}} }.
\end{equation*}
Using the decomposition of $\bK^{(2)}$ in \eqref{eq:K2decomp} and $\| \bK^{-1} \|_\op \le C$ from Lemma~\ref{lem:Kdecomp}, we have
\begin{align*}
&\big\| (\lambda \bI_n + \bK)^{-1} \bK^{(2)} (\lambda \bI_n + \bK)^{-1} \big\|_\op \\
&\le \big\| (\lambda \bI_n + \bK)^{-1}  \bPsi_{\le \ell} \bLambda_{\le \ell}^4 \bPsi_{\le \ell}^\top  (\lambda \bI_n + \bK)^{-1} \big\|_\op + \| (\lambda \bI_n + \bK)^{-1} \|_\op \cdot \big(\| \bDelta^{(2)} \|_\op + \gamma_{>\ell}^{(2)} \big) \cdot \| (\lambda \bI_n + \bK)^{-1} \|_\op \\
&\le \big\| (\lambda \bI_n + \bK)^{-1}  \bPsi_{\le \ell} \bLambda_{\le \ell}^4 \bPsi_{\le \ell}^\top  (\lambda \bI_n + \bK)^{-1} \big\|_\op + \| (\lambda \bI_n + \bK)^{-1} \|_\op^2 \cdot \frac{C}{d^{\ell+1}} \\
&\le \big\| (\lambda \bI_n + \bK)^{-1}  \bPsi_{\le \ell} \bLambda_{\le \ell}^4 \bPsi_{\le \ell}^\top (\lambda \bI_n + \bK)^{-1} \big\|_\op + \frac{C}{n}.
\end{align*}
Now we use the identity \eqref{eq:magic}, in which we set $\bA = \lambda \bI_n + \bK, \bA_0 = \bPsi_{\le \ell} \bLambda_{\le \ell}^2 \bPsi_{\le \ell}^\top +\gamma_{>\ell} \bI_n$, and $\bH = \bPsi_{\le \ell} \bLambda_{\le \ell}^4 \bPsi_{\le \ell}^\top $. It holds that $\| \bA^{-1} \|_\op \le C, \| \bA_0^{-1} \|_\op \le C$, and $\| \bA-  \bA_0 \|_\op \le C$ with high probability. So with high probability, 
\begin{align*}
&\big\| (\lambda \bI_n + \bK)^{-1}  \bPsi_{\le \ell} \bLambda_{\le \ell}^4 \bPsi_{\le \ell}^\top  (\lambda \bI_n + \bK)^{-1} \big\|_\op \\
&\le C \cdot \Big\| \big(  \bPsi_{\le \ell} \bLambda_{\le \ell}^2 \bPsi_{\le \ell}^\top  + \gamma_{>\ell} \bI_n \big)^{-1}  \bPsi_{\le \ell} \bLambda_{\le \ell}^4 \bPsi_{\le \ell}^\top  \big(  \bPsi_{\le \ell} \bLambda_{\le \ell}^2 \bPsi_{\le \ell}^\top  + \gamma_{>\ell} \bI_n \big)^{-1}  \Big\|_\op \\
&\stackrel{(i)}{\le} C \cdot \Big\| \bPsi_{\le \ell} \big( \bPsi_{\le \ell}^\top \bPsi_{\le \ell} + \gamma_{>\ell} \bLambda_{\le \ell}^{-2} \big)^{-2} \bPsi_{\le \ell}^\top \Big\|_\op \\
&\le C \cdot \big\| \bPsi_{\le \ell} \big\|_\op^2 \cdot \lambda_{\min}\big( \bPsi_{\le \ell}^\top \bPsi_{\le \ell} \big)^{-2} \\
&\stackrel{(ii)}{\le} \frac{C}{n},
\end{align*}
where \textit{(i)} follows from the identity \eqref{eq:algebra} and \textit{(ii)} follows form \eqref{eq:quickref}. This completes the proof of \eqref{eq:K21}.

In order to prove \eqref{eq:K22}, we will prove that the following holds w.h.p.
\begin{equation*}
\sup_{\uu \in \S^{n-1}} \Big| \E_\xx \big[ \uu^\top (\lambda \bI_n + \bK)^{-1} \bK(\cdot, \xx) f(\xx) \big] \Big| \le \frac{C}{\sqrt{n}}.
\end{equation*}
We use the Cauchy-Schwarz inequality on the left-hand side. 
Since we assumed,
without loss of generality,  $\| f \|_{L^2} =1$, we only need to show 
\begin{equation*}
\E_\xx \Big[ \big( \uu^\top (\lambda \bI_n + \bK)^{-1} \bK(\cdot, \xx) \big)^2 \Big] \le \frac{C}{n}
\end{equation*}
for any unit vector $\uu$. This is immediate from \eqref{eq:K21} since
\begin{align*}
& \sup_{\| \uu \|=1} \E_\xx \Big[ \big( \uu^\top (\lambda \bI_n + \bK)^{-1} \bK(\cdot, \xx) \big)^2 \Big] \\
&= \sup_{\| \uu \|=1} \uu^\top \E_\xx \Big[ (\lambda \bI_n + \bK)^{-1} \bK(\cdot, \xx) \bK(\cdot, \xx)^\top(\lambda \bI_n + \bK)^{-1}  \Big] \uu \\
&= \big\| (\lambda \bI_n + \bK)^{-1}  \bK^{(2)} (\lambda \bI_n + \bK)^{-1} \big\|_\op.
\end{align*}
The proof is now complete.

\subsection{Proof of Lemma~\ref{lem:keybnd2}}

\begin{proof}[{\bf Proof of Lemma~\ref{lem:keybnd2}}]
\textbf{Step 1: Proving \eqref{ineq:delta12}--\eqref{ineq:delta232}}. Define, for $k \le N$,
\begin{equation*}
K^{(k)}(\xx_1, \xx_2) = \sigma'(\langle \xx_1, \ww_k \rangle) \sigma'(\langle \xx_2, \ww_k \rangle) \frac{\langle \xx_1, \xx_2 \rangle}{d}
\end{equation*}
so we have $\bK_N = N^{-1} \sum_{k \le N} \bK^{(k)}$. We derive
\begin{align*}
\E_{\ww,\zz_1,\zz_2} \big[ (\delta I_{12}^{g, \hh})^2 \big] & = \frac{1}{N^2} \sum_{k,k' \le N} \E_{\ww,\zz_1,\zz_2} \Big[ \vv^\top \big( (\bK^{(k)}(\cdot, \zz_1) - \bK(\cdot, \zz_1)) (\bK^{(k')}(\zz_2, \cdot) - \bK(\zz_2, \cdot)) g(\zz_1) g(\zz_2) \big) \vv \Big] \\
&= \frac{1}{N^2} \sum_{k \le N} \vv^\top \E_{\ww,\zz_1,\zz_2}\Big[  \big(\bK^{(k)}(\cdot, \zz_1) - \bK(\cdot, \zz_1))\big(\bK^{(k)}(\zz_2, \cdot) - \bK(\zz_2, \cdot)) g(\zz_1) g(\zz_2)  \Big] \vv \\
&= \frac{1}{N^2} \sum_{k \le N} \vv^\top \E_{\ww,\zz_1,\zz_2} \Big[ \big( \bK^{(k)}(\cdot, \zz_1) \bK^{(k)}(\zz_2, \cdot) - \bK(\cdot, \zz_1) \bK(\zz_2, \cdot) \big) g(\zz_1) g(\zz_2)  \Big] \vv
\end{align*}
where we used independence of $\bK^{(k)}$ (conditional on $(\xx_i)_{i \le n}$) and $\E_\ww [\bK^{(k)}] = \bK$. Note that 
\begin{align*}
\vv^\top \E_{\ww,\zz_1,\zz_2} \big[  \bK(\cdot, \zz_1) \bK(\zz_2, \cdot) g(\zz_1) g(\zz_2) \big] \vv = \big(\vv^\top \E_\zz \big[ \bK(\cdot, \zz) g(\zz) \big] \big)^2 \ge 0,
\end{align*}
and that 
\begin{align*}
\E_{\ww,\zz_1,\zz_2} \Big[ \bK^{(k)}(\cdot, \zz_1) \bK^{(k)}(\zz_2, \cdot)  g(\zz_1) g(\zz_2)  \Big] = \bH_1.
\end{align*}
This proves the the bound  on $\E_{\ww,\zz_1,\zz_2} \big[ (\delta I_{12}^{g, \hh})^2 \big]$ in \eqref{ineq:delta12}. A similar bound, namely \eqref{ineq:delta231}, holds for $\E_{\ww,\zz_1,\zz_2} \big[ (\delta I_{231}^\hh)^2 \big]$, in which we replace $g$ with $\tilde h$.

Next, we derive
\begin{align*}
\E_{\ww} \big[ \delta I_{232}^\hh \big] &= \frac{1}{N^2} \sum_{k,k' \le N} \vv^\top \E_{\ww, \xx} \Big[ \big( \bK^{(k)}(\cdot, \xx) - \bK(\cdot, \xx) \big) \big(\bK^{(k')}(\xx, \cdot) - \bK(\xx, \cdot) \big)  \Big] \vv \\
&= \frac{1}{N^2} \sum_{k \le N}\vv^\top \E_{\ww, \xx}  \Big[ \big( \bK^{(k)}(\cdot, \xx) - \bK(\cdot, \xx) \big) \big(\bK^{(k)}(\xx, \cdot) - \bK(\xx, \cdot)\big) \Big] \vv \\
&= \frac{1}{N^2} \sum_{k \le N}\vv^\top \E_{\ww, \xx}  \Big[ \bK^{(k)}(\cdot, \xx) \bK^{(k)}(\xx, \cdot)  - \bK(\cdot, \xx)\bK(\xx, \cdot) \Big] \vv.
\end{align*}
Note that $\bK(\cdot, \xx)\bK(\xx, \cdot)$ is p.s.d., and that 
\begin{equation*}
\E_{\ww, \xx}  \Big[ \bK^{(k)}(\cdot, \xx) \bK^{(k)}(\xx, \cdot)  \Big] = \bH_3.
\end{equation*}
This proves the bound on $\E_{\ww} \big[ \delta I_{232}^\hh \big]$ in \eqref{ineq:delta232}.

\textbf{Step 2: proving the bounds on $\bH_1,\bH_2,\bH_3$}. We define the function $\qq(\ww) \in \R^d$ as follows.
\begin{equation*}
\qq(\ww) = \frac{1}{d} \E_\zz \big[\sigma'(\langle \zz, \ww \rangle) g(\zz) \zz \big].
\end{equation*}
We observe that, for any unit vector $\uu \in \R^d$, 
\begin{equation*}
\big|\langle \qq(\ww), \uu \rangle\big| \le 
\frac{1}{d} \Big\{ \E_\zz \big[ (\sigma'(\langle \zz, \ww \rangle) )^4 \big] \Big\}^{1/4} \cdot 
\Big\{ \E_\zz \big[g(\zz)^2\big] \Big\}^{1/2} \cdot \Big\{ \E_\zz \big[\langle \uu, \zz \rangle^4 \big]\Big\}^{1/4}
 \le \frac{C}{d} \| g\|_{L^2}^2\, ,
\end{equation*}
where we used H\"{o}lder's inequality and 
\begin{align}
&\lim_{d \to \infty} \E_\zz \big[ (\sigma'(\langle \zz, \ww \rangle) )^4] = \lim_{d \to \infty} \E_\zz \big[ (\sigma'(z_1) )^4] = \E_{G\sim \cN(0,1)}\big[ (\sigma'(G) )^4], \label{eq:sphericalmoments1} \\
& \lim_{d \to \infty}  \E_\zz \big[\langle \uu, \zz \rangle^4 \big] = \lim_{d \to \infty}  \E_\zz \big[z_1^4 \big] = \E_{G\sim \cN(0,1)}\big[ G^4]. \label{eq:sphericalmoments2}
\end{align}
Since $\| g \|_{L^2}$ is independent of $\ww$, it follows that 
\begin{equation*}
\sup_{\| \ww \|=1} \big\| \qq(\ww) \big\| = \sup_{\|\ww\|=\|\uu\|=1} \big|\langle \qq(\ww), \uu \rangle\big| \le \frac{C}{d} \| g \|_{L^2}^2.
\end{equation*}
Using this function, we express $\bH_1$ as
\begin{align*}
H_1(\xx_1,\xx_2) = \E_\ww \Big[ \sigma'(\langle \xx_1, \ww\rangle)\sigma'(\langle \xx_2, \ww\rangle) \langle \qq(\ww), \xx_1 \rangle \langle \qq(\ww), \xx_2 \rangle \Big].
\end{align*}
Let $\balpha=(\alpha_1,\ldots,\alpha_n)^\top$ be any vector, then
\begin{align*}
\balpha^\top \bH_1 \balpha &=  \E_\ww \Big[ \qq(\ww)^\top  \Big( \sum_{i,j \le n} \alpha_i \alpha_j \sigma'(\langle \xx_i, \ww\rangle)\sigma'(\langle \xx_j, \ww\rangle) \xx_i \xx_j^\top  \Big) \qq(\ww) \Big] \\
&\le \E_\ww \Big[ \big\| \qq(\ww) \big\|^2 \cdot \Big\| \sum_{i,j \le n} \alpha_i \alpha_j \sigma'(\langle \xx_1, \ww\rangle)\sigma'(\langle \xx_2, \ww\rangle) \xx_i \xx_j^\top   \Big\|_\op \Big] \\
&\le \frac{C\| g \|_{L^2}^2 }{d} \E_\ww \Big[ \Tr \Big( \sum_{i,j \le n} \alpha_i \alpha_j \sigma'(\langle \xx_1, \ww\rangle)\sigma'(\langle \xx_2, \ww\rangle) \xx_i \xx_j^\top\Big) \Big] \\
&= C\| g \|_{L^2}^2\sum_{i,j \le n} \alpha_i \alpha_j \E_\ww \Big[\sigma'(\langle \xx_i, \ww \rangle) \sigma'(\langle \xx_j,\ww \rangle) \frac{\langle \xx_i, \xx_j \rangle}{d} \Big] \\
&=C \| g \|_{L^2}^2 \cdot \balpha^\top \bK \balpha
\end{align*}
where we used $\| \bA \|_\op \le \Tr(\bA)$ for a p.s.d.~matrix $\bA$. This shows $\bH_1 \preceq Cd^{-1} \| g \|_{L^2}^2 \bK$. The proof for $\bH_2$ is similar.

We also define $\QQ(\ww) \in \R^{n \times n}$:
\begin{align*}
\QQ(\ww) = \E_\ww \Big[ \big(\sigma'(\langle \ww, \zz \rangle)\big)^2 \frac{\zz \zz^\top}{d} \Big].
\end{align*}
By definition, $\QQ(\ww)$ is p.s.d.~for all $\ww$. For any $\ww$ and unit vector $\uu$, 
\begin{align*}
\uu^\top \QQ(\ww) \uu &= \frac{1}{d} \E_\zz \Big[ \big(\sigma'(\langle \ww, \zz \rangle) \big)^2 \langle \uu, \zz \rangle^2 \Big] \\
&\le \frac{1}{d} \Big\{ \E_\zz \Big[ \big(\sigma'(\langle \ww, \zz \rangle) \big)^4 \Big] \Big\}^{1/2} \cdot \Big\{ \E_\ww \Big[ \langle \uu, \zz \rangle^4 \Big] \Big \}^{1/2} \le \frac{C}{d}
\end{align*}
where we used the Cauchy-Schwarz inequality and \eqref{eq:sphericalmoments1}--\eqref{eq:sphericalmoments2}. This implies 
\begin{equation*}
\sup_{\| \ww \|=1} \big\| \QQ(\ww) \big\|_\op \le \frac{C}{d}.
\end{equation*}
For any vector $\balpha=(\alpha_1,\ldots,\alpha_n)^\top$, we derive
\begin{align*}
\balpha^\top \bH_3  \balpha &= \frac{1}{d} \E_\ww \Big[ \sum_{i,j \le n} \alpha_i \alpha_j \sigma'(\langle\xx_i, \ww\rangle) \sigma'(\langle\xx_j, \ww\rangle) \xx_i^\top \QQ(\ww) \xx_j  \Big] \\
&= \frac{1}{d} \E_\ww \Big[ \Tr \Big( \sum_{i,j \le n} \alpha_i \alpha_j \sigma'(\langle\xx_i, \ww\rangle) \sigma'(\langle\xx_j, \ww\rangle)\xx_j \xx_i^\top \QQ(\ww) \Big) \Big] \\
&\le \frac{1}{d} \E_\ww \Big[ \big\| \QQ(\ww) \big\|_\op \cdot \Tr \Big( \sum_{i,j \le n} \alpha_i \alpha_j \sigma'(\langle\xx_i, \ww\rangle) \sigma'(\langle\xx_j, \ww\rangle) \xx_j \xx_i^\top \Big) \Big] \\
&\le \frac{C}{d} \E_\ww \Big[ \sum_{i,j \le n} \alpha_i \alpha_j \sigma'(\langle\xx_i, \ww\rangle) \sigma'(\langle\xx_j, \ww\rangle)\frac{\langle \xx_i, \xx_j \rangle}{d} \Big] \\
&= \frac{C}{d} \balpha^\top \bK \balpha,
\end{align*}
where we used $\Tr(\bA_1 \bA_2) \le \| \bA_1 \|_\op \Tr(\bA_2)$ for p.s.d.~matrices $\bA_1, \bA_2$. This proves $\bH_3 \preceq Cd^{-1} \bK$.

\textbf{Step 3: Proving the ``as a consequence'' part}. Now we derive bounds on $\delta I_{12}^{g, \hh}$ 
and $\delta I_{23}^{\hh}$. By assumption, $\| g \|_{L^2}$ is bounded by a constant. 

Also, by assumption  $\| \hh \|^2 / n \le C$ with high probability. Therefore
\begin{align*}
\E_{\ww}\big[(\delta I_{12}^{g, \hh})^2 \big] &\le \frac{1}{Nd} \hh^\top (\lambda \bI_n + \bK)^{-1} \bK (\lambda \bI_n + \bK)^{-1} \hh \le \frac{1}{Nd} \hh^\top (\lambda \bI_n + \bK)^{-1} \hh \\
&\le \frac{1}{cNd} \| \hh \|^2 \le \frac{Cn}{Nd}.
\end{align*}
Similar bounds hold for $\E_{\ww}\big[(\delta I_{231}^\hh)^2 \big]$ and $\E_{\ww}\big[ \delta I_{232}^\hh \big]$. Note that $\delta I_{232}^\hh$ is always nonnegative. By Markov's inequality, we have w.h.p.,
\begin{equation*}
|\delta I_{12}^{g,\hh} | \le \sqrt{\frac{Cn \log n}{Nd}}, \qquad |\delta I_{231}^\hh | \le \sqrt{\frac{Cn \log n}{Nd}}, \qquad |\delta I_{232}^\hh | \le \frac{Cn \log n}{Nd}.
\end{equation*}
Combining the last two bounds then leads to the bound on $|\delta I_{23}^\hh|$.
\end{proof}

\subsection{Proof of Theorem~\ref{thm:reduceprr}}

By homogeneity we can and will assume, without loss of generality, $\|f\|_{L^2}=\sigma_{\veps}=1$. It is convenient to introduce the following matrix $\bK^{(p,2)} \in\reals^{n\times n}$
\begin{equation}\label{def:Kp2}
\bK^{(p,2)} = \big(\E_\xx[ K^p(\xx_i, \xx)K^p(\xx, \xx_j) ] \big)_{i,j\le n}\, .
\end{equation}
We further define $\obK^p := \bK^p+\gamma_{>\ell}\id_n$ or, equivalently,
\begin{align}
\bar K^p_{ij} &=  \sum_{k=0}^\ell \gamma_k  Q_k^{(d)}(\langle \xx_i, \xx_j \rangle) + \gamma_{>\ell}\delta_{ij}\, .
\end{align}
We will use the following lemma. 
\begin{lem}\label{lem:toPRR}
Suppose that $C_1>0$ is a constant and that $\hh_1, \hh_2 \in \R^n$ are  random vectors that satisfy $\max\{\| \hh_1 \|, \| \hh_2 \|\} \le C_1 \sqrt{n}$ with high probability. Then, there exists a constant $C'>0$ such that the following bounds hold with high probability.
\begin{align}
&\big| \hh_1^\top (\lambda \bI_n + \bK)^{-1} \bK^{(2)} (\lambda \bI_n + \bK)^{-1} \hh_2 - \hh_1^\top (\lambda \bI_n + \obK^{p})^{-1} \bK^{(p,2)} (\lambda \bI_n + \obK^{p})^{-1} \hh_2 \big| \le \sqrt{\frac{C'(\log n)^{C'} n}{d^{\ell+1}}},  \label{ineq:toPRR1}\\
&\big| \hh_1^\top (\lambda \bI_n + \bK)^{-1} \E_\xx \big[ \bK(\cdot, \xx) f(\xx) \big] - \hh_1^\top (\lambda \bI_n + \obK^p)^{-1} \E_\xx \big[ \bK^p(\cdot, \xx) f(\xx) \big] \big| \le \sqrt{\frac{C'(\log n)^{C'} n}{d^{\ell+1}}}\|f\|_{L^2} \label{ineq:toPRR2}
\end{align}
\end{lem}

We first show that Theorem~\ref{thm:reduceprr} follows from this lemma. After that, we will prove Lemma~\ref{lem:toPRR}.
\begin{proof}[{\bf Proof of Theorem~\ref{thm:reduceprr}}]
First we observe that 
\begin{align*}
E_{\bias} - E_{\bias}^p &= -2\Big( \ff^\top (\lambda \bI_n + \bK)^{-1} \E_\xx \big[ \bK(\cdot, \xx) f(\xx) \big] - \ff^\top (\lambda \bI_n + \obK^p)^{-1} \E_\xx \big[ \bK^p(\cdot, \xx) f(\xx) \big] \Big) \\
&+ \Big(\ff^\top (\lambda \bI_n + \bK)^{-1} \bK^{(2)} (\lambda \bI_n + \bK)^{-1} \ff - \ff^\top (\lambda \bI_n + \obK^{p})^{-1} \bK^{(p,2)} (\lambda \bI_n + \obK^{p})^{-1} \ff\Big).
\end{align*}
In Lemma~\ref{lem:toPRR} Eqs.~\eqref{ineq:toPRR1} and \eqref{ineq:toPRR2}, we set $\hh_1 = \hh_2 = \ff$, which yields $|E_{\bias} - E_{\bias}^p| \le  \eta'$. For the variance term, we apply Lemma~\ref{lem:toPRR} Eq.~\eqref{ineq:toPRR1} with $\hh_1 = \hh_2 = \bveps$, which yields $|E_{\var} - E_{\var}^p| \le  \eta'$. 

For the cross term, we observe that 
\begin{align*}
E_{\cross} - E_{\cross}^p &= \Big( \bveps^\top (\lambda \bI_n + \bK)^{-1} \E_\xx \big[ \bK(\cdot, \xx) f(\xx) \big] - \bveps^\top (\lambda \bI_n + \obK^p)^{-1} \E_\xx \big[ \bK^p(\cdot, \xx) f(\xx) \big] \Big) \\
&- \Big(\bveps^\top (\lambda \bI_n + \bK)^{-1} \bK^{(2)} (\lambda \bI_n + \bK)^{-1} \ff - \bveps^\top (\lambda \bI_n + \obK^{p})^{-1} \bK^{(p,2)} (\lambda \bI_n + \obK^{p})^{-1} \ff\Big).
\end{align*}
We apply Lemma~\ref{lem:toPRR} with $\hh_1 = \bveps$ and $\hh_2 = \ff$. This leads to $|E_{\cross} - E_{\cross}^p| \le  \eta'$, which completes the proof. 
\end{proof}

\begin{proof}[{\bf Proof of Lemma~\ref{lem:toPRR}}]
Define differences
\begin{align*}
\delta I_1' &=  \hh_1^\top (\lambda \bI_n + \obK^{p})^{-1} \big(\bK^{(2)}- \bK^{(p,2)}) \big (\lambda \bI_n + \obK^{p})^{-1} \hh_2 \\
\delta I_2' &= \hh_1^\top (\lambda \bI_n + \bK)^{-1}  \bK^{(2)} \big (\lambda \bI_n + \bK)^{-1} \hh_2 - \hh_1^\top (\lambda \bI_n + \obK^{p})^{-1} \bK^{(2)} \big (\lambda \bI_n + \obK^{p})^{-1} \hh_2, \\
\delta I_3' &= \hh_1^\top (\lambda \bI_n + \obK^p)^{-1} \E_\xx \big[ (\bK(\cdot, \xx) - \bK^p(\cdot, \xx) ) f(\xx) \big], \\
\delta I_4' &= \hh_1^\top \big[ (\lambda \bI_n + \bK)^{-1} - (\lambda \bI_n + \obK^p)^{-1} \big] \E_\xx \big[ \bK(\cdot, \xx) f(\xx) \big].
\end{align*}
We have
\begin{align*}
& \hh_1^\top (\lambda \bI_n + \bK)^{-1} \bK^{(2)} (\lambda \bI_n + \bK)^{-1} \hh_2 - \hh_1^\top (\lambda \bI_n + \obK^{p})^{-1} \bK^{(p,2)} (\lambda \bI_n + \obK^{p})^{-1} \hh_2  = \delta I_1' + \delta I_2' , \\
 & \hh_1^\top (\lambda \bI_n + \bK)^{-1} \E_\xx \big[ \bK(\cdot, \xx) f(\xx) \big] - \hh_1^\top (\lambda \bI_n + \obK^p)^{-1} \E_\xx \big[ \bK^p(\cdot, \xx) f(\xx) \big] = \delta I_3' + \delta I_4'.
\end{align*}
We observe that 
\begin{align*}
\bar K^p_{ij} &=  \sum_{k=0}^\ell \gamma_k  Q_k^{(d)}(\langle \xx_i, \xx_j \rangle) + \gamma_{>\ell} \delta_{ij}, \\
K^{(p,2)}_{ij} & = \sum_{k,m=0}^\ell \gamma_k \gamma_m \E_\xx \big[ Q_k^{(d)}(\langle \xx_i, \xx \rangle)Q_k^{(d)}(\langle \xx, \xx_j \rangle) \big] = \sum_{k=0}^\ell \frac{\gamma_k^2}{B(d,k)}Q_k^{(d)}(\langle \xx_i, \xx_j \rangle) .
\end{align*}
It follows that
\begin{align}
&\bK - \obK^p = \sum_{k>\ell} \gamma_k(\QQ_k - \bI_n), \qquad \bK^{(2)} - \bK^{(p,2)} = \sum_{k > \ell} \frac{\gamma_k^2}{B(d,k)}\QQ_k,  \label{ineq:KKpdiff}\\
& \E_\xx \big[ (\bK(\cdot, \xx) - \bK^p(\cdot, \xx)) (\bK(\cdot, \xx) - \bK^p(\cdot, \xx))^\top   \big] = \sum_{k > \ell} \frac{\gamma_k^2}{B(d,k)}  \QQ_k.\label{ineq:KKpdiff2}
\end{align}
Therefore, with high probability,
\begin{align}
\big\| \bK  - \obK^p \big\|_\op &\le \Big\| \sum_{k > \ell} \gamma_k ( \QQ_k - \bI_n) \Big\|_\op \le \gamma_{>\ell} \sup_{k > \ell} \big\| \QQ_k - \bI_n \big\|_\op \stackrel{(i)}{\le} \sqrt{\frac{C(\log n)^C n}{d^{\ell+1}}}, \label{ineq:KKp} \\
\big\| \bK^{(2)} - \bK^{(p,2)} \big\|_\op & = \Big\| \sum_{k>\ell} \frac{\gamma_k^2}{B(d,k)} \QQ_k \Big\|_\op \le \sum_{k>\ell}\gamma_k^2 \cdot \frac{C }{d^{\ell+1}} \sup_{k > \ell} \big\|  \QQ_k \big\|_\op \\
&\stackrel{(ii)}{\le} \gamma_{>\ell}^2  \cdot \frac{C }{d^{\ell+1}} \cdot (1 + o_{d,\P}(1)) \le \frac{C }{d^{\ell+1}}, \label{ineq:K2Kp2}
\end{align}
where in \emph{(i), (ii)} we used Proposition~\ref{prop:Qconctr}. Also, note that $\bar \bK^p \succeq \gamma_{>\ell} \bI_n$, so we must have $\| (\lambda \bI_n + \bar \bK^{p})^{-1} \|_\op \le \gamma_{>\ell}^{-1}$. Thus, w.h.p.,
\begin{equation*}
\big| \delta I_1' \big| \le \| \hh_1 \|\cdot \|\hh_2\| \cdot \big\| (\lambda \bI_n + \bar \bK^{p})^{-1}  \big\|_\op^2 \cdot \big\| \bK^{(2)} - \bK^{(p,2)} \big\|_\op \le \frac{Cn}{d^{\ell+1}}.
\end{equation*}
In the identity \eqref{eq:magic}, we set $\bA_0 = \lambda \bI_n + \obK^p$, $\bA = \lambda \bI_n + \bK$, and $\bH = \bK^{(2)}$. It holds that w.h.p., 
$\| \bA_0^{-1} \|_\op ,\| \bA^{-1} \|_\op \le C$, $\| \bA - \bA_0\|_\op \le \sqrt{C(\log n)^C n / d^{\ell+1}}$. Thus, w.h.p.,
\begin{align*}
\big| \delta I_2'  \big| &\le C\| \hh_1 \|\cdot \|\hh_2\| \cdot  \sqrt{\frac{C(\log n)^C n}{d^{\ell+1}}} \cdot \big\| (\lambda \bI_n + \bK)^{-1}  \bK^{(2)} \big (\lambda \bI_n + \bK)^{-1} \big\|_\op \le Cn \cdot \sqrt{\frac{(\log n)^C n}{d^{\ell+1}}} \cdot \frac{C}{n} \\
&\le \sqrt{\frac{C(\log n)^C n}{d^{\ell+1}}}
\end{align*}
where we used Lemma~\ref{lem:K2} Eq.~\eqref{eq:K21}. Combining the upper bounds on $|\delta I_1'|$ and $|\delta I_2'|$, we arrive at the first inequality in the lemma.

Next, we bound $|\delta I_3'|$. With high probability,
\begin{align*}
| \delta I_3' |^2 &\stackrel{(i)}{\le} \E_\xx \Big[ \big( \hh_1^\top (\lambda \bI_n + \bar \bK^p)^{-1} \big( \bK(\cdot, \xx) - \bK^p(\cdot, \xx) \big)^2  \Big] \cdot \E_\xx \big[ (f(\xx))^2 \big] \\
&\le \hh_1^\top (\lambda \bI_n + \bK^p)^{-1}  \E_\xx \big[ \big( \bK(\cdot, \xx) - \bK^p(\cdot, \xx) \big) \big( \bK(\cdot, \xx) - \bK^p(\cdot, \xx) \big)^\top \big] (\lambda \bI_n + \bK^p)^{-1} \hh_1 \cdot \| f \|_{L^2}^2 \\
&\stackrel{(ii)}{\le} \sum_{k > \ell} \frac{C\gamma_k^2}{B(d,k)} \cdot \hh_1^\top (\lambda \bI_n + \bar \bK^p)^{-1} \QQ_k (\lambda \bI_n + \bar \bK^p)^{-1} \hh_1  \\
&\le \frac{C}{d^{\ell+1}} \cdot  \sum_{k > \ell} \gamma_k^2 \cdot \| \hh_1\|^2 \cdot \| (\lambda \bI_n + \bar \bK^p \big)^{-1} \|^2_{\op} \cdot \| \QQ_k \|_\op \\
&\stackrel{(iii)}{\le} \frac{Cn}{d^{\ell+1}} \cdot \gamma_{>\ell}^2 \cdot  \sup_{k > \ell} \| \QQ_k \|_\op \stackrel{(iv)}{\le} \frac{Cn}{d^{\ell+1}}
\end{align*}
where \textit{(i)} follows from the Cauchy-Schwarz inequality, \textit{(ii)} follows 
from \eqref{ineq:KKpdiff2}, \textit{(iii)} is because $\| \hh \| \le C \sqrt{n}$ w.h.p.~by
 assumption and $\| (\lambda\bI_n + \bar \bK^p)^{-1} \|_\op \le \gamma_{>\ell}^{-1}$, and \textit{(iv)} 
 follows from Proposition~\ref{prop:Qconctr}. Finally,
\begin{align*}
| \delta I_4' | &\le \big\| \hh_1^\top (\lambda \bI_n + \bar \bK^p)^{-1} \big\| \cdot \| \bar \bK^p - \bK \|_\op \cdot \big\| (\lambda \bI_n + \bK)^{-1} \E_\xx \big[\bK(\cdot,\xx) f(\xx) \big] \big\|_\op \\
&\stackrel{(i)}{\le} C \sqrt{n} \cdot   \sqrt{\frac{C(\log n)^C n}{d^{\ell+1}}} \cdot \frac{C}{\sqrt{n}}  \stackrel{(ii)}{\le} \sqrt{\frac{C(\log n)^C n}{d^{\ell+1}}}
\end{align*}
where in \textit{(i)} we used  $\| (\lambda \bI_n + \bar \bK^{p})^{-1} \|_\op \le \gamma_{>\ell}^{-1}$, Eq.~\eqref{ineq:KKp}, and Lemma~\ref{lem:K2}~Eq.~\eqref{eq:K22}. Combining the upper bounds on $|\delta I_3'|$ and $|\delta I_4'|$, we arrive at the second inequality of this lemma.
\end{proof}

\section{Generalization error: improved analysis for $\ell = 1$}
\label{sec:Ell1}

We will show that for the case $\ell = 1$, we can relax the condition 
\begin{equation*}
C_0 (\log d)^{C_0} d  \le n \le \frac{d^2}{C_0(\log d)^{C_0}} \qquad \text{to} \qquad c_0 d  \le n \le \frac{d^2}{C_0(\log d)^{C_0}}.
\end{equation*}
The proof of the generalization error under this relaxed condition follows mostly the proof  in Section~\ref{sec:proofkeybnd}, with modifications we show in this subsection.

Throughout, we suppose that $\ell = 1$ and correspondingly $D = d+1$. The matrix $\bPsi_{\le \ell} \in \R^{n \times D}$ is given by 
\begin{equation*}
\bPsi_{\le \ell} = \big( \bone_n, \XX \big) .
\end{equation*}
An important difference under the relaxed condition is that 
$n^{-1} \bPsi_{\le \ell}^\top \bPsi_{\le \ell}$ does not necessarily converge to $\bI_D$ (cf.~Eq.~\ref{eq:psi-conctr}). (In fact, if $n,d$ satisfy $n/d \to \kappa$, the spectra of $n^{-1}\XX^\top \XX$ is characterized by the Marchenko-Pastur distribution.) Nevertheless, if $n \ge C d$ where $C$ is sufficiently large, we have a good control of $n^{-1} \bPsi_{\le \ell}^\top \bPsi_{\le \ell}$. 

\begin{lem}\label{lem:spectracontrol}
For any constant $\delta \in (0,0.1)$, there exists certain $C_\delta>1$ such that the following holds. If $n \ge C_\delta d$, then with very high probability,
\begin{equation}\label{eq:sepctracontrol}
\Big\| \frac{1}{n} \bPsi_{\le \ell}^\top \bPsi_{\le \ell} - \bI_{D} \Big\|_\op \le \delta.
\end{equation}
\end{lem}

Fix the constant $\delta = 0.01$. Let $C_\delta > 1$ be the constant in Lemma~\ref{lem:spectracontrol}. First, we establish some useful results under the condition 
\begin{align}\label{cond:improved}
C_\delta d \le n  \le \frac{d^2}{C (\log d)^C}, \qquad n \le \frac{Nd}{(\log(Nd))^{C}}\, ,
\end{align}
for a sufficiently large $C>0$ such that the following inequalities hold by 
Lemma~\ref{lem:Kdecomp},~\ref{lem:spectracontrol} and Theorem~\ref{thm:invert2}: in the decomposition $\bK = \gamma_{>\ell} \bI_n + \bPsi_{\le \ell} \bLambda_{\le \ell}^2 \bPsi_{\le \ell}^\top  + \bDelta \,$, we have
\begin{align}
 &\big\| \bDelta \big\|_\op = \breve{o}_{d,\P}(1), \quad \big\| n^{-1} \bPsi_{\le \ell}^\top \bPsi_{\le \ell} - \bI_n \big\|_\op \le \delta + \breve{o}_{d,\P}(1), \quad
\big\| \bK^{-1/2} (\bK_N - \bK) \bK^{-1/2} \big\|_\op = \breve{o}_{d,\P}(1)\, . \label{eq:quickref2}
\end{align}

We strengthen Lemma~\ref{lem:eigval} for the case $\ell=1$.

\begin{lem}[Kernel eigenvalue structure: case $\ell=1$]\label{lem:eigval1}
Assume that Eq.~\eqref{cond:improved} holds and set $\ell=1$ and $D=d+1$.
Then the eigen-decomposition of $\bK - \gamma_{>\ell} \bI_n$ and $\bK_N - \gamma_{>\ell} \bI_n $ 
can be written in the following form:
\begin{align}
&\bK - \gamma_{>\ell} \bI_n = \bU \bD \bU^\top + \bDelta^{(\res)}, \label{eq:KeigdecompL1}\\
&\bK_N - \gamma_{>\ell} \bI_n = \bU_N \bD_N \bU_N^\top + \bDelta^{(\res)}_N,
\label{eq:KNeigdecompL1}
\end{align}
where $\bD, \bD_N\in\reals^{D\times D}$ are diagonal matrices that contain $D$ eigenvalues of $\bK, \bK_N$ respectively, 
columns of $\bU, \bU_N \in \R^{n \times D}$ are the corresponding eigenvectors 
and $\bDelta^{(\res)}, \bDelta_N^{(\res)}$ correspond to the other eigenvectors 
(in particular $\bDelta^{(\res)}\bU=\bDelta_N^{(\res)}\bU_N= \bzero$).

 Further defining 
\begin{equation*}
\bD^* = \diag \Big(\gamma_0 n, \underbrace{\frac{\gamma_1 n}{d}, \ldots, \frac{\gamma_1 n}{d}}_{d}\Big).
\end{equation*}
the eigenvalues have the following structure:
\begin{align}
& (1-2\delta - \breve{o}_{d,\P}(1)) \cdot \bD^* - \breve{o}_{d,\P}(1) \le \bD \le  (1+2\delta + \breve{o}_{d,\P}(1)) \cdot \bD^* + \breve{o}_{d,\P}(1), \label{eq:Keig-L1}\\
&(1-2\delta - \breve{o}_{d,\P}(1)) \cdot \bD^* - \breve{o}_{d,\P}(1) \le \bD_N \le  (1+2\delta + \breve{o}_{d,\P}(1)) \cdot \bD^* + \breve{o}_{d,\P}(1). \label{eq:KNeig-L1}
\end{align}
Here, $\le$ denotes entrywise comparisons ($\bA \le \bA'$ if and only if  $A_{ij} \le A'_{ij}$ for all $i,j$). Moreover, the remaining components satisfy $\| \bDelta^{(\res)} \|_\op = \breve{o}_{d,\P}(1)$ and $\| \bDelta^{(\res)}_N \|_\op = \breve{o}_{d,\P}(1)$. 
\end{lem} 
\begin{proof}[{\bf Proof of Lemma~\ref{lem:eigval1}}]
In the proof of Lemma~\ref{lem:eigval}, instead of claiming \eqref{claim:lambdaQ}, we will show that the following modified claim holds.
\begin{equation*}
(1-2\delta - \breve{o}_{d,\P}(1)) \cdot \bD^* \le \blambda(\QQ^\top \bPsi_{\le \ell} \bLambda_{\le \ell}^2 \bPsi_{\le \ell}^\top \QQ) \le (1+2\delta + \breve{o}_{d,\P}(1)) \cdot \bD^*.
\end{equation*}
To prove this claim, we note that Lemma~\ref{lem:spectracontrol} implies that $n^{-1/2} \| \bPsi_{\le \ell} \|_\op \le \sqrt{1 + \delta} \le 1 + \delta$ and $n^{-1/2} \| \bPsi_{\le \ell} \|_\op \ge \sqrt{1 - \delta} \ge 1 - \delta$ with very high probability. Following the notations in the proof of Lemma~\ref{lem:eigval}, we have 
\begin{align*}
\lambda_k(\QQ^\top \bPsi_{\le \ell} \bLambda_{\le \ell}^2 \bPsi_{\le \ell}^\top \QQ) &= 
\max_{\cV: \dim(\cV) = k} \min_{\bzero \neq \uu \in \cV} \frac{ \uu^\top\QQ^\top \bPsi_{\le \ell} \bLambda_{\le \ell}^2 \bPsi_{\le \ell}^\top \QQ \uu}{\| \uu \|^2} \\
&\ge (1 - 2\delta -  \breve{o}_{d,\P}(1))\min_{\bzero \neq \uu \in \cV_1'} \frac{\uu^\top\QQ^\top \bPsi_{\le \ell} \bLambda_{\le \ell}^2 \bPsi_{\le \ell}^\top \QQ \uu}{\| \bPsi_{\le \ell}^\top \QQ \uu \|^2 / n} \\
&\ge  (1 - 2\delta -  \breve{o}_{d,\P}(1))\min_{\bzero \neq \xx \in \cV_1} \frac{\xx^\top \bLambda_{\le \ell}^2 \xx}{\| \xx \|^2 / n}\\
& \ge n(1- 2\delta -\breve{o}_{d,\P}(1)) \cdot \lambda_k(\bLambda_{\le \ell}^2).
\end{align*}
Similarly, 
\begin{align*}
\lambda_k(\QQ^\top \bPsi_{\le \ell} \bLambda_{\le \ell}^2 \bPsi_{\le \ell}^\top \QQ) 
&= \min_{\cV: \dim(\cV) = D-k+1} \max_{\uu \in \cV} \frac{\uu^\top\QQ^\top \bPsi_{\le \ell} \bLambda_{\le \ell}^2 \bPsi_{\le \ell}^\top \QQ \uu}{\| \uu \|^2} \\
&\le (1 + 2\delta +\breve{o}_{d,\P}(1)) \max_{\uu \in \cV_2'} \frac{\uu^\top \QQ^\top \bPsi_{\le \ell} \bLambda_{\le \ell}^2 \bPsi_{\le \ell}^\top \QQ \uu}{\| \bPsi_{\le \ell}^\top \QQ \uu \|^2 / n} \\
&\le(1 + 2\delta +\breve{o}_{d,\P}(1))  \max_{\xx \in \cV_2} \frac{\xx^\top \bLambda_{\le \ell}^2 \xx}{\| \xx \|^2 / n}\\
& \le n(1+ 2\delta + \breve{o}_{d,\P}(1)) \cdot \lambda_k(\bLambda_{\le \ell}^2).
\end{align*}
We omit the rest of proof, as it follows closely the proof of Lemma~\ref{lem:eigval}.
\end{proof}

We prove a slight modification of 
Lemma~\ref{lem:eigvec} for  the case $\ell=1$.
\begin{lem}[Kernel eigenvector structure: case $\ell=1$]\label{lem:eigvec1}
Assume that Eq.~\eqref{cond:improved} holds and set $\ell=1$ and $D=d+1$.
Let $k \neq k' \in \{ 0,1,\ldots,\ell+1\}$. Denote $\lambda_k = \gamma_{>\ell} + 
\gamma_k (k!) n/(d^k)$ for $k = 0, \ldots, \ell$ and $\lambda_{\ell+1} = \gamma_{>\ell} $.  
\begin{enumerate}
\item[(a)] Suppose that $\min\{\lambda_k/ \lambda_{k'}, \lambda_{k'} / \lambda_k \} \le 1/4$. Then, 
\begin{equation*}
\big\| (\bV_0^{(k')})^\top \bV^{(k)} \big\|_\op = \breve{o}_{d,\P}(1). 
\end{equation*} 
\item[(b)] Recall $\bDelta^{(\res)}$ defined in \eqref{eq:KNeigdecompL1}. If $\big|\lambda_k - \lambda_{k'}\big| \ge 4.1\delta\max\{\lambda_k, \lambda_{k'}\}$, then
\begin{equation*}
\big\| (\bV_s^{(k')})^\top \bV_0^{(k)} \big\|_\op \le \frac{2 \| \bDelta^{(\res)}\|_\op}{\big|\lambda_k - \lambda_{k'}\big| - 4\delta\max\{\lambda_k, \lambda_{k'}\} - \breve{o}_{d,\P}(1) }.
\end{equation*}
\end{enumerate}
\end{lem}
\begin{proof}[{\bf Proof of Lemma~\ref{lem:eigvec1}}]
Instead of Eqs.~\eqref{ineq:eigvec0}--\eqref{ineq:eigvec1}, we have
\begin{align*}
\min_{i,j} \big| (\bD_0^{(k')})_{ii} - (\bD^{(k)})_{jj} \big| & \ge \big| \gamma_{>\ell} + (1+\breve{o}_{d,\P}(1)) ( \lambda_k - \gamma_{>\ell}) - \gamma_{>\ell} - (1+\breve{o}_{d,\P}(1)) ( \lambda_{k'} - \gamma_{>\ell}) \big|  \\
&- 2\delta|\lambda_k-\gamma_{>\ell}| - 2\delta|\lambda_{k'} -\gamma_{>\ell}|  - \breve{o}_{d,\P}(1) \\
&\ge \big| (1+\breve{o}_{d,\P}(1)) \lambda_k - (1+\breve{o}_{d,\P}(1)) \lambda_{k'} \big| -4\delta\max\{\lambda_k, \lambda_{k'}\} - \breve{o}_{d,\P}(1). 
\end{align*}
By the assumptions on $\lambda_k$ and $\lambda_{k'}$, with very high probability, 
\begin{equation*}
\min_{i,j} \big| (\bD_0^{(k')})_{ii} - (\bD^{(k)})_{jj} \big| \ge (1/2 - 4\delta) \max\{\lambda_k, \lambda_{k'} \big\}. 
\end{equation*}
Also, instead of Eqs.~\eqref{ineq:V0K}--\eqref{ineq:KV}, we have
 \begin{equation*} 
\big\| (\bV_0^{(k')})^\top \bK^{1/2} \big\|_\op = \big\| \big[ \bD_0^{(k')}\big]^{1/2} (\bV_0^{(k')})^\top \big\|_\op \le (\sqrt{1+2\delta}\, + \breve{o}_{d,\P}(1)
) \cdot \sqrt{\lambda_{k'}}.
\end{equation*}
Similarly, 
\begin{equation*} 
\big\| \bK_N^{1/2} \bV^{(k)} \big\|_\op = \big\| \bV^{(k)} (\bD^{(k)})^{1/2} \big\|_\op \le (1 + \sqrt{1+2\delta}\, + \breve{o}_{d,\P}(1)) \cdot \sqrt{\lambda_k}.
\end{equation*} 
These inequalities lead to
\begin{equation*}
\big\| (\bV_0^{(k')})^\top  \bV^{(k)} \big\|_\op \le \frac{\breve{o}_{d,\P}(1) \cdot  \sqrt{\lambda_k \lambda_{k'}}}{(1-1/2-4\delta)\max\{ \lambda_k, \lambda_{k'}\}} \le \breve{o}_{d,\P}(1),
\end{equation*}
which results in the same conclusion as in Lemma~\ref{lem:eigvec}~(a). To show the modified inequality in (b), we only need to notice that 
\begin{equation*}
\min_{i,j}\big| (\bD_s^{(k')})_{ii} - (\bD_s^{(k)})_{jj} \big| \le |\lambda_k - \lambda_{k'} | - 4\delta \max\{ \lambda_k, \lambda_{k'} \} -  \breve{o}_{d,\P}(1).
\end{equation*}
The proof is complete.
\end{proof}

Recall that we denote the projection matrix  $\bP = [\ee_1,\ldots,\ee_m][\ee_1,\ldots,\ee_m]^{\top}$, 
and that the top-$D$ eigenvectors of $\bPsi_{\le \ell}^\top \bLambda_{\le \ell}^2\bPsi_{\le \ell}$ 
form $\bU_s$. The next lemma adapts Lemma~\ref{lem:rowselect} to the present case.
\begin{lem}\label{lem:rowselect0}
Suppose that $c' n \le m \le C' n$, where $c', C' \in (0,1)$ are constants. Also suppose that the condition \eqref{cond:improved} is satisfied. Then there exist $C_3,c>0$ such that if $m > (C_3+1)d$, then with very high probability,
\begin{equation*}
\sigma_{\min}( \bP \bU_s ) \ge c.
\end{equation*}
\end{lem}
\begin{proof}[{\bf Proof of Lemma~\ref{lem:rowselect0}}]
Following the proof of Lemma~\ref{lem:rowselect}, we have
\begin{equation*}
\sigma_{\min}(\bP \bU_{\bPsi}) \ge \frac{\sigma_{\min}(\bP \bPsi_{\le \ell} \bV_{\bPsi}) }{\sigma_{\max}(\bD_{\bPsi})} = \frac{\sigma_{\min}(\bP \bPsi_{\le \ell}) } {\sigma_{\max}(\bD_{\bPsi})}.
\end{equation*}
By \eqref{eq:quickref2}, we have $n^{-1/2}\sigma_{\max}(\bD_\bPsi) \le 1+\delta + \breve{o}_{d, \P}(1)$. Let $\XX' \in \R^{m \times d}$ be the matrix formed by the top-$m$ rows of $\XX$. Note that 
\begin{equation*}
 \sigma_{\min}(\bP \bPsi_{\le \ell}) = \min_{\| \uu \|=1} \big\| \bP \bPsi_{\le \ell} \uu   \big \| =  \min_{\| \uu \|=1}\big\| [ \bone_m, \XX' ] \uu \big\| = \sigma_{\min} \big(  [ \bone_m, \XX' ]\big)
\end{equation*}
where the last equality holds since $m \ge d+1 = D$. It suffices to show $ \sigma_{\min} \big(  [ \bone_m, \XX' ]\big) \ge c\sqrt{m} > 0$ for certain constant $c$ with very high probability. 

Let $\bG:=  [ \bone_m, \XX' ]^{\top}[ \bone_m, \XX' ]/m$, and note that
\begin{align*} 
\bG = \left[\begin{matrix}1 & \vv^\top\\
\vv & \bG_0\end{matrix}\right]\, ,\;\;\;\; \vv :=\frac{1}{m}(\XX')^{\top}\bone_m\, ,
\;\;\;\; \bG_0 := \frac{1}{m}(\XX')^\top\XX'\, .
\end{align*}
By  concentration of the norm of subgaussian vectors, we have 
$\|\vv\|_2\le 2\sqrt{d/m}$ with very high probability. By concentration of the eigenvalues of empirical 
covariance matrices \cite{vershynin2018high}, we have $\lambda_0:=\lambda_{\min}(\bG_0) 
\ge 1-3\sqrt{d/m}$ with very high probability.
Further, by the Cauchy-Schwarz inequality,
\begin{align*}
\lambda_{\min}(\bG)\ge  \lambda_{\min}\left(\left[
\begin{matrix}
1 & \|\vv\|_2\\
\|\vv\|_2 & \lambda_0
\end{matrix}\right]\right) = \frac{1}{2}\big[1+\lambda_0-\sqrt{(1-\lambda_0)^2+4\|\vv\|_2^2}\big]\, .
\end{align*}
It follows from the above bounds that $\lambda_{\min}(\bG)\ge 1/10$
with very high probability, provided we choose $C_3$ a large enough constant.
Since $\sigma_{\min} \big(  [ \bone_m, \XX' ]\big) = \sqrt{m\lambda_{\min}(\bG)}$,
this implies the desired lower bound.
\end{proof}

Consider the condition $c_0 d \le n \le d^2/(C (\log d)^C)$. We choose $n'$ to be the smallest integer such that $n' \ge n + (c_0)^{-1} C_\delta n$. As defined in \eqref{def:aug}, the augmented kernel matrix $\tilde \bK$ has size $n_0 \times n_0$, where $n_0$ is guaranteed to satisfy $n_0 \ge  C_\delta d $.

By Jensen's inequality,
\begin{align*}
&\big\| \bK_N^{-1} \bK_N^{(2)} \bK_N^{-1} \big\|_\op  \le \frac{1}{n'} \E_{\xx_{n+1},\ldots,\xx_{n_0}}  \big\| (\tilde \bK_{11})^{-1} (\tilde \bK^2)_{11} (\tilde \bK_{11})^{-1} - \bI_{n} \big\|_\op, \\
&\big\| \bK_N^{-1} \E_\xx \big[ \bK_N(\cdot, \xx) f(\xx) \big] \big\| \le \frac{1}{n'} \E_{\xx_{n+1},\ldots,\xx_{n_0}} \big\| \tilde \bK_{11}^{-1} \tilde \bK_{12} \tilde \ff_2 \big\|
\end{align*}
where $\tilde \ff$ is defined in \eqref{def:tildef}. 

\begin{lem}\label{lem:reduceeig1}
There exist constant $C,C_2>0$ such that the following holds. If $c_0 d \le n \le d^2/(C (\log d)^C)$, then with high probability, 
\begin{align}
& \frac{1}{n'} \E_{\xx_{n+1},\ldots,\xx_{2n}}  \big\| (\tilde \bK_{11})^{-1} (\tilde \bK^2)_{11} (\tilde \bK_{11})^{-1} - \bI_{n} \big\|_\op \le \frac{C_2}{n}, \label{ineq:3K3L1} \\
&\frac{1}{n'} \E_{\xx_{n+1},\ldots,\xx_{2n}} \big\| \tilde \bK_{11}^{-1} \tilde \bK_{12} \tilde \ff_2 \big\| \le \frac{C_2}{\sqrt{n}}. \label{ineq:3K4L1}
\end{align}
Consequently, we obtain \eqref{bnd:key1} and \eqref{bnd:key2} in Lemma~\ref{lem:keybnd} under the relaxed condition $c_0 d \le n \le d^2/(C (\log d)^C)$.
\end{lem}
We note that the remaining proof of Lemma~\ref{lem:keybnd} is the same as before. Once Lemma~\ref{lem:keybnd} is proved, we will obtain Theorem~\ref{thm:gen} under the relaxed condition for $\ell = 1$.
\begin{proof}[{\bf Proof of Lemma~\ref{lem:reduceeig1}}]
Following the proof of Lemma~\ref{lem:reduceeig1}, we have with very high probability,
\begin{align*}
&\big\| (\tilde \bK_{11})^{-1} (\tilde \bK^2)_{11} (\tilde \bK_{11})^{-1} - \bI_{n} \big\|_\op \le C + C \cdot \Big\| \tilde \bU_1\big( \tilde \bU_1^\top  \tilde \bU_1  + \gamma_{>\ell} (\tilde \bD_\gamma)^{-1} \big)^{-2}\tilde \bU_1^\top \Big\|_\op \\
& \big\| \tilde \bK_{11}^{-1} \tilde \bK_{12} \tilde \ff_2 \big\| \le C\| \tilde \ff_2 \|+ C \cdot \Big\| \big( \tilde \bU_1^\top \tilde \bU_1 + \gamma_{>\ell} (\tilde \bD_\gamma)^{-1} \big)^{-1} \Big\|_\op \cdot \| \tilde \ff_2 \| .
\end{align*}
As in the proof of Lemma~\ref{lem:reduceeig1}, it suffices to show 
\begin{equation}\label{ineq:UUlbnd}
\lambda_{\min}\Big( \tilde \bU_1^\top \tilde \bU_1 + \gamma_{>\ell} (\tilde \bD_\gamma)^{-1} \Big) \ge c
\end{equation}
for certain constant $c>0$. In the case $n > (C_3 + 1)d$, the assumptions in Lemma~\ref{lem:eigval1}, \ref{lem:eigvec1}, \ref{lem:rowselect0} are satisfied. Following the proof of Lemma~\ref{lem:reduceeig1}, we have 
\begin{equation}\label{ineq:UUlbnd1}
\lambda_{\min}\Big( \tilde \bU_1^\top \tilde \bU_1 \Big) \ge c.
\end{equation}
Below we consider the case $n \le (C_3 + 1)d$. One difficulty is that we cannot apply Lemma~\ref{lem:rowselect0}; moreover, if $n < d$, then the matrix $ \tilde \bU_1^\top \tilde \bU_1$ is not full rank, so we do not expect \eqref{ineq:UUlbnd1} to hold. 

In order to resolve this issue, we first notice that if $\gamma_0 n_0 \le 2 \gamma_1 n_0 / d$, then \eqref{ineq:UUlbnd} holds with very high probability. In fact, 
\begin{equation*}
\lambda_{\min}\Big( \tilde \bU_1^\top \tilde \bU_1 + \gamma_{>\ell} (\tilde \bD_\gamma)^{-1} \Big) \ge \lambda_{\min}\Big(  \gamma_{>\ell} (\tilde \bD_\gamma)^{-1} \Big) \ge \frac{\gamma_{>\ell}}{\max\{\gamma_0 n_0,  \gamma_1 n_0 / d\}} \ge \frac{\gamma_{>\ell}}{2\gamma_1 n_0 / d} > c.
\end{equation*}
From now on, we assume $\gamma_0 n_0 > 2 \gamma_1 n_0 / d$. Conforming to our notations in the previous proof, we denote the top eigenvector (which corresponds to the eigenvalue closest to $\gamma_0 n_0$) of $\bK_N, \bK, \tilde \bPsi_{\le \ell} \tilde \bLambda_{\le \ell}^2 \tilde \bPsi_{\le \ell}^\top$, respectively, by $\tilde \vv^{(1)}, \tilde \vv_0^{(1)}, \tilde \vv_s^{(1)}$. Also, denote $\tilde \bU = [\tilde \vv^{(1)}, \tilde \bV^{(2)}]$ where $\tilde \bV^{(2)} \in \R^{n_0 \times d}$.

First, we make a claim about  $\vv_s^{(1)}$. There exists a constant $C_3 > 0$ such that if $\gamma_0 n_0 \ge C_3$, then for certain constant $c_0'>0$, with very high probability,
\begin{equation}\label{claim:vs}
\big\| \bP \tilde \vv_s^{(1)} \big\| \ge c_0'.
\end{equation}
To prove this claim, we express $ \tilde \vv_s^{(1)}$ as
\begin{equation*}
\tilde \vv_s^{(1)} = \argmax_{\| \vv \|=1} \Big[ \vv^\top \tilde \bPsi_{\le \ell} \tilde \bLambda_{\le \ell}^2 \tilde \bPsi_{\le \ell}^\top \vv \Big] = \argmax_{\| \vv \|=1} \Big[ \gamma_0 \langle \vv , \bone_{n_0} \rangle^2 + \frac{\gamma_1}{d} \| \tilde \XX \vv \|^2  \Big] 
\end{equation*}
where each row of $\tilde \XX \in \R^{n_0 \times d}$ is a uniform vector in $\S^{d-1}(\sqrt{d})$. Evaluating the maximization problem at $\vv = \tilde \vv_s^{(1)}$ and $\vv = \bone_{n_0} / \sqrt{n_0}$, we obtain
\begin{equation*}
\gamma_0 \langle \tilde \vv_s^{(1)}, \bone_{n_0} \rangle^2 + \frac{\gamma_1}{d} \| \tilde \XX \tilde \vv_s^{(1)} \|^2 \ge \gamma_0 n_0 +  \frac{\gamma_1}{dn_0} \| \tilde \XX \bone_{n_0} \|^2.
\end{equation*}
Since $\| \tilde \XX \|_\op \le 2(\sqrt{n_0}+\sqrt{d}) \le C' \sqrt{n_0}$ with very high probability, we have
\begin{equation*}
\gamma_0 \langle \tilde \vv_s^{(1)}, \bone_{n_0} \rangle^2 \ge \gamma_0 n_0 - C_4\frac{n_0}{d} \ge \gamma_0 n_0 - C_4'
\end{equation*} 
where $C_4,C_4'>0$ are some constants. Thus, we obtain
\begin{equation*}
\langle \tilde \vv_s^{(1)}, \bone_{n_0} / \sqrt{n_0} \rangle^2 \ge 1 -  \frac{C_4'}{\gamma_0 n_0}.
\end{equation*}
We observe that
\begin{align*}
\langle \tilde \vv_s^{(1)}, \bone_{n_0} / \sqrt{n_0} \rangle &\le \langle \bP\tilde  \vv_s^{(1)}, \bone_{n_0} / \sqrt{n_0} \rangle + \langle \bP^\perp \tilde \vv_s^{(1)}, \bone_{n_0} / \sqrt{n_0} \rangle \\
&\stackrel{(i)}{\le} \big\| \bP  \tilde  \vv_s^{(1)} \big\| + \langle \tilde \vv_s^{(1)}, \bP^\perp \bone_{n_0} / \sqrt{n_0} \rangle \\
&\le \big\| \bP  \tilde  \vv_s^{(1)} \big\| + \sqrt{\frac{n_0 - n}{n_0}}
\end{align*}
where in \textit{(i)} we used the fact $\langle \bP \aa, \bb \rangle = \langle  \aa, \bb \bP\rangle$ for any projection matrix $\bP$. Thus, we deduce that 
\begin{align*}
\big\| \bP  \tilde  \vv_s^{(1)} \big\| & \ge \sqrt{1 - C_4'/(\gamma_0 n_0)} - \sqrt{1 - (n/n_0)} \\
&\stackrel{(i)}{\ge} 1 - C_4'/(\gamma_0 n_0)  - (1 - (n/(2n_0)) \\
&\stackrel{(ii)}{\ge} c_0 - C_4'/(\gamma_0 n_0)
\end{align*}
where in \textit{(i)} we used $1-a \le \sqrt{1 - a} \le 1 - (a/2)$ for $a \in (0,1)$, and in \textit{(ii)} $c_0>0$ is certain small constant. Therefore, if $\gamma_0 n_0 \ge 2C_4'$, then $\big\| \bP  \tilde  \vv_s^{(1)} \big\|  \ge c_0/2$. This proves the claim \eqref{claim:vs}.

In order to bound the smallest eigenvalue in \eqref{ineq:UUlbnd}, we consider two cases.

\noindent \textbf{Case 1: $\gamma_0 n_0 < C_3$}. We have
\begin{align*}
\lambda_{\min}\Big( \tilde \bU_1^\top \tilde \bU_1 + \gamma_{>\ell} (\tilde \bD_\gamma)^{-1} \Big)& \ge \lambda_{\min}\Big(  \gamma_{>\ell} (\tilde \bD_\gamma)^{-1} \Big) \ge \frac{\gamma_{>\ell}}{\max\{\gamma_0 n_0,  \gamma_1 n_0 / d\}}  \\
&\ge \frac{\gamma_{>\ell}}{\max\{C_3, C (n/d) \} } \ge  c.
\end{align*}
\noindent \textbf{Case 2: $\gamma_0 n_0 \ge C_3$}. We derive two lower bounds on its variational form. Let $\zz \in \S^{n_0-1}$ by any vector, and we denote $\zz = [z_1, \zz_{-1}]^\top$.  The first lower bound is 
\begin{align}
&~~ \zz^\top \Big[ \tilde \bU_1^\top \tilde \bU_1 + \gamma_{>\ell} (\tilde \bD_\gamma)^{-1} \Big] \zz \notag \\
&\ge z_1^2 \big\|  \bP  \tilde  \vv_s^{(1)} \big\|^2 + 2z_1[ \tilde \vv^{(1)}]^\top \bP \tilde \bV^{(2)} \zz_{-1} + \zz_{-1}^\top [\tilde \bV^{(2)}]^\top \bP \tilde \bV^{(2)} \zz_{-1} + \gamma_{>\ell} \zz^\top (\tilde \bD_\gamma)^{-1} \zz \notag\\
&\ge z_1^2 \big\|  \bP  \tilde  \vv_s^{(1)} \big\|^2 - 2|z_1| \cdot \| \zz_{-1} \| - \| \zz_{-1} \|^2  \notag\\
&\stackrel{(i)}{\ge} (1 - \| \zz_{-1}\|^2) \cdot c_0'- 3\| \zz_{-1}\|. \label{ineq:firstlbnd}
\end{align}
where we used the claim \eqref{claim:vs} in \textit{(i)}. The second lower bound is
\begin{align}
\zz^\top \Big[ \tilde \bU_1^\top \tilde \bU_1 + \gamma_{>\ell} (\tilde \bD_\gamma)^{-1} \Big] \zz &\ge \zz^\top \Big[ \gamma_{>\ell} (\tilde \bD_\gamma)^{-1} \Big] \zz  \ge \gamma_{>\ell} \| \zz_{-1} \|^2 \cdot \frac{c}{\gamma_1 n_0 /d} \notag\\
&\ge \frac{c\gamma_{>\ell} \| \zz_{-1}\|^2d}{\gamma_1 n} \ge c_1 \| \zz_{-1}\|^2 \label{ineq:secondlbnd}
\end{align}
where $c_1>0$ is certain small constant, and in the last inequality we used our assumption $n \le (C_3+1)d$.
If $\| \zz_{-1} \| \le \min\{1/2, c_0'/12\}$, then by the first lower bound \eqref{ineq:firstlbnd}, we have 
\begin{equation*}
\zz^\top \Big[ \tilde \bU_1^\top \tilde \bU_1 + \gamma_{>\ell} (\tilde \bD_\gamma)^{-1} \Big] \zz \ge c_0/4.
\end{equation*}
If $\| \zz_{-1} \| \ge \min\{1/2, c_0'/12\}$ instead, then the second lower bound \eqref{ineq:secondlbnd} implies that the left-hand side above is lower bounded by a constant. Combining the two cases, we obtain 
\begin{equation*}
\lambda_{\min} \Big( \tilde \bU_1^\top \tilde \bU_1 + \gamma_{>\ell} (\tilde \bD_\gamma)^{-1} \Big) \ge c.
\end{equation*}
This proves the desired inequality \eqref{ineq:UUlbnd}. The remaining proof is similar proof of Lemma~\ref{lem:reduceeig1}.
\end{proof}

\section{Generalization error for linear model: proof of Corollary~\ref{cor:linear}}

Throughout this subsection, let the assumptions of Corollary~\ref{cor:linear} hold. Denote $\bar \lambda := \lambda + v(\sigma)$. First, we state and prove a lemma. We will use the simple identity 
\begin{equation*}
\bA^{-1} = \bA_0^{-1} - \bA^{-1} (\bA - \bA_0) \bA_0^{-1}, \qquad \text{for all}~\bA,\bA_0 \in \R^{n \times n}
\end{equation*}
multiple times.

\begin{lem}\label{lem:interceptbnd}
There exist constants $c,C>0$ such that the following holds. With high probability,
\begin{align*}
& \bone_n^\top \big( \bar \lambda \bI_n + \frac{\gamma_1}{d} \XX \XX^\top \big)^{-1} \bone_n \ge \frac{cn}{\bar \lambda }, \\
&\Big\| \gamma_0 \big( \bar \lambda  \bI_n + \gamma_0 \bone_n \bone_n^\top + \frac{\gamma_1}{d} \XX \XX^\top \big)^{-1} \bone_n \Big\| \le \frac{C}{\sqrt{n}}.
\end{align*}
\end{lem}
\begin{proof}[{\bf Proof of Lemma~\ref{lem:interceptbnd}}]
Since $\| \XX \XX^\top \|_\op \le C(n+d)$ with very high probability (by Lemma~\ref{lem:concntr1}), if $n \le 16 d$, we have 
\begin{equation*}
\bone_n^\top \big( \bar \lambda \bI_n + \frac{\gamma_1}{d} \XX \XX^\top \big)^{-1} \bone_n \ge \frac{n}{\lambda_{\max} \big( \bar \lambda \bI_n + \gamma_1 \XX \XX^\top / d \big) } \ge \frac{n}{\bar \lambda + C\gamma_1} \ge \frac{n}{C'\bar \lambda}.
\end{equation*}
where in the last inequality we used $\bar \lambda \ge c$ for a constant $c>0$ so $\bar \lambda + C \lambda_1 \le C' \bar \lambda$. If $n > 16d$, we observe that
\begin{align*}
\bone_n^\top \big( \bar \lambda \bI_n + \frac{\lambda_1}{d} \XX \XX^\top \big)^{-1} \bone_n - \bone_n^\top (\bar \lambda \bI_n )^{-1} \bone_n &= - \bone_n^{-1} (\bar \lambda \bI_n)^{-1} \frac{\lambda_1}{d} \XX \XX^\top \big( \bar \lambda \bI_n + \frac{\lambda_1}{d} \XX \XX^\top \big)^{-1} \bone_n \\
&= - \frac{\lambda_1}{\bar \lambda d } \bone_n^\top \XX \big( \bar \lambda \bI_d + \frac{\lambda_1}{d} \XX^\top \XX \big)^{-1} \XX^\top \bone_n.
\end{align*}
Since $\lambda_{\min}(\XX^\top \XX) \ge n/4$ and $\| \XX^\top \bone_n \|^2 \le 2nd$ with high probability due to Lemma~\ref{lem:concntr1}, we get 
\begin{align*}
\bone_n^\top \big( \bar \lambda \bI_n + \frac{\lambda_1}{d} \XX \XX^\top \big)^{-1} \bone_n  &\ge \frac{n}{\bar \lambda} - \frac{\lambda_1}{\bar \lambda d} \cdot \| \XX^\top \bone_n \|^2 \cdot \frac{4d}{\lambda_1 n} \\
&\ge \frac{n}{\bar \lambda} - \frac{\lambda_1}{\bar \lambda d} \cdot 2nd \cdot \frac{4d}{\lambda_1 n} \ge \frac{n}{\bar \lambda} - \frac{8d}{\bar \lambda} \\
&\ge \frac{n}{2\bar \lambda}.
\end{align*}
In either case, the first inequality of the lemma holds with high probability. Next, we notice that 
\begin{align*}
&~~\big(\bar \lambda  \bI_n + \gamma_0 \bone_n \bone_n^\top + \frac{\gamma_1}{d} \XX \XX^\top \big)^{-1} \bone_n - \big( \bar \lambda  \bI_n + \frac{\gamma_1}{d} \XX \XX^\top \big)^{-1} \bone_n \\
&= - \big( \bar \lambda \bI_n + \gamma_0 \bone_n \bone_n^\top + \frac{\gamma_1}{d} \XX \XX^\top \big)^{-1}  \gamma_0 \bone_n \bone_n^\top \big( \bar \lambda  \bI_n + \frac{\gamma_1}{d} \XX \XX^\top \big)^{-1} \bone_n.
\end{align*}
By re-arranging the equality, we get 
\begin{align*}
\Big\|  \big( \bar \lambda  \bI_n + \gamma_0 \bone_n \bone_n^\top + \frac{\gamma_1}{d} \XX \XX^\top \big)^{-1} \bone_n  \Big\| &= \Big[1 + \gamma_0 \bone_n^\top \big( \bar \lambda  \bI_n + \frac{\gamma_1}{d} \XX \XX^\top \big)^{-1} \bone_n  \Big]^{-1} \Big\|  \big( \bar \lambda  \bI_n + \frac{\gamma_1}{d} \XX \XX^\top \big)^{-1} \bone_n \Big\| \\
&\le \frac{\sqrt{n}/\bar \lambda }{1+c\gamma_0n/\bar \lambda } = \frac{\sqrt{n}}{\bar \lambda +c\gamma_0n} \le \frac{C}{\gamma_0 \sqrt{n}},
\end{align*}
which proves the lemma.
\end{proof}

Building on this lemma, we prove some useful bounds.
For convenience, we define 
\begin{equation}\label{def:A}
\bA= \bar \lambda \bI_n + \gamma_0 \bone_n \bone_n^\top + \frac{\gamma_1}{d} \XX \XX^\top, \qquad \bA_0 = \bar \lambda \bI_n + \frac{\gamma_1}{d} \XX \XX^\top.
\end{equation}
\begin{lem}\label{lem:usefulbnd}
We have
\refstepcounter{equation}\label{ineq:frobbnd}
\begin{align}
&\gamma_1 \big\| \bA^{-1} \XX \big\|_\op = O_{d,\P}\big( \sqrt{d}  \big), \qquad \gamma_1 \big\| \bA^{-1} \XX \big\|_F = O_{d,\P}\big( d  \big), \tag{\theequation a}\label{claim:frobbnd1} \\
&\gamma_1 \big\| \XX^\top \bA^{-1} \XX \big\|_\op = O_{d,\P}\big( d\big), \qquad \gamma_1 \big\| \XX^\top \bA^{-1} \XX \big\|_F = O_{d,\P}\big( d^{3/2}\big),\tag{\theequation b} \label{claim:frobbnd2} \\
& \gamma_1 \big\| \XX  \XX^\top \bA_0^{-1} \big\|_\op = O_{d,\P}(d), \qquad \gamma_1 \big\| \XX  \XX^\top \bA^{-1} \big\|_\op = O_{d,\P}(d). \tag{\theequation c}\label{claim:frobbnd3}
\end{align}
\end{lem}
\begin{proof}[{\bf Proof of Lemma~\ref{lem:usefulbnd}}]
In this proof, we will use the bounds on the singular values of $\XX$ from Lemma~\ref{lem:concntr1}. To prove the  two bounds in \eqref{claim:frobbnd1}, we observe that
\begin{align*}
\big\| \bA^{-1} \XX - \bA_0^{-1} \XX\big\|_F &= \big\| \bA^{-1} \gamma_0 \bone_n \bone_n^\top \bA_0^{-1}\XX \big\|_F \le \gamma_0 \big\| \bA^{-1} \bone_n \big\| \cdot \big\| \XX^\top \bA_0^{-1} \bone_n \big\| \\
&\stackrel{(i)}{\le}  \frac{C}{\sqrt{n}} \cdot \sqrt{n}\, \big\| \bA_0^{-1} \XX \big\|_\op \le C \big\| \bA_0^{-1} \XX \big\|_\op
\end{align*}
where in \textit{(i)} we used Lemma~\ref{lem:interceptbnd}. So by the triangle inequality, we have
\begin{equation*}
\big\| \bA^{-1} \XX \big\|_\op \le C \big\| \bA_0^{-1} \XX \big\|_\op, \qquad \big\| \bA^{-1} \XX \big\|_F \le \big\| \bA_0^{-1} \XX \big\|_F + C \big\| \bA_0^{-1} \XX \big\|_\op \le C\big\| \bA_0^{-1} \XX \big\|_F.
\end{equation*}
If $n/d > 1+c$ for any constant $c>0$, then by Lemma~\ref{lem:concntr1} with high probability $\lambda_{\min}(\XX^\top \XX) > c'n$ for certain constant $c'>0$. Since $n>c_0 d$, we also have $\| \XX \|_\op \le C \sqrt{n}$ and $\| \XX \|_F \le C n$ with high probability. Thus,
\begin{align}
& \big\| \bA_0^{-1} \XX \big\|_\op = \big\| \XX \big(\bar \lambda \bI_d + \frac{\gamma_1}{d} \XX^\top \XX \big)^{-1} \big\|_\op \le \big\| \XX \big\|_\op \cdot \big\| \big(\bar \lambda \bI_d + \frac{\gamma_1}{d} \XX^\top \XX \big)^{-1}\big\|_\op \stackrel{(i)}{\le} C\sqrt{n} \cdot \frac{d}{\gamma_1 n}, \label{ineq:tmpbnd} \\
& \big\| \bA_0^{-1} \XX \big\|_F  \le \big\| \XX \big\|_F \cdot \big\| \big(\bar \lambda \bI_d + \frac{\gamma_1}{d} \XX^\top \XX \big)^{-1}\big\|_\op \le Cn \cdot \frac{d}{\gamma_1 n} \le \frac{Cd}{\gamma_1}.\notag
\end{align}
where in \textit{(i)} we used $\| \bM_1 \bM_2 \|_F \le \| \bM_1 \|_F \cdot \| \bM_2 \|_\op$ for matrices $\bM_1, \bM_2$. If $n/d \le 1+c$ instead, then  $\big\| \bA_0^{-1} \XX \big\|_\op \le C \big\| \XX \big\|_\op \le C\sqrt{n} \le O_{d,\P}(\sqrt{d})$, and $\big\| \bA_0^{-1} \XX \big\|_F \le C \big\| \XX \big\|_F \le Cn \le O_{d,\P}(d)$. In either cases, the two bounds in \eqref{claim:frobbnd1} are true. To prove the  two bounds in \eqref{claim:frobbnd2}, we note that 
\begin{align*}
\big\| \XX^\top \bA^{-1} \XX - \XX^\top \bA_0^{-1} \XX \big\|_\op &=  \big\| \XX^\top\bA^{-1} \gamma_0 \bone_n \bone_n^\top \bA_0^{-1}\XX \big\|_\op \stackrel{(i)}{\le} \big\| \XX \big\|_\op \cdot \big\| \gamma_0 \bA^{-1} \bone_n \big\| \cdot \sqrt{n} \, \big\| \bA_0^{-1} \XX \big\|_\op \\
&\le O_{d,\P}(\sqrt{n}) \cdot O_{d,\P}(n^{-1/2}) \cdot O_{d,\P}(\gamma_1^{-1} d) = O_{d,\P}(\gamma_1^{-1} d)\\
\big\| \XX^\top \bA_0^{-1} \XX \big\|_\op &= \big\| \XX^\top \XX \big( \bar \lambda \bI_d + \frac{\lambda_1}{d} \XX^\top \XX \big)^{-1} \big\|_\op \le \frac{d}{\gamma_1}.
\end{align*}
where we used the bound \eqref{ineq:tmpbnd} in \textit{(i)}. This leads to $\gamma_1 \| \XX^\top \bA^{-1} \XX \|_\op = O_{d,\P}(d)$. Also, the rank of $\XX^\top \bA^{-1} \XX$ is at most $d$, so 
\begin{equation*}
\gamma_1 \big\| \XX^\top \bA^{-1} \XX \big\|_F \le \gamma_1 \sqrt{d}\, \big\| \XX^\top \bA^{-1} \XX \big\|_\op = O_{d,\P}(d^{3/2}).
\end{equation*}
To prove the two bounds in \eqref{claim:frobbnd3}, we let the eigenvalue decomposition of $\XX \XX^\top$ be $\XX \XX^\top = \bU_\XX \bD_\XX \bU_\XX^\top$ and we find
\begin{equation*}
\gamma_1 \XX \XX^\top \bA_0^{-1} = \bU_\XX \diag \Big\{ \frac{\lambda_1(\XX \XX^\top)}{\gamma_1\lambda_1(\XX \XX^\top)/d + \bar \lambda}, \ldots,  \frac{\lambda_d(\XX \XX^\top)}{\gamma_1\lambda_d(\XX \XX^\top)/d + \bar \lambda}  \Big\} \bU_\XX^\top.
\end{equation*}
The diagonal entries are all no larger than $d$, so this proves the first bound in \eqref{claim:frobbnd3}. Moreover,
\begin{align*}
\big\| \XX \XX^\top \bA^{-1} - \XX \XX^\top \bA_0^{-1} \big\|_\op &= \big\| \XX \XX^\top \bA_0^{-1} \gamma_0 \bone_n \bone_n^\top \bA^{-1} \big\|_\op \\
&\le  \big\| \XX \XX^\top \bA_0^{-1} \big\|_\op \cdot \sqrt{n} \cdot \big\| \gamma_0 \bA^{-1} \bone_n \big\| \\
&\le O_{d,\P}(d) \cdot \sqrt{n} \cdot O_{d,\P}(n^{-1/2}) = O_{d,\P}(d),
\end{align*}
which proves the second bound.
\end{proof}

To study the risk $R_{\PRR}(\bar \lambda)$ in the linear case, we first notice that $K^p(\xx_i, \xx_j) = \gamma_0 + \frac{\gamma_1}{d} \xx_i^\top \xx_j$. Thus, we have
\begin{equation*}
\bK^p = \gamma_0 \bone_n \bone_n^\top + \frac{\gamma_1}{d} \XX \XX^\top, \qquad \bK^{(p,2)} = \gamma_0^2 \bone_n \bone_n^\top + \frac{\gamma_1^2}{d^2} \XX \XX^\top
\end{equation*}
where the second equality is due to $\E [ \xx \xx^\top ] = \bI_d$.  Now we break down the risk.
\begin{align*}
R_{\PRR}(\bar \lambda) = E_{\bias}^p(\bar \lambda) + E_{\var}^p(\bar \lambda) + E_{\cross}^p(\bar \lambda).
\end{align*}
The expressions for the bias/variance/cross terms are given below. 
\begin{align*}
E_{\bias}^p &= \E_\xx \Big[ \big( f(\xx) - \bK^p(\cdot, \xx)^\top (\bar \lambda \bI_n + \bK^p)^{-1} \ff \big)^2 \Big] \\
&= \| f \|_{L^2}^2 - 2 \E_\xx \Big[ f(\xx) \bK^p(\cdot, \xx)^\top (\bar \lambda \bI_n + \bK^p )^{-1} \ff \Big] + \ff^\top (\bar  \lambda \bI_n+ \bK^p)^{-1} \bK^{(p,2)} (\bar \lambda  \bI_n+ \bK^p)^{-1} \ff \\
&= \| f \|_{L^2}^2 - 2 \frac{\gamma_1}{d} \bbeta_*^\top \XX^\top (\bar \lambda \bI_n + \bK^p )^{-1}  \XX \bbeta_* \\
& + \bbeta_*^\top \XX^\top (\bar  \lambda \bI_n+ \bK^p)^{-1} \Big(\gamma_0^2 \bone_n \bone_n^\top + \frac{\gamma_1^2}{d^2} \XX \XX^\top \Big)(\bar  \lambda \bI_n+ \bK^p)^{-1} \XX \bbeta_* =: \| f \|_{L^2}^2 - 2 I_1^{(1)} + I_{1,1}^{(2)}  ,\\
E_{\var}^p  &=  \bveps^\top (\bar  \lambda \bI_n+ \bK^p)^{-1} \Big(\gamma_0^2 \bone_n \bone_n^\top + \frac{\gamma_1^2}{d^2} \XX \XX^\top \Big)(\bar  \lambda \bI_n+ \bK^p)^{-1}\bveps =: I_{\veps, \veps}^{(2)}, \\
E_{\cross}^p  &= \frac{\gamma_1}{d} \bveps^\top (\bar \lambda \bI_n + \bK^{p})^{-1} \XX \bbeta_* - \bveps^\top (\bar \lambda \bI_n + \bK^{p})^{-1} \Big(\gamma_0^2 \bone_n \bone_n^\top + \frac{\gamma_1^2}{d^2} \XX \XX^\top \Big)(\bar \lambda \bI_n + \bK^{p})^{-1}  \XX \bbeta_* \\
&=: I_{\veps}^{(1)} - I_{\veps,1}^{(2)}.
\end{align*}

To simplify these expressions, we notice that we can assume $\bbeta_*$ to be uniformly random on a sphere with given radius $r:=\|\bbeta_*\|_2$: $\bbeta_* \sim \mathrm{Unif}(\S^{d-1}(r))$. To understand why this is the case, let us denote the risk $R_{\PRR} = R_{\PRR}(\xx_1,\ldots,\xx_n, \bbeta_*)$ to emphasize its dependence on $\XX$, $\bbeta_*$. Let $\bR \in  \cO(n)$ be a random orthogonal matrix with uniform (Haar) distribution. Then, 
\begin{align*}
R_{\PRR}(\xx_1,\ldots,\xx_n,  \bbeta_*) &= R_{\PRR}(\bR\xx_1,\ldots,\bR\xx_n, \bR\bbeta_*) \stackrel{d}{=} R_{\PRR}(\xx_1,\ldots,\xx_n, \bR\bbeta_*).
\end{align*}
The first equality follows from straightforward calculation, and the distributional equivalence $\stackrel{d}{=}$ holds because the distribution of $\xx_i$ is orthogonally invariant.
Considering therefore, without loss of generality, $\bbeta_* \sim \mathrm{Unif}( \S^{d-1}(r))$, it is easy to further simplify the terms $E_{\bias}^p , E_{\var}^p , E_{\cross}^p $. Indeed, we expect these terms to concentrate around their expectation with respect to $\bbeta_*$,
$\bveps$. This is formally stated in the next lemma.
\begin{lem}[Concentration to the trace]\label{lem:quadraticform}
We have
\refstepcounter{equation}\label{ineq:Idiff}
\begin{align}
&\max\big\{ \big| I^{(1)}_{\veps} \big|,  \big|I^{(2)}_{\veps, 1}\big|  \big\} = O_{d,\P}\Big( \frac{ \tau^2 }{\sqrt{d}} \Big), \tag{\theequation a} \label{ineq:Idiffa}\\
& \Big| I_1^{(1)} - \frac{\gamma_1\| \bbeta_* \|^2 }{d^2} \Tr\Big( \XX^\top \bA^{-1} \XX \Big) \Big| =   O_{d,\P}\Big( \frac{\tau^2}{\sqrt{d}}  \Big), \tag{\theequation b} \label{ineq:Idiffb} \\
& \Big| I_{1,1}^{(2)} - \frac{\| \bbeta_* \|^2 }{d} \Tr\Big( \XX^\top \bA^{-1} \Big(\gamma_0^2 \bone_n \bone_n^\top + \frac{\gamma_1^2}{d^2} \XX \XX^\top \Big) \bA^{-1} \XX  \Big) \Big| =  O_{d,\P}\Big( \frac{\tau^2}{\sqrt{d}}  \Big), \tag{\theequation c} \label{ineq:Idiffc} \\
&  \Big| I_{\veps,\veps}^{(2)} - \sigma_\veps^2 \Tr\Big( \bA^{-1} \Big(\gamma_0^2 \bone_n \bone_n^\top + \frac{\gamma_1^2}{d^2} \XX \XX^\top \Big) \bA^{-1} \Big) \Big| =  O_{d,\P}\Big( \frac{\tau^2}{\sqrt{d}}  \Big). \tag{\theequation d} \label{ineq:Idiffd}
\end{align}
\end{lem}

To further simplify the terms, we will show the following lemma.
\begin{lem}\label{lem:linreduce2}
We have
\refstepcounter{equation}\label{eq:linreduce}
\begin{align}
&\Big| \Tr \Big( \bA^{-1} \big( \gamma_0^2 \bone_n \bone_n^\top + \frac{\gamma_1^2}{d^2} \XX \XX^\top \big) \bA^{-1} \Big) - \frac{\gamma_1^2}{d^2} \Tr \Big( \bA_0^{-1} \XX \XX^\top \bA_0^{-1} \Big)\Big| = O_{d,\P}\Big(\frac{1}{d}\Big), \tag{\theequation a}\label{eq:linreduce1} \\
&\Big|\frac{1}{d} \Tr \Big( \XX^\top \bA^{-1} \big( \gamma_0^2 \bone_n \bone_n^\top + \frac{\gamma_1^2}{d^2} \XX \XX^\top \big) \bA^{-1} \XX \Big) - \frac{\gamma_1^2}{d^3} \Tr \Big( \XX^\top \bA_0^{-1} \XX \XX^\top\bA_0^{-1} \XX \Big)\Big| = O_{d,\P}\Big( \frac{1}{d} \Big),  \tag{\theequation b}\label{eq:linreduce2} \\
& \Big| \frac{\gamma_1}{d^2} \Tr \Big( \XX^\top \bA^{-1} \XX \Big) -  \frac{\gamma_1}{d^2}\Tr \Big( \XX^\top \bA_0^{-1} \XX \Big) \Big| = O_{d,\P}\Big( \frac{1}{d} \Big). \tag{\theequation c}\label{eq:linreduce3}
\end{align}
\end{lem}
Once these bounded are proved, then by Lemma~\ref{lem:algebra} we have
\refstepcounter{equation}\label{eq:S}
\begin{align}
&\frac{\gamma_1^2}{d^2} \Tr \Big( \bA_0^{-1} \XX \XX^\top \bA_0^{-1} \Big) = \frac{\gamma_1^2}{d^2} \Tr\Big(\XX \big( \bar \lambda \bI_d + \frac{\gamma_1}{d} \XX^\top \XX \big)^{-2} \XX^\top \Big) = \frac{1}{d} \Tr \Big( \bS \big( \gamma \bI_d + \bS\big)^{-2} \Big), \tag{\theequation a}\label{eq:Sa} \\
&\frac{\gamma_1^2}{d^3} \Tr \Big( \XX^\top \bA_0^{-1} \XX \XX^\top \bA_0^{-1} \XX \Big) =\frac{\gamma_1^2}{d^3} \Tr\Big(\XX^\top \XX \big( \bar \lambda \bI_d + \frac{\gamma_1}{d} \XX^\top \XX \big)^{-2} \XX^\top \XX \Big) = \frac{1}{d} \Tr\Big(\bS^2\big( \gamma \bI_d + \bS\big)^{-2} \Big) , \tag{\theequation b}\label{eq:Sb}\\
& \frac{\gamma_1}{d^2}\Tr \Big( \XX^\top \bA_0^{-1} \XX \Big) = \frac{\gamma_1}{d^2}\Tr \Big( \XX^\top\XX \big( \bar \lambda \bI_d + \frac{\gamma_1}{d} \XX^\top \XX \big)^{-1}  \Big) = \frac{1}{d} \Tr \Big( \bS \big( \gamma \bI_d + \bS\big)^{-1} \Big).\tag{\theequation c}\label{eq:Sc}
\end{align}
where we recall $\gamma := \gamma_{\seff}(\lambda, \sigma) = \bar \lambda / \gamma_1$ and denote $\bS = \XX^\top \XX / d$. Hence, we derive
\begin{align*}
E_{\bias}^p &= \| f \|_{L^2}^2 - \frac{2\| \bbeta_* \|^2}{d} \Tr \Big( \bS (\gamma \bI_d + \bS)^{-1} \Big) +  \frac{\| \bbeta_* \|^2}{d} \Tr \Big( \bS^2 (\gamma \bI_d + \bS)^{-2}\Big) + O_{d,\P}(\tau^2/\sqrt{d}) \\
&= \frac{\| \bbeta_* \|^2}{d} \Tr\Big( \bI_d \Big) -  \frac{2\| \bbeta_* \|^2}{d} \Tr \Big( \bS (\gamma \bI_d + \bS)^{-1} \Big) +  \frac{\| \bbeta_* \|^2}{d} \Tr \Big( \bS^2 (\gamma \bI_d + \bS)^{-2} \Big) + O_{d,\P}(\tau^2/\sqrt{d}) \\
&= \frac{\gamma^2 \| \bbeta_* \|^2 }{d} \Tr \Big( (\gamma \bI_d + \bS)^{-2}\Big) + O_{d,\P}(\tau^2/\sqrt{d})
\end{align*}
where the last equality is due to $(\gamma \bI_d + \bS)^2 - 2\bS (\gamma \bI_d + \bS) - \bS^2 = \gamma^2 \bI_d$. Also, 
\begin{align*}
E_{\var}^p = \frac{\sigma_\veps^2}{d} \Tr \Big( \bS ( \gamma \bI_d + \bS)^{-2} \Big) + O_{d,\P}(\tau^2/\sqrt{d}), \qquad  E_{\cross}^p = O_{d,\P}(\tau^2/\sqrt{d}).
\end{align*}
This proves that
\begin{align}\label{eq:PRRfinal}
R_{\PRR}(\bar \lambda) = \frac{\gamma^2 \| \bbeta_* \|^2 }{d} \Tr \Big( (\gamma \bI_d + \bS)^{-2}\Big) + \frac{\sigma_\veps^2}{d} \Tr \Big( \bS ( \gamma \bI_d + \bS)^{-2} \Big) + O_{d,\P}(\tau^2/\sqrt{d})
\end{align}
Similar to the decomposition of the risk $R_{\PRR}(\bar \lambda)$ and Lemma~\ref{lem:linreduce2}, we can show that $R_{\slin}(\gamma)$ has a similar variance-bias decomposition and that the components concentrate to their respective trace, which yield Eqs.~\eqref{eq:Lin-NT}--\eqref{eq:Lin-NT2}. Combining these with Theorem~\ref{thm:gen} and \eqref{eq:PRRfinal}, we arrive at the claims in Corollary~\ref{cor:linear}. 

To complete the proof, below we first prove Lemma~\ref{lem:linreduce2} and then Lemma~\ref{lem:quadraticform}.
 
\begin{proof}[{\bf Proof of Lemma~\ref{lem:linreduce2}}]
To show the claim \eqref{eq:linreduce1}, first we notice that 
\begin{equation*}
0 \le \Tr \Big( \bA^{-1} \gamma_0^2 \bone_n \bone_n^\top \bA^{-1} \Big) = \big\| \gamma_0 \bA^{-1}  \bone_n \big\|^2 \le O_{d,\P}\Big( \frac{1}{n} \Big)
\end{equation*}
where we used Lemma~\ref{lem:interceptbnd}. Further, 
\begin{align*}
\Big| \Tr \Big( \bA^{-1} \frac{\gamma_1^2}{d^2} \XX \XX^\top \bA^{-1} \Big) - \Tr \Big( \bA_0^{-1} \frac{\gamma_1^2}{d^2} \XX \XX^\top \bA^{-1} \Big) \Big| &= \Big| \Tr \Big(\bA^{-1} \gamma_0 \bone_n \bone_n^\top \bA_0^{-1} \frac{\gamma_1^2}{d^2} \XX \XX^\top \bA^{-1} \Big) \Big| \\
&\le \big\| \gamma_0\bA^{-1} \bone_n \big\| \cdot \frac{\gamma_1^2 }{d^2} \big\| \bA^{-1} \XX \XX^\top \bA_0^{-1} \bone_n \big\| \\
&\stackrel{(i)}{\le} \frac{C}{\sqrt{n}} \cdot  \frac{\gamma_1^2 }{\bar \lambda d^2} \cdot \big\| \XX \XX^\top \bA_0^{-1} \big\|_\op\cdot  \sqrt{n} \\
&\le \frac{C\gamma_1  }{\bar 
\lambda d} = O_{d,\P}\Big( \frac{1}{d} \Big)
\end{align*}
where in \textit{(i)} we used Lemma~\ref{lem:interceptbnd} and $\| \bA^{-1} \|_\op \le \bar \lambda^{-1}$, in \textit{(ii)} we used $\frac{\gamma_1}{d} \| \XX \XX^\top \bA_0^{-1} \|_\op \le 1$ from Lemma~\ref{lem:usefulbnd}.
Similarly, we have
\begin{equation*}
\Big| \Tr \Big( \bA_0^{-1} \frac{\gamma_1^2}{d^2} \XX \XX^\top \bA^{-1} \Big) - \Tr \Big( \bA_0^{-1} \frac{\gamma_1^2}{d^2} \XX \XX^\top \bA_0^{-1} \Big) \Big|  = O_{d,\P}\Big( \frac{1}{d} \Big)
\end{equation*}
which proves \eqref{eq:linreduce1}. Next,, we have
\begin{align*}
\frac{1}{d} \Big|\Tr \Big( \XX^\top \bA^{-1} \gamma_0^2 \bone_n \bone_n^\top \bA^{-1} \XX \Big)\Big| &\le \frac{1}{d} \big\| \gamma_0 \XX^\top \bA^{-1} \bone_n \big\|^2 \le \frac{1}{d} \big\| \XX \big\|_\op^2 \cdot \big\| \gamma_0\bA^{-1} \bone_n \big\|^2 \\
&\le O_{d,\P}\Big( \frac{n}{d} \Big) \cdot \big\|\gamma_0 \bA^{-1} \bone_n \big\|^2 \le  O_{d,\P}\Big( \frac{1}{d} \Big)
\end{align*}
where we used Lemma~\ref{lem:interceptbnd}. Also, 
\begin{align*}
&\frac{1}{d} \Big| \Tr \Big( \XX^\top \bA^{-1} \frac{\gamma_1^2}{d^2} \XX \XX^\top \bA^{-1}\XX \Big) - \Tr \Big( \XX^\top \bA_0^{-1} \frac{\gamma_1^2}{d^2} \XX \XX^\top \bA^{-1} \XX\Big) \Big| \\
& = \frac{\gamma_1^2}{d^3} \Big| \Tr \Big( \XX^\top \bA^{-1} \gamma_0 \bone_n \bone_n^\top \bA_0^{-1} \XX \XX^\top \bA^{-1} \XX \Big) \Big| \\
&=  \frac{\gamma_1^2}{d^3} \Big| \Tr \Big( \bA^{-1} \gamma_0 \bone_n \bone_n^\top \bA_0^{-1} \XX \XX^\top \bA^{-1} \XX \XX^\top \Big) \Big| \\
&\le \frac{1}{d^3} \big\| \gamma_0 \bA^{-1} \bone_n \big\| \cdot \| \gamma_1 \XX \XX^\top \bA^{-1} \big\|_\op \cdot \big\| \gamma_1 \XX \XX^\top \bA_0^{-1} \big\|_\op \cdot \sqrt{n} \\
&\le O_{d,\P}\Big(\frac{1}{d}\Big)
\end{align*}
where we used Lemma~\ref{lem:interceptbnd} and~\ref{lem:usefulbnd} in the last inequality. Similarly, we can also prove 
\begin{equation*}
\frac{1}{d} \Big| \Tr \Big( \XX^\top \bA^{-1} \frac{\gamma_1^2}{d^2} \XX \XX^\top \bA_0^{-1}\XX \Big) - \Tr \Big( \XX^\top \bA_0^{-1} \frac{\gamma_1^2}{d^2} \XX \XX^\top \bA_0^{-1} \XX\Big) \Big| \le O_{d,\P}\Big(\frac{1}{d}\Big),
\end{equation*}
which proves \eqref{eq:linreduce2}.
Finally, we observe
\begin{align*}
\Tr \Big( \XX^\top \bA^{-1} \XX \Big) -  \Tr \Big( \XX^\top \bA_0^{-1} \XX \Big) & = \Tr \Big( \XX^\top \bA^{-1} \gamma_0 \bone_n \bone_n^\top \bA_0^{-1} \XX \Big) \\
&= \Tr \Big(  \bA^{-1} \gamma_0 \bone_n \bone_n^\top \bA_0^{-1} \XX \XX^\top\Big) .
\end{align*}
Thus we get 
\begin{align*}
\frac{\gamma_1}{d^2} \Big| \Tr \Big( \XX^\top \bA^{-1} \XX \Big) -  \Tr \Big( \XX^\top \bA_0^{-1} \XX \Big)  \Big| \le \frac{\sqrt{n}}{d} \big\| \gamma_0 \bA^{-1} \bone_n \big\| \cdot \big\| \frac{\gamma_1}{d} \bA_0 ^{-1} \XX \XX^\top \big\|_\op \le O_{d,\P}\Big( \frac{1}{d} \Big).
\end{align*}
which proves \eqref{eq:linreduce3}.
\end{proof}

\begin{proof}[{\bf Proof of Lemma~\ref{lem:quadraticform}}]

\textbf{Step 1.} 
We will use the variance calculation of quadratic forms from \cite{mei2019generalization}: suppose $\bgg$ is a vector of i.i.d.~random variables $g_i \sim \P_g$ which have zero mean and variance $\sigma_g^2$, and $\hh$ is another vector of i.i.d.~random variables $h_i \sim \P_h$ with zero mean and variance $\sigma_h^2$, then for matrices $\bA, \bB$,
\begin{equation}\label{eq:var}
\var [\bgg^\top \bA \hh] = \sigma_g^2\sigma_h^2\| \bA \|_F^2, \qquad \var [\bgg^\top \bB \bgg] = \sum_{i=1}^n B_{ii}^2 ( \E[ g^4]-3\sigma_g^4) + \sigma_g^4\| \bB \|_F^2 + \sigma_g^4\Tr(\bB^2).
\end{equation}
In particular, if $\bB$ is symmetric and $\E(g_i^4) \le C \sigma_g^4$, then the second variance is bounded by $O(\sigma_g^4) \cdot \| \bB \|_F^2$. In order to apply these identities, we use the Gaussian vector $\bgg \in \cN(\bzero, \frac{\| \bbeta_*\|^2}{d} \bI_d)$ as a surrogate for $\bbeta_* \sim \mathrm{Unif}(\| \bbeta_*\| \S^{d-1})$, which only incurs small errors since Gaussian vector norm concentrates in high dimensions. We use $\tilde I_\veps^{(1)}, \tilde I_1^{(1)}, \tilde I_{\veps, \veps}^{(2)}, \tilde I_{1,1}^{(2)}, \tilde I_{\veps, 1}^{(2)}$ to denote the surrogate terms after we replace $\bbeta_*$ by $\bgg$ in the definitions of $I_\veps^{(1)}, I_1^{(1)}, I_{\veps, \veps}^{(2)}, I_{1,1}^{(2)}, I_{\veps, 1}^{(2)}$. By homogeneity, we assume that $\| \bbeta_*\| = \sigma_\veps = 1$ without loss of generality.

Using Lemma~\ref{lem:usefulbnd}, we have 
\begin{align*}
\var \big[ \tilde I_\veps^{(1)} \big] \le \frac{C\gamma_1 }{d^3} \big\| \bA^{-1} \XX \big\|_F^2 = O_{d,\P}\Big( \frac{1}{d} \Big), \qquad \var \big[ \tilde I_1^{(1)} \big] \le \frac{C\gamma_1^2}{d^4} \big\| \XX^\top \bA^{-1} \XX \big\|_F^2 = O_{d,\P}\Big( \frac{1}{d} \Big).
\end{align*}
Also, using Lemma~\ref{lem:usefulbnd} and Lemma~\ref{lem:linreduce2},
\begin{align*}
 \frac{1}{d} \big\| \bA^{-1} \gamma_0^2 \bone_n \bone_n^\top  \bA^{-1} \XX \big\|_F^2  &= \frac{1}{d} \big\|  \gamma_0 \bA^{-1} \bone_n \big\|^2 \cdot \big\| \gamma_0 \XX^\top \bA^{-1} \bone_n \big\|^2
\le \frac{1}{d} \big\| \gamma_0 \bA^{-1}\bone_n \|^4 \cdot \big\| \XX \big\|_\op^2 \\
&\le O_{d,\P}\Big(\frac{1}{dn^2}\Big) \cdot O_{d,\P}(n) = O_{d,\P}\Big( \frac{1}{dn} \Big). \\
\frac{1}{d} \big\| \bA^{-1} \frac{\gamma_1^2}{d^2} \XX \XX^\top  \bA^{-1} \XX \big\|_F^2 & \le \frac{1}{d^5} \big\| \gamma_1 \bA^{-1} \XX \big\|_\op^2 \cdot \big\| \gamma_1\XX^\top \bA^{-1} \XX \big\|_F^2 = O_{d,\P}\Big(\frac{1}{d}\Big).
\end{align*}
Combining the two bounds proves 
\begin{equation*}
\var \big[ \tilde I_{\veps, 1}^{(2)}  \big]  \le \frac{C}{d} \Big\| \bA^{-1} \big( \gamma_0^2 \bone_n \bone_n^\top +\frac{\gamma_1^2}{d^2} \XX \XX^\top \big)  \bA^{-1} \XX \big\|_F^2 = O_{d,\P}\Big(\frac{1}{d}\Big).
\end{equation*}
Similar to this, we derive 
\begin{align*}
& \big\| \bA^{-1} \gamma_0^2 \bone_n \bone_n^\top  \bA^{-1} \big\|_F^2 = \big\| \gamma_0 \bA^{-1} \bone_n \big\|^4 = O_{d,\P}\Big( \frac{1}{n^2} \Big) \le O_{d,\P}\Big( \frac{1}{d^2} \Big), \\
&  \big\| \bA^{-1} \frac{\gamma_1^2}{d^2} \XX \XX^\top  \bA^{-1} \big\|_F^2 = 
\frac{1}{d^4} \big\| \gamma_1 \bA^{-1} \XX \big\|_F^2 \cdot  \big\| \gamma_1 \bA^{-1} \XX \big\|_\op^2 = O_{d,\P}\Big( \frac{1}{d} \Big), \\
& \frac{1}{d^2}\big\| \XX^\top \bA^{-1} \gamma_0^2 \bone_n \bone_n^\top  \bA^{-1} \XX \big\|_F^2 = \frac{1}{d^2}\big\| \gamma_0 \XX^\top \bA^{-1} \bone_n \big\|^4 \le \frac{1}{d^2}\| \XX \|_\op^2 \cdot \big\| \gamma_0 \bA^{-1} \bone_n \big\|^4 \le O_{d,\P}\Big( \frac{1}{n^2} \Big), \\
&  \frac{1}{d^2}\big\| \XX^\top \bA^{-1} \frac{\gamma_1^2}{d^2} \XX \XX^\top  \bA^{-1} \XX \big\|_F^2 = 
\frac{1}{d^6} \big\| \gamma_1 \XX^\top \bA^{-1} \XX \big\|_F^2 \cdot  \big\| \gamma_1 \XX^\top\bA^{-1} \XX \big\|_\op^2 = O_{d,\P}\Big( \frac{1}{d} \Big).
\end{align*}
Combining these two bounds give
\begin{align*}
&\var \big[ \tilde I_{\veps, \veps}^{(2)} \big] \le C\big\| \bA^{-1} \big( \gamma_0^2 \bone_n \bone_n^\top +\frac{\gamma_1^2}{d^2} \XX \XX^\top \big)  \bA^{-1} \big\|_F^2 = O_{d,\P}\Big(\frac{1}{d}\Big), \\
&\var \big[ \tilde I_{1,1}^{(2)} \big]  \le \frac{C}{d^2}\big\| \XX^\top \bA^{-1} \big( \gamma_0^2 \bone_n \bone_n^\top +\frac{\gamma_1^2}{d^2} \XX \XX^\top \big)  \bA^{-1} \XX \big\|_F^2 = O_{d,\P}\Big(\frac{1}{d}\Big).
\end{align*}

\textbf{Step 2.} Finally, to prove the desired results, we connect $\tilde I_\veps^{(1)}, \tilde I_1^{(1)}, \tilde I_{\veps, \veps}^{(2)}, \tilde I_{1,1}^{(2)}, \tilde I_{\veps, 1}^{(2)}$ back to $I_\veps^{(1)}, I_1^{(1)}, I_{\veps, \veps}^{(2)}, I_{1,1}^{(2)}, I_{\veps, 1}^{(2)}$. By definition, $\bbeta_*$ and $\| \bbeta_*\| \bgg / \| \bgg\|$ have the same distribution, so there is no loss of generality by writing
\begin{equation*}
I_\veps^{(1)} = \frac{\| \bbeta_*\|}{\| \bgg\| } \tilde I_\veps^{(1)},  \quad I_1^{(1)} = \frac{\| \bbeta_*\|^2}{\| \bgg\|^2 } \tilde I_1^{(1)}, \quad I_{1,1}^{(2)} = \frac{\| \bbeta_*\|^2}{\| \bgg\|^2 } \tilde I_{1,1}^{(2)}, \quad I_{\veps,1}^{(2)} = \frac{\| \bbeta_*\|}{\| \bgg\| } \tilde I_{\veps,1}^{(2)}.
\end{equation*}
And also $\tilde I_{\veps,\veps}^{(2)} = I_{\veps,\veps}^{(2)}$ because there is no $\bbeta_*$ involved. By concentration of Gaussian vector norms \cite{vershynin2010introduction, VanHandel}, we have $\| \bgg \| / \| \bbeta_* \| = 1 + O_{d,\P}(d^{-1/2})$. We observe that
\begin{align*}
\E_{\bveps, \bgg}[ \tilde I_\veps^{(1)} ] =0, \quad \var_{\bveps, \bgg}[ \tilde I_\veps^{(1)} ] = O_{d,\P}\Big(\frac{1}{d} \Big)  \qquad \Rightarrow \qquad \big|\tilde I_\veps^{(1)} \big| = O_{d,\P}\Big(\frac{1}{\sqrt{d}} \Big)
\end{align*}
by Chebyshev's inequality, and thus $\big| I_\veps^{(1)} \big| = O_{d,\P}(d^{-1/2})$. The same argument applies to the term $I_{\veps,1}^{(2)}$ which leads to the bound $\big| I_{\veps,1}^{(2)} \big| = O_{d,\P}(d^{-1/2})$. Also, by a similar argument,
\begin{align*}
& \frac{\| \bbeta_*\|^2}{\| \bgg\|^2 }\Big| \tilde I_1^{(1)} - \E_{\bgg} [\tilde I_1^{(1)} ] \Big| = O_{d,\P}\Big(\frac{1}{\sqrt{d}} \Big),  \qquad \frac{\| \bbeta_*\|^2}{\| \bgg\|^2 } \Big| \tilde I_{1,1}^{(2)} - \E_{\bgg} [\tilde  I_{1,1}^{(2)}  ] \Big| = O_{d,\P}\Big(\frac{1}{d} \Big), \\
&\Big| I_{\veps,\veps}^{(2)} - \E_{\bveps} [I_{\veps,\veps}^{(2)} ] \Big| = O_{d,\P}\Big(\frac{1}{\sqrt{d}} \Big).
\end{align*}
Therefore, we deduce
\begin{align*}
& I_1^{(1)} = \frac{\| \bbeta_*\|^2}{\| \bgg\|^2 } \E_{\bgg} [\tilde I_1^{(1)} ] + O_{d,\P}\Big(\frac{1}{\sqrt{d}} \Big), \qquad I_{1,1}^{(2)} = \frac{\| \bbeta_*\|^2}{\| \bgg\|^2 } \E_{\bgg} [\tilde  I_{1,1}^{(2)}  ]  + O_{d,\P}\Big(\frac{1}{\sqrt{d}} \Big), \\
& I_{\veps,\veps}^{(2)} = \E_{\bveps} [I_{\veps,\veps}^{(2)} ] +  O_{d,\P}\Big(\frac{1}{\sqrt{d}} \Big).
\end{align*}
Note that $\| \bgg \|^2 / \| \bbeta_* \|^2 = 1 + O_{d,\P}(d^{-1/2})$ and  $\E_{\bgg} [\tilde I_1^{(1)} ], \E_{\bgg} [\tilde  I_{1,1}^{(2)}  ] , \E_{\bveps} [I_{\veps,\veps}^{(2)} ]$ are exactly the trace quantities in the statement of Lemma~\ref{lem:quadraticform}. Also note that $|\E_{\bgg} [\tilde I_1^{(1)} ]|, |\E_{\bgg} [\tilde  I_{1,1}^{(2)}  ]| , |\E_{\bveps} [I_{\veps,\veps}^{(2)} ]|$ are bounded by a constant due to Lemma~\ref{lem:linreduce2} and Eqs.~\eqref{eq:Sa}--\eqref{eq:Sc}. This completes the proof.

\end{proof}

\section{Supporting lemmas and proofs}

\subsection{Basic lemmas}

\begin{lem}\label{lem:concntr1}
Fix any constant $\veps_0  \in (0,0.1)$. With very high probability, $\| \XX \|_\op \le (1+\veps_0) (\sqrt{n} + \sqrt{d})$, $\sigma_{\min}(\XX) \ge (\sqrt{n} - (1+\veps_0)\sqrt{d})/(1+\veps_0)$, and $\| \XX^\top \bone_n\| \le (1+\veps_0)\sqrt{nd}$.
\end{lem}
\begin{proof}[{\bf Proof of Lemma~\ref{lem:concntr1}}]
Let $\tilde \XX = [\tilde \xx_1, \ldots, \tilde \xx_n]^\top \in \R^{n \times d}$ be a matrix of i.i.d.~standard normal variable, and $\tilde \bD_\XX = d^{-1/2} \diag\{\|\tilde \xx_1\|, \ldots, \| \tilde \xx_n \| \}$. By \cite{vershynin2010introduction}[Corollary~5.35], $\| \tilde \XX  \|_\op \le (1+\veps_0/2)(\sqrt{n} + \sqrt{d})$ with probability at least $1 - 2 \exp(-c(n+d))$. By the subgaussian concentration, $d^{-1/2} \| \tilde \xx_i \| \in (1-\veps_0/10,1+\veps_0/10)$ with probability at least $1 - 2 \exp(-cd)$. Taking the union bound, we deduce that with probability at least $1-Cn \exp(-cd)$, 
\begin{align*}
\big\| \XX \big\|_\op &= \big\| \tilde \bD_\XX^{-1} \tilde \XX \big\|_\op \le  \big\| \tilde \XX \big\|_\op /  \lambda_{\min}(\tilde\bD_\XX) < (1+\veps_0)(\sqrt{n} + \sqrt{d}), \\
\sigma_{\min}(\XX)  &= \sigma_{\min} \big( \tilde \bD_\XX^{-1} \tilde \XX \big) \ge \sigma_{\min}(\tilde \XX) / \lambda_{\max}(\tilde\bD_\XX) > (\sqrt{n} - (1+\veps_0)\sqrt{d})/(1+\veps_0),\\
\| \XX^\top \bone_n\| &= \| \tilde \bD_\XX^{-1} \tilde \XX^\top \bone_n\| \le \|\tilde \XX^\top \bone_n\| / \lambda_{\min}(\tilde\bD_\XX) < (1+\veps_0)\sqrt{nd},
\end{align*}
where we used concentration of Gaussian vector norms \cite{vershynin2010introduction, VanHandel} in the last inequality. This finishes the proof.
\end{proof}

\begin{lem}\label{lem:xinnerprod}
Let $a>0$ be any constant, $\xx' \in \S^{d-1}(\sqrt{d})$ be fixed, and $\xx \sim \Unif(\S^{d-1}(\sqrt{d}))$. Then $|\langle \xx, \xx' \rangle| \le a \sqrt{d} \log d$ holds with very high probability.
\end{lem}
\begin{proof}[{\bf Proof of Lemma~\ref{lem:xinnerprod}}]
By orthogonal invariance, we assume without loss of generality that $\xx' = \sqrt{d}\, \ee_1$. It suffices to prove that $| \langle \xx, \ee_1 \rangle | \le a \log d$ holds with very high probability. Note that $\xx \stackrel{d}{=} \sqrt{d}\, \bgg / \| \bgg\|$ where $\bgg \sim \cN(0,\bI_d)$. From basic concentration results \cite{vershynin2010introduction}, we know $\|\bgg\| / \sqrt{d}\ge 1/2$ and $|\langle \xx, \ee_1 \rangle| \le a\log d/2$ with very high probability. This then proves the desired result.
\end{proof}

\begin{lem}\label{lem:algebra}
Let $\bB \in \R^{m \times q}$ be a matrix and $a > 0$. Then,
\begin{equation}\label{eq:algebra}
\big( \bB \bB^\top + a \bI_m \big)^{-1} \bB = \bB \big( \bB^\top\bB  + a \bI_q \big)^{-1}
\end{equation}
\end{lem}
This simple lemma is straightforward to check (left-multiplying $\big( \bB \bB^\top + a \bI_m \big)^{-1}$ and right-multiplying $\big( \bB^\top\bB  + a \bI_q \big)^{-1}$). 

\begin{lem}\label{lem:L2conv}
Let $\cH$ be an $L^2$ space and $\langle \cdot, \cdot \rangle_\cH$ be the associated inner product. If $f, f_m, g, g_m\in \cH$ for $m \ge 1$, and $f_m \xrightarrow{L^2} f$, $g_m \xrightarrow{L^2} g$, then $\langle f_m, g_m \rangle_{\cH} \xrightarrow{m \to \infty} \langle f, g \rangle_{\cH}$.
\end{lem}
This is a simple convergence result of $L^2$ spaces.

\subsection{An operator norm bound}
Recall that $Q_k^{(d)}$ is the degree-$k$ Gegenbauer polynomial in $d$ dimensions. The next lemma states that when the degree is high, the Gegenbauer inner-product kernels are very close to the identity matrices. Throughout this subsection, we assume that $d^{\ell+1} \ge n$. 
\begin{prop}\label{prop:Qconctr}
Denote $\QQ_k = (Q_k^{(d)}(\langle \xx_i, \xx_j \rangle))_{i,j \le n}$. Then, there exists a constant $C>0$ such that with very high probability,
\begin{align*}
\sup_{k > \ell} \big\| \QQ_k - \bI_n \big\|_\op \le C \sqrt{\frac{n (\log d)^C}{d^{\ell+1}}}.
\end{align*}
\end{prop}

This lemma is a quantitative refinement of \cite{ghorbani2019linearized}[Prop.~3]. Difference from the moment method employed in \cite{ghorbani2019linearized}[Prop.~3], the proof of this lemma uses a specific matrix concentration inequality, namely the matrix version of Freedman's inequality for martingales, first
established by \cite{Tropp11} (see also \cite{oliveira2009concentration}). 

\begin{thm}\label{thm:freedman}
Consider a matrix-valued martingale $\{ \YY_i: i=0,1,\ldots\}$ with respect to the filtration $(\cF_i)_{i=0}^\infty$. The values of $\YY_i$ are symmetric matrices with dimension $n$. Let $\ZZ_i = \YY_i - \YY_{i-1}$ for all $i \ge 1$. Assume that $\lambda_{\max}( \ZZ_i) \le \bar L$ almost surely for all $i \ge 1$. Then, for all $t \ge 0$ and $v \ge 0$, 
\begin{equation*}
\P \Big( \exists\, k \ge 0: \lambda_{\max}(\YY_k) \ge t ~\text{and}~ \| \bV_k \|_\op \le v \Big) \le n \cdot \exp \left( - \frac{ t^2/2}{v + \bar Lt/3} \right)
\end{equation*}
where $\bV_i$ is defined as $\bV_i = \sum_{j = 1}^i \E[ \ZZ_j^2 | \cF_{i-1} ]$.
\end{thm} 

First, we fix $k>\ell$ and treat it as a constant. We also suppress the dependence on $k$ if there is no confusion. Define a filtration $(\cF_i)_{i=0}^n$ and random matrices $(\ZZ_i)_{i=1}^n$ as follows. We define $\cF_j$ to be the $\sigma$-algebra generated by $\xx_1,\ldots,\xx_j$; particularly, $\cF_0$ is the trivial $\sigma$-algebra. 
We also define the truncated version of $\QQ$. 
\begin{equation}\label{def:truncate}
\bar \QQ = \big(Q_{ij} \xi_{ij} \bone\{ i \neq j \}\big)_{i,j \le n}, \qquad \where~ \xi_{ij} = \bone\{ | \langle \xx_i, \xx_j \rangle| \le \sqrt{d} \log d \}
\end{equation}
By a simple concentration result (Lemma~\ref{lem:xinnerprod}), we have, with very high probability,
\begin{equation*}
\max_{i \neq j} | \langle \xx_i, \xx_j \rangle| \le  \sqrt{d}\, \log d.
\end{equation*}
Note $Q_{ii} = 1$ always holds. So $\bar \QQ = \QQ - \bI_n$ holds with very high probability. Define
\begin{equation*}
\ZZ_i := \E[\bar \QQ | \cF_i] - \E[\bar \QQ | \cF_{i-1}].
\end{equation*}
 In other words, $(\E[\bar \QQ | \cF_i])_{i=1}^n$ is a Doob martingale with respect to the filtration $(\cF_i)_{i=0}^n$, and $(\ZZ_i)_{i=1}^n$ is the resulting martingale difference sequence. Note that trivially $\bar \QQ = \E[\bar \QQ | \cF_n]$, $\E[\bar \QQ] = \E[\bar \QQ | \cF_0]$. Using this definition, we can express $\bar \QQ $ as
\begin{equation}\label{eq:Kbar}
\bar \QQ = \E[ \bar \QQ]  + \sum_{i=1}^n \ZZ_i.
\end{equation}
By construction, $\E[\QQ] = \bI_d$. 
It is not difficult to show that $\big\| \E[ \bar \QQ] \big\|_{\op} \le Cd^{-k}$; 
see Lemma \ref{lem:Zi} below. Crucially, we claim that the matrix $\ZZ_i$ is a rank-$2$ matrix plus 
a small perturbation (namely, a matrix with very small operator norm). The perturbation term is 
due to the effect of truncation. 

\begin{lem}\label{lem:Zi}
Let $k \ge \ell +1 $ be a constant. 
Then we have decomposition
\begin{equation}\label{eq:Z}
\ZZ_i = \ee_i \vv_i^\top + \vv_i \ee_i^\top + \bDelta_i, \qquad \where~(\vv_i)_j = \bar Q_{ij} \bone\{ j < i\}.
\end{equation}
Here $\bDelta_i \in \R^{n \times n}$ is certain matrix that satisfies $\| \bDelta_i \|_\op \le Cd^{-k}$ almost surely, and $\vv_i$ satisfies $\| \vv_i \| \le C\sqrt{\frac{n (\log d)^{2k}}{d^k}}$ almost surely. In addition, $\| \E [\bar \QQ] \|_\op \le C d^{-k}$.
\end{lem}
\begin{proof}[{\bf Proof of Lemma~\ref{lem:Zi}}] 
In this proof, we will use the identities
\begin{equation}\label{eq:verysimpleeq}
\E_\xx \big[ Q(\langle \xx, \xx' \rangle) \big] = 0, \qquad \E_\xx \big[ Q(\langle \xx, \xx' \rangle)^2 \big] = B(d,k)^{-1},
\end{equation}
where $\xx, \xx' \sim_{\iid} \Unif(\S^{d-1}(\sqrt{d}))$. These identities can be derived from \eqref{eq:addtheorem}. 

Let $j, m \in [n]$ be indices. Since $\ZZ_i$ is symmetric, there is no loss of generality in only considering that $j < m$ (note $(\ZZ_i)_{jj} = 0$ if $j=m$). If $m < i$, then $\bar Q_{jm} \in \cF_{i-1}$ so $[\ZZ_i]_{j m} = 0$. Next we consider $m > i$. Observe that 
\begin{align*}
& ~~~~\Big| \E_{\xx_m} \big[ Q(\langle \xx_j, \xx_m \rangle) \bone\big\{ |\langle \xx_j, \xx_m \rangle| \le C_1 \sqrt{d} \, \log d \big \} \big] \Big| \\
&\stackrel{(i)}{=} \Big| \E_{\xx_m} \big[ Q(\langle \xx_j, \xx_m \rangle) \bone\big\{ |\langle \xx_j, \xx_m \rangle| > C_1 \sqrt{d} \, \log d \big \} \big] \Big| \\
& \stackrel{(ii)}{\le}  \Big\{ \E_{\xx_m}\big[ Q(\langle \xx_j , \xx_m \rangle)^2 \big] \Big\}^{1/2} \cdot \Big\{ \P_{\xx_m} \big( |\langle \xx_j, \xx_m \rangle| > C_1 \sqrt{d}\, \log d\big) \Big\}^{1/2} \\
& \stackrel{(iii)}{\le } B(d,k)^{-1/2} \cdot O_d(d^{-2k}) = O_d(d^{-2k}),
\end{align*}
where in \textit{(i)} we used \eqref{eq:verysimpleeq}, in \textit{(ii)} we used the Cauchy-Schwarz inequality, and in \textit{(iii)} we used \eqref{eq:verysimpleeq} and Lemma~\ref{lem:xinnerprod}. This implies
\begin{equation}\label{ineq:Qxi}
\E \big[ Q_{jm} \xi_{j m} \big| \cF_i \big] = O_d(d^{-2k}), \qquad \E \big[ Q_{jm} \xi_{j m} \big| \cF_{i-1} \big] = O_d(d^{-2k}).
\end{equation}
Together, these bounds imply $| [\ZZ_i]_{jm}| =O_d(d^{-2k})$. Finally, consider $m = i$. A similar argument gives
\begin{equation*}
\E \big[ Q_{jm} \xi_{j m} \big| \cF_i \big] = Q_{jm} \xi_{j m} = \bar Q_{jm}, \qquad \E \big[ Q_{jm} \xi_{j m} \big| \cF_{i-1} \big] = O_d(d^{-2k}).
\end{equation*}
Combining the three cases, we derive the expression \eqref{eq:Z}. The residual satisfies
\begin{equation*}
\big\| \bDelta \big\|_\op \le n \big\| \bDelta \big\|_{\max} = n \cdot O_d(d^{-2k}) \le O_d(d^{-k}) 
\end{equation*}
where the last inequality is due to $d^k \ge n$. Note that \eqref{ineq:Qxi} also implies
\begin{equation*}
\| \E [ \bar \QQ ] \|_\op \le n \cdot \max_{j m} \big| \E \big[ Q_{j m} \xi_{j,  m}  \big] \big| \le O_d(n d^{-2k}) \le O_d(d^{-k}).
\end{equation*}

By the property \eqref{eq:Hermiteconnect} of the Gegenbauer polynomials, we know that the coefficients of $B(d,k)^{1/2}Q(\cdot)$ are of constant order, so we have the deterministic bound
\begin{equation*}
\big| B(d,k)^{1/2}Q(\langle \xx_i, \xx_j \rangle) \big| \xi_{ij} \le (1+o_d(1))\cdot h_k\Big( \frac{\langle \xx_i, \xx_j \rangle}{\sqrt{d}} \Big) \xi_{ij} \le O_d\big( (\log d)^k \big).
\end{equation*}
Since $B(d,k) = (1+o_d(1))\cdot d^k / k!$, this gives $\max_{ij} |\bar Q_{ij}| \le C\sqrt{(\log d)^{k}/d^k}$ and thus 
\begin{equation*}
\| \vv_i \| \le \sqrt{i} \cdot \max_{j<i}|\bar Q_{ij}| \le C \sqrt{\frac{n (\log d)^{2k}}{d^k}}.
\end{equation*}
This completes the proof.
\end{proof}

Before applying the matrix concentration inequality to the sum $\sum_{i=1}^n \ZZ_i$, let us define
\begin{equation*}
L := \max_{i \le n} \| \ZZ_i \|_\op, \qquad V := \Big\| \sum_{i=1}^n \E[ \ZZ_i^2| \cF_{i-1}] \Big\|_\op.
\end{equation*}
Using Lemma~\ref{lem:Zi}, we obtain deterministic bounds on $L$ and  $V$, as stated below.
\begin{lem}\label{lem:keyclaim1}
The following holds almost surely.
\begin{align}
&\max_{i \le n} \| \ZZ_i \|_\op \le C\sqrt{\frac{n (\log d)^{2k}}{d^k}}, \label{bdn:L} \\
&\Big\| \sum_{i=1}^n \E[ \ZZ_i^2| \cF_{i-1}] \Big\|_\op \le C \frac{n (\log d)^{2k}}{d^k}+  \frac{n}{B(d,k)} \|  \QQ - \bI_n\|_\op. \label{bnd:V}
\end{align}
\end{lem}
\begin{proof}[{\bf Proof of Lemma~\ref{lem:keyclaim1}}]
The first inequality follows directly from Lemma~\ref{lem:Zi}:
\begin{equation*}
\max_{i \le n} \| \ZZ_i \|_\op \le 2 \max_{i\le n} \| \vv_i \| + \max_{i \le n}\| \bDelta_i \|_\op \le C \sqrt{\frac{n (\log d)^{2k}}{d^k}} + Cd^{-k} \le C \sqrt{\frac{n (\log d)^{2k}}{d^k}}.
\end{equation*}
To prove the second inequality, we use the decomposition \eqref{eq:Z}.
\begin{align*}
\E\big[ \ZZ_i^2 | \cF_{i-1} \big] &= \E \big[ \vv_i \vv_i^\top \big| \cF_{i-1} \big] + \E \big[ \ee_i \ee_i^\top \| \vv_i \|^2 \big| \cF_{i-1} \big] + \tilde \bDelta_i
\end{align*}
where $\tilde \bDelta_i \in \R^{n \times n}$ is given by
\begin{align*}
\tilde \bDelta_i = \E \big[ \vv_i \ee_i^\top \bDelta_i + \ee_i \vv_i^\top \bDelta_i +  \bDelta_i\vv_i \ee_i^\top + \bDelta_i \ee_i \vv_i^\top + \bDelta_i^2 \big| \cF_{i-1} \big].
\end{align*}
Note that we used $\ee_i^\top \vv_i = 0$ in the above calculation. Using the deterministic bounds on $\| \vv_i \|$ and $\| \bDelta_i \|_\op$, we get 
\begin{equation*}
\| \tilde \bDelta_i \|_\op \le 4 \| \tilde \bDelta_i \|_\op \cdot \| \vv_i \| + \| \tilde \bDelta_i \|_\op^2 \le \frac{C}{d^k} \sqrt{\frac{n(\log d)^{2k}}{d^k}} + \frac{C}{d^{2k}}.
\end{equation*}
We also note that $\sum_{i\le n} \E \big[ \ee_i \ee_i^\top \| \vv_i \|^2 \big| \cF_{i-1} \big]$ is a diagonal matrix with its diagonal entries bounded by $\max_{i\le n} \| \vv_i \|^2$. Thus, we get 
\begin{align}
\Big\| \sum_{i\le n} \E\big[ \ZZ_i^2 | \cF_{i-1} \big] \Big\|_\op &\le n \max_{i\le n}  \Big\| \E \big[ \vv_i \vv_i^\top \big| \cF_{i-1} \big] \Big\|_\op + \max_{i\le n} \|\vv_i\|^2 + n \max_{i\le n} \| \tilde \bDelta_i \|_\op \notag \\
&\le  n \max_{i\le n}  \Big\| \E \big[ \vv_i \vv_i^\top \big| \cF_{i-1} \big] \Big\|_\op+ \frac{Cn(\log d)^{2k}}{d^k} + \frac{Cn}{d^k} \sqrt{\frac{n(\log d)^{2k}}{d^k}} + \frac{Cn}{d^{2k}} \notag \\
&\le n \max_{i\le n}  \Big\| \E \big[ \vv_i \vv_i^\top \big| \cF_{i-1} \big] \Big\|_\op+\frac{Cn(\log d)^{2k}}{d^k} \label{expr:vvF}
\end{align}
where the last inequality is because
\begin{equation*}
\frac{n}{d^k} \sqrt{\frac{n(\log d)^{2k}}{d^k}} \le \sqrt{\frac{n}{d^k}} \cdot \sqrt{\frac{n(\log d)^{2k}}{d^k}} \le \frac{n(\log d)^{2k}}{d^k}, \qquad \frac{n}{d^{2k}} \le  \frac{n(\log d)^{2k}}{d^k},
\end{equation*}
due to the assumption $d^k \ge n$. To further simplify \eqref{expr:vvF}, we note that 
\begin{align*}
\E \big[ \vv_i \vv_i^\top \big| \cF_{i-1} \big]_{j m} = 0, \qquad \text{if  } m \ge i ~ \text{or  } j \ge i.
\end{align*}
Now we consider the case where $m < i$ and $j < i$. 
\begin{align*}
\E \big[ \vv_i \vv_i^\top \big| \cF_{i-1} \big]_{j m}  &= \E_{\xx_i} \big[ Q(\langle \xx_i, \xx_j \rangle) Q(\langle \xx_i, \xx_m \rangle) \xi_{i j} \xi_{i m} \big] \\
&= \E_{\xx_i} \big[ Q(\langle \xx_i, \xx_j \rangle) Q(\langle \xx_i, \xx_m \rangle) \big] - \E_{\xx_i} \big[ Q(\langle \xx_i, \xx_j \rangle) Q(\langle \xx_i, \xx_m \rangle) (1-\xi_{i j} \xi_{i m}) \big] \\
&= \frac{1}{B(d,k)} Q(\langle \xx_j , \xx_m \rangle) - \E_{\xx_i} \big[ Q_{i j} Q_{i m} (1-\xi_{i j} \xi_{i m}) \big]
\end{align*}
where we used the identity \eqref{eq:innerproductself} in the last equality. We write $1 - \xi_{i j} \xi_{i m} = (1 - \xi_{ij})\xi_{i m} + 1 - \xi_{i m}$ and use the Cauchy-Schwarz inequality to derive
\begin{align*}
\Big| \E_{\xx_i} \big[ Q_{i j} Q_{i m} (1-\xi_{i j} \xi_{i m}) \big] \Big| &\le \Big\{ \E_{\xx_i} \big[ Q_{ij}^2 Q_{im}^2 \big] \Big\}^{1/2} \cdot \Big\{ \E_{\xx_i} \big[ (1 - \xi_{ij})^2 \xi_{im}^2 \big] \Big\}^{1/2}  \\
&+ \Big\{ \E_{\xx_i} \big[ Q_{ij}^2 Q_{im}^2 \big] \Big\}^{1/2} \cdot \Big\{ \E_{\xx_i} \big[ (1 - \xi_{im})^2 \big] \Big\}^{1/2} \\
&\stackrel{(i)}{\le} \big[ \P_{\xx_i} \big( \xi_{ij} = 0 )\big]^{1/2} + \big[\P_{\xx_i} \big( \xi_{i m} = 0 )\big]^{1/2} \\
&\stackrel{(ii)}{\le} O_d(d^{-2k}). 
\end{align*}
where in \textit{(i)} we used the inequality $|Q_{ij}| \le 1$, and in \textit{(ii)} we used Lemma~\ref{lem:xinnerprod}. Therefore,
\begin{equation*}
\E \big[ \vv_i \vv_i^\top \big| \cF_{i-1} \big] \preceq \frac{1}{B(d,k)} \QQ + O_d(nd^{-2k}) \cdot \bI_n. 
\end{equation*}
Thus, we obtain
\begin{equation*}
\Big\| \sum_{i\le n} \E\big[ \ZZ_i^2 | \cF_{i-1} \big] \Big\|_\op \le \frac{n}{B(d,k)} \| \QQ \|_\op + \frac{Cn(\log d)^{2k}}{d^k} \le \frac{n}{B(d,k)} \| \QQ - \bI_n \|_\op + \frac{Cn(\log d)^{2k}}{d^k}.
\end{equation*}
This completes the proof.
\end{proof}

Once this lemma is established, we apply Theorem~\ref{thm:freedman} to 
the martingales $\YY_m :=\sum_{i\le m} \ZZ_i$ and $\YY'_m :=\sum_{i\le m} (-\ZZ_i)$ (where we simply set $\ZZ_{n+i} = \bzero$ for $i \ge 1$), and we obtain the following result. For any $t \ge 0$ and $v > 0$,
\begin{equation*}
\P\left( \Big\| \sum_{i=1}^n \ZZ_i  \Big\|_\op \ge t, \text{and}~ V \le v\right) \le 2 n \cdot \exp \left( - \frac{t^2/2}{v+ \bar L t/3} \right),
\end{equation*}
where $\bar L = C\sqrt{\frac{n (\log d)^{2k}}{d^k}}$ is the upper bound on $L$ in \eqref{bdn:L}. This inequality implies a probability tail bound on $\| \sum_{i=1}^n \ZZ_i \|_\op$. 
\begin{equation}\label{freedman2}
\P\left( \Big\| \sum_{i=1}^n \ZZ_i  \Big\|_\op \ge t \right) \le 2 n \cdot \exp \left( - \frac{t^2/2}{v+ \bar L t/3} \right) + \P \big(V > v \big).
\end{equation}
By taking $v$ slightly larger the ``typical'' values of $V$, we can make the probability $\P \big(V > v \big)$ very small, which can be bounded by a tail probability of $\| \QQ \|_\op$ according to \eqref{bnd:V}.
This leads to a recursive inequality between tail probabilities. By abuse of notations, in the next lemma we assume that constant $C \ge 1$ is chosen to be no smaller than those that appeared in Lemma~\ref{lem:Zi} and~\ref{lem:keyclaim1} (so that we can invoke both results).

\begin{lem}\label{lem:keyclaim1-2}
Let the constant $C \ge 1$ be no smaller than those in Lemma~\ref{lem:Zi} and~\ref{lem:keyclaim1}. 
Suppose that 
\begin{equation}\label{bnd:tlower}
t \ge \max \Big\{ 8C \sqrt{\frac{n(\log d)^{2k+4}}{d^k}}, 128 \frac{n(\log d)^2}{B(d,k)} \Big\}.
\end{equation}
Then, by setting $\bar L = C\sqrt{\frac{n (\log d)^{2k}}{d^k}}$ and $v = \frac{t^2}{32 (\log d)^2}$, we have
\begin{equation}\label{freedman-recursive}
\P\left( \big\| \QQ - \bI_d \big\|_\op \ge t \right) \le \P\left( \big\| \QQ - \bI_d \big\|_\op \ge 2t \right) + C' \exp\big( -(\log d)^2 \big) + \P\big( \bar \QQ \neq \QQ-  \bI_n \big).
\end{equation}
As a consequence, by iterating the choice of $t$ at most $O_d(\log d)$ times, we obtain that $\big\| \QQ - \bI_d \big\|_\op \ge t$ holds with very high probability.
\end{lem}

\begin{proof}[{\bf Proof of Lemma~\ref{lem:keyclaim1-2}}]
If $t$ satisfies \eqref{bnd:tlower}, then we derive
\begin{align*}
\P \big( \| \QQ-  \bI_n \|_\op > t \big) &\le \P \big(  \| \bar \QQ \|_\op > t  \big) + \P\big( \bar \QQ \neq \QQ-  \bI_n \big) \\
& \le  \P \big(  \| \sum_{i \le n} \ZZ_i  \|_\op + \big\| \E [ \bar \QQ ] \big\|_\op > t  \big) + \P\big( \bar \QQ \neq \QQ-  \bI_n \big) \\
&\stackrel{(i)}{\le}  \P \Big(  \big\| \sum_{i \le n} \ZZ_i \big \|_\op > \frac{t}{2}  \Big)  + \P\big( \bar \QQ \neq \QQ-  \bI_n \big) \\
&\le \P(V > v) + n \exp \Big( - \frac{t^2/8}{v+ \bar L t/6} \Big) + \P\big( \bar \QQ \neq \QQ-  \bI_n \big)
\end{align*}
where \textit{(i)} is because $\big\| \E [ \bar \QQ ] \big\|_\op < t/2$. By Lemma~\ref{lem:keyclaim1} Eq.~\eqref{bnd:V}, we get 
\begin{equation*}
 \P(V > v) \le \P \Big(  C \frac{n (\log d)^{2k}}{d^k}+  \frac{n}{B(d,k)} \|  \QQ - \bI_n\|_\op > \frac{t^2}{32 (\log d)^2} \Big).
\end{equation*}
By the assumption on $t$, we have
\begin{equation*}
C \frac{n (\log d)^{2k}}{d^k} \le \frac{t^2}{64 (\log d)^2}, \qquad \frac{2tn}{B(d,k)} \le \frac{t^2}{64 (\log d)^2}.
\end{equation*}
This leads to $\P(V > v) \le \P(  \|  \QQ - \bI_n\|_\op > 2t)$. Also, $\bar L t / 6 \le t^2 / (32 (\log d)^2)$, so
\begin{align*}
n \exp \Big( - \frac{t^2/8}{v+ \bar L t/6} \Big) &\le n \exp \Big(- \frac{t^2/8}{v+ t^2 / (32 (\log d)^2)}   \Big) \\
&=  n \exp \Big( - 2 (\log d)^2 \Big) \le d^k \cdot \exp \Big( - 2 (\log d)^2 \Big)  \\
&\le C' \exp \Big( - (\log d)^2 \Big).
\end{align*}
This proves the inequality \eqref{freedman-recursive}. We note that there is a naive deterministic bound 
\begin{equation*}
\big\| \QQ - \bI_n \big\|_\op \le n \big\| \QQ - \bI_n \big\|_{\max} \le 2n.
\end{equation*}
We apply  \eqref{freedman-recursive} in which $t$ is set to one of $\{2t, 4t, \ldots, 2^s t$\} where $s \ge k \log d / \log 2$. Summing these inequalities and using the naive deterministic bound, we obtain
\begin{equation*}
\P \left( \big\| \QQ - \bI_n \big\|_\op > t \right) \le O_d( \log d) \cdot \Big( \exp \Big( - (\log d)^2 \Big) + \P\big( \bar \QQ \neq \QQ-  \bI_n \big) \Big).
\end{equation*}
Since $\bar \QQ = \QQ-  \bI_n$ happens with high probability, the above inequality implies that $\| \QQ - \bI_n \|_\op$ also happens with very high probability.
\end{proof}

Finally, we are in a position to prove Proposition~\ref{prop:Qconctr}. Note that Lemma~\ref{lem:keyclaim1-2} already gives the right bound on $\| \QQ_k - \bI_n \|_\op$ for a constant $k$. For very large $k$, we resort to a different approach.
\begin{proof}[{\bf Proof of Proposition~\ref{prop:Qconctr}}]
Let $s \ge 1$ be a constant integer. Our goal is to prove
\begin{align}\label{eq:goal}
d^s \cdot \P \Big( \sup_{k > \ell} \big\| \QQ_k - \bI_d \big\|_\op \le C' \sqrt{\frac{n (\log d)^{C'}}{d^{\ell+1}}} \Big) = o_d(1)
\end{align}
for certain sufficiently large constant $C'$.

Consider $k \ge k_0 := 4\ell + s + 7$. Since $\E \big[ Q_k(\langle \xx_i, \xx_j \rangle)^2 \big] = B(d,k)^{-1}$, by Markov's inequality, for any $t > 0$, we have
\begin{equation*}
\P \Big( \big| Q_k(\langle \xx_i, \xx_j \rangle) \big| > d^{-t} \Big) \le \frac{d^{2t}}{B(d,k)}.
\end{equation*}
Thus, by the union bound,
\begin{align*}
\P \big( \big\| \QQ_k - \bI_d \big\|_\op  > nd^{-t} \big) \le \P \big( n \max_{i \neq j} \big| Q_k(\langle \xx_i, \xx_j \rangle) \big|  > nd^{-t} \big)  \le \frac{n^2d^{2t}}{B(d,k)}.
\end{align*}
We choose $t = \ell+2$, so that
\begin{equation*}
n d^{-t} < \frac{n}{d^{\ell+1}} < C\sqrt{\frac{n (\log d)^C}{d^{\ell+1}}}, \qquad \sum_{k > k_0} \frac{n^2 d^{2t}}{B(d,k)} \le d^{2\ell+2+t} \sum_{k > k_0} \frac{1}{B(d,k)} \le Cd^{-s-1}
\end{equation*}
where we used the inequality $\sum_{k > k_0} \frac{1}{B(d,k)} \le Cd^{-k_0}$. So taking the union bound, we get 
\begin{align}\label{ineq:sup1}
\P \left( \sup_{k \ge k_0} \big\| \QQ_k - \bI_d \big\|_\op  > C\sqrt{\frac{n (\log d)^C}{d^{\ell+1}}} \right) \le \sum_{k > k_0} \frac{n^2 d^t}{B(d,k)} \le Cd^{-s-1}.
\end{align}
Further, we apply Lemma~\ref{lem:keyclaim1-2} with $k$ set to $\{\ell+1, \ell+2,\ldots,k_0-1\}$, and we find 
\begin{align}\label{ineq:sup2}
d^s \cdot \P \left( \sup_{\ell < k < k_0} \big\| \QQ_k - \bI_d \big\|_\op  > C\sqrt{\frac{n (\log d)^C}{d^{\ell+1}}} \right) = o_d(1).
\end{align}
Combining \eqref{ineq:sup1} with \eqref{ineq:sup2} yields the desired goal \eqref{eq:goal}.
\end{proof}

\section{Proof for network capacity upper bound}\label{sec:proofthreshold}

\begin{proof}[{\bf Proof of Theorem~\ref{thm:lowerbound2}}]
We denote by $\btheta = (\bb, \bW)$ the collection of neural network parameters and by $f_{\btheta}$ the associated function.

Let $M = M_d > 1$ be a positive integer to be specified later. We define the discretization set $\cD_M =  \{0, \pm \frac{1}{M}, \pm \frac{2}{M}, \pm \frac{3}{M},  \ldots \}$ and
\begin{align}
&\Theta_L := \Big\{\btheta: b_\ell \in \cD_M, \sqrt{d}\,W_{\ell k} \in \cD_M, ~\forall\,\ell \in [N], \forall\, k \in [d] \Big\}, \\
&\Theta_{L,M} := \Theta_L \cap \Big\{\btheta: b_\ell \in \cD_M, \sqrt{d}\,W_{\ell k} \in \cD_M, ~\forall\,\ell \in [N], \forall\, k \in [d] \Big\}. \label{def:ThetaLM}
\end{align}
Each $f \in \cF_{\NN}^{N,L}$ is associated with an $\btheta \in \Theta_L$.

\textbf{(1) Lower bound on discretized parameter space.} We make the following claim. If with probability at least $\eta_1/2$, there exists certain $\bar \btheta \in \Theta_{L,M} $ such that at least $(1/2+\eta_2/2)n$ examples are correctly classified, i.e.,
\begin{equation}\label{ineq:theta0}
n^{-1} \sum_{i=1}^n \bone \{ y_i f_{\bar \btheta} (\xx_i) > 0 \} \ge \frac{1}{2}+\frac{\eta_2}{2},
\end{equation}
then we have $Nd = \Omega_d\big( \frac{n}{\log (LM)} \big)$.

To prove this claim, we treat the input $\xx_i$ as deterministic. We derive
\begin{align}
&~~~\P_\yy \Big( \exists\, \bar \btheta \in \Theta_{L,M}~\text{such that}~| \{i: \mathrm{sign}(f_{ \bar\btheta}(\xx_i))= y_i \}| \ge (1/2+\eta_2/2)n \Big) \label{event:discrete}\\
&\stackrel{(i)}{\le} \big| \Theta_{L,M} \big| \cdot \P_\yy \Big( | \{i: \mathrm{sign}(f_{\bar \btheta}(\xx_i))= y_i \}| \ge (1/2+\eta_2/2)n \Big)  \notag \\
&\stackrel{(ii)}{\le} O_d(1) \cdot \big[ \sqrt{2\pi e}(4ML + 1) \big]^{Nd+N} \cdot \P_\yy \Big( \sum_{i=1}^n \xi_i \ge (1/2 + \eta_2/2) n \Big) \notag
\end{align}
where $\xi_i = \bone\{\mathrm{sign}(f_{\bar \btheta}(\xx_i))= y_i \}$ denotes a Bernoulli random variable with mean $1/2$. Here, (i) used the union bound and (ii) used the following bound on the cardinality of $\Theta_{L,M}$
\begin{equation*}
| \Theta_{L,M} | = O_d(1) \cdot \big[ \sqrt{2\pi e}(4ML + 1) \big]^{Nd+N},
\end{equation*}
which is proved via the covering number argument (see Lemma~\ref{lem:cover}). By Hoeffding's inequality, we have
\begin{equation*}
\P_\yy \Big( \sum_{i=1}^n \xi_i \ge (1/2 + \eta_2/2) n \Big) \le \exp\big( -2n\eta_2^2 \big).
\end{equation*}
Since we assume that the probability of the event in \eqref{event:discrete} is at least $\eta_1$, we take the logarithm and deduce
\begin{align*}
& \log ( \eta_1/2) \le O_d(1) \cdot N(d+1) \log(4LM+1) - n\eta_2^2/2, \quad \text{or simply} \\
& n = O_d\big( Nd \log(LM+1) \big).
\end{align*}
It then follows that $Nd = O_d\big( \frac{n}{\log (LM+1)} \big)$. 

\textbf{(2) Projecting $\btheta$ onto the discretized set.} 
For any $z \in \R$, let $\tau_1(z)$ and $\tau_2(z)$ be the elements in $\cD_M$ such that the distances to $z$ are the smallest one and the second smallest one (break ties in an arbitrary way). For a given $z$, define a random variable $\xi(z) \in \{\tau_1(z), \tau_2(z) \}$ by
\begin{align*}
\xi(z) = \begin{cases}\tau_1(z), & \text{with probability}~ \displaystyle \frac{|z - \tau_2(z)|}{|z - \tau_1(z)| + |z - \tau_2(z)|};  \\ \tau_2(z), & \text{with probability}~ \displaystyle \frac{|z - \tau_1(z)|}{|z - \tau_1(z)| + |z - \tau_2(z)|}. \end{cases}
\end{align*}
This definition ensures that $\E \xi(z) = z$. Indeed, $ z - \tau_1(z)$ and $z - \tau_2(z)$ have the opposite signs, so
\begin{equation*}
z - \E\xi(z)  = \frac{(z - \tau_1(z))|z - \tau_2(z)|}{|z - \tau_1(z)| + |z - \tau_2(z)|} + \frac{(z - \tau_2(z))|z - \tau_1(z)|}{|z - \tau_1(z)| + |z - \tau_2(z)|} = 0.
\end{equation*}
Denote $\tilde b_\ell = \xi(b_\ell)$ and $\tilde W_{\ell k} = d^{-1/2}\xi(d^{1/2}W_{\ell k})$ for all $\ell$ and $k$ (they are independent random variables), and $\tilde \btheta = (\tilde \bb, \tilde \bW)$. Then $\tilde \btheta$ is a random projection of $\btheta \in \Theta_{L}$ onto the discretized set $\Theta_{L,M}$.

\textbf{(3) Bounding the approximation error}. By assumption, with probability greater than $\eta_1$, there exists $\btheta \in \Theta_{L}$ (which depends on the data $(\xx_i)_{i \le 1}$ and $(y_i)_{i \le n}$) such that it gives more than $\delta$ margin on at least $(1/2+\eta_2)n$ examples. Denote by $\cI := \cI(\bX,\yy) \subset [n]$ the set of indices $i$ such that $y_if_\btheta(\xx_i) > \delta$. Then, the assumption is equivalent to $| \cI | \ge (1/2+ \eta_2)n$ with probability great than $\eta_1$.

We claim  
\begin{equation}\label{claim:tildetheta}
\P_{\bX,\yy,\tilde \btheta}\Big( \frac{1}{n}\sum_{i=1}^n\bone\big\{| f_\btheta(\xx_i) - f_{\tilde \btheta}(\xx_i) | < \delta\big\} \ge  1-\frac{\eta_2}{2} \Big) \ge 1 - \frac{\eta_1}{2}.
\end{equation}
Here we emphasize that $\P_{\bX,\yy,\tilde \btheta}$ is the probability measure over all random variables (namely the data $\bX$, $\yy$ and also $\tilde \btheta$). Once this claim is proved, we can immediately prove the desired result. In fact, for a given $\bar \btheta \in \Theta_{L,M}$, let us define the indicator variable
\begin{equation*}
I_{\bar \btheta} = \bone\Big\{ \frac{1}{n}\sum_{i=1}^n\bone\big\{| f_\btheta(\xx_i) - f_{\bar\btheta}(\xx_i) | < \delta\big\} \ge  1-\frac{\eta_2}{2} \Big\}
\end{equation*}
We observe that
\begin{align*}
&~~~~\E_{\bX, \yy, \tilde \btheta} [I_{\tilde \btheta}] = \E_{\bX, \yy} \E_{\tilde \btheta| \bX, \yy} [I_{\tilde \btheta}] \stackrel{(i)}{\le} \E_{\bX, \yy}  \Big[ \max_{\bar \btheta \in \Theta_{L,M}} [I_{\bar \btheta}] \Big]  \stackrel{(ii)}{=}  \P_{\bX, \yy} \Big[ \max_{\bar \btheta \in \Theta_{L,M}} [I_{\bar \btheta}] = 1 \Big] \\
&=  \P_{\bX, \yy} \Big[ \exists\, \bar \btheta \in \Theta_{L,M}~\text{such that}~ I_{\bar \btheta} = 1 \Big]
\end{align*}
where (i) is because the mean is no larger than the maximum and (ii) is because $\max_{\bar \btheta \in \Theta_{L,M}} [I_{\bar \btheta}]$ takes value $0$ or $1$. Combining this with \eqref{claim:tildetheta}, we deduce that with probability at least $1 - \eta_1/2$, there exists an $\bar \btheta\in \Theta_{L,M}$ such that $| f_\btheta(\xx_i) - f_{\bar \btheta}(\xx_i) | < \delta$ for all $i \in \cI'$ where $\cI' \subset [n]$ satisfies $| \cI' | \ge (1-\eta_2/2)n$. 

We use the fact that $\P(\cA_1 \cap \cA_2) \ge \P(\cA_1) + \cP(\cA_2) - 1$ for arbitrary events $\cA_1,\cA_2$ to deduce that with probability at least $\eta_1/2$, there exists an $\bar \btheta\in \Theta_{L,M}$ such that $| f_\btheta(\xx_i) - f_{\bar \btheta}(\xx_i) | < \delta$ for all $i \in \cI'$, and in the meantime $|\cI | \ge (1/2+\eta_2)n$. By the triangle inequality, for every $i \in \cI \cap \cI'$, 
\begin{equation*}
y_i f_{\bar \btheta} \ge y_i f_{\btheta} - |y_i f_\btheta(\xx_i) - y_i f_{\bar \btheta}(\xx_i)| =  y_i f_{\btheta} - |f_\btheta(\xx_i) - f_{\bar \btheta}(\xx_i)| > 0.
\end{equation*}
Note that $| \cI \cap \cI'| \ge |\cI| + |\cI'| - n \ge (1/2 + \eta_2/2)n$. Therefore, the assumption of the discretized case, namely (1), is satisfied; and hence we obtain the desired lower bound.

Below we prove the claim \eqref{claim:tildetheta}. Denote $h_{i \ell} = \sigma(\langle \ww_\ell, \xx_i \rangle)$ and $\Delta h_{i \ell} = \sigma(\langle \ww_\ell, \xx_i \rangle )  - \sigma(\langle \tilde \ww_\ell, \xx_i \rangle ) $. By the triangle inequality,
\begin{align}
| f_{\btheta}(\xx_i) - f_{\tilde \btheta}(\xx_i) | &\le \frac{1}{N} \Big|  \sum_{\ell = 1}^N  (b_\ell - \tilde b_\ell) \sigma(\langle \ww_\ell, \xx_i \rangle ) \Big| + \frac{1}{N} \sum_{\ell=1}^N \Big| \tilde b_\ell \big[ \sigma(\langle \ww_\ell, \xx_i \rangle )  - \sigma(\langle \tilde \ww_\ell, \xx_i \rangle )  \big] \Big| \notag \\
&\le \frac{1}{N} \Big| \sum_{\ell=1}^N (b_\ell - \tilde b_\ell) h_{i\ell} \Big| + \frac{L}{N} \sum_{\ell=1}^N | \Delta h_{i\ell}|, \notag\\
&\stackrel{(i)}{\le} \frac{|\sigma(0)|}{N}\Big| \sum_{\ell=1}^N(b_\ell - \tilde b_\ell) \Big| +  \frac{1}{N} \Big| \sum_{\ell=1}^N (b_\ell - \tilde b_\ell) (h_{i\ell} - \sigma(0)) \Big| + \frac{L}{\sqrt{N}} \Big( \sum_{\ell=1}^N | \Delta h_{i\ell}|^2 \Big)^{1/2}, \notag\\
&=: T_{i,1} + T_{i,2} + T_{i,3},\label{ineq:threedelta}
\end{align}
where we used Cauchy-Schwarz inequality in (i). 

\textbf{i) Bounding $T_{i,1}$.} Conditioning on $\bX$ and $\yy$, the random variable $b_\ell - \tilde b_\ell$ is independent across $\ell$, has zero mean, and satisfies $|b_\ell - \tilde b_\ell| \le 1/M$. So we can use Hoeffding's inequality to control the first term $T_1$ in \eqref{ineq:threedelta}.
\begin{equation*}
\P_{\tilde \btheta}\Big( \frac{|\sigma(0)|}{N}\big| \sum_{\ell=1}^N(b_\ell - \tilde b_\ell) \big| > \frac{|\sigma(0)|t}{M} \Big) \le 2e^{-Nt^2/2} \le 2e^{-t^2/2}.
\end{equation*}
Taking $t = \sqrt{2\log (16/\eta_1)}$ yields the bound
\begin{equation*}
T_{i,1} \le \frac{C_0\sqrt{2 \log (16/\eta_1)}}{M}
\end{equation*}
with probability at least $1- \eta_1/8$, where we used the assumption that $|\sigma(0)| \le C_0$.

\textbf{ii) Bounding $T_{i,2}$.} By assumption, $\sigma$ has weak derivatives, so we have
\begin{equation*}
\sigma(\langle \ww_\ell, \xx_i \rangle) - \sigma(0) = \langle \ww_\ell, \xx_i \rangle \int_0^1 \sigma'\big( t\langle \ww_\ell, \xx_i \rangle \big) \; dt
\end{equation*}
Since $\max_{i\in [n]} \| \xx_i \|_2 \le 2\sqrt{d}$ with probability at least $1 - 2ne^{-cd}$, we have $|\langle \ww_\ell, \xx_i \rangle| \le \| \ww_\ell \|_2 \cdot \| \xx_i \|_2 \le 2L\sqrt{d}$. Thus
\begin{equation*}
| \sigma( \langle \ww_\ell, \xx_i \rangle) - \sigma(0) | \le  \max_{|z| \le 2L\sqrt{d}} | \sigma'(z) | \cdot | \langle \ww_\ell, \xx_i \rangle|.
\end{equation*}
Recall the Assumption~\ref{ass:Sigma} on the activation function $|\sigma'(z)| \le B(1+|z|)^B$. We have $\max_{|z| \le 2L\sqrt{d}} | \sigma'(z) | \le B [ 1 + (2L\sqrt{d})^B ] \le 2B(2L\sqrt{d})^B$. Therefore, 
\begin{align*}
\sum_{i=1}^n \sum_{\ell=1}^N | \sigma(\langle \ww_\ell, \xx_i \rangle) - \sigma(0) |^2 &\le 4B^2(2L\sqrt{d})^{2B} \cdot \sum_{i=1}^n \sum_{\ell=1}^N | \langle \ww_\ell, \xx_i \rangle|^2 \\
&= 4B^2(2L\sqrt{d})^{2B} \cdot \sum_{\ell=1}^N \| \bX \ww_\ell \|^2 \\
&\le 4B^2(2L\sqrt{d})^{2B} \cdot \sum_{\ell=1}^N \| \bX \|_{\op}^2 \| \ww_\ell \|^2
\end{align*}
By standard random matrix theory \cite[Cor.~5.35]{vershynin2010introduction}, $\| \bX \|_{\op} \le 2\sqrt{n} + \sqrt{d} \le 3\sqrt{n}$ with probability at least $1-2e^{-n^2/2}$. Also, $\| \ww_\ell \|^2 \le L^2$. So we get
\begin{equation}\label{ineq:sumil}
\sum_{i=1}^n \sum_{\ell=1}^N | \sigma( \langle \ww_\ell, \xx_i \rangle) - \sigma(0) |^2 \le 36B^2L^2(2L\sqrt{d})^{2B} Nn.
\end{equation}
Let $\cI_1:= \cI_1(\bX,\yy) \subset [n]$ be the set of $i \in [n]$ such that 
\begin{equation*}
\sum_{\ell=1}^N | h_{i\ell} - \sigma(0) |^2 \le 32(\eta_1\eta_2)^{-1}\times 36B^2 L^2(2L\sqrt{d})^{2B} N.
\end{equation*}
Recall the definition $h_{i\ell} =  \sigma( \langle \ww_\ell, \xx_i \rangle)$. Then, with probability at least $ 1 -2ne^{-cd} - 2e^{-n^2/2}$, we have $|\cI_1| \ge (1-\eta_1\eta_2/32)n$. Indeed, if this is the not true, then the set of $i$ that does not satisfy the above inequality will exceed $\eta_1\eta_2n/32$ and thus \eqref{ineq:sumil} will be violated.

Now, conditioning on $\bX,\yy$, we can view $\sum_{\ell=1}^N (b_\ell - \tilde b_\ell) (h_{i \ell} - \sigma(0)) $ as a weighted sum of independent sub-gaussian variables $b_\ell - \tilde b_\ell$. By Hoeffding's inequality for sub-gaussian variables, we have
\begin{equation*}
\P_{\tilde \btheta} \Big( \big| \sum_{\ell=1}^N (b_\ell - \tilde b_\ell) ( h_{i \ell} - \sigma(0) ) \big| > M^{-1} t  \Big) \le 2\exp \Big( - t^2  \big/  \sum_{\ell=1}^N ( h_{i \ell} - \sigma(0))^2  \Big).
\end{equation*}
We take $t = C L (2L\sqrt{d})^B \sqrt{N}$ where $C>0$ is a sufficiently large constant so that the above probability upper bound (right-hand side) is smaller than $\eta_1 \eta_2/16$. Denote the event $\cE_i$ as 
\begin{equation*}
\cE_i = \Big\{ \big| \sum_{\ell=1}^N (b_\ell - \tilde b_\ell) ( h_{i \ell} - \sigma(0) ) \big| > C M^{-1} L (2L\sqrt{d})^B \sqrt{N}  \Big\}.
\end{equation*}
We have proved $\P_{\tilde \btheta}(\cE_i) \le \eta_1\eta_2/32$ for $i \in \cI_1$. Conditioning on $|\cI_1| \ge (1-\eta_1\eta_2/32)n$, we obtain, by Markov's inequality,
\begin{align*}
\P_{\tilde \btheta} \Big( \sum_{i=1}^n \bone\{ \cE_i \} > \frac{\eta_2 n}{2} \Big) &\le \frac{2}{\eta_2 n} \E_{\tilde \btheta} \Big[ \sum_{i=1}^n \bone\{ \cE_i \} \Big] = \frac{2}{\eta_2 n} \sum_{i=1}^n \P_{\tilde \btheta}(\cE_i)   \\
&\le \frac{2}{\eta_2 n} \sum_{i \in \cI_1}\P_{\tilde \btheta}(\cE_i)  +  \frac{2}{\eta_2 n} \sum_{i \not\in\cI_1} 1 \le \frac{\eta_1}{8}.
\end{align*}

Let $\cI' = \cI'(\bX,\yy,\tilde \btheta) \subset [n]$ be the set of $i$ such that the complement $\cE_i^c$ holds. Then with probability at least $ 1 -\eta_1/8 -2ne^{-cd} - 2e^{-n^2/2}$, we have $|\cI'| \ge (1-\eta_2/2)n$, and for every $i \in  \cI'$, 
\begin{equation*}
T_{i,2} \le \frac{1}{\sqrt{N}} C M^{-1} L (2L\sqrt{d})^B \le C  M^{-1} L (2L\sqrt{d})^B .
\end{equation*}

\textbf{iii) Bounding $T_{i,3}$.} By the assumption on $\sigma'$, we can write
\begin{equation*}
\Delta h_{i\ell} = \int_0^1\sigma' \big(t \langle \ww_\ell, \xx_i \rangle + (1-t) \langle \tilde \ww_\ell, \xx_i \rangle \big) \; dt \cdot \langle \ww_\ell - \tilde \ww_\ell, \xx_i \rangle
\end{equation*}
Similarly as before, with probability at least $1 - 2ne^{-cd}$, we have $|\langle \ww_\ell, \xx_i \rangle| \le \| \ww_\ell \|_2 \cdot \| \xx_i \|_2 \le 2L\sqrt{d}$ and similarly $|\langle \tilde \ww_\ell, \xx_i \rangle| \le 2L\sqrt{d}$ for all $\ell$ and $i$. This leads to
\begin{equation*}
| \Delta h_{i\ell} | \le \max_{|z| \le 2L\sqrt{d}} | \sigma'(z) | \cdot | \langle \ww_\ell - \tilde \ww_\ell, \xx_i \rangle|.
\end{equation*}
By the assumption on $\sigma'$, we have $\max_{|z| \le 2L\sqrt{d}} | \sigma'(z) | \le B [ 1 + (2L\sqrt{d})^B ] \le 2B(2L\sqrt{d})^B$. Therefore, we can bound $T_{i,3}$ as follows. For each $i \in [n]$,
\begin{align}
\frac{L}{\sqrt{N}} \Big( \sum_{\ell=1}^N | \Delta h_{i\ell}|^2 \Big)^{1/2} &\le \frac{L}{\sqrt{N}} \cdot 2B(2L\sqrt{d})^B  \cdot \Big( \sum_{\ell=1}^N \langle \ww_\ell - \tilde \ww_\ell, \xx_i \rangle^2 \Big)^{1/2} \notag \\
&\le \frac{2BL(2L\sqrt{d})^B}{\sqrt{N}}\cdot \| (\bW - \tilde \bW) \xx_i \| \notag \\
&\le 4BCM^{-1}L(2L\sqrt{d})^B \notag
\end{align}
with probability at least $1-2e^{-cd}$. This is because $\| \xx_i \| \le 2\sqrt{d}$ holds with probability $1-2e^{-cd}$; and also, since each entry of $\bW - \tilde \bW$ is independent and bounded by $(M\sqrt{d})^{-1}$,
\begin{equation}\label{ineq:T3bound}
\P_{\tilde \btheta}\Big( \| (\bW - \tilde \bW) \xx_i \| \le 2CM^{-1}\sqrt{N} \Big) \ge 1 - 2e^{-cN},
\end{equation}
which is a consequence of Bernstein's inequality (see \cite{vershynin2010introduction} Thm.~5.39 for example). Taking the union bound over $i$, \eqref{ineq:T3bound} holds for all $i \in [n]$ with probability at least $1 - 2ne^{-cd} - 2ne^{-cN}$.

\textbf{(iv) Combining three bounds.} Finally, combining the bounds on $T_{i,1}$, $T_{i,2}$, and $T_{i,3}$, we obtain that with probability at least $1- \eta_1/4 - 4n(e^{-cN} + e^{-cd}) - 2e^{-n^2/2}$, 
\begin{align}
\max_{i \in I'} |f_\btheta(\xx_i) - f_{\tilde \btheta}(\xx_i)| &\le CM^{-1} + C M^{-1} L (2L\sqrt{d})^B + CM^{-1}L(2L\sqrt{d})^B \notag \\
&\le C M^{-1} L (L\sqrt{d})^B. \label{ineq:CM-1}
\end{align}
Under the asymptotic assumption $\log n = o_d\big(\min\{N,d\} \big)$, we have $4n(e^{-cN} + e^{-cd}) = o_d(1)$, so there exists a sufficient large $d_0$ such that $4n(e^{-cN} + e^{-cd}) + 2e^{-n^2/2} \le \eta_1/4$ if $d > d_0$. 
We choose $M = C \delta^{-1} \cdot 2L(L\sqrt{d})^B$ (where the $C$ has the same value as in Eq.~\ref{ineq:CM-1}) and get 
\begin{equation*}
\max_{i \in \cI'} |f_\btheta(\xx_i) - f_{\tilde \btheta}(\xx_i)| \le \delta/2 < \delta.
\end{equation*}
with probability $1-\eta_1/2$. This proves the claim. Note that with this choice of $M$, $\log (LM+1) = O_d \big( \log(d L/\delta)\big)$ so the lower bound is obtained.

\end{proof}

\begin{lem}\label{lem:cover}
The cardinality of the discrete set $\Theta_{L,M}$ defined in \eqref{def:ThetaLM} have the bound
\begin{equation*}
| \Theta_{L,M} | \le O\Big( \big[ \sqrt{2\pi e} (4ML+1)\big]^{Nd + N} \Big).
\end{equation*}
\end{lem}
\begin{proof}
Denote $B^d(r) = \{\xx \in \R^d: \| \xx \| \le r \}$. Recall the definition of the packing number. We say the set $\cA \in B^d(r)$ is an $\veps$-packing of $(B^d(r), \| \cdot \|_\infty)$ if $\| \xx_1 - \xx_2 \|_\infty > \veps$ for every $\xx_1, \xx_2 \in \cA$, $\xx_1 \neq \xx_2$. We define the packing number
\begin{equation*}
D(B^d(r), \| \cdot \|_{\infty}, \veps) := \sup\big\{ |\cA|: \cA~ \text{is an $\veps$-packing of}~(B^d(r), \| \cdot \|_{\infty} \big\}
\end{equation*}
By the volume argument, we have the bound
\begin{equation}\label{ineq:packing}
D(B^d(r), \| \cdot \|_{\infty}, 2\veps) \le \frac{\mathrm{Vol}\big( B^d(r+\sqrt{d}\, \veps) \big)}{\veps^d} \stackrel{(i)}{=} o_d(1) \cdot \Big( \frac{2\pi e}{d} \Big)^{d/2} \Big( \frac{ r+\sqrt{d}\veps}{\veps} \Big)^d
\end{equation}
where in (i) we used the formula of the volume of a Euclidean ball and Stirling's approximation. 

Recall that $\cD_M = \{0, \pm \frac{1}{M}, \pm \frac{2}{M}, \ldots\}$. To apply this result to bounding $| \Theta_{L,M} |$, first we observe that
\begin{align}
| \Theta_{L,M} | \le &\big| \big\{ \aa \in \R^N: \| \aa \| \le \sqrt{N}\, L, a_\ell \in \cD_M~\forall\, \ell\in[N] \big\} \big| \label{set:a} \\
\cdot & \prod_{\ell=1}^N \big| \big\{ \ww_\ell \in \R^d: \| \ww_\ell \| \le  L, \sqrt{d}\, [\ww_\ell]_j \in \cD_M~\forall\, j \in[d] \big\} \big|. \label{set:w}
\end{align}
Suppose $\aa_1,\aa_2$ lie in the set on the RHS of \eqref{set:a} and $\aa_1 \neq \aa_2$. Then, by the definition of $\cD_M$, we must have $\| \aa_1 - \aa_2 \|_\infty > \frac{1}{2M}$. This means that the cardinality of the set on the RHS of \eqref{set:a} is bounded by the packing number $D(B^N(\sqrt{N}\, L), \| \cdot \|_{\infty}, \frac{1}{2M})$. By \eqref{ineq:packing}, we get an upper bound on the cardinality
\begin{equation}\label{ineq:packinga}
o_N\Big(\frac{2 \pi e}{N} \Big)^{N/2} \cdot \big( \frac{\sqrt{N}\,L + \sqrt{N} \frac{1}{4M} }{\frac{1}{4M}} \Big)^N = o_N(1) \cdot (2 \pi e)^{N/2} (4ML + 1)^N.
\end{equation}
Similarly, we can bound the cardinality of each of the $N$ sets in \eqref{set:w} by
\begin{equation}\label{ineq:packingw}
o_d\Big(\frac{2 \pi e}{d} \Big)^{d/2} \cdot \big( \frac{L + \sqrt{d} \frac{1}{4M\sqrt{d}} }{\frac{1}{4M\sqrt{d}}} \Big)^d = o_d(1) \cdot (2 \pi e)^{d/2} (4LM+1)^d.
\end{equation}
Recall that $N = N(d) \to \infty$ as $d \to \infty$, so $o_N(1)$ in \eqref{ineq:packinga} can be replaced by $o_d(1)$. Combining \eqref{ineq:packinga} and \eqref{ineq:packingw}, we obtain
\begin{align*}
| \Theta_{L,M} | & \le o_d(1)  \cdot \big[ \sqrt{2\pi e}\, (4ML+1) \big]^N \prod_{\ell=1}^N  \big[ \sqrt{2\pi e}\, (4ML+1) \big]^d \\
&\le o_d(1) \cdot \big[ \sqrt{2\pi e}\, (4ML+1) \big]^{Nd+N}.
\end{align*}
\end{proof}

\bibliographystyle{amsalpha}
\bibliography{all-bibliography}

\end{document}